\documentclass[10pt]{article} 
\usepackage{amsmath,amssymb,amsthm,mathtools}

\usepackage{newtxtext}
\usepackage{newtxmath}
\usepackage[margin=1in]{geometry}
\usepackage{booktabs}
\usepackage{algorithm}
\usepackage{multirow}
\usepackage{algpseudocode}
\usepackage{enumitem}
\usepackage{subcaption}

\usepackage{amsmath,amsfonts,bm}

\def\eqref#1{equation~\ref{#1}}

\def\1{\bm{1}}

\DeclareMathAlphabet{\mathsfit}{\encodingdefault}{\sfdefault}{m}{sl}
\SetMathAlphabet{\mathsfit}{bold}{\encodingdefault}{\sfdefault}{bx}{n}

\newcommand{\abs}[1]{\left\lvert#1\right\rvert}

\newcommand{\E}{\mathbb{E}}

\DeclareMathOperator*{\argmax}{arg\,max}
\DeclareMathOperator*{\argmin}{arg\,min}

\usepackage[backref,colorlinks,citecolor=blue,bookmarks=true]{hyperref}
\usepackage{mathtools, amssymb, amsthm, dsfont}
\numberwithin{equation}{section}
\usepackage{enumitem}
\usepackage[nameinlink,capitalize]{cleveref}
\usepackage{algorithm}
\usepackage[table]{xcolor}

\usepackage{parskip}
\usepackage{mathtools, amssymb, amsthm, bbm}
\usepackage[nameinlink,capitalize]{cleveref}

\theoremstyle{plain}
\newtheorem{theorem}{Theorem}[section]
\newtheorem{lemma}{Lemma}[section]
\newtheorem{corollary}{Corollary}[section]
\newtheorem{proposition}{Proposition}[section]

\newtheorem{assumption}{Assumption}[section]
\newtheorem*{claim}{Claim}

\theoremstyle{definition}
\newtheorem{definition}{Definition}[section]

\theoremstyle{remark}

\newtheorem{remark}[theorem]{Remark}

\newcommand{\norm}[1]{\|#1\|}

\renewcommand{\Pr}{\mathbf{Pr}}

\usepackage{url}
\usepackage{etoc}
\usepackage[title]{appendix}

\definecolor{light-gray}{gray}{0.95}
\newcommand{\code}[1]{\colorbox{light-gray}{\texttt{#1}}}

\definecolor{HHgreen}{HTML}{34C759}

\newcommand{\pmgreen}[1]{%
  {\begingroup\color{HHgreen}\scriptsize\textbf{(\(\,\pm\,\)#1)}\endgroup}%
}

\definecolor{citeblue}{rgb}{0,0,1}

\usepackage{xcolor}
\newcommand{\fakecite}[1]{\textcolor{citeblue}{[#1]}}

\newlist{thmassumptions}{enumerate}{1}
\setlist[thmassumptions,1]{label=(\alph*), ref=\thetheorem(\alph*)}

\crefname{thmassumptionsi}{Assumption}{assumptions}
\crefname{thmassumptionsi}{Assumption}{Assumptions}

\title{Online Distributionally Robust LLM Alignment via Regression to Relative Reward}

\author{ Sharan Sahu \\ \texttt{ss4329@cornell.edu} \\ Department of Statistics and Data Science \\ Cornell University \and Martin T. Wells \\ \texttt{mtw1@cornell.edu} \\ Department of Statistics and Data Science \\ Cornell University}



\begin{document}

\maketitle
\begin{abstract} 
Reinforcement Learning with Human Feedback (RLHF) has become crucial for aligning Large Language Models (LLMs) with human intent. However, existing offline RLHF approaches suffer from overoptimization, where language models degrade by overfitting inaccuracies and drifting from preferred behaviors observed during training.  Distributionally robust optimization (DRO) is a natural solution, but existing DRO-DPO methods are sample-inefficient, ignore heterogeneous preferences, and lean on brittle heuristics. We introduce \emph{DRO-REBEL}, a family of robust online REBEL updates built on type-$p$ Wasserstein, Kullback-Leibler (KL), and $\chi^2$ ambiguity sets. Strong duality reduces each update to a relative-reward regression, retaining REBEL's scalability without PPO-style clipping or value networks. Under linear rewards, log-linear policies, and a standard coverage condition, we prove $\widetilde{\mathcal{O}}(\sqrt{d/n})$ bounds on squared parameter error, with sharper constants than prior DRO-DPO analyses, and give the first parametric $\widetilde{\mathcal{O}}(d/n)$ rate for DRO-based alignment under preference shift, matching non-robust RLHF in benign regimes. Each divergence yields a tractable SGD-based algorithm: gradient regularization for Wasserstein, importance weighting for KL, and a 1-D dual solve for $\chi^2$. On Emotion Alignment, the ArmoRM multi-objective benchmark, and HH-Alignment, DRO-REBEL outperforms prior robust and non-robust baselines across unseen preference mixtures, model sizes, and dataset scales. \end{abstract}

\section{Introduction}
RLHF has emerged as one of the most important stages of aligning LLMs with human intent \cite{christiano2017deep, ziegler2019fine}. Typically, after supervised fine-tuning (SFT), an additional alignment phase is often required to refine their behavior based on human feedback. The alignment of LLMs with human values and preferences is a central objective in machine learning, enabling these models to produce outputs that are useful, safe, and aligned with human intent. In RLHF, human evaluators provide preference rankings that are subsequently utilized to train a reward model, guiding a policy optimization step to maximize learned rewards \cite{ouyang2022training}. From a statistical perspective, RLHF can be viewed as a multi-stage estimation and optimization pipeline built on noisy human preference data. Despite its success, standard RLHF methodologies are fragile mainly due to three reasons: \textit{(i) Assumption that one reward model can model diverse human preferences:} Many RLHF methodologies including popular methods such as Direct Preference Optimization (DPO) \cite{rafailov2023direct} and Proximal Policy Optimization (PPO) \cite{schulman2017proximal}  assume that a single reward function can model and accurately capture diverse human preferences.  In reality, human preferences are highly diverse, context-dependent, and distributional, making it infeasible to represent them within one single reward function. To this end, there has been work done in creating Bayesian frameworks for robust reward modeling \cite{yan2024rewardrobustrlhfllms}, modeling loss as a weighted combination of different topics and using out-of-distribution detection to reject bad behavior \cite{bai2022traininghelpfulharmlessassistant}, or formulating a mixture of reward models \cite{chakraborty2024maxmin}. \textit{(ii) Reward hacking:} Alignment depends on the quality of the human preference data collected. Unfortunately, this process is inherently noisy and prone to bias, conflicting opinions, and inconsistency that leads to a misaligned preference estimation. This issue is exacerbated by reward hacking, where instead of learning reward functions that are aligned with genuine human intent, models learn undesirable shortcuts to maximize the estimated reward function. Subsequently, these models appear to generate responses that appear aligned
but deviate from human intent. Some works directly address this, such as \cite{bukharin2024robustreinforcementlearningcorrupted}.  \textit{(iii) Distribution shift:} Standard RLHF alignment algorithms use static preference datasets for training, collected under controlled conditions. However, the preferences of real-world users can often be out-of-distribution from those of the training data, depending on several factors such as geographic location, demographics, etc. Thus, a language model facing distribution shift can experience catastrophic performance degradation due to inaccuracies of overfitting and divergence from the preferred human responses encountered in training data \cite{levine2024baselineanalysisrewardmodels, kirk2024understandingeffectsrlhfllm, casper2023openproblemsfundamentallimitations}. We focus on the problem of distribution shift, also known as \textit{overoptimization} \cite{huang2025correcting}. 

\begin{figure}[h!]
    \centering
    \includegraphics[width=0.9\textwidth]{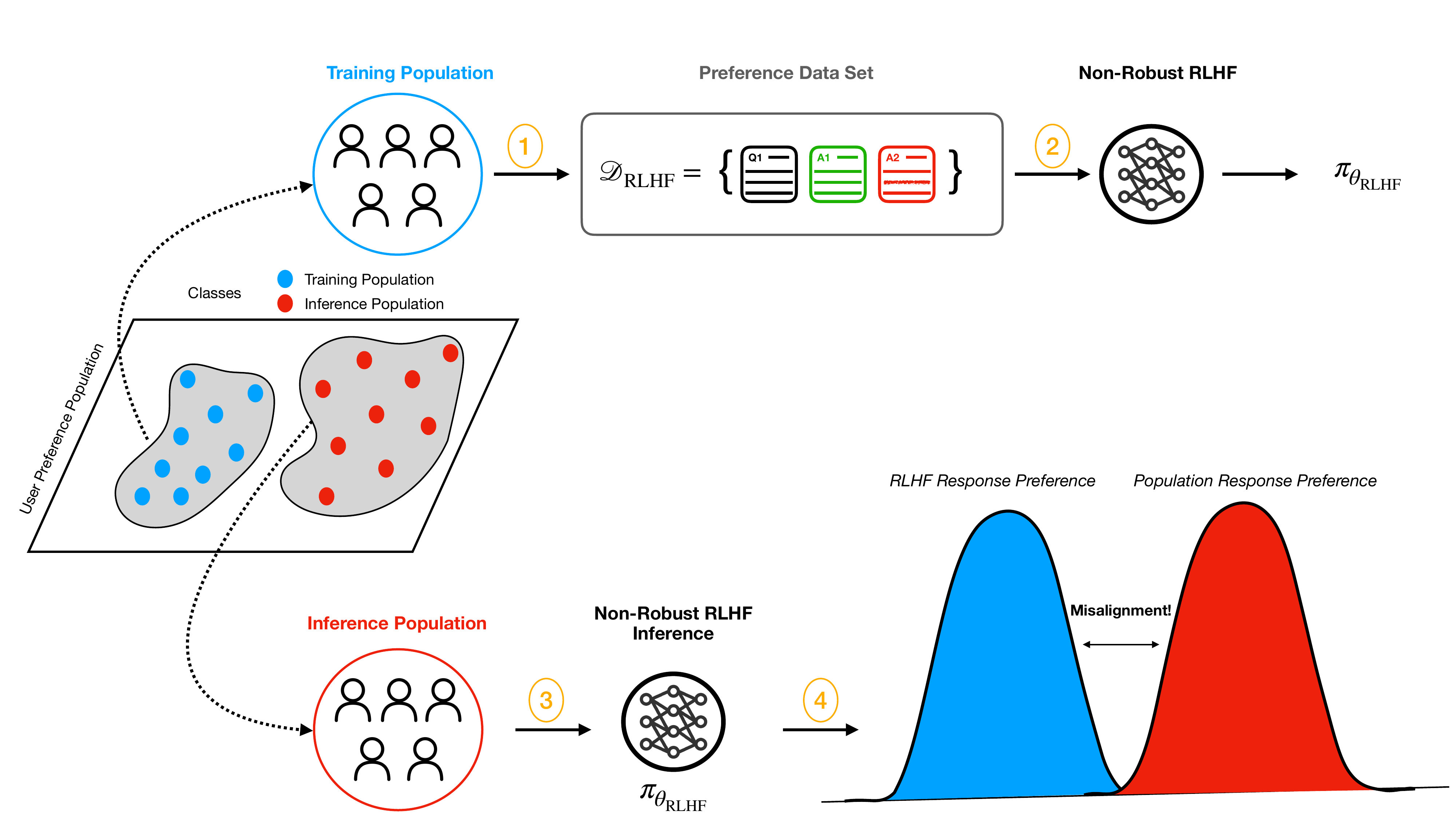}
    \caption{Non‐robust RLHF under distributional shift. Pairwise preference data $\mathcal{D}_{\rm RLHF}$ collected from a training population (blue) are used to learn a policy $\pi_{\theta_{\rm RLHF}}$ via standard RLHF. In latent preference–feature space, the inference population (red) occupies a disjoint region from the training cohort (blue), indicating a shift. When $\pi_{\theta_{\rm RLHF}}$ is deployed on these out‐of‐distribution users, its induced response‐preference distribution (blue) diverges from the true population preferences (red), resulting in systematic misalignment.}
    \label{fig:rlhf-pitfalls}
\end{figure}

Recently, distributionally robust RLHF methods have emerged to tackle robustness challenges under distributional shifts in prompts and preferences \cite{mandal2025distributionally, xu2025distributionallyrobustdirectpreference}. Specifically, \cite{mandal2025distributionally} and \cite{xu2025distributionallyrobustdirectpreference} introduced DRO variants of popular RLHF methods, namely DPO and PPO, employing uncertainty sets defined via Chi-Squared ($\chi^{2}$), type-p Wasserstein, and Kullback–Leibler (KL) divergences. Unfortunately, it is known that PPO requires multiple heuristics to enable stable convergence (e.g. value networks, clipping), and is notorious for its sensitivity to the precise implementation of these components. Offline RLHF methods such as DPO, while not explicitly using a reward model, learn an implicit reward model of the form $r_{\phi}(x, y) = \beta \log \left( \pi_{\theta}(y \mid x) / \pi_{\mathrm{SFT}}(y \mid x) \right)$ that provides higher reward to preferred responses over dispreferred responses. This can lead to a brittle solution that overfits the preference distribution seen during training. When faced with a new type of prompt (a distributional shift), the learned policy may fail because the implicit reward signal on which it is based does not generalize \cite{xu2024dposuperiorppollm}. Recently, \cite{gao2024rebel} proposed REBEL, an algorithm that cleanly reduces the problem of policy optimization to regressing the relative reward between two completions to a prompt in terms of
the policy. They find that REBEL avoids the use of "unjustified" heuristics like PPO and enjoys strong convergence and regret guarantees, similar to Natural Policy Gradient \cite{NIPS2001_4b86abe4}, while also being scalable due to not requiring inversion of the Fisher information matrix. REBEL is much more sample efficient compared to methods like DPO and PPO and is better equipped to generalize under distributional shifts. Given the fragility of PPO-style updates and the slower rates observed for robust DPOs, this motivated us to answer the following questions.

\begin{center}
\begin{minipage}{0.9\linewidth}
\raggedright\itshape
Can we obtain distributional robustness to preference shift \underline{without} sacrificing sample efficiency and stability by adopting a simpler, regression-based algorithm such as REBEL? Concretely, does combining DRO with REBEL yield a learning rule that is (i) theoretically sound, (ii) practically scalable, and (iii) empirically more \underline{generalizable} under realistic distribution shift?
\end{minipage}
\end{center}

\textbf{Our contributions}. Inspired by the strong theoretical guarantees of REBEL in terms of sample efficiency and simplicity, we introduce \emph{DRO–REBEL}, a family of robust REBEL updates for RLHF under distributional shifts, and make the following advances:
\begin{enumerate}
  \item \textbf{Sharp slow-rate guarantees.} Under standard linear-reward and log-linear policy assumptions with a data-coverage condition, we prove an $\widetilde{\mathcal{O}}(\sqrt{d/n})$ bounds on squared parameter error for all DRO–REBEL variants. By replacing logistic links with linear regression, we remove hidden exponential curvature factors and tighten constants compared to prior DRO-DPO analyses \cite{xu2025distributionallyrobustdirectpreference, mandal2025distributionally}. The strong-convexity modulus depends only on the coverage constant $\lambda$ and step size $\eta$, rather than Bradley–Terry curvature as in WDPO/KLDPO.
  \item \textbf{Minimax-optimal fast rates.} To the best of our knowledge, we provide the first proof that DRO-based LLM alignment under preference shift attains the parametric \(\widetilde{\mathcal{O}}(d/n)\) rate, matching non-robust RLHF in benign regimes via a measurable selection reduction to empirical risk minimization (ERM) and localized Rademacher complexity argument. Applying the same machinery to WDPO and KLDPO yields $\widetilde{\mathcal{O}}(d/n)$ estimation rates, improving upon the $\widetilde{\mathcal{O}}(\sqrt{d/n})$ bounds reported in prior DRO–DPO theory \cite{xu2025distributionallyrobustdirectpreference, mandal2025distributionally}.
  \item \textbf{Extensive empirical validation.} We evaluate on (i) a controlled Emotion Alignment task \cite{saravia-etal-2018-carer}, (ii) a large-scale ArmoRM multi-objective setting \cite{wang2024interpretablepreferencesmultiobjectivereward}, and (iii) a zero-shot \emph{HH–RLHF} domain-adaptivity evaluation \cite{bai2022traininghelpfulharmlessassistant}. The first two replicate the benchmarks of \cite{xu2025distributionallyrobustdirectpreference} for direct comparison, the third is out-of-domain adaptivity to a popular dataset benchmark. Empirically, DRO-REBEL demonstrates superior performance in maintaining alignment across preference shifts and generalizing to unseen objectives. 
\end{enumerate}

\subsection{Related Work}
\textbf{Robust RLHF:} There has been some recent work in this area that aims to address RLHF overoptimization. \cite{bai2022traininghelpfulharmlessassistant} propose addressing distribution shift by adjusting the weights on the combination of loss functions based
on different topics (harmless vs. helpful) for robust reward learning.  They also propose using out-of-distribution detection to filter and reject known types of bad behavior. \cite{chakraborty2024maxmin} proposes a MaxMin approach to RLHF, using mixtures of reward models to honor diverse human preference distributions through an expectation-maximization approach, and a robust policy based on these rewards via a max-min optimization. In a similar vein, \cite{padmakumar2024binarycapturingdiversepreferences} attempts to augment the human preference datasets with synthetic preference judgments to estimate the diversity of user preferences. There has also been some foundational theoretical work towards this problem. \cite{yan2024rewardrobustrlhfllms} proposed a Bayesian reward model ensemble to model the uncertainty set of the reward functions and systematically choose rewards in the uncertainty set with the tightest confidence band. Another line of work focuses on robust reward modeling as an alternative to distributionally robust optimization. For example, \cite{bukharin2024robustreinforcementlearningcorrupted} propose R3M, a method that explicitly models corrupted preference labels as sparse outliers. They formulate reward learning as an $\ell_1$ regularized maximum likelihood estimation problem, allowing robust recovery of the underlying reward function even in the presence of noisy or inconsistent human feedback. While our work focuses on embedding robustness at the policy optimization level using distributional uncertainty sets (e.g., $\chi^2$, and Wasserstein), R3M represents a complementary direction that enhances robustness by improving the reliability of the reward model itself. 

\textbf{Robust DPO:} There have been several works that approach this problem using DRO.  \cite{huang2025correcting} proposed $\chi$PO that implements the principle of pessimism in the face
of uncertainty via regularization with the $\chi^{2}$-divergence for avoiding reward hacking/overoptimization with respect to the estimated reward. \cite{wu2024robustalignmentlanguagemodels} focus on noisy preference data and categorize the types of noise in DPO, introducing Dr. DPO to improve pairwise robustness through a DRO formulation with a tunable reliability parameter. \cite{hong2024adaptivepreferencescalingreinforcement} propose an adaptive preference loss grounded in DRO that adjusts scaling weights across preference pairs to account for ambiguity in human feedback, enhancing the flexibility of reward estimation and policy performance. Separately, \cite{zhang2024overcomingrewardoveroptimizationadversarial} introduces a lightweight uncertainty-aware approach called AdvPO, combining last-layer embedding-based uncertainty estimation with a DRO formulation to address overoptimization in reward-based RLHF. There are two related works that are most similar to our approach. \cite{xu2025distributionallyrobustdirectpreference} develop Wasserstein and KL-based DRO formulations of Direct Preference Optimization (WDPO and KLDPO), providing sample complexity bounds and scalable gradient-based algorithms. Their methods achieve improved alignment performance under shifting user preference distributions. Similarly, \cite{mandal2025distributionally} propose robust variants of both reward-based and reward-free RLHF methods, incorporating DRO into the reward estimation and policy optimization phases using Total Variation and Wasserstein distances. Their algorithms retain the structure of existing RLHF pipelines while providing theoretical convergence guarantees and demonstrating robustness to out-of-distribution (OOD) tasks.

\textbf{Distributionally Robust Learning:} Distributionally robust optimization (DRO) has emerged as a general framework for learning under distribution shift, ambiguity, and tail risk. It has been studied extensively in supervised learning, where robust objectives can improve performance under covariate shift and heavy-tailed or heterogeneous data \cite{namkoong2017variance, shah2020robust}, in reinforcement learning, where ambiguity sets over transition or occupancy distributions are used to obtain policies with stronger out-of-distribution guarantees \cite{pmlr-v119-zhang20j, yang2021wdr}, and in multi-armed bandits, where robust formulations help control worst-case regret under model misspecification or adversarial perturbations \cite{gao2022distributionally, zhou2022distributionally}. At the same time, a rich statistical theory has been developed for DRO based on $f$-divergences and Wasserstein distances, including duality, generalization, asymptotic, and finite-sample guarantees \cite{duchi2016statistics, shapiro2022wasserstein}. These results make DRO especially appealing for RLHF, where training and deployment distributions may differ substantially across prompts, annotators, or user populations. In our setting, DRO provides a principled way to hedge against such preference shift while retaining a tractable learning objective.

\textbf{Statistical reinforcement learning:}
Recent work has also developed reinforcement learning from a distinctly statistical perspective, emphasizing estimation, inference, and heterogeneity rather than only algorithmic performance. This includes semiparametrically efficient off-policy evaluation from logged data through balancing-based methods \cite{wang2023projected}, estimation and inferential theory for return distributions in distributional reinforcement learning \cite{zhang2025estimation}, and individualized offline policy learning from heterogeneous data sources \cite{miao2025reinforcement}. Collectively, these works show that reinforcement learning is increasingly studied as a problem of statistical efficiency, uncertainty quantification, and robust decision-making under heterogeneous populations. 

\section{Preliminaries}

\subsection{Notations}
We will denote sets using calligraphic letters, that is, $\mathcal{S}, \mathcal{A}, \mathcal{Z}$. For a measure $P$, we refer to the empirical measure $P_{n}$ to mean drawing samples $x_{1}, \dots, x_{n} \stackrel{\mathrm{i.i.d}}{\sim} P$ with $P_{n} = 1 / n \sum_{i=1}^{n} \delta_{x_{i}}$ where $\delta$ is the Dirac measure. We denote $\ell(z; \theta)$ to be the loss incurred by sample $z$ with policy parameter $\theta$. We denote $\mathcal{M} \left(\mathcal{Z} \right)$ to be the set of Borel measures supported on set $\mathcal{Z}$. Lastly, we denote $\lambda_{\mathrm{min}}(A)$ to be the minimum eigenvalue of a symmetric matrix $A \in \mathbb{S}^{n}$. We adopt the standard big-oh notation, and write $a \lesssim b$ as shorthand for $a = \mathcal{O}(b)$. We write $a_n = \widetilde{\mathcal{O}}(b_n)$ to mean that there exists a constant $C>0$ that is independent of $n$, and a poly-logarithmic factor $\mathrm{polylog}(\Pi)$ in the relevant problem parameters $\Pi$, such that $|a_n| \le C\,b_n \cdot \mathrm{polylog}(\Pi)$. 


\subsection{Divergences}
In this section, we define the divergences that we use to define our ambiguity sets.  

\vspace{1em}

\begin{definition}[Type-p Wasserstein Distance]
The type-p ($p \in [1, \infty))$ Wasserstein distance between two probability measures $P, Q \in \mathcal{M} \left( \mathcal{Z} \right)$ is defined as

    \[
        \mathcal{W}_{p} \left(P, Q \right) = \left( \inf_{\pi \in \Pi \left( P, Q \right)} \int_{\mathcal{Z} \times \mathcal{Z}} d(\xi, \eta)^{p} \pi(d\xi, d\eta) \right)^{1/p},
    \]

    where $\pi$ is a coupling between the marginal distributions $\xi \sim P$ and $\eta \sim Q$ and $d$ is a pseudometric defined on $\mathcal{Z}$.
\end{definition}

\vspace{1em}

\begin{definition}[Kullback-Leibler (KL) Divergence]
For any two probability measures $P, Q \in \mathcal{M} \left( \mathcal{Z} \right)$, the Kullback-Liebler (KL) Divergence is defined as 
\[
D_{\mathrm{KL}}(P \, || \, Q) = \int_{\mathcal{Z}} \log \left( \frac{d P}{d Q} \right) dP.
\]
\end{definition}

\vspace{1em}

\begin{definition}[Chi-Squared Divergence]
For any two probability measures $P, Q \in \mathcal{M} \left( \mathcal{Z} \right)$ such that $P \ll Q$ i.e. $P$ is absolutely continuous with respect to $Q$, the Chi-Squared ($\mathcal{\chi}^{2}$) divergence is defined as
\[
D_{\chi^2}(P \, || \, Q) = \int_{\mathcal{Z}} \left( \frac{dP}{dQ} - 1 \right)^2 dQ.
\]
\end{definition}

Using these, we can define our ambiguity sets as follows.

\vspace{1em}

\begin{definition}[Distributional Uncertainty Sets]
Let $\varepsilon > 0$ and $P^{\circ} \in \mathcal{M} \left( \mathcal{Z} \right)$. Then, we define the ambiguity set as 

\[
    \mathcal{B}_{D} \left( P^{\circ} ; \varepsilon \right) = \left\{ P \in \mathcal{M} \left( \mathcal{Z} \right) \; : \; D \left( P || P^{\circ} \right) \leq \varepsilon \right\},
\]

where $D(\cdot || \cdot)$ is a distance metric between two probability measures, for instance, type-p Wasserstein, KL, $\chi^{2}$, etc.
\end{definition}

\subsection{Reinforcement Learning from Human Feedback (RLHF)}
\label{subsection:RLHF}
Reinforcement Learning from Human Feedback (RLHF), as introduced by \cite{christiano2017deep} and later adapted by \cite{ouyang2022training}, consists of two primary stages: (1) learning a reward model from preference data, and (2) optimizing a policy to maximize a KL-regularized value function. We assume access to a preference dataset $\mathcal{D}_{\text{src}} = \{(x, a^{1}, a^{2})\}$ where $x \in \mathcal{S}$ is a prompt and $a^{1}, a^{2} \in \mathcal{A}$ are two possible completions of the prompt $x$ generated from a reference policy $\pi_{\text{SFT}}(\cdot \mid x)$ (e.g., a supervised fine-tuned model).  $\pi_{\text{SFT}}(\cdot \mid x)$ involves fine-tuning a pre-trained LLM through supervised learning on high-quality data, curated for downstream tasks. A human annotator provides preference feedback that indicates $a^{1} \succ a^{2} \mid x$. The most common model for preference learning is the Bradley-Terry (BT) model, which assumes that
\begin{align*}
    \mathcal{P}^{*}(a^{1} \succ a^{2} \mid x) &= \sigma\left(r^\star(x, a^{1}) - r^\star(x, a^{2})\right) \\
    &= \frac{1}{1 + \exp\left(r^\star(x, a^{2}) - r^\star(x, a^{1})\right)},
\end{align*}
where $r^\star$ is the underlying (unknown) reward function used by the annotator. The first step in RLHF is to learn a parametrized reward model $r_{\phi}(s, a)$ by solving the following maximum likelihood estimation problem:
\begin{align*}
    r_{\phi} \leftarrow \arg\min_{r_{\phi}} \; -\mathbb{E}_{(x, a^{1}, a^{2}) \sim \mathcal{D}_{\text{src}}} \left[ \log \sigma\left(r_{\phi}(x, a^{1}) - r_{\phi}(x, a^{2})\right) \right].
\end{align*}

Given the learned reward model $r_{\phi}$, the second step is to solve the KL-regularized policy optimization problem:
\begin{align*}
    \pi_{\theta} \leftarrow \arg\max_{\pi_{\theta}} \; \mathbb{E}_{x \sim \mathcal{D}_{\text{src}}, a \sim \pi_{\theta}(\cdot \mid x)} \left[ r_{\phi}(x, a) - \beta \log \frac{\pi_{\theta}(a \mid x)}{\pi_{\text{SFT}}(a \mid x)} \right],
\end{align*}

where $\beta$ controls the deviation between the learned and reference policy.

\subsection{Direct Preference Optimization (DPO)}
The DPO \cite{rafailov2023direct} procedure is a form of offline RLHF which avoids training a reward model by identifying a mapping (and reparametrization) between language model policies and reward functions that enables training language models to satisfy human preferences directly. That is, DPO uses the following policy objective
\label{dpo-objective}
\begin{equation}
    \ell_{\mathrm{DPO}} \left( \pi_{\theta} ; \pi_{\mathrm{SFT}} \right) = -\mathbb{E}_{\left(x, a^{1}, a^{2} \right) \sim \mathcal{D}} \left[ \log \sigma \left( \beta \log \frac{\pi_{\theta} \left( a^{1} \mid x \right)}{\pi_{\mathrm{SFT}}\left( a^{1} \mid x \right)} - \beta \log \frac{\pi_{\theta} \left( a^{2} \mid x \right)}{\pi_{\mathrm{SFT}}\left( a^{2} \mid x \right)}  \right)   \right].
\end{equation}

One can arrive at this policy objective by observing that the objective in the RL fine-tuning phase has a closed form solution of the form

\[
    \pi_{\theta} \left( a \mid x \right) = \frac{1}{Z(x)} \pi_{\mathrm{SFT}} \left( a \mid x \right) \exp \left( \frac{1}{\beta} r_{\phi}(x, a) \right)
\]

where $Z(x) = \sum_{a \in \mathcal{A}} \pi_{\mathrm{SFT}} \left( a \mid x \right) \exp \left( \frac{1}{\beta} r_{\phi}(x, a) \right)$ is the partition function. Taking logs and moving terms around, we get

\begin{align}
\label{eq:reward-reparametrization}
    r_{\phi}(x, a) = \beta \log \frac{\pi_{\theta}(a \mid x)}{\pi_{\mathrm{SFT}} \left( a \mid x \right)} + \beta \log Z(x).
\end{align}

Using the BT model and plugging in this reparameterization, we get the DPO policy objective. One might note that this is not a proper reparametrization, since $Z(x)$ is dependent on $r$. However, it turns out that defining this reparametrization as a function from an equivalence class of reward functions to a particular policy is well defined, does not constrain the class of learned reward models, and allows for the exact recovery of the optimal policy.

\subsection{REBEL: Regression-Based Policy Optimization}

Let $(x, a)$ represent a \emph{prompt-response} pair, where $x \in \mathcal{S}$ is a context or prompt, and $a \in \mathcal{A}$ is a response (e.g., a sequence of tokens or actions). We assume access to a reward function $r(x, a)$, which may be a learned preference model \cite{christiano2017deep}. We note that this reward function is not the underlying (unknown) reward function used by the annotator and thus can be considered noisy. Let $\pi \colon \mathcal{S} \rightarrow \Delta(\mathcal{A})$ be a stochastic policy mapping prompts to distributions over responses. We denote the prompt distribution as $\rho$, and let $\pi_\theta(a \mid x)$ denote a parameterized policy with parameters $\theta$. Now instead of using the reward parametrization (\ref{eq:reward-reparametrization}) in the BT model, instead notice that we can get rid of the partition function $Z(x)$ by sampling two responses $a^{1}, a^{2} \sim \pi_{\theta_{t}} \left( \cdot \mid x \right)$ at time step $t$ and taking the difference of the reparametrized reward function

\[
    r(x, a^{1}) - r(x, a^{2}) = \frac{1}{\eta} \left( \log \frac{\pi_{\theta}(a^{1} \mid x)}{\pi_{\theta_{t}} \left( a^{1} \mid x \right)} - \log \frac{\pi_{\theta}(a^{2} \mid x)}{\pi_{\theta_{t}} \left(a^{2} \mid x \right)} \right).
\]

Then we can simply regress the difference in rewards and update the policy parameters as follows

\begin{align}
    \label{eq:rebel-regression}
    \theta_{t+1} = \arg\min_{\theta \in \Theta} \left( 
        \frac{1}{\eta} \left[
        \log \frac{\pi_\theta(a^{1} \mid x)}{\pi_{\theta_{t}}(a^{1} \mid x)} -
        \log \frac{\pi_\theta(a^{2} \mid x)}{\pi_{\theta_{t}}(a^{2} \mid x)}
        \right]
        - \left[r(x, a^{1}) - r(x, a^{2})\right]
        \right)^2.
\end{align}

The \textbf{REBEL} (REgression to RElative REward-Based RL) \cite{gao2024rebel} algorithm directly regresses to relative reward differences through KL-constrained updates. The REBEL algorithm is detailed in Algorithm~\ref{alg:rebel}.

\begin{algorithm}[H]
\caption{REBEL: REgression to RElative REward-Based RL}
\label{alg:rebel}
\begin{algorithmic}[1]
\State \textbf{Input:} Reward function $r$, policy class $\Pi = \{\pi_\theta : \theta \in \Theta\}$, prompt distribution $\rho$, learning rate $\eta$
\State Initialize policy $\pi_{\theta_0}$
\For{$t = 0$ to $T-1$}
    \State Collect dataset $\mathcal{D}_t = \{(x, a^{1}, a^{2})\}$ with $x \sim \rho$, $a^{1} \sim \pi_{\theta_{t}}(\cdot \mid x)$, $a^{2} \sim \pi_{\theta_{t}}(\cdot \mid x)$
    \State Update policy by solving:
    \begin{align*}
        \theta_{t+1} = \arg\min_{\theta \in \Theta} \sum_{(x, a^{1}, a^{2}) \in \mathcal{D}_t} \left( 
        \frac{1}{\eta} \left[
        \log \frac{\pi_\theta(a^{1} \mid x)}{\pi_{\theta_{t}}(a^{1} \mid x)} -
        \log \frac{\pi_\theta(a^{2} \mid x)}{\pi_{\theta_{t}}(a^{2} \mid x)}
        \right]
        - \left[r(x, a^{1}) - r(x, a^{2})\right]
        \right)^2
    \end{align*}
\EndFor
\end{algorithmic}
\end{algorithm}

At each iteration, REBEL approximates the solution to a KL-constrained policy optimization objective:
\begin{align*}
    \pi_{t+1} = \arg\max_{\pi \in \Pi} \; \mathbb{E}_{x \sim \rho, a \sim \pi(\cdot \mid x)} \left[ r(x, a) \right] 
    - \frac{1}{\eta} \, \mathbb{E}_{x \sim \rho} \left[ \text{KL}\left( \pi(\cdot \mid x) \, \| \, \pi_t(\cdot \mid x) \right) \right],
\end{align*}
which encourages reward maximization while regularizing the policy to remain close to the previous iterate. This objective is particularly well-suited for fine-tuning language models using learned or noisy reward signals while maintaining stability in training. 

Adapting REBEL for distributionally robust RLHF is particularly appealing because it offers distinct theoretical and practical advantages over existing methods. Actor-critic methods such as PPO rely on complex and often unstable heuristic mechanisms (e.g., clipping, value baselines) to ensure stability while offline RLHF methods such as DPO, while not explicitly using a reward model, learn an implicit reward model of the form $r_{\phi}(x, y) = \beta \log \left( \pi_{\theta}(y \mid x) / \pi_{\mathrm{SFT}}(y \mid x) \right)$ which provides higher reward to preferred responses over dispreferred responses.  This can lead to a brittle solution that overfits the preference distribution seen during training. When faced with a new type of prompt (a distributional shift), the learned policy may fail because the implicit reward signal on which it relies does not generalize \cite{xu2024dposuperiorppollm}. In contrast, REBEL takes a more direct and robust approach. It reduces policy optimization to a sequence of simple regression problems on explicit relative rewards. REBEL's regression objective learns a cardinal signal that captures how much better one response is than another, and since the reward model is trained explicitly on the task of "predicting preference differences," it is more likely to generalize to unseen prompts.

This fundamental simplicity also eliminates the need for explicit value functions or constrained optimization, translating into significantly improved stability and sample complexity. In particular, REBEL can achieve convergence guarantees comparable to or better than NPG, with a sample complexity that scales favorably due to its variance-reduced gradient structure. Empirically, REBEL has been shown to converge faster than PPO and outperform DPO in both language and image generation tasks. Building on this regression‐based perspective, our DRO–REBEL algorithms simply replace the standard squared‐error loss in each REBEL update with its robust counterpart under the chosen divergence (Wasserstein, KL, or $\chi^{2}$). As a result, DRO–REBEL inherits REBEL’s stability and low sample complexity while gaining worst‐case robustness guarantees under distributional shifts.

\section{Distributionally Robust REBEL}
In this section, we formally define the DRO variants of REBEL under type-$p$ Wasserstein, KL, and $\chi^{2}$
ambiguity sets. We begin by specifying the nominal data-generating distribution at round $t$.
Recall the sampling procedure in Section~\ref{subsection:RLHF}: a prompt $x\in\mathcal{S}$ is drawn from a
prompt distribution $\rho$, and we sample two responses $a^{1}\sim \pi_{\theta_t}(\,\cdot\mid x)$ and
$a^{2}\sim \mu_t(\,\cdot\mid x)$, where $\mu_t$ is a base distribution. The base distribution $\mu_t$ can be
a fixed offline behavior distribution (e.g., an instruction fine-tuning dataset) or the current policy
$\pi_{\theta_t}$. Throughout, we focus on the fully on-policy online setting and take $\mu_t=\pi_{\theta_t}$.
We assume access to a reward oracle $r(x,a)$ that returns a scalar score for a prompt--completion pair $(x,a)$.
This reward can be pre-defined (e.g., ROUGE against reference responses) or produced by a learned reward model
trained from offline demonstrations or preference data (e.g. the RLHF paradigm \cite{christiano2017deep, ziegler2019fine}).
We form the observed relative-reward signal $\Delta r := r(x,a^{1})-r(x,a^{2})$.

Importantly, we allow the observed reward difference $\Delta r$ to be noisy, e.g., due to reward-model error or
human inconsistency. Concretely, we assume an additive observation model $\Delta r \;=\; \Delta r^{\star}(x,a^{1},a^{2}) \;+\; \xi$ where $\Delta r^{\star}(x,a^{1},a^{2}) := r^{\star}(x,a^{1})-r^{\star}(x,a^{2})$ is the latent (true) reward difference and $\xi$ is mean-zero noise, possibly conditional on $(x,a^{1},a^{2})$. Using this sampling process, we can now define the nominal data-generating distribution.

\vspace{1em}

\begin{definition}[Joint data-generating distribution for REBEL]
\label{defn:nominal-data-generating-dist}
Let $\mathcal{Z}=\mathcal{S}\times\mathcal{A}\times\mathcal{A}\times\mathbb{R}$.
Fix an iteration $t$ and a base distribution $\mu_t(\cdot\mid x)$ (e.g.\ $\mu_t=\pi_{\theta_t}$ in the fully
on-policy online setting). Let $\Delta r^{\star}(x,a^1,a^2):=r^{\star}(x,a^1)-r^{\star}(x,a^2)$ denote the latent
(true) reward difference. Let $\nu(\cdot\mid x,a^1,a^2)\in\mathcal{P}(\mathbb{R})$ be a Markov kernel
describing the conditional law of the observation noise $\xi$ given $(x,a^1,a^2)$, and define the induced
conditional law of the observed reward difference by the shifted kernel
\[
Q_t^\circ(A\mid x,a^1,a^2)
\;:=\;
\nu(A-\Delta r^\star(x,a^1,a^2)\mid x,a^1,a^2),
\;\; \forall A\in\mathcal{B}(\mathbb{R}),
\]
where $A-c:=\{u\in\mathbb{R}:u+c\in A\}$. We define the nominal distribution $P_t^\circ$ over $z=(x,a^1,a^2,\Delta r)\in\mathcal{Z}$ by the sampling process:
\[
x\sim\rho,\;\; a^1\sim\pi_{\theta_t}(\cdot\mid x),\;\; a^2\sim\mu_t(\cdot\mid x),\;\;
\Delta r \sim Q_t^\circ(\cdot\mid x,a^1,a^2).
\]
Equivalently, for any measurable $A\subseteq\mathbb{R}$,
\[
P_t^\circ\!\bigl(x,a^1,a^2, A\bigr)
=
\rho(x)\,\pi_{\theta_t}(a^1\mid x)\,\mu_t(a^2\mid x)\,Q_t^\circ\!\bigl(A\mid x,a^1,a^2\bigr).
\]
We assume $P_t^\circ$ generates the round-$t$ batch $\mathcal{D}_t=\{z_{t,i}\}_{i=1}^{n_t}$ i.i.d., i.e.\
$z_{t,i}\sim P_t^\circ$, and write the empirical measure
$P_{n,t}^\circ \;:=\; \frac{1}{n_t}\sum_{i=1}^{n_t}\delta_{z_{t,i}}$.
\end{definition}




\begin{figure}[h!]
    \centering
    \includegraphics[width=0.9\textwidth]{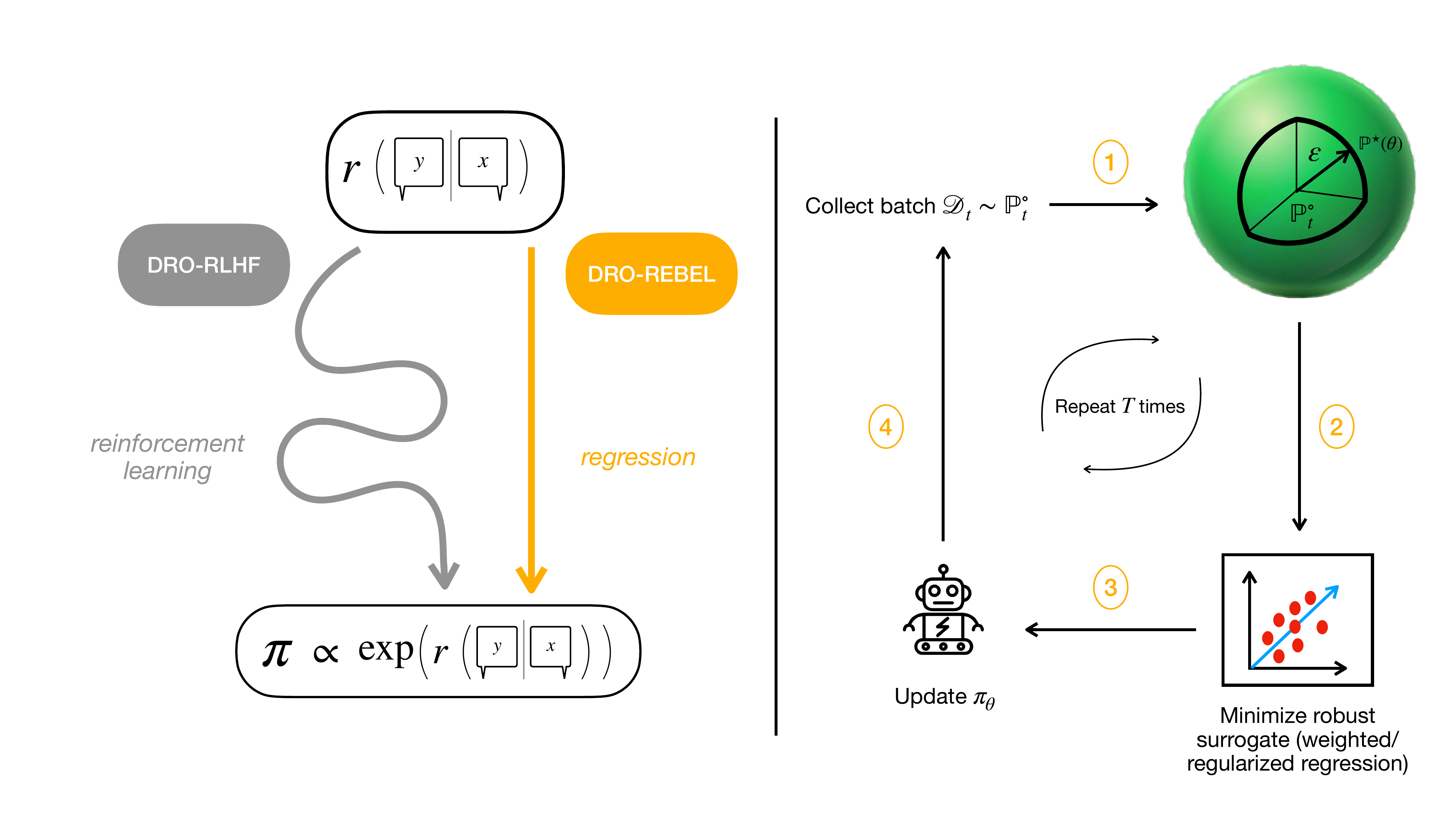}
    \caption{Illustration of the DRO–REBEL update loop. At each iteration: \textbf{(1)} a batch of preference tuples $z=(x,a^1,a^2)$ is sampled from the nominal distribution $P_{n,t}^\circ$ induced by the current policy, \textbf{(2)} the per-sample squared-error loss $\ell_{\rm REBEL}(z;\theta)$ is computed, \textbf{(3)} a distributionally robust objective is formed over an ambiguity set $\mathcal{B}_\varepsilon\!\left(P_{n,t}^\circ\right)$ (e.g., Wasserstein, KL, or $\chi^2$), aligning the predicted change in log-probabilities $\Delta \log \pi_\theta$ with the observed reward differences $\Delta r$, and \textbf{(4)} the policy parameters $\theta$ are updated via a gradient step. The left-side figure is credited to \cite{gao2024rebel}.}
    \label{fig:dro-rebel-figure}
\end{figure}

\subsection{Distributionally Robust REBEL}
From the REBEL update (Equation (\ref{eq:rebel-regression})), we define the pointwise loss as follows

\[
    \ell_{t}(z ; \theta) =  \left(\frac{1}{\eta} \left[
        \log \frac{\pi_\theta(a^{1} \mid x)}{\pi_{\theta_t}(a^{1} \mid x)} -
        \log \frac{\pi_\theta(a^{2} \mid x)}{\pi_{\theta_t}(a^{2} \mid x)}
        \right]
        - \left[r(x, a^{1}) - r(x, a^{2})\right]
        \right)^2.
\]
For $\varepsilon > 0$, define the ambiguity set as $\mathcal{B}_{\varepsilon} \left( P^{\circ} ; D \right)$ for the nominal distribution $P^{\circ}$ and the distance measure $D$. Using the DRO framework, we consider the following distributionally robust optimization problem: 

\[
    \min_{\theta} \max_{P \in \mathcal{B}_{D} \left( P_{t}^{\circ} ; 
    \varepsilon \right)} \mathbb{E}_{z \sim P} \left[ \ell_{t}(z;\theta) \right],
\]

which directly captures our objective: finding the best policy under worst-case distributional shift. Now, let us define the following \( D \)-DRO-REBEL loss function:
\[
    \mathcal{L}^{D}\left(\theta; \varepsilon\right) = \sup_{P \in \mathcal{B}_{D} \left( P_{t}^{\circ} ; 
    \varepsilon \right)} \mathbb{E}_{z \sim P} \left[ \ell_{t}(z;\theta) \right],
\]
where \( \mathcal{B}_{\varepsilon}(P_{t}^{\circ}; D) \) denotes an ambiguity set centered at the nominal distribution \( P_{t}^{\circ} \), defined using a divergence or distance \( D \). This formulation is general and allows us to instantiate a family of distributionally robust REBEL objectives by choosing different \( D \)---such as the type-\( p \) Wasserstein distance, Kullback--Leibler (KL) divergence, or chi-squared (\( \chi^2 \)) divergence. Each choice of \( D \) yields a different robustness profile and tractable dual formulation, enabling us to tailor the algorithm to specific distributional shift scenarios. When the nominal distribution $P_{t}^{\circ}$ is replaced with its empirical counterpart, i.e., $P_{t, n}^{\circ} := \frac{1}{n} \sum_{i=1}^n \delta_{z_i}$, where $z_1, \ldots, z_n$ are $n$ i.i.d. samples from $P_{t}^{\circ}$, we use $\mathcal{L}^{D}_n(\theta; \varepsilon)$ to denote the empirical $D$-REBEL loss incurred by the policy parameter $\theta$.

\section{Theoretical Results}
In this section, we will provide several sample complexity results for DRO-REBEL under the ambiguity sets previously mentioned. First, we state some assumptions we make in our analysis.

\vspace{1em}

\begin{assumption}[Linear reward class]
\label{assumption:linear-reward-class}
Let $\phi: \mathcal{S} \times \mathcal{A} \rightarrow \mathbb{R}^{d}$ be a known $d$-dimensional feature mapping with $\sup_{x, a} ||\phi(x,a)||_{2} \leq 1$ and $\omega \in \mathbb{R}^{d}$ such that $||\omega||_{2} \leq F$ for $F > 0$. We consider the following class of linear reward functions:
\[
    \mathcal{F} : = \left\{  r_{\omega} : r_{\omega}(x, a) = \phi(x, a)^{\top}\omega \right\}.
\]

\end{assumption}

\begin{remark}
These are standard assumptions in the theoretical analysis of inverse reinforcement learning (IRL) \cite{DBLP:journals/corr/AbbeelN04, DBLP:journals/corr/Ng00algorithmsfor}, imitation learning \cite{DBLP:journals/corr/HoE16}, and areas such as RLHF \cite{pmlr-v202-zhu23f} where reward functions are learned. Our analysis can be extended to neural reward classes where $\phi(x, a)^\top \omega$ is replaced by $f_\omega(x, a)$, where $f_\omega$ is a neural network satisfying twice differentiability, smoothness, and boundedness of the $f_{\omega}$ and $\nabla_\omega f_\omega(x, a)$.
\end{remark}

\vspace{1em}

\begin{assumption}[Log-linear policy class]
\label{assumption:log-linear-policy}
Let $\psi : \mathcal{S} \times \mathcal{A} \rightarrow \mathbb{R}^d$ be a known $d$-dimensional feature mapping with $\sup_{x,a} \| \psi(x, a) \|_2 \leq 1$. Assume a bounded policy parameter set $\Theta := \{ \theta \in \mathbb{R}^d : \| \theta \|_2 \leq B \}$. We consider the following class of log-linear policies:
\[
\Pi : = \left\{ \pi_\theta : \pi_\theta(a \mid x) = \frac{\exp\left( \theta^\top \psi(x, a) \right)}{\sum_{a' \in \mathcal{A}} \exp\left( \theta^\top \psi(x, a') \right)} \right\}.
\]
\end{assumption}

\begin{remark}
These are standard assumptions in the theoretical analysis of RL algorithms \cite{JMLR:v22:19-736, pmlr-v108-modi20a}, RLHF \cite{pmlr-v202-zhu23f}, and DPO \cite{nika2024rewardmodellearningvs, chowdhury2024provablyrobustdpoaligning}. Our analysis can be extended to neural policy classes where $\theta^\top \psi(s, a)$ is replaced by $f_\theta(s, a)$, where $f_\theta$ is a neural network satisfying twice differentiability and smoothness assumptions.
\end{remark}

We also make the following policy-induced data coverage assumption along the REBEL trajectory.

\vspace{1em}

\begin{assumption}[Policy-induced regularity condition (trajectory coverage)]
\label{assumption:data-coverage}
Let $\{\pi_{\theta_t}\}_{t\ge 0}$ denote the sequence of policies generated by REBEL, and let $\rho$ be the context distribution.
For each iteration $t$, define the nominal on-policy distribution $P_t^\circ$ over quadruples $(x,a^{1},a^{2},\Delta r)$ by
\[
x \sim \rho, \;\; a^{1},a^{2} \sim \pi_{\theta_t}(\cdot \mid x),
\;\; \Delta r = r(x, a^{1}) - r(x, a^{2}),
\]
For any distribution $P$ over $(x,a^{1},a^{2},\Delta r)$, define
\[
\Sigma_{P, t}
:= \mathbb{E}_{(x,a^{1},a^{2},\Delta r)\sim P}
\Big[
\big(\psi(x,a^{1})-\psi(x,a^{2})\big)\big(\psi(x,a^{1})-\psi(x,a^{2})\big)^\top
\Big].
\]
There exists a constant $\kappa >0$ (uniform in $t$) such that, for every iteration $t\ge 0$,
\[
\Sigma_{P, t} \succeq \kappa I,
\;\; \forall\, P \in \mathcal{B}_{\varepsilon}\!\left(P_t^\circ; D\right).
\]
\end{assumption}

\begin{remark}
Related coverage / persistence-of-excitation assumptions under linear function approximation are standard in the RL literature \cite{agarwal21, wang2020statisticallimitsofflinerl, jin2022pessimismprovablyefficientoffline}. 
In our on-policy setting, the relevant notion of coverage is policy-induced and must hold uniformly along the REBEL trajectory. Concretely, Assumption~\ref{assumption:data-coverage} requires that for every iteration $t$ the nominal on-policy distribution $P_t^\circ$ (induced by $x\sim\rho$ and $a^1,a^2\sim\pi_{\theta_t}(\cdot\mid x)$) yields a well-conditioned feature-difference covariance, robustly over the ambiguity set.
\end{remark}

\subsection{"Slow Rate" Estimation Errors }
Define $\theta_{t}^{\mathcal{W}_{p}} \in \mathrm{arg}\min_{\theta \in \Theta} \mathcal{L}_{t}^{\mathcal{W}_{p}}(\theta)$ be the true optimal policy estimate and the empirical estimate as $\hat{\theta}_{n, t}^{\mathcal{W}_{p}} \in \mathrm{arg}\min_{\theta \in \Theta} \mathcal{L}_{n, t}^{\mathcal{W}_{p}}(\theta)$. First, we provide a sample complexity result for convergence of robust policy estimation using REBEL. Our proof technique hinges on showing that $\mathcal{L}^{\mathcal{W}_{p}}$ is strongly convex.

\vspace{1em}

\begin{lemma}[Strong convexity of $\mathcal{L}_{t}^{\mathcal{W}_{p}}$]
    Let $\ell_{t}(z;\theta)$ be as defined in the REBEL update. The Wasserstein-DRO-REBEL loss 
    \[
        \mathcal{L}_{t}^{\mathcal{W}_{p}}\left(\theta; \varepsilon\right) = \sup_{P \in \mathcal{B}_{\mathcal{W}_{p}} \left( P_{t}^{\circ} ; 
    \varepsilon \right)} \mathbb{E}_{z \sim P} \left[ \ell_{t}(z;\theta) \right],
    \]
    is $2\lambda/\eta^{2}$-strongly convex where $\lambda$ is from the regularity condition in Assumption \ref{assumption:data-coverage} and $\eta$ is from the step size defined in the DRO update \ref{eq:rebel-regression}
\end{lemma}

We now present our "slow rate" results on the sample complexity for the convergence of the robust policy parameter.

\vspace{1em}

\begin{theorem} ["Slow" Estimation error of $\theta^{\mathcal{W}_{p}}$]
\label{theorem:Wasserstein-DRO-REBEL}
Let $\delta \in (0,1)$. Then for all $t \in [T]$, with probability at least $1-\delta$,
\[
    ||\theta_{t}^{\mathcal{W}_{p}} - \hat{\theta}_{n, t}^{\mathcal{W}_{p}}||_{2}^{2} \lesssim \frac{\eta^{2}K_{g}^{2}}{\kappa}  \sqrt{\frac{d \log n_{t} + \log(\overline{\Delta}) + \log\left(T / \delta \right)}{n_{t}} },
\]
where $\kappa$ is from the regularity condition in Assumption \ref{assumption:data-coverage} and $K_{g} = 8B/\eta + 2F$ where $B$ is from the assumption that the policy parameter set is bounded in Assumption \ref{assumption:log-linear-policy}, $F$ is from Assumption \ref{assumption:linear-reward-class}, and $T$ is the number of iterations we run REBEL.
\end{theorem}

\begin{proof}[Proof sketch]
We first prove that $\ell_{t}(z;\theta)$ is uniformly bounded and is $4K_{g}/\eta$-Lipschitz in $\theta$ where $K_{g} = 4B / \eta + 2F$ (Appendix \ref{appendix:uniform-and-lipschitz-bound}). Using this, we can prove that $\mathbb{E}_{z \sim P} \left[ \ell_{t}(z;\theta) \right]$ is $2/\eta$-strongly convex in $||\cdot||_{\Sigma_{P, t}}$. Intuitively, taking the supremum over $P$ should preserve the convex combination and the negative quadratic term, and doing this analysis formally allows us to show that $\mathcal{L}_{t}^{\mathcal{W}_{p}}\left(\theta; \varepsilon\right)$ is $2\kappa/\eta^{2}$-strongly convex in $||\cdot||_{2}$. The detailed proof for strong convexity can be found in Lemma \ref{appendix:strong-convexity-of-wasserstein-loss}. 

The strong duality of Wasserstein DRO \cite{gao2022distributionallyrobuststochasticoptimization} (Corollary \ref{appendix:gao-strong-duality-wasserstein}) allows us to reduce the difference $\lvert \mathcal{L}_{t}^{\mathcal{W}_{p}}\left(\theta; \varepsilon\right) - \mathcal{L}_{n, t}^{\mathcal{W}_{p}}\left(\theta; \varepsilon\right) \rvert$ to the concentration $\left \lvert \mathbb{E}_{z \sim P_{n, t}^{\circ}} \left[ \ell_{\Delta, t}(z; \theta) \right] - \mathbb{E}_{z \sim P_{t}^{\circ}} \left[ \ell_{\Delta, t}(z; \theta) \right] \right \rvert$ where $\ell_{\Delta, t}(z; \theta)$ is the Moreau envelope of $-\ell_{t}$. We then use conditional Hoeffding’s inequality and a covering argument to obtain a concentration result that is uniform on $\theta \in \Theta$ and $\Delta$. We can now do a "three-term" decomposition of $\mathcal{L}^{\mathcal{W}_{p}}\left(\theta_{t}^{\mathcal{W}_{p}}; \varepsilon\right)- \mathcal{L}^{\mathcal{W}_{p}}\left(\hat{\theta}_{n, t}^{\mathcal{W}_{p}}; \varepsilon\right)$ into $\mathcal{L}^{\mathcal{W}_{p}}\left(\theta_{t}^{\mathcal{W}_{p}}; \varepsilon\right) - \mathcal{L}_{n, t}^{\mathcal{W}_{p}}\left(\theta_{t}^{\mathcal{W}_{p}}; \varepsilon\right)$, $\mathcal{L}_{n, t}^{\mathcal{W}_{p}}\left(\theta_{t}^{\mathcal{W}_{p}}; \varepsilon\right) - \mathcal{L}_{n, t}^{\mathcal{W}_{p}}\left(\hat{\theta}_{n, t}^{\mathcal{W}_{p}}; \varepsilon\right)$, and $\mathcal{L}_{n, t}^{\mathcal{W}_{p}}\left(\hat{\theta}_{n, t}^{\mathcal{W}_{p}}; \varepsilon\right) - \mathcal{L}^{\mathcal{W}_{p}}\left(\hat{\theta}_{n, t}^{\mathcal{W}_{p}}; \varepsilon\right)$ and bound the first and last term by Hoeffding and the second term by 0. Using strong convexity of $\mathcal{L}^{\mathcal{W}_{p}}$, we can get the estimation error. The detailed proof for the "slow rate" estimation error can be found in Appendix \ref{appendix:proof-slow-rate-wasserstein}.
\end{proof}

We prove similar results for KL and $\chi^{2}$ ambiguity sets using the same ideas used in the Wasserstein ambiguity set setting. We state the "slow rate" estimation rates below and defer the proofs to Appendix \ref{appendix:Slow-KL-DRO-REBEL} and Appendix \ref{appendix:Slow-chisquared-DRO-REBEL}.

\vspace{1em}

\begin{theorem} ["Slow" Estimation error of $\theta^{\mathrm{KL}}$]
\label{theorem:KL-DRO-REBEL}
Let $\delta \in (0,1)$. Then for all $t \in [T]$, with probability at least $1-\delta$,
\[
    \lVert \theta_{t}^{\mathrm{KL}} - \hat{\theta}_{n, t}^{\mathrm{KL}} \rVert_{2}^{2} \lesssim \frac{\eta^2 \bar{\lambda} \exp(K_g^2 / \underline{\lambda})}{\kappa} \sqrt{\frac{d \log n_{t} + \log(\overline{\lambda}) + \log\left(T / \delta \right)}{n_{t}} },
\]
where $\kappa$ is from the regularity condition in Assumption \ref{assumption:data-coverage}, $K_{g} = 8B/\eta + 2F$, $\bar{\lambda}, \underline{\lambda}$ is from Assumption \ref{Assumption:KL-DRO-Assumption}, $\eta$ is the stepsize defined in the DRO update \ref{eq:rebel-regression}, and $T$ is the number of iterations we run REBEL.
\end{theorem}

\vspace{1em}

\begin{theorem} ["Slow" Estimation error of $\theta^{\chi^2}$]
\label{theorem:chi-squared-DRO-REBEL}
Let $\delta \in (0,1)$. Then for all $t \in [T]$, with probability at least $1-\delta$,
\[
\|\theta_{t}^{\chi^2}-\hat\theta_{n, t}^{\chi^2}\|_{2}^{2}
\;\lesssim\;
\frac{\eta^2}{\kappa}
\Big(K_\ell+\frac{K_\ell^2}{\underline{\lambda}}\Big)
\sqrt{\frac{d\log n_t+\log(T/\delta)}{n_t}},
\]
where $\kappa$ is from the regularity condition in Assumption \ref{assumption:data-coverage} and $K_{g} = 8B / \eta + 2F$ where $B$ is from the assumption that the policy parameter set is bounded in Assumption \ref{assumption:log-linear-policy}, $F$ is from Assumption \ref{assumption:linear-reward-class}, $\eta$ is from the step size defined in the DRO update \ref{eq:rebel-regression}, and $T$ is the number of iterations we run REBEL.
\end{theorem}

\vspace{1em}

\begin{remark}[Uniformity in prior analyses] A subtle but essential point in any DRO finite-sample analysis is that both the primal minimizer $\hat\theta_{n,t}\in\arg\min_{\theta\in\Theta}\mathcal{L}_{n,t}(\theta)$ and the dual optimizer—e.g., $\hat\Delta_{n,t}$ for Wasserstein or $\hat\lambda_{n,t}$ for KL—are measurable functions of the dataset $\mathcal{D}t$, hence random. Controlling quantities such as $|\mathcal{L}_{t}(\hat\theta_{n,t})-\mathcal{L}_{n,t}(\hat\theta_{n,t})|$ therefore cannot rely on concentration inequalities proved at a fixed $\theta$ or a fixed dual variable. One must either (i) upgrade pointwise concentration to a uniform statement over $\Theta$ and the dual domain via covering numbers and Lipschitzness, or (ii) substitute an alternative proof techniques such as algorithmic stability. We take route (i), building $\varepsilon$-nets jointly in the primal and dual variables. Prior WDPO/KLDPO analyses, including \cite{xu2025distributionallyrobustdirectpreference}, establish their loss-convergence lemmas pointwise and then apply them at data-dependent minimizers---such as $\hat\theta_{n,t}$, $\hat\Delta_{n,t}$, and $\hat\lambda_{n,t}$---without an explicit uniformity argument. Because these quantities all depend on $\mathcal{D}_t$, this step requires a covering or stability justification that the original arguments omit. A related gap concerns dimensional dependence: prior bounds are often written in terms of curvature or coverage parameters, which can make the resulting rates appear dimension-free. This appearance is misleading, since under standard feature normalization such coverage quantities typically deteriorate with dimension, rendering the dependence implicit rather than absent. Our proofs address both issues simultaneously: we supply the missing uniformity argument that licenses applying pointwise convergence at data-dependent minimizers, and we make the hidden dimensional dependence explicit throughout.
\end{remark}

With the above analysis, building on the techniques of \cite{xu2025distributionallyrobustdirectpreference}, we recover the same \(\widetilde{\mathcal{O}}(\sqrt{d/n})\) squared parameter estimation‐error rate but with substantially sharper constants. {\cite{xu2025distributionallyrobustdirectpreference}} \ observe that WDPO’s squared paramter estimation error decays at $\widetilde{\mathcal{O}}(\sqrt{d/n})$, while non‐robust DPO already achieves $\widetilde{\mathcal{O}}(d/n)$. This slowdown arises because, in the non-robust setting, one can compute the closed-form expression for $\nabla_\theta(1/n) \sum_{i=1}^n l(z_i; \theta)$. This allows us to write $\|\nabla_{\theta}(1/n) \sum_{i=1}^n l(z_i; \theta^*)\|_{(\Sigma_D +\lambda I)^{-1}}$ in quadratic form and then obtain a concentration using Bernstein's inequality. Coupled with a Taylor linear approximation argument and the strong convexity of the empirical DPO loss, we get the claimed $\widetilde{\mathcal{O}}(d/n)$. However, for the robust setting, one cannot exchange the supremum over distributions with the gradient operator on the empirical robust loss. That is, $\nabla_\theta\mathcal{L}_{n}^{\mathcal{W}_{p}} (\theta^{\mathcal{W}_{p}})\neq \sup_{P\in\mathcal{B}_{\varepsilon} \left( P^{\circ} ; \mathcal{W}_{p} \right)} \nabla_\theta\mathbb{E}_{z\sim P}[l(z; \theta^{\mathcal{W}_{p}})]$. Thus, the approach taken in the non-robust setting will not work for the robust setting. As a result, the proof in the robust setting relies on the strong convexity of the robust DPO loss itself and a concentration bound on the loss function's convergence (obtained via Hoeffding's inequality for bounded random variables in this case). This indirect approach of bounding the loss and then translating it to a parameter bound through strong convexity results in a slower $\widetilde{\mathcal{O}}(\sqrt{d/n})$ convergence rate for the policy parameter. Closing this gap and restoring the optimal inverse square root rate for robust DPO remains an open problem. To close this gap, we develop a \emph{localized Rademacher complexity} analysis for DRO-REBEL which recovers the optimal $\widetilde{\mathcal{O}}(d/n)$ convergence rate even under various ambiguity sets.

\section{"Fast Rate" Estimation Errors}
We now present the fast estimation error rates achieved by distributionally robust relative reward regression. We first present the result for ambiguity sets measured under type-p Wasserstein distance.

\vspace{1em}

\begin{theorem}["Fast" estimation error of $\theta^{\mathcal{W}_p}$]
\label{thm:fast-wasserstein-dro-rebel}
Fix $\delta \in (0, 1)$. Assume Assumption~\ref{assumption:data-coverage} (data coverage) and the
Wasserstein dual regularity and selection conditions in
Appendix~\ref{appendix:proofs-fast-wasserstein}, namely that $\mathcal{Z}$ is a proper Polish space,
the loss is regular, measurable maximizer and minimizer selectors exist, and the dual domain satisfies
$\Delta^\star \in [0, \overline{\Delta}_{p}]$. Suppose moreover that the $\Delta$-indexed gradient class satisfies
the H\"{o}lder entropy condition in
Assumption~\ref{appendix:assumption-holder-delta-entropy} with parameters $(\alpha, L_\Delta)$.
If the batch size satisfies $n_t \gtrsim \log(dT/\delta) / \kappa$, then with probability at least
$1 - \delta$, simultaneously for all $t \in [T]$,
\[
    \bigl\| \theta_t^{\mathcal{W}_p} - \hat{\theta}_{n,t}^{\mathcal{W}_p} \bigr\|_2^2
    \;\lesssim\;
    \frac{\eta^2 K_g^2 }{\kappa^2}
    \cdot
    \frac{d \Bigl( \log(2dT/\delta)
    + \frac{1}{\alpha_{p}} \log\bigl(1 + L_{\Delta, p} \overline{\Delta}_{p}^{\alpha_{p}} \eta / K_{g}\bigr) \Bigr)}{n_t}.
\]
Here $\kappa$ is the coverage constant from Assumption~\ref{assumption:data-coverage}, and
$K_g = 8B/\eta + 2F$ where $B$ is from Assumption~\ref{assumption:log-linear-policy} and $F$ is from
Assumption~\ref{assumption:linear-reward-class}.
\end{theorem}

\begin{proof}[Proof sketch]
By strong convexity of $\theta \mapsto \mathcal{L}_t^{\mathcal{W}_p}(\theta; \varepsilon)$ with modulus
$2\kappa/\eta^2$ (Lemma~\ref{appendix:strong-convexity-of-wasserstein-loss}), the parameter error
$\|\theta_t^{\mathcal{W}_p} - \hat{\theta}_{n,t}^{\mathcal{W}_p}\|_2^2$ reduces to a gradient deviation
at the population minimizer (Lemma~\ref{appendix:reduction-to-gradient-concentration-wasserstein}).
The measurable selection conditions in Appendix~\ref{appendix:proofs-fast-wasserstein} then allow us
to represent a valid subgradient of $\mathcal{L}_t^{\mathcal{W}_p}$ as an empirical process indexed
only by $\Delta \in [0, \bar{\Delta}]$. The uniform concentration on this one-dimensional dual index,
combined with the H\"{o}lder entropy condition in
Assumption~\ref{appendix:assumption-holder-delta-entropy}, yields the stated rate
(Lemma~\ref{appendix:uniform-conditional-concentration}). The detailed proof can be found in
Appendix~\ref{appendix:proofs-fast-wasserstein}.
\end{proof}

We now present the result for KL ambiguity sets.

\vspace{1em}

\begin{theorem}["Fast" estimation error of $\theta^{\mathrm{KL}}$]
\label{thm:fast-kl-dro-rebel}
Fix $\delta \in (0,1)$. Assume Assumption~\ref{Assumption:KL-DRO-Assumption} (bounded KL-dual domain
and strong duality), Assumption~\ref{assumption:data-coverage} (coverage constant $\kappa$), and
Assumption~\ref{assumption:noisy-reward-function} (noise level $\tau_\xi^2 > 0$). With probability at least $1 - \delta$, simultaneously for all $t \in [T]$,
\[
    \bigl\| \hat{\theta}^{\mathrm{KL}}_{n,t} - \theta^{\mathrm{KL}}_t \bigr\|_2^2
    \;\lesssim\;
    \frac{\eta^2}{\kappa}
    \cdot
    \frac{C_{\mathrm{KL}} \left( d \log n_t + \log(T/\delta) \right)}{n_t}.
\]
Here $K_g = 8B/\eta + 2F$ (with $B$ from Assumption~\ref{assumption:log-linear-policy} and $F$ from
Assumption~\ref{assumption:linear-reward-class}), $\mathcal{J}_{\max} := \exp(K_\ell /
\underline{\lambda})$, and the constant $C_{\mathrm{KL}}$ corresponds to a constant as follows:
\[
    C_{\mathrm{KL}}
    :=
    \underbrace{
        \mathcal{J}_{\max}^2 \Bigl( 1 + \frac{K_\ell^2}{\underline{\lambda}} \Bigr)
    }_{\text{empirical-process terms}}
    +
    \underbrace{
        \overline{\lambda}\,(\mathcal{J}_{\max} - 1)
        \Biggl( 1 +
        \frac{\overline{\lambda}^3 \mathcal{J}_{\max}}{\tau_\xi^2}
        \Bigl( \frac{K_g}{\eta} \Bigr)^2
        \Bigl( \frac{2}{\underline{\lambda}} + \frac{K_\ell}{2\underline{\lambda}^2} \Bigr)^2
        \Biggr)
    }_{\lambda\text{-mismatch via strong convexity in }\lambda}.
\]
\end{theorem}

\begin{proof}[Proof sketch]
We decompose the excess robust risk $\mathcal{L}^{\mathrm{KL}}_t(\theta_t^{\mathrm{KL}}; \varepsilon)
- \mathcal{L}^{\mathrm{KL}}_t(\hat{\theta}_{n,t}^{\mathrm{KL}}; \varepsilon)$ into two terms:
(i) $\theta$-suboptimality at the random empirical dual optimizer $\hat{\lambda}$, and
(ii) $\lambda$-mismatch. For term (i), we apply a Talagrand--Bousquet inequality to the centered, scaled exponential-tilt
class $\gamma_t(\cdot, \lambda; \theta)$, yielding an $\mathcal{O}(d \log n_t / n_t)$ excess-risk
bound. For term (ii), Assumption~\ref{assumption:noisy-reward-function} implies that $\lambda \mapsto
\Psi_t(\theta, \lambda)$ is $m_\lambda$-strongly convex uniformly in $\theta$ via a variance lower
bound under exponential tilting, so the $\lambda$-mismatch
reduces to a squared deviation of $\nabla_\lambda \Psi_t$. This squared deviation is controlled by
(a) a standard empirical-process concentration bound on $\nabla_\lambda \Psi_{n,t} - \nabla_\lambda
\Psi_t$, and (b) a sharp $\theta$-Lipschitz bound for $\nabla_\lambda \Psi_{n,t}(\theta, \lambda)$
(Lemma~\ref{lem:theta_lip_partial_lambda_Psi}). Finally, strong convexity of $\theta \mapsto
\mathcal{L}_t^{\mathrm{KL}}(\theta; \varepsilon)$ with modulus $\kappa/\eta^2$
(Lemma~\ref{appendix:strong-convexity-of-KL-loss}) converts the fast $\mathcal{O}(d \log n_t / n_t)$
excess-risk rate into the stated parameter estimation rate. The detailed proofs can be found in Appendix \ref{appendix:proofs-fast-kl}.
\end{proof}

\paragraph{Discussion.} Theorems~\ref{thm:fast-wasserstein-dro-rebel}--\ref{thm:fast-kl-dro-rebel} show that DRO-REBEL attains the parametric rate $\|\hat\theta_{n,t}-\theta_t\|_2^2=\widetilde{\mathcal{O}}(d/n_t)$ under each ambiguity set, uniformly along the online trajectory. Prior robust preference-optimization analyses (WDPO/KLDPO) report only $\widetilde{\mathcal{O}}(\sqrt{d/n_{t}})$ parameter rates, a consequence of the mismatch between the distributional supremum and differentiation. REBEL sidesteps this: the per-round loss is quadratic in the log-policy increment, so trajectory coverage gives strong convexity of the robust population objective and parameter error reduces to controlling an empirical subgradient. The bound cleanly separates a statistical term $\widetilde{\mathcal{O}}(d/n_t)$ from a robustness term that enters through dual geometry (Wasserstein) or exponential tilting (KL). Distributional robustness in REBEL therefore carries no statistical efficiency cost.

The technical challenges differ across ambiguity sets. For Wasserstein, measurability and selection issues arise from the dual supremum over perturbed samples; the fast rate relies on representing the subgradient as an empirical process indexed by the one-dimensional dual variable $\Delta$ and establishing uniform concentration over $\Delta$. For KL, the finite-dimensional smooth dual admits a localized Rademacher/Talagrand--Bousquet argument for the $\widetilde{\mathcal{O}}(d/n_t)$ excess-risk bound, with strong convexity of the dual variable controlled via a variance lower bound under exponential tilting.

We conjecture the fast rate extends to $\chi^2$ and, more broadly, to general $f$-divergence ambiguity sets as the dual reformulation and localized complexity tools from the KL proof should carry over with some modification. The same analysis should also go through for DPO-based objectives, provided one can verify strong convexity and the covering conditions under the DPO loss.

\section{Approximate Tractable Algorithms for Robust LLM Alignment}
While our Distributionally Robust REBEL (DRO-REBEL) formulations benefit from finite-sample guarantees which are minimax optimal, directly solving the minimax objective using stochastic gradient descent methods can be computationally challenging. As \cite{xu2025distributionallyrobustdirectpreference} also points out in the context of robust DPO, this challenge arises because we do not have direct control over the data distribution $P \in \mathcal{B}_{\varepsilon} \left( P_{t}^{\circ} ; D \right)$ within the uncertainty set, as it is not parameterized in a straightforward manner. Furthermore, preference data are generated according to the nominal distribution $P^{\circ}$, which means that we lack samples from any other distribution within the uncertainty set $\mathcal{B}_{\varepsilon} \left( P_{t}^{\circ} ; D \right)$. To overcome this, we introduce principled tractable algorithms that approximate the solution to our DRO-REBEL objectives. Our algorithms for solving Wasserstein-DRO-REBEL and KL-DRO-REBEL are largely the same as those of \cite{xu2025distributionallyrobustdirectpreference}'s WDPO and KLDPO. However, we will derive and propose an algorithm for $\chi^{2}$-DRO-REBEL that can be efficiently solved using stochastic gradient descent methods. 

\subsection{Tractable Wasserstein DRO-REBEL (WD-REBEL)}
The connection between Wasserstein distributionally robust optimization (DRO) and regularization has been established previously in the literature, see \cite{shafieezadehabadeh2019regularizationmasstransportation} for example. We take advantage of recent progress in the Wasserstein theory in connecting the Wasserstein DRO to regularization. For $p$-Wasserstein DRO, $p \in (1, \infty]$, \cite{gao2022distributionallyrobuststochasticoptimization} (Theorem 1) shows that for a broad class of loss functions (potentially non-convex and non-smooth), with high probability, Wasserstein DRO is asymptotically equivalent to a variation regularization. In particular, an immediate consequence is that, when $p = 2$:
\[
\min_{\theta \in \Theta} \sup_{P \in \mathcal{B}_{\varepsilon_{n}} \left( P_{n, t}^{\circ} ; \mathcal{W}_{p} \right)} \mathbb{E}_{z \sim P} \left[ \ell(z;\theta) \right] = \min_{\theta \in \Theta} \left\{\mathbb{E}_{z \sim P_{n, t}^{\circ}} [\ell(z; \theta)] + \varepsilon_{n} \sqrt{(1/n) \sum_{i=1}^n\|\nabla_z \ell(z_i; \theta)\|_2^2}\right\} + O_P(1/n),
\]
where $\varepsilon_n = \mathcal{O}(1/\sqrt{n})$. This indicates that one can approximate the Wasserstein DRO objective by adding a gradient regularization term to the empirical loss of risk minimization (ERM), $\mathbb{E}_{z \sim P_{n, t}^{\circ}}[\ell(z; \theta)]$. Based on this, we propose a tractable WD-REBEL algorithm in Algorithm \ref{alg:wd-rebel}.

\begin{algorithm}[H]
\caption{Online WD-REBEL (Wasserstein-DRO REBEL)}
\label{alg:wd-rebel}
\begin{algorithmic}[1]
\State \textbf{Input:} Reward function $r$, policy class $\Pi=\{\pi_\theta:\theta\in\Theta\}$, context distribution $\rho$, robustness radius $\rho_0$, step size $\tilde{\eta}$, inner steps $K$
\State Initialize policy parameters $\theta_0$
\For{$t = 0$ to $T-1$}
    \State Collect batch $\mathcal{D}_t=\{z_i\}_{i=1}^{n_t}$ where $z_i=(x_i,a_i^{1},a_i^{2})$,
    \Statex \hspace{1.5em} $x_i \sim \rho$, $a_i^{1}\sim\pi_{\theta_t}(\cdot\mid x_i)$, $a_i^{2}\sim\pi_{\theta_t}(\cdot\mid x_i)$
    \State Set reference policy $\pi_{\mathrm{ref}} \leftarrow \pi_{\theta_t}$
    \State Initialize inner iterate $\vartheta^{(0)} \leftarrow \theta_t$
    \For{$k = 0$ to $K-1$}
        \State For each $z_i\in\mathcal{D}_t$, compute the (non-robust) REBEL pointwise loss
        \Statex \hspace{1.5em}
        $\ell_t(z_i;\vartheta^{(k)}) \;=\;
        \Big(
        \frac{1}{\eta}\Big[
        \log\frac{\pi_{\vartheta^{(k)}}(a_i^{1}\mid x_i)}{\pi_{\mathrm{ref}}(a_i^{1}\mid x_i)}
        -
        \log\frac{\pi_{\vartheta^{(k)}}(a_i^{2}\mid x_i)}{\pi_{\mathrm{ref}}(a_i^{2}\mid x_i)}
        \Big]
        -
        (r(x_i,a_i^{1})-r(x_i,a_i^{2}))
        \Big)^2.$
        \State Compute empirical REBEL loss
        \Statex \hspace{1.5em}
        $L_{\mathrm{REBEL},t}(\vartheta^{(k)})=\frac{1}{n_t}\sum_{i=1}^{n_t}\ell_t(z_i;\vartheta^{(k)})$
        \State Compute gradient regularizer (WD approximation)
        \Statex \hspace{1.5em}
        $R_t(\vartheta^{(k)})=\rho_0\Big(\frac{1}{n_t}\sum_{i=1}^{n_t}\|\nabla_{z_i}\ell_t(z_i;\vartheta^{(k)})\|_2^2\Big)^{1/2}$
        \State Form WD-robust surrogate loss
        \Statex \hspace{1.5em}
        $L_{W,t}(\vartheta^{(k)},\rho_0)=L_{\mathrm{REBEL},t}(\vartheta^{(k)})+R_t(\vartheta^{(k)})$
        \State Gradient step:
        $\vartheta^{(k+1)} \leftarrow \vartheta^{(k)} - \tilde{\eta}\,\nabla_\theta L_{W,t}(\vartheta^{(k)},\rho_0)$
    \EndFor
    \State Update policy parameters: $\theta_{t+1}\leftarrow \vartheta^{(K)}$
\EndFor
\State \Return $\pi_{\theta_T}$
\end{algorithmic}
\end{algorithm}

\subsection{Tractable KL-DRO-REBEL (KL-REBEL)}
We utilize the following proposition established by \cite{xu2025distributionallyrobustdirectpreference} to show that
we can approximate the worst-case probability distribution in a KL uncertainty set with respect to a given loss function.

\vspace{1em}

\begin{proposition}[Worst-case distribution]
\label{prop:kl-worst-case}
Let $P \in \mathbb{R}^n$ be the worst-case distribution with respect to a loss function $\ell$ and KL uncertainty around the empirical distribution $P_{n, t}^{\circ}$, defined as $P = \sup_{P : D_{\mathrm{KL}}(P \| P_{n, t}^{\circ}) \le \rho} \mathbb{E}_{z \sim P}[\ell(z; \theta)]$. The worst-case distribution $P$ is related to $P_{n, t}^{\circ}$ through
\[
P(i) \propto P_{n, t}(i) \cdot \exp\left(\frac{1}{\tau} (\ell(z_i; \theta) - \sum_{j=1}^n P_{n,t}(j)\ell(z_j; \theta))\right),
\]
where $\tau > 0$ is a constant.
\end{proposition}
A proof of this proposition can be found in Appendix D of \cite{xu2025distributionallyrobustdirectpreference}. Based on Proposition \ref{prop:kl-worst-case}, we propose a tractable KL-REBEL algorithm in Algorithm \ref{alg:kl-rebel}.

\begin{algorithm}[H]
\caption{Online KL-REBEL (KL-DRO REBEL)}
\label{alg:kl-rebel}
\begin{algorithmic}[1]
\State \textbf{Input:} Reward function $r$, policy class $\Pi=\{\pi_\theta:\theta\in\Theta\}$, context distribution $\rho$, robustness temperature $\tau$, step size $\tilde{\eta}$, REBEL scaling $\eta$, inner steps $K$
\State Initialize policy parameters $\theta_0$
\For{$t = 0$ to $T-1$}
    \State Collect batch $\mathcal{D}_t=\{z_i\}_{i=1}^{n_t}$ where $z_i=(x_i,a_i^{1},a_i^{2})$,
    \Statex \hspace{1.5em} $x_i \sim \rho$, $a_i^{1}\sim\pi_{\theta_t}(\cdot\mid x_i)$, $a_i^{2}\sim\pi_{\theta_t}(\cdot\mid x_i)$
    \State Set reference policy $\pi_{\mathrm{ref}} \leftarrow \pi_{\theta_t}$
    \State Initialize inner iterate $\vartheta^{(0)} \leftarrow \theta_t$
    \For{$k = 0$ to $K-1$}
        \State Compute pointwise REBEL losses on $\mathcal{D}_t$:
        \Statex \hspace{1.5em}
        $\ell_t(z_i;\vartheta^{(k)}) \;=\;
        \Big(
        \frac{1}{\eta}\Big[
        \log\frac{\pi_{\vartheta^{(k)}}(a_i^{1}\mid x_i)}{\pi_{\mathrm{ref}}(a_i^{1}\mid x_i)}
        -
        \log\frac{\pi_{\vartheta^{(k)}}(a_i^{2}\mid x_i)}{\pi_{\mathrm{ref}}(a_i^{2}\mid x_i)}
        \Big]
        -
        (r(x_i,a_i^{1})-r(x_i,a_i^{2}))
        \Big)^2.$
        \State Compute empirical mean loss:
        $\bar{\ell}_t(\vartheta^{(k)})=\frac{1}{n_t}\sum_{j=1}^{n_t}\ell_t(z_j;\vartheta^{(k)})$
        \State Compute unnormalized exponential (worst-case) weights:
        \Statex \hspace{1.5em}
        $\tilde{P}_t^{(k)}(i)=\exp\!\Big(\frac{1}{\tau}\big(\ell_t(z_i;\vartheta^{(k)})-\bar{\ell}_t(\vartheta^{(k)})\big)\Big)$
        \State Normalize weights:
        \Statex \hspace{1.5em}
        $P_t^{(k)}(i)=\frac{\tilde{P}_t^{(k)}(i)}{\sum_{m=1}^{n_t}\tilde{P}_t^{(k)}(m)}$
        \State Form KL-robust surrogate loss on the current batch:
        \Statex \hspace{1.5em}
        $L_{\mathrm{KL},t}(\vartheta^{(k)};\tau)=\sum_{i=1}^{n_t}P_t^{(k)}(i)\,\ell_t(z_i;\vartheta^{(k)})$
        \State Gradient step:
        $\vartheta^{(k+1)} \leftarrow \vartheta^{(k)} - \tilde{\eta}\,\nabla_\theta L_{\mathrm{KL},t}(\vartheta^{(k)};\tau)$
    \EndFor
    \State Update policy parameters: $\theta_{t+1}\leftarrow \vartheta^{(K)}$
\EndFor
\State \Return $\pi_{\theta_T}$
\end{algorithmic}
\end{algorithm}

\subsection{Tractable $\chi^2$-DRO-REBEL ($\chi^2$-REBEL)}
We exploit the dual formulation of the $\chi^2$‐DRO objective (e.g.\ \cite{NIPS2017_5a142a55}) to obtain a one‐dimensional inner solve and closed‐form worst‐case weights.

\vspace{1em}

\begin{proposition}[Dual form \& worst‐case weights]
\label{prop:chi2-worst-case}
Let $\ell_i = \ell(z_i;\theta)$ for $i=1,\dots,n$ and set $\mathcal{L}^{\chi^2}_n(\theta;\varepsilon_{n})=\sup_{P: D_{\chi^2}(P \|P_{n, t}^{\circ})\le\varepsilon_{n}}\mathbb{E}_{z\sim P}[\ell(z;\theta)]$.  Then
\[
\mathcal{L}^{\chi^2}_n(\theta;\varepsilon_{n})
=\inf_{\eta\in\mathbb{R}}
\left\{\,
\eta \;+\;\sqrt{\frac{2\varepsilon_{n}}{n}\sum_{i=1}^n(\ell_i-\eta)_+^2}
\right\},
\]
\end{proposition}

We defer the proof to Appendix \ref{appendix:tractable-chi-squared-algo}. Based on Proposition \ref{prop:chi2-worst-case}, we propose a tractable $\chi^2$-REBEL Algorithm in Algorithm \ref{alg:chi2-rebel}.

\begin{algorithm}[H]
\caption{Online $\chi^2$-REBEL ($\chi^2$-DRO REBEL)}
\label{alg:chi2-rebel}
\begin{algorithmic}[1]
\State \textbf{Input:} Reward function $r$, policy class $\Pi=\{\pi_\theta:\theta\in\Theta\}$, context distribution $\rho$, $\chi^2$-radius $\rho_0$, step size $\alpha$, REBEL scaling $\eta$, inner steps $K$
\State Initialize policy parameters $\theta_0$
\For{$t=0$ to $T-1$}
  \State Collect batch $\mathcal{D}_t=\{z_i\}_{i=1}^{n_t}$ where $z_i=(x_i,a_i^{1},a_i^{2})$,
  \Statex \hspace{1.5em} $x_i\sim\rho$, $a_i^{1}\sim\pi_{\theta_t}(\cdot\mid x_i)$, $a_i^{2}\sim\pi_{\theta_t}(\cdot\mid x_i)$
  \State Set reference policy $\pi_{\mathrm{ref}} \leftarrow \pi_{\theta_t}$
  \State Initialize inner iterate $\vartheta^{(0)} \leftarrow \theta_t$
  \For{$k=0$ to $K-1$}
    \State Compute pointwise REBEL losses on $\mathcal{D}_t$:
    \Statex \hspace{1.5em}
    $\ell_t(z_i;\vartheta^{(k)}) \;=\;
    \Big(
    \frac{1}{\eta}\Big[
    \log\frac{\pi_{\vartheta^{(k)}}(a_i^{1}\mid x_i)}{\pi_{\mathrm{ref}}(a_i^{1}\mid x_i)}
    -
    \log\frac{\pi_{\vartheta^{(k)}}(a_i^{2}\mid x_i)}{\pi_{\mathrm{ref}}(a_i^{2}\mid x_i)}
    \Big]
    -
    (r(x_i,a_i^{1})-r(x_i,a_i^{2}))
    \Big)^2.$
    \State \emph{(Inner 1-D solve on current batch)} find
    \[
      \eta_t^{*}(\vartheta^{(k)}) \;=\;\arg\min_{u\in\mathbb{R}}\left\{u + 
      \sqrt{\frac{2\rho_0}{n_t}\sum_{i=1}^{n_t}\big(\ell_t(z_i;\vartheta^{(k)})-u\big)_+^2}\right\}
    \]
    \Statex \hspace{1.5em} via sorting $\{\ell_t(z_i;\vartheta^{(k)})\}_{i=1}^{n_t}$ and binary search (or a 1-D convex solver)
    \State Form $\chi^2$-robust surrogate loss on the current batch:
    \[
      L_{\chi^2,t}(\vartheta^{(k)};\rho_0)
      \;=\;
      \eta_t^{*}(\vartheta^{(k)})
      +
      \sqrt{\frac{2\rho_0}{n_t}\sum_{i=1}^{n_t}\big(\ell_t(z_i;\vartheta^{(k)})-\eta_t^{*}(\vartheta^{(k)})\big)_+^2}
    \]
    \State Gradient step:
    $\vartheta^{(k+1)} \leftarrow \vartheta^{(k)} - \alpha\,\nabla_\theta L_{\chi^2,t}(\vartheta^{(k)};\rho_0)$
  \EndFor
  \State Update policy parameters: $\theta_{t+1}\leftarrow \vartheta^{(K)}$
\EndFor
\State \Return $\pi_{\theta_T}$
\end{algorithmic}
\end{algorithm}
Each outer iteration $t$ requires (i)~evaluating $\{\ell_t(z_i;\vartheta^{(k)})\}_{i=1}^{n_t}$ and their gradients, and (ii)~solving a one-dimensional convex problem for $\eta_t^{\star}$. A single inner update therefore costs $\mathcal{O}(n_t\log n_t + n_t\,c_{\text{grad}})$, where $c_{\text{grad}}$ is the per-sample cost of computing $\ell_t$ and $\nabla_\theta \ell_t$. Under the log-linear model with constant-size action normalization, treating log-probability and gradient evaluation as $\mathcal{O}(d)$ operations gives $c_{\text{grad}}=\Theta(d)$, so the per-update cost reduces to $\mathcal{O}(n_t\log n_t + n_t d)$, and the per-iteration cost with $K$ inner updates is $\mathcal{O}(K(n_t\log n_t + n_t d))$. A detailed derivation appears in Appendix~\ref{appendix:tractable-chi-squared-algo}; the algorithm itself was originally proposed by \cite{namkoong2017variance}.

Since $\mathcal{L}^{\chi^2}_n(\theta;\rho_0)$ is $\mu$-strongly convex over the compact set $\Theta$, it admits a unique minimizer $\theta^{\star}$. Moreover, $\mathcal{L}^{\chi^2}_n$ is convex and Lipschitz (Appendix~\ref{appendix:uniform-and-lipschitz-bound}), so projected first-order methods converge to $\theta^{\star}$ under standard step-size conditions---projected subgradient descent in the nonsmooth case, or projected gradient descent with linear convergence when $\nabla \mathcal{L}^{\chi^2}_n$ is Lipschitz. Analogous guarantees extend to the KL and Wasserstein robust objectives via subgradient-based updates.

\section{Experiments}
We perform a comprehensive experimental evaluation of DRO-REBEL, comparing its performance against non-robust DPO and REBEL baselines, as well as established Distributionally Robust Optimization (DRO) variants, namely WDPO and KLDPO \cite{xu2025distributionallyrobustdirectpreference}. Our empirical evaluations spans two distinct alignment tasks, designed to investigate model performance under varying dataset scales, model sizes, the complexity of the reward function, and degrees of distributional shift: (i) an Emotion Alignment task with controlled, synthetic shifts, (ii) an ArmoRM Multi-objective Alignment \cite{wang2024interpretablepreferencesmultiobjectivereward} task featuring large-scale, real-world shifts, and (iii) an HH-RLHF pairwise preference benchmark \cite{bai2022traininghelpfulharmlessassistant}, where we train Llama-1B and Llama-8B on HelpSteer2 (ArmoRM-scored) and assessed on HH–RLHF only at test time, with held-out Win/Lose rates computed via a DPO-style margin. The Emotion and ArmoRM settings mirror those in \cite{xu2025distributionallyrobustdirectpreference} and indeed we will use these as a baseline to compare the performance of DRO variants of REBEL and DPO in the face of distribution shift. The HH-RLHF study complements these experiments with a widely used dataset in alignment literature to see if DRO variants of REBEL outperform DRO variants of DPO in terms of win-rates. Code and full hyperparameters can be found at \url{https://github.com/sharansahu/distributionally_robust_rebel}.

\subsection{Experimental Setup}

\subsubsection{Emotion Alignment}
For the Emotion Alignment task, our experimental setup is designed to precisely control and simulate distributional shifts in user preferences. We begin by training a reward model based on a \code{GPT-2} architecture \cite{Radford2019LanguageMA} augmented with a classification head. This model is fine-tuned on the Emotion dataset \cite{saravia-etal-2018-carer} to perform multi-label classification across five distinct emotions: sadness, joy, love, anger, and fear. The sigmoid outputs from this classification head serve as our multi-objective reward signals throughout the experiment. In parallel, a base policy model, also a \code{GPT-2} instance, undergoes supervised fine-tuning (SFT) on the same Emotion dataset, establishing our initial policy for subsequent preference alignment.

To generate preference data and systematically introduce distributional shifts, we define two distinct reward mixing functions using two chosen reward objectives, $r_1$ (anger) and $r_2$ (fear):
\begin{enumerate}
    \item \textbf{Convex Mixing:} $r^*_{\text{convex}}(\alpha) := \alpha \cdot r_1 + (1 - \alpha) \cdot r_2$,
    \item \textbf{Geometric Mixing:} $r^*_{\text{geometric}}(\alpha) := r_1^{\alpha} \cdot r_2^{1-\alpha}$.
\end{enumerate}
For training all alignment methods, preference pairs are constructed by generating two completions per prompt. Preference labels are then assigned using the Bradley-Terry (BT) model, parameterized by the mixed reward function $r^*(\alpha_0)$ at a fixed nominal mixing coefficient $\alpha_0 = 0.1$. This $\alpha_0$ represents the training-time preference distribution. To evaluate robustness, we introduce a controlled distributional shift by sweeping the mixing coefficient $\alpha$ across the entire range $[0, 1]$ at test time. Performance is measured by the average mixture reward $r^*(\alpha)$ obtained over 64 held-out prompts. More comprehensive details about these experiments are given in Appendix \ref{appendix:experiment-training-details}.

\subsubsection{ArmoRM Multi-objective Alignment}
To assess performance in a more complex, real-world setting, we conduct experiments on a multi-objective alignment task utilizing the Absolute-Rating Multi-Objective Reward Model (ArmoRM) \cite{wang2024interpretablepreferencesmultiobjectivereward}. ArmoRM provides 19 distinct first-stage objective outputs, allowing for a rich and diverse set of preferences. Our base policy for this task is Meta LLaMA-3.2-1B-Instruct.

For training, we generate two completions per prompt from the HelpSteer2 dataset \cite{wang2024helpsteer2opensourcedatasettraining}. Reward preferences are derived by selecting pairs of equally weighted objectives (e.g., honesty, verbosity, safety) from ArmoRM's 19-dimensional output space and combining them via convex mixing. Models are then trained on these preferences, again at a nominal mixing coefficient of $\alpha_0 = 0.1$. We measure the performance of all aligned policies on five individual ArmoRM objectives. Crucially, three of these five objectives are \emph{unseen} during the training process, specifically designed to simulate real-world scenarios where models encounter new or unweighted preferences. The evaluation is conducted over 128 test prompts. All fine-tuning runs across both the Emotion Alignment and ArmoRM tasks adhere to identical hyperparameters, as comprehensively detailed in Appendix \ref{appendix:experiment-training-details}.

\subsubsection{HH-RLHF Alignment}
We use \textsc{HH\texorpdfstring{-}{-}RLHF}~\cite{bai2022traininghelpfulharmlessassistant} for evaluation in order to test domain adaptivity. Our base policies are Meta LLaMA-3.2-1B-Instruct and LLaMA-3-8B-Instruct.  Rather than fitting on HH pairs, we align the policies on a different source domain: HelpSteer2 prompts scored by the second-stage ArmoRM reward model. 
For each HelpSteer2 prompt, we sample \(k=10\) on-policy candidate completions using nucleus decoding (temperature \(0.7\), top-\(p=1.0\)). We then deduplicate candidates, filter out responses longer than \(1024\) tokens, score the remaining completions with ArmoRM, and construct a single preference pair by selecting the highest- and lowest-scoring completions \((a^+, a^-)\). All variants optimize response-token likelihood conditioned on the prompt under identical context length and training budget. 

We intentionally do not fine-tune on HH-RLHF so it serves as a held-out target domain, allowing us to measure true domain adaptivity under preference shift without any target leakage. Improvements on HH-RLHF then reflect robustness of the optimization method rather than memorization or dataset-specific tuning. We evaluate zero-shot on a held-out split of \textsc{HH\texorpdfstring{-}{-}RLHF} using the standard pairwise criterion: \emph{Win} is the fraction of pairs where the policy increases the chosen–rejected log-probability margin relative to the reference
and \emph{Lose} is the fraction with a negative margin (ties ignored). More comprehensive details about these experiments are given in   Appendix~\ref{appendix:experiment-training-details}.

\subsection{Results}

\subsubsection{Emotion Alignment}
\begin{figure}[htbp!]
    \centering
    \includegraphics[width=0.9\linewidth]{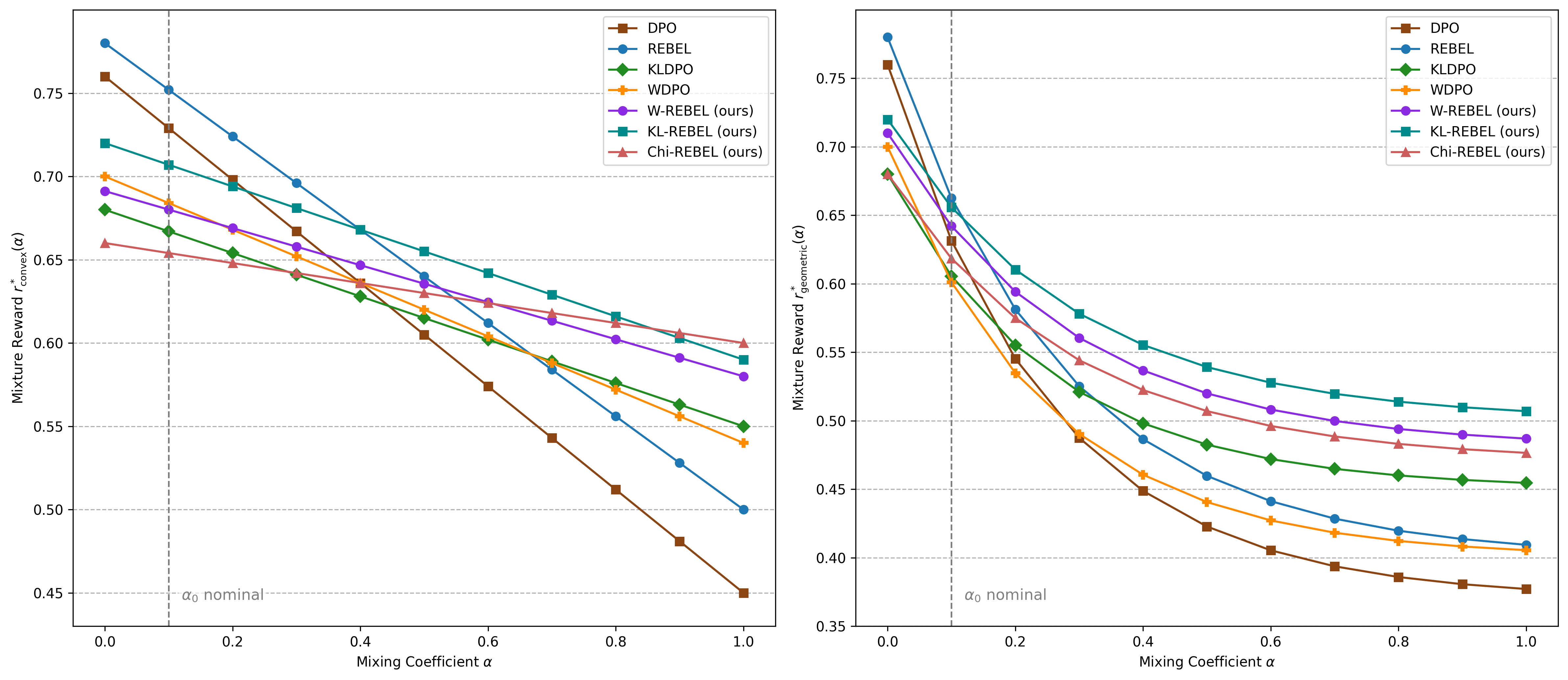}
    \caption{Emotion alignment performance under convex (left) and geometric (right) reward mixing. Models are trained at a nominal mixing coefficient $\alpha_0=0.1$, and evaluated across a range of $\alpha \in [0,1]$ to simulate preference shift. This figure compares the mixture reward for DPO, REBEL, and various DRO variants (KLDPO, WDPO, KL-REBEL, W-REBEL, and $\chi^2$-REBEL).}
    \label{fig:emotion-alignment}
\end{figure}
The results for Emotion Alignment, presented in Figure \ref{fig:emotion-alignment}, clearly illustrate the robustness of various methods under varying degrees of preference shift for both convex and geometric reward mixing. Non-robust baselines, DPO and REBEL, achieve high rewards at the training-time nominal $\alpha_0=0.1$. However, their performance degrades significantly linearly as $\alpha$ deviates from $\alpha_0$. The same can be seen for geometric mixing, and this implies there is a susceptibility to overoptimization on the training distribution. Prior DRO variants applied to DPO, namely WDPO and KLDPO, offer modest improvements in robustness compared to plain DPO. Our robust REBEL variants, W-REBEL and KL-REBEL, demonstrate even more substantial gains in mixture reward beyond the DPO and REBEL baselines. Among all the methods tested, $\chi^2$ -REBEL emerges as the most stable, consistently retaining a high reward across the entire range of $\alpha$.

\subsubsection{ArmoRM Multi-objective Alignment}
\begin{figure}[h!]
    \centering
    \includegraphics[width=.8\linewidth]{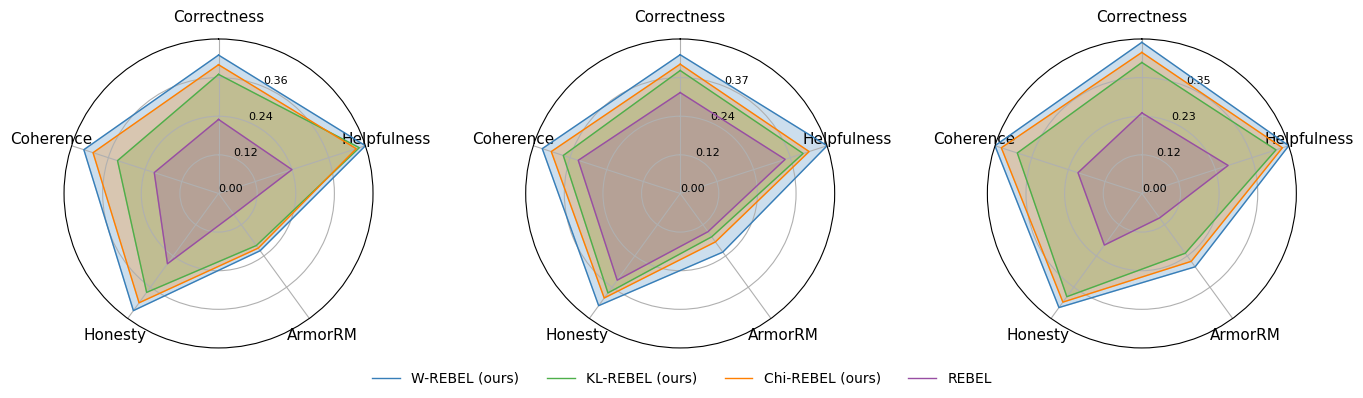}
    \includegraphics[width=.8\linewidth]{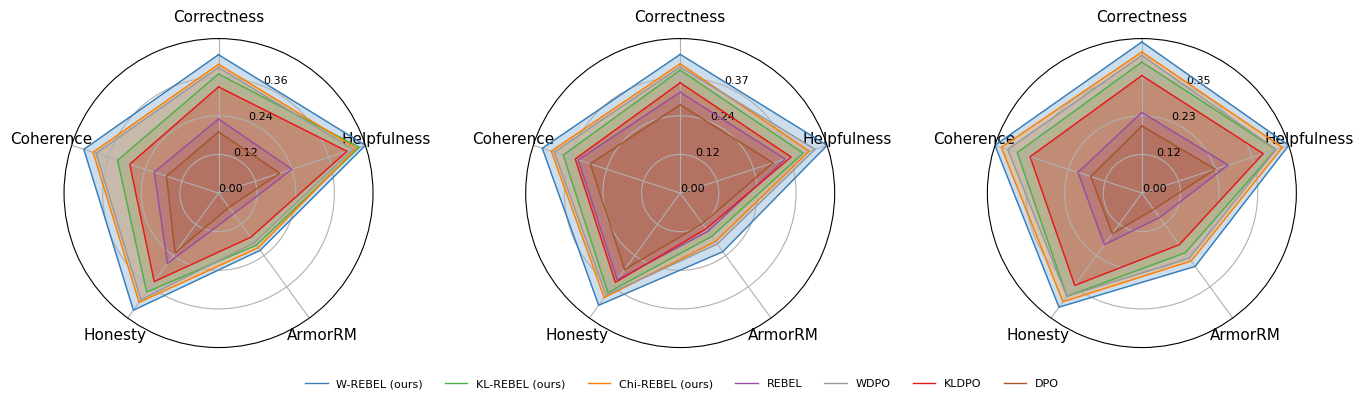}
    \caption{Performance on five ArmoRM objectives (Correctness, Helpfulness, Honesty, ArmoRM, Coherence), including three (Honesty, ArmoRM, Coherence) that were unseen during training. The radar charts illustrate the alignment across different objectives, where larger shaded areas indicate stronger robustness and better overall performance under preference shift across diverse objectives.}
    \label{fig:armorm-alignment}
\end{figure}
Figure \ref{fig:armorm-alignment} illustrates the multi-objective alignment performance on the ArmoRM task. The size of the shaded area directly correlates with overall alignment and robustness. DPO and REBEL baselines exhibit noticeable performance drops on objectives unseen during training, demonstrating a clear tendency towards overoptimization on the specific objective combinations present in the training data, leading to poor generalization. WDPO and KLDPO, as prior DRO variants of DPO, show the capacity to recover performance on these unseen objectives. Our REBEL-based robust variants, W-REBEL, KL-REBEL, and $\chi^2$-REBEL, achieve further significant gains on the unseen objectives compared to their DPO counterparts. These results also support and are fully consistent with the theoretical and empirical insights of \cite{huang2025correcting} as they show that $\chi^{2}$ DRO in RLHF can achieve broad generalization and maintain robustness across diverse objectives, including those not explicitly encountered or optimized during training.

\subsubsection{HH-RLHF Alignment}
From the results in Table \ref{tab:hh-rlhf}, we see that the non-robust baselines (DPO and REBEL) perform reasonably. Introducing distributional robustness consistently improves performance as we observe all robust methods outperform their non-robust counterparts, confirming the value of robustness in preference learning. Prior DRO variants applied to DPO (WDPO and KL-DPO) yield modest gains over vanilla DPO. However, robust \textsc{REBEL} variants provide substantially larger improvements. These results illustrate the broader story from our preference-shift experiments: inducing robustness with alignment algorithms provides the strongest protection against over-optimization to the training pair distribution while preserving nominal performance.

\begin{table}[!h]
\centering
\setlength{\tabcolsep}{10pt}
\begin{tabular}{lcccc}
\toprule
\multirow{2}{*}{\textbf{Baselines}} &
\multicolumn{2}{c}{\textbf{Llama-1B}} &
\multicolumn{2}{c}{\textbf{Llama-8B}} \\
\cmidrule(lr){2-3}\cmidrule(lr){4-5}
& \textbf{Win(↑)} & \textbf{Lose(↓)} & \textbf{Win(↑)} & \textbf{Lose(↓)} \\
\midrule
DPO \fakecite{RSM\textsuperscript{+}24}    & 72.1 \pmgreen{1.62} & 25.6 \pmgreen{1.49} & 72.9 \pmgreen{1.40} & 23.7 \pmgreen{1.65} \\
WDPO \fakecite{XVP\textsuperscript{+}25}     & 73.5 \pmgreen{1.58} & 24.1 \pmgreen{1.33} & 74.2 \pmgreen{1.53} & 22.9 \pmgreen{1.28} \\
KL\mbox{-}DPO \fakecite{XVP\textsuperscript{+}25} & 74.1 \pmgreen{1.54} & 23.7 \pmgreen{1.29} & 75.4 \pmgreen{1.41} & 21.9 \pmgreen{1.21} \\
REBEL \fakecite{GCZ\textsuperscript{+}24}           & 73.8 \pmgreen{1.52} & 24.0 \pmgreen{1.27} & 75.0 \pmgreen{1.37} & 22.2 \pmgreen{1.22} \\
W\mbox{-}REBEL (Ours)               & 76.1 \pmgreen{1.36} & 21.9 \pmgreen{1.18} & 76.9 \pmgreen{1.18} & 20.9 \pmgreen{1.08} \\
KL\mbox{-}REBEL (Ours)              & 76.4 \pmgreen{1.31} & 21.6 \pmgreen{1.14} & 77.3 \pmgreen{1.14} & 20.5 \pmgreen{1.03} \\
\rowcolor{black!6}\textbf{CHI\mbox{-}REBEL} (Ours) &
\textbf{77.3} \pmgreen{1.24} & \textbf{20.8} \pmgreen{1.09} &
\textbf{78.0} \pmgreen{1.08} & \textbf{19.7} \pmgreen{1.05} \\
\bottomrule
\end{tabular}
\vspace{6pt}
\caption{Comparison with pairwise robust baselines on HH-RLHF \cite{bai2022traininghelpfulharmlessassistant}. Win(↑) is win-rate on held-out pairs and Lose(↓) is negative-margin rate. All REBEL variants outperform all DPO variants in terms of win and lose rates.}
\label{tab:hh-rlhf}
\end{table}

\section{Conclusion} We introduced \emph{DRO-REBEL}, a family of distributionally robust REBEL updates for online RLHF built on type-$p$ Wasserstein, KL, and $\chi^2$ ambiguity sets. Strong duality reduces each robust update to a scalable relative-reward regression, preserving REBEL's simplicity without PPO-style clipping or value networks. On the theoretical side, we first recover the $\widetilde{\mathcal{O}}(\sqrt{d/n})$ slow rate of prior DRO-DPO analyses with substantially sharper constants, and then, via a localized Rademacher complexity argument, establish the first parametric $\widetilde{\mathcal{O}}(d/n)$ fast rate for robust preference optimization under Wasserstein and KL ambiguity. Distributional robustness in REBEL therefore carries no statistical efficiency cost relative to the non-robust setting. Empirically, DRO-REBEL maintains alignment across preference shifts and generalizes to unseen objectives on Emotion Alignment, ArmoRM, and HH-RLHF, with $\chi^2$-REBEL the most consistent performer, suggesting a favorable bias–variance trade-off under moderate shift.

There are several interesting direction for future work. In this work, we consider the KL divergence as the primary $f$-divergence ambiguity set of interest and derived $\widetilde{\mathcal{O}}(d/n)$ fast rates. Given the strong empirical success of $\chi^2$-REBEL, it would be interesting to extend the fast rate analysis to $\chi^2$ and, more broadly, to general $f$-divergence ambiguity sets using a similar analysis. We also considered REBEL updates in this work and we conjecture that a similar analysis can be done DPO-based objectives, provided one can verify strong convexity and the covering conditions under the DPO loss. Beyond these, our guarantees rely on the trajectory coverage condition (Assumption~\ref{assumption:data-coverage}), which provides strong convexity of the robust population loss. Relaxing this condition or replacing it with a partial-coverage or a pessimism-based alternative while retaining fast rates is an interesting and important direction.

\bibliographystyle{alpha}
\bibliography{refs}

\clearpage
\appendix
\crefalias{section}{appendix}
\crefalias{subsection}{appendix}

\part*{Appendix} 
\addcontentsline{toc}{part}{Appendix} 

\begingroup
\etocsettocstyle{%
  \section*{Table of Contents}%
  \vspace{-0.25em}%
  \noindent\rule{\linewidth}{0.4pt}\par
  \vspace{0.75em}%
}{}
\localtableofcontents
\endgroup
\clearpage

\section{Auxiliary Technical Tools}

\subsection{Wasserstein Theory}
\begin{lemma}[\cite{gao2022distributionallyrobuststochasticoptimization}, Theorem 1; Strong Duality for DRO with Wasserstein Distance]
\label{appendix:gao-strong-duality-wasserstein}
Consider any $p\in[1,\infty)$, any $\nu\in\mathcal P(\Xi)$, any $\rho>0$, and any $\Psi\in L^1(\nu)$ such that the growth rate $\kappa$ of $\Psi$ satisfies
\begin{equation}\label{eq:kappa}\tag{13}
\kappa :=
\inf\bigl\{\eta\ge 0 :
\int_{\Xi}\Phi(\eta,\zeta)\,\nu(d\zeta)>-\infty
\bigr\}<\infty,
\end{equation}
where
\[
\Phi(\eta,\zeta)
:=
\inf_{\xi\in\Xi}\bigl\{\eta\,d^p(\xi,\zeta)-\Psi(\xi)\bigr\}.
\]
Then strong duality holds with finite optimal value $v_p=v_D\le\infty$, where the primal and dual problems are
\begin{align}
v_p &= \sup_{\mu\in\mathcal P(\Xi)}
\Bigl\{\!\int_{\Xi}\Psi(\xi)\,\mu(d\xi)\colon
W_p(\mu,\nu)\le\rho\Bigr\},\quad\text{(Primal)}\\
v_D &= \inf_{\eta\ge0}
\Bigl\{\eta\,\rho^p \;-\;\int_{\Xi}
\inf_{\xi\in\Xi}\bigl[\eta\,d^p(\xi,\zeta)-\Psi(\xi)\bigr]\,
\nu(d\zeta)\Bigr\}.\quad\text{(Dual)}
\end{align}
\end{lemma}

\begin{lemma}[\cite{gao2022distributionallyrobuststochasticoptimization}, Lemma 2(ii); Properties of the growth $\kappa$]
\label{appendix:gao-strong-duality-wasserstein-growth}
Suppose that $\nu\in\mathcal P_p(\Xi)$. Then the growth rate $\kappa$ in \eqref{eq:kappa} is finite if and only if there exists $\zeta_0\in\Xi$ and constants $L,M>0$ such that
\begin{equation}\label{eq:growth}\tag{14}
\Psi(\xi)-\Psi(\zeta_0)\;\le\;L\,d^p(\xi,\zeta_0)+M,
\quad\forall\,\xi\in\Xi.
\end{equation}
\end{lemma}

\begin{corollary}
Consider any bounded loss function $\ell$ over a bounded space $\Xi$.  Then the duality in Lemma 2 holds.
\end{corollary}
\begin{proof}
Immediate from Lemma \ref{appendix:gao-strong-duality-wasserstein-growth} by choosing $L=\mathrm{diam}(\Xi)^p$ and $M=\sup_{\xi\in\Xi}|\Psi(\xi)|$.
\end{proof}

\subsection{Optimization}

\begin{lemma}[\cite{beck2023introduction}, Theorem 1.24; Linear Approximation Theorem]
\label{appendix:beck-linear-approx}
Let $f\colon U\to\mathbb R$ be twice continuously differentiable on an open set $U\subseteq\mathbb R^n$, and let $x,y\in U$ satisfy $[x,y]\subset U$.  Then there exists $\xi\in[x,y]$ such that
\[
f(y) \;=\; f(x) \;+\;\nabla f(x)^\top(y-x)
\;+\;\frac12\,(y-x)^\top\nabla^2f(\xi)\,(y-x).
\]
\end{lemma}

\begin{lemma}[\cite{beck2017first}, Theorem 5.24; First-order characterizations of strong convexity]
\label{appendix:beck-strong-convexity}
Let $f\colon E\to(-\infty,\infty]$ be a proper, closed, convex function, and let $\sigma>0$.  The following are equivalent:
\begin{enumerate}
  \item For all $x,y\in\mathrm{dom}(f)$ and $\lambda\in[0,1]$,
  \[
    f\bigl(\lambda x+(1-\lambda)y\bigr)
    \le
    \lambda f(x)+(1-\lambda)f(y)
    -\frac\sigma2\,\lambda(1-\lambda)\,\|x-y\|^2.
  \]
  \item For all $x\in\mathrm{dom}(\partial f)$, $y\in\mathrm{dom}(f)$ and $g\in\partial f(x)$,
  \[
    f(y)\;\ge\;f(x)\;+\;\langle g,y-x\rangle\;+\;\frac\sigma2\,\|y-x\|^2.
  \]
\end{enumerate}
\end{lemma}

\begin{lemma}[\cite{beck2017first}, Theorem 5.25; Existence and uniqueness of minimizer]
\label{appendix:beck-strong-convexity-unique-min}
Let $f\colon E\to(-\infty,\infty]$ be proper, closed, and $\sigma$-strongly convex with $\sigma>0$.  Then:
\begin{enumerate}
  \item $f$ has a unique minimizer $x^*$.
  \item For all $x\in\mathrm{dom}(f)$,
  \[
    f(x)-f(x^*)\;\ge\;\frac\sigma2\,\|x-x^*\|^2.
  \]
\end{enumerate}
\end{lemma}

\subsection{Distributionally Robust Optimization}

The $f$-divergence between the distributions $P$ and $P_0$ in $\mathcal X$ is
\begin{equation}\label{eq:fdiv}\tag{15}
D_f(P\Vert P_0)
\;=\;
\int_{\mathcal X}
f\!\left(\frac{dP}{dP_0}\right)\,dP_0,
\end{equation}
where $f$ is a convex function (e.g., $f(t)=t\log t$ gives KL divergence).  For a loss $\ell\colon\mathcal X\to\mathbb R$ the following holds.

\vspace{1em}

\begin{lemma}[\cite{duchi2020learningmodelsuniformperformance}, Proposition 1]
\label{appendix:duchi-strong-duality-KL}
Let $D_f$ be as in \eqref{eq:fdiv}.  Then
\begin{equation}\label{eq:dual-fdiv}\tag{16}
\sup_{P\colon D_f(P\Vert P_0)\le\rho}
\mathbb E_P[\ell(X)]
\;=\;
\inf_{\substack{\lambda\ge0\\\eta\in\mathbb R}}
\Bigl\{
\lambda\,f^*\!\left(\frac{\ell(X)-\eta}{\lambda}\right)
\;+\;\lambda\,\rho\;+\;\eta
\Bigr\},
\end{equation}
where $f^*(s)=\sup_{t\ge0}\{st-f(t)\}$ is the Fenchel conjugate of $f$.
\end{lemma}

\vspace{1em}

\begin{lemma}[\cite{vanhandelprob}, Lemma 4.10; Gibbs Variational Principle]
\label{appendix:gibbs-variational}
Let $\mu, \nu \in \mathcal{P} \left( \Xi \right)$ be Borel probability measures supported on $\Xi$. Then
\[
    \log \mathbb{E}_{\mu} \left[ e^{f}\right] = \sup_{\nu} \left\{  \mathbb{E}_{\nu} \left[ f \right] - D_{\mathrm{KL}}\left(\mu \mid \mid \nu \right)\right\}.
\]
\end{lemma}

\subsection{Empirical Process Theory}
\begin{lemma}[\cite{van1996weak}, Lemma 2.3.1; Symmetrization]
\label{appendix:symmetrization}
For every nondecreasing, convex $\Phi:\mathbb{R}\to\mathbb{R}$ and class of measurable functions $\mathcal{F}$,
\[
  \mathbb{E}^*\!\Bigl[\Phi\bigl(\|P_n - P\|_{\mathcal{F}}\bigr)\Bigr]
  \;\le\;
  \mathbb{E}^*\!\Bigl[\Phi\bigl(2\,\|P_n^{\circ}\|_{\mathcal{F}}\bigr)\Bigr],
\]
where the outer expectations $\mathbb{E}^*$ are taken over the data generating distribution and Rademacher random variables and $P_n^{\circ}$ is the symmetrized process.
\end{lemma}

\vspace{1em}

\begin{corollary}[\cite{boucheron2013concentration}, Corollary 13.2; Dudley's Entropy Integral]
\label{cor:dudley}
Let $\mathcal T$ be a finite pseudometric space and let $(X_t)_{t\in\mathcal T}$ be a collection of random variables such that, for all $t,t'\in\mathcal T$ and all $\lambda>0$,
\[
  \log \mathbb{E}\bigl[e^{\lambda\,(X_t - X_{t'})}\bigr]
  \;\le\;
  \frac{\lambda^2\,d^2(t,t')}{2}\,.
\]
Then for any fixed $t_0\in\mathcal T$, if we set $\delta \;=\;\sup_{t\in\mathcal T} d(t,t_0)$, it holds that
\[
  \mathbb{E}\Bigl[\sup_{t\in\mathcal T}\bigl(X_t - X_{t_0}\bigr)\Bigr]
  \;\le\;
  12 \int_{0}^{\delta/2} \sqrt{\log \mathcal{N}_{\epsilon}(\mathcal{T}; d\bigr)}\;\mathrm{d}u\,,
\]
where $\mathcal{N}_{\epsilon}(\mathcal{T}; d\bigr)$ is the covering number for $\mathcal{T}$ at scale $\epsilon$ under metric $d$. Consequently for a b-uniformly bounded class of functions $\mathcal{F}$,
\[
    \mathbb{E}_{\varepsilon} \left[ \sup_{f \in \mathcal{F}} \abs{\frac{1}{n} \sum_{i=1}^{n} \varepsilon_{i}f(x_{i})} \right] \leq \frac{12}{\sqrt{n}} \int_{0}^{2b} \sqrt{\log \mathcal{N}_{\epsilon}(\mathcal{T}; d\bigr)} du,
\]
where $\varepsilon_i$ are i.i.d Rademacher random variables.
\end{corollary}

\vspace{1em}


\begin{theorem}[\cite{shalev2014understanding}; Theorem 26.5]
\label{theorem:uniform-rademacher-deviation}
Let \( \mathcal{F} \) be a class of real-valued functions \( f : \mathcal{X} \to \mathbb{R} \), and let \( X_1, \ldots, X_n \) be independent samples from a distribution \( P \). Define the empirical measure
\[
P_n f \;:=\; \frac{1}{n} \sum_{i=1}^n f(X_i).
\]
Let $\varepsilon_{i} \stackrel{\mathrm{i.i.d}}{\sim} \mathrm{Unif}\left( \left\{ -1, 1 \right\} \right)$ be independent Rademacher random variables. The empirical Rademacher complexity of \( \mathcal{F} \) is defined by
\[
\mathfrak{R}_n(\mathcal{F})
\;:=\;
\mathbb{E}_{\varepsilon}\!\left[
\sup_{f \in \mathcal{F}}
\frac{1}{n}\sum_{i=1}^n \varepsilon_i f(X_i)
\right],
\]
where the expectation is taken conditional on \( X_1,\dots,X_n \). Then with probability at least \( 1-\delta \), the following holds uniformly for all \( f \in \mathcal{F} \):
\[
\sup_{f \in \mathcal{F}} \bigl( P f - P_n f \bigr)
\;\le\;
2 \mathfrak{R}_n(\mathcal{F})
\;+\;
4 \sup_{f \in \mathcal{F}} \|f\|_{\infty}
\sqrt{\frac{2\ln(2/\delta)}{n}}.
\]
\end{theorem}

\subsection{Concentration Inequalities}

\begin{lemma}[\cite{boucheron2013concentration}, Lemma 2.2; Hoeffding's Lemma]
\label{appendix:hoeffding-lemma}
Let $Y$ be a random variable with $\mathbb{E}[Y]=0$ and almost surely $Y\in[a,b]$. Define $\psi_Y(\lambda) = \log \mathbb{E}\bigl[e^{\lambda Y}\bigr]$. Then for all $\lambda\in\mathbb{R}$, $\psi_Y''(\lambda)\le\frac{(b - a)^2}{4}$ and consequently $Y$ is sub-Gaussian with proxy variance $(b-a)^2/4$, i.e. $Y \sim \mathcal{SG} \; \!\left(\frac{b-a}{2}\right)$.
\end{lemma}

Using Hoeffding's lemma, one can prove Hoeffding's inequality using a standard Chernoff bound argument.

\vspace{1em}

\begin{lemma}[Hoeffding’s inequality]
\label{appendix:hoeffding}
Let $X_1,\dots,X_n$ be independent with $X_i\in[a_i,b_i]$ almost surely, and define
\[
S=\sum_{i=1}^n (X_i - \mathbb E[X_i]).
\]
Then for every $t>0$,
\[
\Pr \left(S \geq t \right) \leq \exp \left(- \frac{2t^2}{ \sum_{i=1}^n(b_i-a_i)^2} \right).
\]
In particular, if $X_1,\dots,X_n$ are i.i.d.\ with mean $\mu$ and support $[a,b]$, then for all $t>0$
\[
\Pr \left(S \geq t \right) \leq \exp \left(- \frac{2t^2}{ \sum_{i=1}^n(b_i-a_i)^2} \right).
\left(\abs{\frac1n\sum_{i=1}^n X_i \;-\;\mathbb{E}\left[X \right]} \right)\ge t\bigr)
\;\le\;
2\exp\left(-\,\frac{2nt^2}{(b-a)^2}\right).
\]
\end{lemma}

\vspace{1em}

\begin{theorem}[\cite{Tropp2015AnIT}, Theorem 5.1.1; Matrix Chernoff]
    \label{appendix:matrix-chernoff}
    Consider a finite sequence $\{X_k\}$ of independent, random, Hermitian matrices with common dimension $d$. Assume that
\[
0 \le \lambda_{\min}(X_k)
\quad\text{and}\quad
\lambda_{\max}(X_k) \le L
\;\;\text{for each index }k.
\]
Introduce the random matrix
\[
Y := \sum_k X_k .
\]
Define the minimum eigenvalue $\mu_{\min}$ and maximum eigenvalue $\mu_{\max}$ of the expectation $\mathbb{E}Y$ by
\begin{align}
\mu_{\min}
&:= \lambda_{\min}(\mathbb{E}Y)
= \lambda_{\min}\!\Bigl(\sum_k \mathbb{E}X_k\Bigr),
\label{eq:chernoff-mumin}\\
\mu_{\max}
&:= \lambda_{\max}(\mathbb{E}Y)
= \lambda_{\max}\!\Bigl(\sum_k \mathbb{E}X_k\Bigr).
\label{eq:chernoff-mumax}
\end{align}
Then, for $\theta>0$,
\begin{align}
\mathbb{E}\,\lambda_{\min}(Y)
&\ge \frac{1-e^{-\theta}}{\theta}\,\mu_{\min}
\;-\;\frac{1}{\theta}\,L\log d,
\label{eq:chernoff-exp-min}\\
\mathbb{E}\,\lambda_{\max}(Y)
&\le \frac{e^{\theta}-1}{\theta}\,\mu_{\max}
\;+\;\frac{1}{\theta}\,L\log d.
\label{eq:chernoff-exp-max}
\end{align}
Furthermore,
\begin{align}
\Pr\!\left(\lambda_{\min}(Y)\le (1-\varepsilon)\mu_{\min}\right)
&\le
d\left[\frac{e^{-\varepsilon}}{(1-\varepsilon)^{\,1-\varepsilon}}\right]^{\mu_{\min}/L},
\;\; \varepsilon\in[0,1),
\label{eq:chernoff-tail-min}\\
\Pr\!\left(\lambda_{\max}(Y)\ge (1+\varepsilon)\mu_{\max}\right)
&\le
d\left[\frac{e^{\varepsilon}}{(1+\varepsilon)^{\,1+\varepsilon}}\right]^{\mu_{\max}/L},
\;\; \varepsilon\ge 0.
\label{eq:chernoff-tail-max}
\end{align}
\end{theorem}

\subsection{Real and Convex Analysis}

\begin{theorem}[\cite{rudin1953principles}, Theorem 4.16; Weierstrass Extreme Value Theorem]
    \label{appendix:extreme-value-theorem}
    Suppose $f$ is continuous real function on a compact metric space $X$, and 
    \[
        M = \sup_{p \in X} f(p), \quad m = \inf_{p \in X} f(p).
    \]
    Then there exists points $p, q \in X$ such that $f(p) = M$ and $f(q) = m$.
\end{theorem}

\vspace{1em}

\begin{theorem}[\cite{aliprantis2007infinite}, Theorem 17.31; Berge Maximum Theorem]
    \label{appendix:berge-maximum-theorem}
    Let $\varphi : X \rightrightarrows Y$ be a continuous correspondence between topological spaces with nonempty compact values, and suppose $f : \operatorname{Gr}(\varphi) \to \mathbb{R}$ is continuous. Define the value function $m : X \to \mathbb{R}$ by
\[
m(x) = \max_{y \in \varphi(x)} f(x,y),
\]
and the correspondence of maximizers $\mu : X \rightrightarrows Y$ by
\[
\mu(x) = \{\, y \in \varphi(x) : f(x,y) = m(x) \,\}.
\]
Then:
\begin{enumerate}
\item The value function $m$ is continuous.
\item The correspondence $\mu$ has nonempty compact values.
\item If either $f$ has a continuous extension to all of $X \times Y$ or $Y$ is Hausdorff, then $\mu$ is upper hemicontinuous.
\end{enumerate}
\end{theorem}

\vspace{1em}

\begin{theorem}[\cite{aliprantis2007infinite}, Theorem 18.13; Kuratowski–Ryll-Nardzewski Selection Theorem]
    \label{appendix:krn-selection-theorem}
    Let $X$ be a Polish space, let ${\mathcal B}(X)$ denote the Borel $\sigma$-algebra on $X$, let $(\Omega,{\mathcal F})$ be a measurable space, and let $\psi$ be a multifunction on $\Omega$ taking values in the collection of nonempty closed subsets of $X$. Suppose that $\psi$ is ${\mathcal F}$-weakly measurable, that is, for every open subset $U \subseteq X$,
\[
\{\omega \in \Omega : \psi(\omega)\cap U \neq \emptyset\} \in {\mathcal F}.
\]
Then $\psi$ admits a selection $s:\Omega \to X$ such that $s(\omega)\in \psi(\omega)$ for all $\omega\in\Omega$, and $s$ is ${\mathcal F}$--${\mathcal B}(X)$ measurable.
\end{theorem}

\vspace{1em}

\begin{theorem}[\cite{rockafellar2009variational}, Exercise 8.31; Subgradients is the convex hull of active gradients]
    \label{appendix:subgradient-convex-hull}
    Suppose
\[
f(x) = \sup_{\alpha \in \mathcal{A}} f_\alpha(x)
\]
for a finite collection of smooth upper semi-continuous functions $\alpha \mapsto f_{\alpha}(x)$ for each $x$. Then
\[
\partial f(\bar x) = \operatorname{co}\bigl\{\, \nabla f_\alpha(\bar x) \;:\; f_{\alpha}(x) = f(x) \,\bigr\},
\;\;
\partial^\infty f(\bar x) = \{0\},
\]
and the directional derivative satisfies, for every $w\in\mathbb{R}^n$,
\[
df(\bar x)(w)
= \sup_{\alpha \in \mathcal{A}} \bigl\langle \nabla f_\alpha(\bar x), w \bigr\rangle
= \sup_{\alpha \in \mathcal{A}} df_\alpha(\bar x)(w).
\]

\end{theorem}

\section{Proofs of Uniform Boundedness and Lipschitzness of $\ell_{t}(z; \theta)$}

\subsection{Uniform Boundedness of $\ell_{t}(z;\theta)$}
\label{appendix:uniform-and-lipschitz-bound}
We first prove that $\ell_{t}(z; \theta)$ is uniformly bounded.
\vspace{1em}
\begin{lemma}[Uniform bound on $\ell_{t}(z;\theta)$]
\label{appendix:uniform-bound-l}
Let $K_{g} = \sup_{z, \theta} |g_{t}(z; \theta)| \leq 8B / \eta + 2F$ where $\ell_{t}(z;\theta) = g_{t}(z; \theta)^{2}$ with $z = (x, a^{1}, a^{2}) \sim P^{\circ}$. Then $\sup_{z, \theta}|\ell_{t}(z;\theta)| = K_{\ell} = K_{g}^{2}$.
\end{lemma}

\begin{proof}[Proof of Lemma~\ref{appendix:uniform-bound-l}]
    Since we have that $\pi_{\theta}, \pi_{\theta_{t}} \in \Pi$, notice that
    \begin{align*}
        \log \left( \frac{\pi_{\theta}(a \mid x)}{\pi_{\theta_{t}}(a \mid x)}\right) &= \log \pi_{\theta}(a \mid x)  - \log \pi_{\theta_{t}}(a \mid x) \\
        &= \log \left( \frac{\exp \left(\theta^{\top} \psi(x, a) \right)}{\sum_{a^{\prime} \in \mathcal{A}}\exp \left(\theta^{\top} \psi(x, a^{\prime}) \right)}\right) - \log \left( \frac{\exp \left(\theta_{t}^{\top} \psi(x, a) \right)}{\sum_{a^{\prime} \in \mathcal{A}}\exp \left(\theta_{t}^{\top} \psi(x, a^{\prime}) \right)}\right) \\
        &= \log \left( \exp \left( \left(\theta - \theta_{t} \right)^{\top} \psi(x, a) \right) \right) + \log \left( \sum_{a^{\prime} \in \mathcal{A}} \exp \left( \theta_{t}^{\top}\psi(x, a^{\prime}) \right) \right) - \log \left( \sum_{a^{\prime} \in \mathcal{A}} \exp \left( \theta^{\top}\psi(x, a^{\prime}) \right) \right) \\
        &= \left(\theta - \theta_{t} \right)^{\top} \psi(x, a) + \log \left( \sum_{a^{\prime} \in \mathcal{A}} \exp \left( \theta_{t}^{\top}\psi(x, a^{\prime}) \right) \right) - \log \left( \sum_{a^{\prime} \in \mathcal{A}} \exp \left( \theta^{\top}\psi(x, a^{\prime}) \right) \right).
    \end{align*}
Now since $\theta, \theta_{t} \in \Theta$ and $\max_{x, a} \norm{\psi(x, a)}_{2} \leq 1$, it follows that $\log \left( \sum_{a^{\prime} \in \mathcal{A}} \exp \left( \theta^{\top}\psi(x, a^{\prime}) \right) \right) \in [\log(\abs{\mathcal{A}}) - B, \log(\abs{\mathcal{A}}) + B]$ by Cauchy-Schwartz. Note that we are assuming finite action spaces i.e. $\abs{\mathcal{A}} < \infty$. We can easily generalize this to countably infinite action spaces by assuming the policy $\pi_{\theta}(\cdot \mid x)$ is an exponential tilt of the base distribution $\mu(\cdot \mid x)$. That is, assume
\[
    \pi_{\theta}(a \mid x) = \frac{\exp(\theta^{\top}\psi(x, a))}{\int_{a^{\prime}} \exp(\theta^{\top}\psi(x, a^{\prime})) d\mu(a^{\prime} \mid x)}.
\]
The argument above will still hold similarly. We continue assuming finite action spaces. Using this, we find

\begin{align*}
    \log \left( \frac{\pi_{\theta}(a \mid x)}{\pi_{\theta_{t}}(a \mid x)}\right) &= \left(\theta - \theta_{t} \right)^{\top} \psi(x, a) + \log \left( \sum_{a^{\prime} \in \mathcal{A}} \exp \left( \theta_{t}^{\top}\psi(x, a^{\prime}) \right) \right) - \log \left( \sum_{a^{\prime} \in \mathcal{A}} \exp \left( \theta^{\top}\psi(x, a^{\prime}) \right) \right) \\
    &\leq \max_{x, a} \norm{\psi(x, a)}_{2} \left( \norm{\theta}_{2} + \norm{\theta_{t}}_{2} \right) + \log(\abs{\mathcal{A}}) + B - (\log(\abs{\mathcal{A}}) - B) \\
    &\leq 4B,
\end{align*}

where the first inequality holds from the CS and triangle inequality. Now, we also have that $r \in \mathcal{F}$. Thus,
\begin{align*}
    r(x, a) - r(x, a^{\prime}) &= \phi(x, a)^{\top}\omega - \phi(x, a^{\prime})^{\top}\omega \\
    &\leq \left( ||\phi(x, a)||_{2} + ||\phi(x, a^{\prime})||_{2} \right) ||\omega||_{2} \\
    &\leq 2F,
\end{align*}
where the first inequality holds from Cauchy-Schwartz and Triangle inequality. Now recall the REBEL update \ref{eq:rebel-regression}. Using these facts we have that $|g_{t}(z; \theta)| \leq 8B / \eta + 2F$ so $K_{g} = \sup_{z, \theta} |g_{t}(z; \theta)| \leq 8B / \eta + 2F$. Since $\ell_{t}(z;\theta) = g_{t}(z; \theta)^{2}$, $K_{\ell} = K_{g}^{2}$.
\end{proof}

\subsection{Lipschitz bound on $\ell(z;\theta)$}
Now we prove that $\ell_{t}(z; \theta)$ is $4K_{g} / \eta$-Lipschitz in $\theta$. 
\vspace{1em}
\begin{lemma}[Lipschitz bound on $\ell(z;\theta)$]
\label{appendix:lipschitz-bound-l}
$\ell_{t}(z; \theta)$ is $\frac{4K_{g}}{\eta}$-Lipschitz in $\theta$. 
\end{lemma}
\begin{proof}[Proof of Lemma~\ref{appendix:lipschitz-bound-l}]
First, we compute the gradient $\nabla_{\theta} g_{t}(z;\theta)$. Since we are looking at updates with respect to $\theta$, notice that we have the following:
\[
    \nabla_{\theta} g(z;\theta) = \nabla_{\theta} \left( \frac{1}{\eta} \left[ \log \pi_{\theta} (a \mid x) - \log \pi_{\theta} (a^{\prime} \mid x) \right] \right).
\]
Now notice that
\begin{align*}
    \log \pi_{\theta} (a \mid x) - \log \pi_{\theta} (a^{\prime} \mid x) &= \log \left( \exp \left( \theta^{\top} \psi(x, a) \right) \right) - \log \left( \exp \left( \theta^{\top} \psi(x, a^{\prime}) \right) \right)  \\
    &= \theta^{\top} \left( \psi(x, a) - \psi(x, a^{\prime}) \right).
\end{align*}
Thus we find that 
\[
    \nabla_{\theta} g(z;\theta) = \frac{1}{\eta} \left( \psi(x, a) - \psi(x, a^{\prime}) \right).
\]
Thus by triangle inequality, $\sup_{x, a} ||\nabla_{\theta} g(z;\theta)||_{2} \leq 2/\eta$. Now since $\ell(z;\theta) = g(z; \theta)^{2}$, we have that $\nabla_{\theta} \ell(z;\theta) = 2 g(z;\theta) \nabla_{\theta}g(z;\theta)$. From Lemma \ref{appendix:uniform-bound-l}, we know that $K_{g} = \sup_{z, \theta} |g(z; \theta)|$ so we see that $\sup_{x, a} ||\nabla_{\theta} \ell(z;\theta)||_{2} \leq 4K_{g} / \eta$. Thus we can conclude that $\ell_{t}(z;\theta)$ is $4K_g/\eta$-Lipschitz in $\theta$.
\end{proof}

\section{Proof of "Slow Rate" Wasserstein-DRO-REBEL}
\label{appendix:Slow-Wasserstein-DRO-REBEL}
First we prove that $h_{t}(\theta;P) = \mathbb{E}_{z \sim P} \left[ \ell_{t}(z;\theta) \right]$ is strongly convex for any $P$. 

\subsection{Proof of Strong  Convexity of WDRO-REBEL}
\begin{lemma}[Strong convexity of $h$]
\label{appendix:strong-convexity-of-h}
    Let $\ell_{t}(z;\theta)$ be the REBEL loss function. Assume that Assumption \ref{assumption:data-coverage} holds. Then $h_{t}(\theta;P) = \mathbb{E}_{z \sim P} \left[ \ell_{t}(z;\theta) \right]$ is $2/\eta^{2}$-strongly convex with respect to norm $||\cdot||_{\Sigma_{P}}$ where $$\Sigma_{P, t} := \mathbb{E}_{(x,a^{1},a^{2}, \Delta r) \sim P} \left[ \left( \psi(x, a^{1}) - \psi(x, a^{2}) \right) \left( \psi(x, a^{1}) - \psi(x, a^{2}) \right)^\top \right].$$
\end{lemma}

\begin{proof}[Proof of Lemma \ref{appendix:strong-convexity-of-h}]
We begin by computing the Hessian of $\ell(z;\theta)$ with respect to $\theta$. From Lemma \ref{appendix:lipschitz-bound-l}, we know that $\nabla_{\theta} \ell(z;\theta) = 2 g(z;\theta) \nabla_{\theta}g(z;\theta)$.
Differentiating again with respect to $\theta$ using the product rule, we get:
\[
\nabla_{\theta}^{2} \ell(z;\theta) = 2 \nabla_{\theta}g(z;\theta) \nabla_{\theta}g(z;\theta)^{\top} + 2 g(z;\theta) \nabla_{\theta}^{2}g(z;\theta).
\]
From Lemma \ref{appendix:lipschitz-bound-l}, we know that $\nabla_{\theta}g(z;\theta) = \frac{1}{\eta} \left[ \psi(x, a) - \psi(x, a^{\prime}) \right]$.
Crucially, this gradient does not depend on $\theta$. Therefore, $\nabla_{\theta}^{2}g(z;\theta) = \mathbf{0}$. Substituting this into the Hessian expression, the second term vanishes:

\begin{align*}
    \nabla_{\theta}^{2} \ell(z;\theta) = \frac{2}{\eta^{2}} \left( \psi(x, a) - \psi(x, a^{\prime}) \right) \left( \psi(x, a) - \psi(x, a^{\prime}) \right)^{\top} 
    = \frac{2}{\eta^{2}}\Sigma_z,
\end{align*}

where $\Sigma_z := \left( \psi(x, a) - \psi(x, a^{\prime}) \right) \left( \psi(x, a) - \psi(x, a^{\prime}) \right)^{\top}$. Note that $\Sigma_z$ is a positive semi-definite matrix. Let $\theta, \theta^{\prime} \in \Theta$. By the linear approximation theorem (Lemma \ref{appendix:beck-linear-approx}), there exists $\alpha \in [0, 1]$ and $\tilde{\theta} = \alpha \theta + (1-\alpha)\theta^{\prime}$ such that
\[
    \ell_{t}(z;\theta^{\prime}) - \ell_{t}(z;\theta) - \langle \nabla_{\theta} \ell_{t}(z;\theta), \Delta \rangle = \frac{1}{2}\Delta^{\top} \nabla_{\theta}^{2} \ell_{t}(z;\widetilde{\theta})\Delta = \frac{\mu}{2} ||\Delta||_{\Sigma_{P, t}},
\]

where $\mu = 2/\eta^{2}$. By Lemma \ref{appendix:beck-strong-convexity}, since $\ell_{t}(z;\theta)$ is a convex function of $\theta$ (as its Hessian is positive semi-definite): 

\begin{align*}
    h_{t}(\alpha\theta + (1 - \alpha)\theta^{\prime}) &= \mathbb{E}_{z \sim P} \left[ \ell_{t}\left(z; \alpha\theta + (1 - \alpha)\theta^{\prime}\right)  \right] \\
    &\leq \mathbb{E}_{z \sim P} \left[ \alpha \ell_{t}\left(z; \theta \right) + (1 - \alpha) \ell_{t}\left(z; \theta^{\prime} \right) - \frac{\mu}{2}\alpha(1-\alpha)\norm{\theta - \theta^{\prime}}_{\Sigma_{z}}^{2} \right] \\
    &= \alpha h_{t}(\theta) + (1 - \alpha)h_{t}(\theta^{\prime}) - \frac{\mu}{2}\alpha(1-\alpha) \left( \theta - \theta^{\prime} \right)^{\top} \mathbb{E}_{z \sim P} \left[ \Sigma_{z} \right] \left( \theta - \theta^{\prime} \right) \\
    &= \alpha h_{t}(\theta) + (1 - \alpha)h_{t}(\theta^{\prime}) - \frac{\mu}{2}\alpha(1-\alpha) \norm{\theta - \theta^{\prime}}_{\Sigma_{P, t}}^{2}.
\end{align*}

By Assumption \ref{assumption:data-coverage}, we have that $\Sigma_{P, t}$ is positive definite so $\norm{\cdot}_{\Sigma_{P, t}}$ is a norm. This implies $h$ is $\mu$-strongly convex in the $||\cdot||_{\Sigma_{P, t}}$ norm.
\end{proof}

We now establish strong convexity of $\mathcal{L}_{t}^{\mathcal{W}_{p}}\left(\theta; \varepsilon\right) = \sup_{P \in \mathcal{B}_{\mathcal{W}_{p}} \left( P^{\circ} ; \varepsilon \right)} \mathbb{E}_{z \sim P} \left[ \ell(z;\theta) \right]$.

\begin{lemma}[Strong convexity of $\mathcal{L}_{t}^{\mathcal{W}_{p}}$]
\label{appendix:strong-convexity-of-wasserstein-loss}
Let $\ell_{t}(z;\theta)$ be the REBEL loss function. Then $$\mathcal{L}^{\mathcal{W}_{p}}\left(\theta; \varepsilon\right) = \sup_{P \in \mathcal{B}_{\mathcal{W}_{p}} \left( P^{\circ} ; \varepsilon \right)} \mathbb{E}_{z \sim P} \left[ \ell_{t}(z;\theta) \right],$$ is $2\kappa / \eta^{2}$-strongly convex with respect to Euclidean norm $||\cdot||_{2}$ where $\kappa$ is the regularity parameter from Assumption \ref{assumption:data-coverage}.
\end{lemma}

\begin{proof}[Proof of Lemma \ref{appendix:strong-convexity-of-wasserstein-loss}]
In Lemma \ref{appendix:strong-convexity-of-h}, we proved strong convexity of $h$. By Lemma \ref{appendix:beck-strong-convexity}, for $\theta, \theta^{\prime} \in \Theta$ and $\alpha \in [0, 1]$, this is equivalent to

\[
    h_{t}(\alpha\theta + (1-\alpha)\theta^{\prime};P) \leq \alpha h_{t}(\theta;P) + (1-\alpha)h_{t}(\theta^{\prime};P) - \frac{\mu}{2} \alpha(1-\alpha)||\theta - \theta^{\prime}||_{\Sigma_{P, t}}^{2}.
\]

Taking the supremum over $P$ preserves the convex combination and the negative quadratic term so we get 

\begin{align*}
    \mathcal{L}_{t}^{\mathcal{W}_{p}}\left(\alpha\theta + (1-\alpha)\theta^{\prime}; \varepsilon\right) &= \sup_{P \in \mathcal{B}_{\varepsilon} \left( P^{\circ} ; \mathcal{W}_{p} \right)} h_{t}(\alpha\theta + (1-\alpha)\theta^{\prime};P) \\
    &\leq \sup_{P \in \mathcal{B}_{\varepsilon} \left( P^{\circ} ; \mathcal{W}_{p} \right)} \left[  \alpha h_{t}(\theta;P) + (1-\alpha)h_{t}(\theta^{\prime};P) - \frac{\mu}{2} \alpha(1-\alpha)||\theta - \theta^{\prime}||_{\Sigma_{P}}^{2} \right] \\
    &\leq \alpha \mathcal{L}_{t}^{\mathcal{W}_{p}} \left( \theta; \varepsilon \right) + (1-\alpha) \mathcal{L}_{t}^{\mathcal{W}_{p}} \left( \theta^{\prime}; \varepsilon \right) - \frac{\mu}{2} \alpha(1-\alpha) \inf_{P \in \mathcal{B}_{\mathcal{W}_{p}} \left( P^{\circ} ; \varepsilon \right)}||\theta - \theta^{\prime}||_{\Sigma_{P, t}}^{2} \\
    &\leq \alpha \mathcal{L}_{t}^{\mathcal{W}_{p}} \left( \theta; \varepsilon \right) + (1-\alpha) \mathcal{L}_{t}^{\mathcal{W}_{p}} \left( \theta^{\prime}; \varepsilon \right) - \frac{\mu}{2} \alpha(1-\alpha) \inf_{P \in \mathcal{B}_{\mathcal{W}_{p}} \left( P^{\circ} ; \varepsilon \right)} \lambda_{\mathrm{min}} \left( \Sigma_{P, t} \right) ||\theta - \theta^{\prime}||_{2}^{2} \\
    &\leq \alpha \mathcal{L}_{t}^{\mathcal{W}_{p}} \left( \theta; \varepsilon \right) + (1-\alpha) \mathcal{L}_{t}^{\mathcal{W}_{p}} \left( \theta^{\prime}; \varepsilon \right) - \frac{\mu \kappa}{2} \alpha(1-\alpha)  ||\theta - \theta^{\prime}||_{2}^{2},
\end{align*}

where the second inequality holds from $\sup_{x} \left( f(x) + g(x) \right) \leq \sup_{x} f(x) + \sup_{x}g(x)$, the third inequality holds by the fact that $\Sigma_{P, t} \succeq \lambda_{\mathrm{min}}\left( \Sigma_{P, t} \right)I$, and the last inequality holds from Assumption \ref{assumption:data-coverage}. Thus we conclude that $\mathcal{L}_{t}^{\mathcal{W}_{p}}$ is $\mu\kappa$-strongly convex in the $||\cdot||_{2}$ norm.
\end{proof}

\subsection{Proof of Slow Parameter Estimation Rate of WDRO-REBEL}
We are now ready to prove the "slow rate" estimation error of Wasserstein-DRO-REBEL. Before we prove this, we must make an assumption about the dual variable set we are optimizing over and the metric space over which $(\mathcal{Z}, d)$ is defined
\vspace{1em}
\begin{assumption}
    \label{appendix:assumption-compact-optima-bounded-diameter}
    Assume that the optimal $\Delta^{*} \in [0, \overline{\Delta}]$ and the metric space $(\mathcal{Z}, d)$ has bounded diameter $\mathrm{diam}(\mathcal{Z}) = \sup_{u, v \in \mathcal{Z}} d(u, v) \leq D$. 
\end{assumption}

\begin{proof}[Proof of Theorem \ref{theorem:Wasserstein-DRO-REBEL}]
\label{appendix:proof-slow-rate-wasserstein}
    Let $\theta_{t}^{\mathcal{W}_{p}}$ denote the true population minimizer $\mathrm{arg}\min_{\theta \in \Theta} \mathcal{L}_{t}^{\mathcal{W}_{p}}\left(\theta; \varepsilon\right)$ and $\hat{\theta}_{n, t}^{\mathcal{W}_{p}}$ denote the empirical minimizer $\mathrm{arg}\min_{\theta \in \Theta} \mathcal{L}_{n, t}^{\mathcal{W}_{p}}\left(\theta; \varepsilon\right)$. Let $t \in \left\{ 0, \dots, T-1 \right\}$ and for each $t$, we collect a dataset $\mathcal{D}_{t} = \left\{ (x_{t, i}, y_{t, i}, y_{t, i}^{\prime}) \right\}_{i=1}^{n_{t}}$ with $x_{t, i} \sim \rho$, $y_{t, i}, y_{t, i}^{\prime} \stackrel{\mathrm{i.i.d}}{\sim} \pi_{\theta_{t-1}}(\cdot \mid x_{t, i})$. Define $\mathcal{F}_{t} = \sigma(\theta_{0}, \mathcal{D}_{0}, \dots, \theta_{t-1}, \mathcal{D}_{t-1})$ be the sigma-field containing everything revealed up to the start of iteration $t$
    and $\overline{\mathcal{F}}_{t} = \sigma(\mathcal{F}_{t}, \mathcal{D}_{t})$. In particular, $\theta_{t}$ is $\mathcal{F}_{t}$ measurable. By strong duality for Wasserstein DRO [\ref{appendix:gao-strong-duality-wasserstein}] and Assumption 
    \ref{appendix:assumption-compact-optima-bounded-diameter}, for fixed $\theta$ we have 

    \[
        \mathcal{L}_{t}^{\mathcal{W}_{p}}\left(\theta; \varepsilon\right) = \sup_{\mathbb{P} \in \mathcal{B}_{\mathcal{W}_{p} } \left( P_{t}^{\circ} ; \varepsilon \right) } \mathbb{E}_{z \sim \mathbb{P}} \left[ \ell(z;\theta) \right] = \inf_{\Delta \in [0, \overline{\Delta}]} \left\{ \delta \varepsilon^{p} - \mathbb{E}_{z \sim P_{t}^{\circ}} \left[ \ell_{\Delta, t}(z; \theta) \right]\right\},
    \]

    where $\ell_{\Delta}(z; \theta) = \inf_{z^{\prime} \in \mathcal{Z}} \left\{ \Delta d^{p}(z, z^{\prime}) - \ell(z^{\prime};\theta)\right\}$ where $d$ is the metric used to define the type-p Wasserstein distance. Consider the difference between the population and empirical Wasserstein DRO losses: 

    \begin{align*}
        \lvert \mathcal{L}^{\mathcal{W}_{p}}\left(\theta; \varepsilon\right) - \mathcal{L}_{n}^{\mathcal{W}_{p}}\left(\theta; \varepsilon\right) \rvert &= \left\lvert \sup_{\mathbb{P} \in \mathcal{B}_{\varepsilon} \left( \mathbb{P}^{\circ} ; \mathcal{W}_{p} \right)} \mathbb{E}_{z \sim \mathbb{P}} \left[ \ell(z;\theta) \right] - \sup_{\mathbb{P} \in \mathcal{B}_{\varepsilon} \left( P_{n, t}^{\circ} ; \mathcal{W}_{p} \right)} \mathbb{E}_{z \sim \mathbb{P}} \left[ \ell(z;\theta) \right] \right\rvert \\
        &= \left\lvert  \inf_{\Delta \in [0, \overline{\Delta}]} \left\{ \Delta \varepsilon^{p} - \mathbb{E}_{z \sim P_{t}^{\circ}} \left[ \ell_{\Delta, t}(z; \theta) \right]\right\} - \inf_{\Delta \in [0, \overline{\Delta}]} \left\{ \Delta \varepsilon^{p} - \mathbb{E}_{z \sim P_{n, t}^{\circ}} \left[ \ell_{\Delta, t}(z; \theta) \right]\right\}  \right\rvert \\
        &\leq \sup_{\Delta \in [0, \overline{\Delta}]} \left \lvert \mathbb{E}_{z \sim P_{n, t}^{\circ}} \left[ \ell_{\Delta, t}(z; \theta) \right] - \mathbb{E}_{z \sim P_{t}^{\circ}} \left[ \ell_{\Delta, t}(z; \theta) \right] \right \rvert,
    \end{align*}
where the first equality holds from strong duality and the last inequality holds from $\inf_{x}f(x) - \inf_{x}g(x) \leq \sup_{x} |f(x) - g(x)|$. Let $F(\Delta) = \mathbb{E}_{z \sim P_{n, t}^{\circ}} \left[ \ell_{\Delta, t}(z; \theta) \right] - \mathbb{E}_{z \sim P_{t}^{\circ}} \left[ \ell_{\Delta, t}(z; \theta) \right]$. First notice that
\begin{align*}
    \abs{F(\Delta) - F(\Delta^{\prime})} &\leq \mathbb{E}_{z \sim P_{n, t}^{\circ}} \left[ \abs{\ell_{\Delta, t}(z; \theta) - \ell_{\Delta^{\prime}, t}(z; \theta)} \right] + \mathbb{E}_{z \sim P_{t}^{\circ}} \left[ \abs{\ell_{\Delta^{\prime}, t}(z; \theta) -  \ell_{\Delta, t}(z; \theta)}\right] \\
    & \leq \mathbb{E}_{z \sim P_{n, t}^{\circ}} \left[ \abs{\inf_{z^{\prime} \in \mathcal{Z}} \left\{ \Delta d^{p}(z, z^{\prime}) - \ell_{t}(z^{\prime};\theta)\right\} - \inf_{z^{\prime} \in \mathcal{Z}} \left\{ \Delta^{\prime} d^{p}(z, z^{\prime}) - \ell_{t}(z^{\prime};\theta)\right\}} \right]\\
    &+ \mathbb{E}_{z \sim P_{t}^{\circ}} \left[ \abs{\inf_{z^{\prime} \in \mathcal{Z}} \left\{ \Delta d^{p}(z, z^{\prime}) - \ell_{t}(z^{\prime};\theta)\right\} - \inf_{z^{\prime} \in \mathcal{Z}} \left\{ \Delta^{\prime} d^{p}(z, z^{\prime}) - \ell_{t}(z^{\prime};\theta)\right\}} \right] \\
    &\leq \mathbb{E}_{z \sim P_{t}^{\circ}} \left[ \sup_{z^{\prime} \in \mathcal{Z}} \abs{d^p(z, z^{\prime})(\Delta - \Delta^{\prime})} \right] + \mathbb{E}_{z \sim P_{n, t}^{\circ}} \left[ \sup_{z^{\prime} \in \mathcal{Z}} \abs{d^p(z, z^{\prime})(\Delta - \Delta^{\prime})} \right] \\
    &\leq 2D^{p} \abs{\Delta - \Delta^{\prime}}.
\end{align*}
Note that since $\Delta^{*}$, the optimal dual variable for the dual problem, is a function of the data drawn from $P_{t}^{\circ}$, it is not $\mathcal{F}_{t}$ measurable. Thus we need to use a covering argument. Take a $\alpha$-net $\mathcal{N}_{\alpha}([0, \overline{\Delta}] ; \abs{\cdot}) = 1 + \overline{\Delta} / \varepsilon$. Then take a representative member $\Delta^{\sharp} \in \mathcal{N}_{\alpha}$. Then notice that 
\begin{align*}
    \sup_{\delta \in [0, \overline{\Delta}]} \abs{F(\Delta)} &\leq \max_{\Delta^{\sharp} \in \mathcal{N}_{\alpha}} \abs{F(\Delta^{\sharp})} + \abs{F(\Delta) - F(\Delta^{\sharp})}
    \leq \max_{\Delta^{\sharp} \in \mathcal{N}_{\alpha}} \abs{F(\Delta^{\sharp})} + 2D^{p}\alpha.
\end{align*}
From Lemma \ref{appendix:uniform-bound-l}, we showed that $\ell_{t}(z;\theta) \in [0, K_{l}]$. Now notice that
\begin{align*}
    &\ell_{\Delta, t}(z; \theta) = \inf_{z^{\prime} \in \mathcal{Z}} \left\{ \Delta d^{p}(z, z^{\prime}) - \ell_{t}(z^{\prime};\theta)\right\} \leq \inf_{z^{\prime} \in \mathcal{Z}} \left\{ \Delta d^{p}(z, z^{\prime}) \right\} = 0 \quad (\text{since } \ell_{t}(z;\theta) \ge 0) \\
    &\ell_{\Delta, t}(z; \theta) = \inf_{z^{\prime} \in \mathcal{Z}} \left\{ \Delta d^{p}(z, z^{\prime}) - \ell_{t}(z^{\prime};\theta)\right\} \geq \inf_{z^{\prime} \in \mathcal{Z}} \left\{ \Delta d^{p}(z, z^{\prime}) - K_{\ell} \right\} \geq -K_{\ell} \quad (\text{since } d^p \ge 0).
\end{align*}
Thus, $\ell_{\Delta, t} \in [-K_{\ell}, 0]$. Since $\ell_{\Delta}$ is bounded and $z \stackrel{\mathrm{i.i.d}}{\sim} P_{n, t}^{\circ}$, we can apply conditional Hoeffding's inequality (Lemma \ref{appendix:hoeffding}) and union bound over $\mathcal{N}_{\alpha}$. For any $\epsilon > 0$ and fixed $t \in [T]$ and $\theta \in \Theta$:

\[
    \Pr\left( \max_{\Delta^{\sharp} \in \mathcal{N}_{\alpha}}\left \lvert \mathbb{E}_{z \sim P_{n, t}^{\circ}} \left[ \ell_{\Delta, t}(z; \theta) \right] - \mathbb{E}_{z \sim P_{t}^{\circ}} \left[ \ell_{\Delta, t}(z; \theta) \right] \right \rvert \geq \epsilon \mid \mathcal{F}_{t} \right) \leq 2\mathcal{N}_{\alpha}([0, \overline{\Delta}]; \abs{\cdot})\exp \left( -\frac{2n_{t}\epsilon^{2}}{K_{\ell}^{2}} \right).
\]
By setting the right-hand side to $\delta_{t}$ and solving for $\epsilon$, we find that with probability at least $1-\delta_{t}$, for fixed $t \in [T]$ and $\theta \in \Theta$:

\[
    \lvert \mathcal{L}_{t}^{\mathcal{W}_{p}}\left(\theta; \varepsilon\right) - \mathcal{L}_{n, t}^{\mathcal{W}_{p}}\left(\theta; \varepsilon\right) \rvert \leq K_{\ell} \sqrt{\frac{\log (2 \mathcal{N}_{\alpha}([0, \overline{\Delta}] ; \abs{\cdot}) / \delta_{t})}{n_{t}}} + 2D^{p}\alpha.
\]
Choosing $\alpha = K_{\ell}/2D^{p}\sqrt{n_{t}}$, we find for some constant depending on $D, K_{l}$, we have 
\[
    \lvert \mathcal{L}_{t}^{\mathcal{W}_{p}}\left(\theta; \varepsilon\right) - \mathcal{L}_{n, t}^{\mathcal{W}_{p}}\left(\theta; \varepsilon\right) \rvert \lesssim K_{\ell} \sqrt{\frac{\log(n_{t}) + \log(\overline{\Delta}) + \log (1 / \delta_{t})}{n_{t}}}.
\]
Now note that $\ell_{\Delta}(z;\theta)$ is itself Lipschitz in $\theta$. Let $\theta, \theta^{\prime} \in \Theta$. Then
\begin{align*}
    \abs{\ell_{\Delta, t}(z; \theta) - \ell_{\Delta, t}(z; \theta^{\prime})} &= \abs{\inf_{z^{\prime} \in \mathcal{Z}} \left\{ \Delta d^{p}(z, z^{\prime}) - \ell_{t}(z^{\prime};\theta)\right\} - \inf_{z^{\prime} \in \mathcal{Z}} \left\{ \Delta d^{p}(z, z^{\prime}) - \ell_{t}(z^{\prime};\theta^{\prime})\right\}} \\
    &\leq \sup_{z \in \mathcal{Z}} \abs{\ell_{t}(z^{\prime};\theta)- \ell_{t}(z^{\prime};\theta^{\prime})} \\
    &\leq L_{K_{g}, \eta} \norm{\theta - \theta^{\prime}}_{2},
\end{align*}

where the first inequality holds from $\inf_{x}f(x) - \inf_{x}g(x) \leq \sup_{x} |f(x) - g(x)|$ an the last inequality holds from Lemma \ref{appendix:lipschitz-bound-l}. Now let $\mathcal{N}_{\overline{\alpha}}$ be an $\overline{\alpha}$-net of $\Theta = \left\{ \theta \in \mathbb{R}^{d} : \norm{\theta}_{2} \leq B \right\}$ with covering number $\mathcal{N}_{\alpha}(\Theta ; \norm{\cdot}_{2}) \leq \left(3B / \overline{\alpha} \right)^{d}$. Fix $t$. Then choosing failure probability $\delta_{t} / \abs{\mathcal{N}_{\overline{\alpha}}}$ and union bounding over $\theta \in \mathcal{N}_{\overline{\alpha}}$, we find

\[
    \sup_{\theta \in \mathcal{N}_{\alpha}} \abs{\mathcal{L}_{t}^{\mathcal{W}_{p}}\left(\theta; \varepsilon\right) - \mathcal{L}_{n, t}^{\mathcal{W}_{p}}\left(\theta; \varepsilon\right)} \lesssim K_{\ell} \sqrt{\frac{\log(n_{t}) + \log(\overline{\Delta}) + \log (\abs{\mathcal{N}_{\overline{\alpha}}(\Theta ; \norm{\cdot}_{2})} / \delta_{t})}{n_{t}}}.
\]
Choose $\theta^{\sharp} \in \mathcal{N}_{\alpha}$ such for some $\theta \in \Theta$, we have $\norm{\theta - \theta^{\sharp}} \leq \alpha$. Then notice that we have
\begin{align*}
    \mathcal{L}_{t}^{\mathcal{W}_{p}}\left(\theta; \varepsilon\right) - \mathcal{L}_{n, t}^{\mathcal{W}_{p}}\left(\theta; \varepsilon\right) &\leq \mathcal{L}_{t}^{\mathcal{W}_{p}}\left(\theta; \varepsilon\right) - \mathcal{L}_{t}^{\mathcal{W}_{p}}\left(\theta^{\sharp}; \varepsilon\right) + \mathcal{L}_{t}^{\mathcal{W}_{p}}\left(\theta^{\sharp}; \varepsilon\right) - \mathcal{L}_{n, t}^{\mathcal{W}_{p}}\left(\theta^{\sharp}; \varepsilon\right) + \mathcal{L}_{n, t}^{\mathcal{W}_{p}}\left(\theta^{\sharp}; \varepsilon\right) - \mathcal{L}_{n, t}^{\mathcal{W}_{p}}\left(\theta; \varepsilon\right) \\
    &\leq \sup_{\theta \in \mathcal{N}_{\alpha}} \abs{\mathcal{L}_{t}^{\mathcal{W}_{p}}\left(\theta^{\sharp}; \varepsilon\right) - \mathcal{L}_{n, t}^{\mathcal{W}_{p}}\left(\theta^{\sharp}; \varepsilon\right)} + 2L_{K_{g}, \eta}\norm{\theta - \theta^{\prime}}_{2} \\
    &\lesssim K_{\ell} \sqrt{\frac{\log(n_{t}) + \log(\overline{\Delta}) + \log (\abs{\mathcal{N}_{\overline{\alpha}}(\Theta ; \norm{\cdot}_{2})} / \delta_{t})}{n_{t}}} + 2L_{K_{g}, \eta}\alpha,
\end{align*}
where the second inequality holds from us showing $\ell_{\Delta, t}$ is Lipschitz. Choosing $\alpha = K_{l} / L_{K_{g}, \eta} \sqrt{n_{t}}$, we find
\[
    \sup_{\theta \in \Theta} \abs{\mathcal{L}_{t}^{\mathcal{W}_{p}}\left(\theta; \varepsilon\right) - \mathcal{L}_{n, t}^{\mathcal{W}_{p}}\left(\theta; \varepsilon\right)} \lesssim K_{\ell} \sqrt{\frac{\log(n_{t}) + \log(\overline{\Delta}) + \log (\abs{\mathcal{N}_{\overline{\alpha}}(\Theta ; \norm{\cdot}_{2})} / \delta_{t})}{n_{t}}}.
\]
Setting $\delta_{t} = \delta/T$ and using the covering number for $\mathcal{N}_{\alpha}$, we have that for all $t \in [T]$, $\theta \in \Theta$
\begin{align}
    \label{eq:wasserstein-uniform-bound}
        \sup_{\theta \in \Theta} \abs{\mathcal{L}_{t}^{\mathcal{W}_{p}}\left(\theta; \varepsilon\right) - \mathcal{L}_{n, t}^{\mathcal{W}_{p}}\left(\theta; \varepsilon\right)} \lesssim K_{\ell}\sqrt{\frac{d\log n_{t} + \log(\overline{\Delta}) + \log (T / \delta)}  {n_{t}}}.
\end{align}

On the event in \ref{eq:wasserstein-uniform-bound}, for each fixed $t$ we have the standard decomposition

\begin{align*}
    \mathcal{L}^{\mathcal{W}_{p}}_{t}(\hat{\theta}^{\mathcal{W}_{p}}_{n,t};\varepsilon)-\mathcal{L}^{\mathcal{W}_{p}}_{t}(\theta_t^{\mathcal{W}_{p}};\varepsilon)
    &\leq
    \underbrace{\mathcal{L}^{\mathcal{W}_{p}}_{t}(\hat{\theta}^{\mathcal{W}_{p}}_{n,t};\varepsilon)-\mathcal{L}^{\mathcal{W}_{p}}_{n,t}(\hat{\theta}^{\mathcal{W}_{p}}_{n,t};\varepsilon)}_{\le \sup_{\theta}|\mathcal{L}_t-\mathcal{L}_{n,t}|}
    +
    \underbrace{\mathcal{L}^{\mathrm{KL}}_{n,t}(\theta_t^{\mathcal{W}_{p}};\varepsilon)-\mathcal{L}^{\mathcal{W}_{p}}_{t}(\theta_t^{\mathcal{W}_{p}};\varepsilon)}_{\le \sup_{\theta}|\mathcal{L}_t-\mathcal{L}_{n,t}|}\\
    &\le 2\sup_{\theta\in\Theta}\big|\mathcal{L}^{\mathcal{W}_{p}}_{t}(\theta;\varepsilon)-\mathcal{L}^{\mathcal{W}_{p}}_{n,t}(\theta;\varepsilon)\big|.
\end{align*}
Invoking Lemma \ref{appendix:strong-convexity-of-wasserstein-loss}, we have for all $\theta\in\Theta$,
\[
\mathcal{L}^{\mathcal{W}_{p}}_{t}(\theta;\varepsilon)-\mathcal{L}^{\mathcal{W}_{p}}_{t}(\theta_t^{\mathcal{W}_{p}};\varepsilon)
\ge
\frac{\kappa}{\eta^2}\|\theta-\theta_t^{\mathcal{W}_{p}}\|_2^2.
\]
Conditional on $\mathcal{F}_{t}$, applying this inequality at $\theta=\hat{\theta}^{\mathcal{W}_{p}}_{n,t}$ yields, with probability at least $1-\delta$, for all $t\in[T]$,
\[
    \lVert \theta_{t}^{\mathcal{W}_{p}} - \hat{\theta}_{n, t}^{\mathcal{W}_{p}} \rVert^{2} \leq \frac{\eta^{2}K_{g}^{2}}{\kappa} \sqrt{\frac{d\log n_{t} + \log(\overline{\Delta}) + \log (T / \delta)}  {n_{t}}}.
\]
\end{proof}

\section{Proofs for "Fast" Wasserstein-DRO-REBEL}
\label{appendix:proofs-fast-wasserstein}
Throughout this setup, we will need to make a few technical assumptions to work with the Wasserstein dual properly. That is, we will need to make assumptions that guarantee that the dual admits an infimum i.e. the set of minimizers is non-empty and that the minimizers themselves are indeed Borel measurable. We will use these assumptions to reduce the dual DRO objective to a well-defined measurable object that we can treat like an ordinary empirical risk object.

\subsection{Technical results relating to the existence and measurability of maximum and minimum weak correspondence selectors}

The first assumption we make is that the nominal data-generating distribution $P_{t}^{\circ}$ is defined on a measurable space $(\mathcal{Z}, \mathcal{B}(\mathcal{Z}))$ where $\mathcal{Z}$ is a proper Polish space. Note that this is already implicitly when bounding the slow rate as Lemma \ref{appendix:gao-strong-duality-wasserstein} assumes that $\Xi$ is a Polish space (refer to \cite{gao2022distributionallyrobuststochasticoptimization} for details).

\vspace{1em}

\begin{assumption}[Proper Polish Metric space]
\label{appendix:proper-polish-space-wasserstein}
Assume $(\mathcal{Z}, d)$ is a proper Polish metric space: $\mathcal{Z}$ is complete and separable and every closed $d$-ball is compact.
\end{assumption}

We also assume mild conditions on the regularity of $\ell_{t}$: 

\vspace{1em}

\begin{assumption}[Loss regularity]
    \label{appendix:loss-regularity-wasserstein}
    For each $t$, $\ell_{t}: \mathcal{Z} \times \Theta \rightarrow \mathbb{R}$ is bounded and Borel measurable. Moreover, for each $\theta \in \Theta$, the map $z \mapsto \ell_{t}(z; \theta)$ is upper semicontinuous on $\mathcal{Z}$ i.e. for every sequence $z_{n} \rightarrow z$, we have $\limsup_{n \rightarrow \infty} \ell_{t}(z_n ; \theta) \leq \ell_{t}(z; \theta)$.
\end{assumption}

Under Assumption \ref{appendix:proper-polish-space-wasserstein} and \ref{appendix:loss-regularity-wasserstein}, we can show the following:

\vspace{1em}

\begin{lemma}[Existence and compactness of maximizers]
    \label{appendix:existence-and-compactness-of-maximizers-wasserstein}
    Assume Assumption \ref{appendix:proper-polish-space-wasserstein} holds. Fix $(\Delta, \theta, z)$ with $\Delta > 0$. Then the argmax set
    \[
        \mathcal{M}(\Delta, \theta, z) := \argmax_{z^{\prime} \in \mathcal{Z}} \left\{ \ell_{t}(z^{\prime} ; \theta) - \Delta d(z, z^{\prime})^{p} \right\},
    \]
    is nonempty and compact.
\end{lemma}

\begin{proof}
    Suppose $K_{\ell} := \sup_{z, \theta} \abs{\ell(z; \theta)}$ (this holds true by Lemma \ref{appendix:uniform-bound-l}). Then for any $z^{\prime} \in \mathcal{Z}$, we have
    \[
        \ell_{t}(z^{\prime} ; \theta) - \Delta d(z, z^{\prime})^{p} \leq K_{\ell} - \Delta d(z, z^{\prime})^{p}.
    \]
    Now pick a radius $R = (2K_{\ell} / \Delta)^{1/p}$. If $d(z, z^{\prime}) > R$, then 
    \[
        \ell_{t}(z^{\prime} ; \theta) - \Delta d(z, z^{\prime})^{p} < -K_{\ell}.
    \]
    Meanwhile at a specific point $z^{\prime} = z$, we have
    \[
        \ell_{t}(z; \theta) - \Delta d(z, z^{\prime})^{p}  \geq -K_{\ell}.
    \]
    Thus anything outside the closed ball $\overline{B}(z, R)$ has a strictly smaller value, hence it cannot be a maximizer. Therefore we have $\mathcal{M}(\Delta, \theta, z) \subseteq \overline{B}(z, R)$. By Assumption \ref{appendix:proper-polish-space-wasserstein}, we have that every $d$-ball is compact so $\overline{B}(z, R)$ is compact. By Theorem \ref{appendix:extreme-value-theorem} (Weierstrass Extreme Value Theorem)  and Assumption \ref{appendix:loss-regularity-wasserstein}, we conclude that on $\overline{B}(z, R)$, $\ell_{t}(z^{\prime} ; \theta) - \Delta d(z, z^{\prime})^{p}$ attains its maximum. Since $\mathcal{M}(\Delta, \theta, z) \subseteq \overline{B}(z, R)$ and the fact that if a function is upper semi-continuous, the set of maximizers are closed, we can conclude $\mathcal{M}(\Delta, \theta, z)$ is compact since a closed subset of a compact set is indeed compact.
\end{proof}

Using this along with the assumption of the underlying metric space being Polish, we can prove that selectors themselves are Borel measurable.

\vspace{1em}

\begin{lemma}[Measurable argmax selection]
\label{appendix:measurable-argmax-selection}
Let $(\mathcal{Z}, d)$ be proper Polish and $\Theta \subset \mathbb{R}^{d}$ compact. Assume that $K_{\ell} := \sup_{z, \theta} \abs{\ell(z; \theta)}$, $(z, \theta) \rightarrow \ell_{t}(z ; \theta)$ is Borel measurable on $\mathcal{Z} \times \Theta$ and upper semicontinous in $z$. Then there exists a Borel measurable map $z^{\sharp}: (0, \infty) \times \Theta \times \mathcal{Z} \rightarrow \mathcal{Z}$ such that $z^{\sharp}(\Delta, \theta, z) \in \mathcal{M}(\Delta, \theta, z)$ for all $\Delta > 0$, $\theta \in \Theta$, and $z \in \mathcal{Z}$.
\end{lemma}

\begin{proof}
First note that by Lemma \ref{appendix:uniform-bound-l}, Assumption \ref{appendix:proper-polish-space-wasserstein}, and Assumption \ref{appendix:loss-regularity-wasserstein}, the assumptions hold. First let $X := (0, \infty) \times \Theta \times \mathcal{Z}$. Since $(0, \infty)$ is an open subset of $\mathbb{R}$ which under the absolute distance metric is Polish, we have $(0, \infty)$ is Polish. Likewise since $\Theta$ is a compact metric space (by Assumption \ref{assumption:log-linear-policy}) and every compact metric space is complete and separable, we have $\Theta$ is Polish. Thus $X$ is Polish and $(X, \mathcal{B}(X))$ is the standard Borel measurable-space. For each fixed $(\Delta, \theta, z)$, by Lemma \ref{appendix:existence-and-compactness-of-maximizers-wasserstein}, the argmax set $\mathcal{M}(\Delta, \theta, z)$ is nonempty compact. Now let $R(\Delta)= (2K_{\ell} / \Delta)^{1/p}$ and define a correspondence $K: X \rightleftarrows \mathcal{Z}$ by $K(\Delta, \theta, z) = \overline{B}(z, R(\Delta))$. Its graph
\[
    \mathrm{Gr}(K) = \left\{ (\Delta, \theta, z, z^{\prime}) \in X \times \mathcal{Z} : d(z, z^{\prime}) \leq R(\Delta) \right\},
\]

is Borel and closed since the maps $(z, z^{\prime}) \mapsto d(z, z^{\prime})$ and $\Delta \mapsto R(\Delta)$ are continuous. Define the value function
\[
    v(\Delta, \theta, z) = \max_{z^{\prime} \in K(\Delta, z)}  \left\{ \ell_{t}(z^{\prime} ; \theta) - \Delta d(z, z^{\prime})^{p} \right\}.
\]
By Theorem \ref{appendix:berge-maximum-theorem} (Berge Maximum Theorem), we conclude that $v: X \rightarrow \mathbb{R}$ is Borel measurable. Now we can write the graph of $\mathcal{M}$ using $v$ as follows:
\[
    \mathrm{Gr}(\mathcal{M}) = \mathrm{Gr}(K) \cap \left\{ (x, z^{\prime}) : z^{\prime} \in K(x) \; \text{and} \; \ell_{t}(z^{\prime} ; \theta) - \Delta d(z, z^{\prime})^{p} - v(\Delta, \theta, z) = 0\right\}.
\]
Since $\mathrm{Gr}(K)$ is Borel and every map is Borel within the set, we have that $\mathrm{Gr}(\mathcal{M})$ is Borel in $X \times \mathcal{Z}$. Since we have that $X$ is Borel, $\mathcal{Z}$ is Polish by assumption, $\mathrm{Gr}(\mathcal{M})$ is Borel, and $\mathcal{M}$ is nonempty and compact, we have that $\mathcal{M}$ is a weakly measurable correspondence with nonempty compact values. By Theorem \ref{appendix:krn-selection-theorem} (Kuratowski–Ryll-Nardzewski measurable selection theorem), there exists a Borel measurable selector $z^{\sharp}: X \rightarrow \mathcal{Z}$ such that $z^{\sharp}(\Delta, \theta, z) \in \mathcal{M}(\Delta, \theta, z)$ for all $\Delta > 0$, $\theta \in \Theta$, and $z \in \mathcal{Z}$.
\end{proof}

Let us fix $p \geq 1$. For $\Delta \geq 0$, we define the c-transform of $\ell_{t}(\cdot; \theta)$ as follows:
\[
    \Phi(\Delta, \theta; z) := -\ell_{\Delta, t}(z; \theta) = \sup_{z^{\prime} \in \mathcal{Z}} \left\{ \ell_{t}(z^{\prime} ; \theta) - \Delta d(z, z^{\prime})^{p} \right\}.
\]

Note that by a standard fact in convex analysis (Theorem \ref{appendix:subgradient-convex-hull}), since $\mathcal{M}(\Delta, \theta, z)$ is compact (by Lemma \ref{appendix:existence-and-compactness-of-maximizers-wasserstein}) and $\ell_{t}(z^{\prime} ; \theta) - \Delta d(z, z^{\prime})^{p}$ is upper semi-continous (by Assumption \ref{appendix:loss-regularity-wasserstein}), we have that 
\[
    \partial_{\theta} \Phi(\Delta, \theta; z) = \mathrm{conv} \left\{ \nabla_{\theta}\ell_{t}(z^{\prime}; \theta) : z^{\prime} \in   \mathcal{M}(\Delta, \theta, z)\right\},
\]
where $\mathrm{conv}(A)$ denotes the convex hull of a set $A$. Using this fact, if we denote $g_{\Delta}(\theta ; z) := \nabla_{\theta}\ell_{t}(z^{\sharp}(\Delta, \theta, z); \theta)$, then $g_{\Delta}$ is Borel measurable as it is a composition of measurable maps and $g_{\Delta}(\theta ; z) \in \partial_{\theta} \Phi(\Delta, \theta; z)$. Another thing to note is that if for a probability measure $P$ on $\mathcal{Z}$, we define
\[
    \Phi_{P}(\Delta, \theta) = \mathbb{E}_{Z \sim P}[\Phi(\Delta, \theta; z)].
\]
Then since $g_{\Delta}(\theta ; z) \in \partial_{\theta} \Phi(\Delta, \theta; z)$, by definition of the subdifferential:
\[
    \Phi(\Delta, \theta^{\prime}; z) \geq \Phi(\Delta, \theta; z) + \langle g_{\Delta}(\theta ; z), \theta^{\prime} - \theta \rangle.
\]
Taking expectations over $Z \sim P$ and using the fact that $g_{\Delta}$ is integrable gives us that $\mathbb{E}_{Z \sim P}[g_{\Delta}(\theta ; Z)] \in \partial_{\theta}\Phi_{P}(\Delta, \theta)$. These facts will be useful to prove some important properties we will require. First define the robust objective at distribution $P$ as follows:
\[
    \mathcal{L}_{t, P}^{\mathcal{W}_{p}}(\theta ; \varepsilon) = \inf_{\Delta \geq 0} \left\{ \Delta \varepsilon^{p} + \Phi_{P}(\Delta, \theta) \right\}.
\]
We will need to make one more mild assumption before we proceed about domain of the dual.

\vspace{1em}

\begin{assumption}[Compact dual domain and existence of a optimal dual]
    \label{appendix:compact-dual-domain-and-existence-wasserstein}
    Assume that the optimal $\Delta^{*} \in [0, \overline{\Delta}]$ for $\overline{\Delta} < \infty$. Assume $(\Delta, \theta) \mapsto \Delta \varepsilon^{p} + \Phi_{P}(\Delta, \theta)$ is lower semicontinuous in $\Delta$ for each $\theta$. Then by compactness of $[0, \overline{\Delta}]$, we have that the $\argmin_{\Delta \in [0, \overline{\Delta}]} \left\{ \Delta \varepsilon^{p} + \Phi_{P}(\Delta, \theta) \right\}$ is nonempty.
\end{assumption}

Under this assumption, we have measurable selection of $\Delta^{*}$ for the empirical objective

\vspace{1em}

\begin{lemma}[Measurable selection of $\Delta^*$]
    \label{appendix:measurable-selection-of-delta-wasserstein}
    Fix $\theta$. For the empirical objective, define
    \[
        \widehat{\Phi}(\Delta, \theta) = \frac{1}{n} \sum_{i=1}^{n} \Phi(\Delta, \theta; Z_{i}).
    \]
    Assume Assumption \ref{appendix:compact-dual-domain-and-existence-wasserstein} holds. Then the selection function
    \[
        \widehat{\mathcal{D}}(\theta) = \argmin_{\Delta \in [0, \overline{\Delta}]} \left\{ \Delta \varepsilon^{p} + \widehat{\Phi}(\Delta, \theta) \right\},
    \]
    admits a measurable selection $\widehat{\Delta}(\theta)$ as a function of the sample $(Z_{1}, \dots, Z_{n})$. 
\end{lemma}

\begin{proof}[Proof Sketch]
This result follows the same measurable-selection argument as Lemma~\ref{appendix:measurable-argmax-selection}, with two key changes. The first change is that in this case,  the decision variable is the scalar $\Delta\in [0,\overline\Delta]$, rather than $z'\in\mathcal Z$. Since this is clearly compact, we do not require the coercivity/properness truncation argument used earlier to restrict $z'$ to a compact ball. The other difference is applying a measurable minimum theorem in place of a maximum theorem. Define the objective (for fixed $\theta$)
\[
F\bigl((z_1,\dots,z_n),\Delta\bigr)\ :=\ \Delta\varepsilon^{p}+\frac1n\sum_{i=1}^n \Phi(\Delta,\theta;z_i),
\;\; (z_1,\dots,z_n)\in\mathcal Z^n,\ \Delta\in K.
\]
By the assumed measurability of $\Phi$ and the continuity in $\Delta$, $F$ is Borel in the sample and continuous in $\Delta$ (hence lower semicontinuous) on the compact set $K$. Therefore, by Weierstrass' theorem, $\widehat D(\theta)$ is nonempty and compact-valued pathwise. To apply Kuratowski--Ryll-Nardzewski, it remains to check that the graph of $\widehat D(\theta)$ is Borel. This is obtained exactly as in Lemma~\ref{appendix:measurable-argmax-selection}: the value function
\[
v(z_1,\dots,z_n)\ :=\ \min_{\Delta\in K} F\bigl((z_1,\dots,z_n),\Delta\bigr),
\]
is Borel measurable by the measurable minimum theorem (equivalently, by the standard Carath\'eodory/maximum theorem on Polish spaces). Consequently,
\[
\mathrm{Gr}(\widehat D(\theta))
=\Bigl\{((z_1,\dots,z_n),\Delta):\ F((z_1,\dots,z_n),\Delta)=v(z_1,\dots,z_n)\Bigr\},
\]
is a Borel subset of $\mathcal Z^n\times K$. Since $\mathcal Z^n$ is standard Borel and $K$ is Polish, Kuratowski--Ryll-Nardzewski yields a Borel measurable selector $\widehat\Delta(\theta)$ with $\widehat\Delta(\theta)\in\widehat D(\theta)$.
\end{proof}

We will now prove an important result that will finally allow us to do our analysis.

\vspace{1em}

\begin{lemma}[Subgradient rule for a marginal (inf-projection) function]
\label{appendix:marginal-subgradient-rule}
Let $K=[0,\overline{\Delta}]\subset\mathbb{R}$ be nonempty, compact, and convex, and let $\Theta\subset\mathbb{R}^d$ be convex.
Suppose $H:K\times\Theta\to\mathbb{R}$ is jointly convex and lower semicontinuous.\footnote{Lower semicontinuity is taken
with respect to the product topology on $K\times\Theta$.}
Define the marginal function
\[
V(\theta)\ :=\ \inf_{\Delta\in K} H(\Delta,\theta),\;\; \theta\in\Theta,
\]
and fix $\theta\in\Theta$. Assume the minimizer set is nonempty, and pick any
\[
\Delta^\star\in\arg\min_{\Delta\in K} H(\Delta,\theta).
\]
Assume moreover that there exists a joint subgradient $(u^\star,g^\star)\in\partial H(\Delta^\star,\theta)$ such that
\begin{equation}
\label{appendix:normal-cone-condition}
u^\star \in N_K(\Delta^\star),
\end{equation}
where $N_K(\Delta^\star)$ denotes the (convex-analytic) normal cone of $K$ at $\Delta^\star$, i.e.
\[
N_K(\Delta^\star)\ :=\ \left\{u\in\mathbb{R}:\ \langle u,\Delta-\Delta^\star\rangle \le 0\ \ \forall\,\Delta\in K\right\}.
\]
Then $g^\star\in\partial V(\theta)$. In particular, if $H(\Delta,\cdot)$ is differentiable at $\theta$, and if there exists $u^\star$ such that
$(u^\star,\nabla_\theta H(\Delta^\star,\theta))\in\partial H(\Delta^\star,\theta)$ and $u^\star\in N_K(\Delta^\star)$,
then
\[
\nabla_\theta H(\Delta^\star,\theta)\ \in\ \partial V(\theta).
\]
\end{lemma}

\begin{proof}
Fix $\theta\in\Theta$ and $\Delta^\star\in\arg\min_{\Delta\in K}H(\Delta,\theta)$. Let $(u^\star,g^\star)\in\partial H(\Delta^\star,\theta)$.
By the definition of the (convex) subdifferential on the product space $K\times\Theta$, for every $\Delta\in K$ and every
$\theta'\in\Theta$ we have the subgradient inequality
\begin{equation}
\label{appendix:joint-subgrad-ineq}
H(\Delta,\theta')\ \ge\ H(\Delta^\star,\theta)\ +\ \langle u^\star,\Delta-\Delta^\star\rangle\ +\ \langle g^\star,\theta'-\theta\rangle.
\end{equation}
Taking the infimum over $\Delta\in K$ on both sides of \eqref{appendix:joint-subgrad-ineq} yields
\[
\inf_{\Delta\in K}H(\Delta,\theta')\ \ge\ H(\Delta^\star,\theta)\ +\ \langle g^\star,\theta'-\theta\rangle\ +\ \inf_{\Delta\in K}\langle u^\star,\Delta-\Delta^\star\rangle.
\]
Since $u^\star\in N_K(\Delta^\star)$ by \eqref{appendix:normal-cone-condition}, we have $\langle u^\star,\Delta-\Delta^\star\rangle\le 0$ for all
$\Delta\in K$, and therefore
\[
\inf_{\Delta\in K}\langle u^\star,\Delta-\Delta^\star\rangle\ =\ 0,
\]
with equality achieved at $\Delta=\Delta^\star$. Using also $V(\theta)=\inf_{\Delta\in K}H(\Delta,\theta)=H(\Delta^\star,\theta)$ (since $\Delta^\star$
is a minimizer at $\theta$), we conclude that for every $\theta'\in\Theta$,
\[
V(\theta')\ \ge\ V(\theta)\ +\ \langle g^\star,\theta'-\theta\rangle,
\]
which is exactly the subgradient inequality $g^\star\in\partial V(\theta)$. For the final claim, if $H(\Delta,\cdot)$ is differentiable at $\theta$ and there exists $u^\star$ such that
$(u^\star,\nabla_\theta H(\Delta^\star,\theta))\in\partial H(\Delta^\star,\theta)$ with $u^\star\in N_K(\Delta^\star)$, then the argument above applies
with $g^\star=\nabla_\theta H(\Delta^\star,\theta)$, giving $\nabla_\theta H(\Delta^\star,\theta)\in\partial V(\theta)$.
\end{proof}

\begin{remark}
\label{appendix:normal-cone-condition-remark}
Condition \eqref{appendix:normal-cone-condition} is precisely the first-order optimality condition of $\Delta^\star$ for the constrained minimization
$\min_{\Delta\in K} H(\Delta,\theta)$, written in subdifferential form. Equivalently,
\[
0\ \in\ \partial_\Delta H(\Delta^\star,\theta)\ +\ N_K(\Delta^\star).
\]
In particular, if $\Delta^\star$ lies in the interior of $K$, then $N_K(\Delta^\star)=\{0\}$ and one may take $u^\star=0$.
If $\Delta^\star$ lies on the boundary, $N_K(\Delta^\star)$ is nontrivial and the term $u^\star\in N_K(\Delta^\star)$ is essential.
\end{remark}

\subsection{Technical results relating to the strong convexity of the empirical robust risk}

We will now prove a result that establishes that with enough samples, one can show that the empirical covariance denoted
\[
    \widehat{\Sigma}_{t} = \frac{1}{n} \sum_{i=1}^{n} \left( \psi(x_{t ,i}, a_{t, i}^{1}) -  \psi(x_{t ,i}, a_{t, i}^{2})  \right) \left( \psi(x_{t ,i}, a_{t, i}^{1}) -  \psi(x_{t ,i}, a_{t, i}^{2}) \right)^{\top},
\]
has a minimum eigenvalue of $\lambda/2$. With this, by defining the empirical (sample) robust REBEL loss as 
\[
    \ell_{\mathcal{D}_{t}}(\theta) = \frac{1}{n_t} \sum_{i=1}^{n_{t}} \ell_{t}(z_{i}; \theta) = \frac{1}{n_t} \sum_{i=1}^{n_{t}} \left( 
        \frac{1}{\eta} \left[
        \log \frac{\pi_\theta(a_{i, t}^{1} \mid x_{i, t})}{\pi_{\theta_{t}}(a_{i, t}^{1} \mid x)} -
        \log \frac{\pi_\theta(a_{i, t}^{2} \mid x_{i, t})}{\pi_{\theta_{t}}(a_{i, t}^{2} \mid x_{i, t})}
        \right]
        - \left[r(x_{i, t}, a_{i, t}^{1}) - r(x_{i, t}, a_{i, t}^{2})\right]
        \right)^2.
\]
We can show that $\sup_{P \in \mathcal{B}_{\mathcal{W}_{p}} \left( P_{n, t}^{\circ} ; \varepsilon \right)} \mathbb{E}_{z \sim P}[\ell_{\mathcal{D}_{t}}(\theta)]$ is $\kappa / \eta^2$-strongly convex with probability atleast $1 - \delta$ for any $\delta > 0$.

\vspace{1em}

\begin{lemma}[Empirical covariance matrix satisfies data coverage]
    \label{appendix:empirical-covariance-matrix-data-coverage}
    Define the empirical covariance matrix as 
    \[
    \widehat{\Sigma}_{t} = \frac{1}{n_{t}} \sum_{i=1}^{n_{t}} \left( \psi(x_{t ,i}, a_{t, i}^{1}) -  \psi(x_{t ,i}, a_{t, i}^{2})  \right) \left( \psi(x_{t ,i}, a_{t, i}^{1}) -  \psi(x_{t ,i}, a_{t, i}^{2}) \right)^{\top},
    \]
    where $x_{t, i} \sim \rho(x_{t, i})$ and $a_{t, i}^{1}, a_{t, i}^{2} \sim \pi_{\theta_{t}}(\cdot \mid x_{t, i})$. Then if $n_{t} \geq \mathcal{O}(\log(dT / \delta) / \kappa)$ samples where $\kappa > 0$ comes from Assumption \ref{assumption:data-coverage}, we have $\lambda_{\mathrm{min}}(\widehat{\Sigma}_{t}) \geq \kappa / 2$.
\end{lemma}

\begin{proof}
    The proof is a rather straightforward application of a matrix Chernoff concentration inequality (Theorem \ref{appendix:matrix-chernoff}). First notice that by Assumption \ref{assumption:log-linear-policy}, we have
    \begin{align*}
        \norm{( \psi(x_{t ,i}, a_{t, i}^{1}) -  \psi(x_{t ,i}, a_{t, i}^{2}) ) ( \psi(x_{t ,i}, a_{t, i}^{1}) -  \psi(x_{t ,i}, a_{t, i}^{2}) )^{\top}}_{\mathrm{op}} &= \norm{\psi(x_{t ,i}, a_{t, i}^{1}) -  \psi(x_{t ,i}, a_{t, i}^{2})}_{2}^{2} \leq 4.
    \end{align*}
    Thus if we let $\Psi_{t, i} = \left( \psi(x_{t ,i}, a_{t, i}^{1}) -  \psi(x_{t ,i}, a_{t, i}^{2})  \right) \left( \psi(x_{t ,i}, a_{t, i}^{1}) -  \psi(x_{t ,i}, a_{t, i}^{2}) \right)^{\top}$, we have $\lambda_{\mathrm{max}}(\Psi_{t, i}) \leq 4$. Additionally, notice that 
    \begin{align*}
        \mu_{\mathrm{min}} = \lambda_{\mathrm{min}} \left( \sum_{i=1}^{n_{t}} \mathbb{E}_{z \sim P}[\Psi_{t, i}] \right) \geq n_{t}\kappa,
    \end{align*}
    for any $P \in \mathcal{B}_{D}(P_{t}^{\circ} ; \varepsilon)$. Thus by matrix Chernoff, we have
    \begin{align*}
        \Pr \left( \lambda_{\mathrm{min}} \left( \sum_{i=1}^{n_{t}} \Psi_{t, i} \right) \leq \frac{\kappa n_{t}}{2} \right) &\leq d \left( \frac{e^{-1/2}}{(1/2)^{1/2}} \right)^{\kappa n_{t}/4} \\
        &\leq d\exp \left( -\frac{\kappa n_{t}}{8} \right).
    \end{align*}
    Setting the left right side to $\delta_{t}$ and solving for $n_{t}$ yields $n_{t} \geq \mathcal{O}(\log(d / \delta_t) / \kappa)$. By union bound over $t \in [T]$, we have for all $t \in [T]$, with probability atleast $1 - \delta$, if we collect more than $n_{t} \geq \mathcal{O}(\log(dT/ \delta) / \kappa)$ samples, $\lambda_{\mathrm{min}}(\widehat{\Sigma}_{t}) \geq \kappa / 2$.
\end{proof}

We can now prove $\widehat{h}_{t}(\theta;P) = \mathbb{E}_{z \sim P} \left[ \ell_{\mathcal{D}_{t}}(\theta) \right]$ is strongly convex for any $P$. Denote $\mathcal{E}_{t}$ denote the event that we collect sufficient samples for Lemma \ref{appendix:empirical-covariance-matrix-data-coverage} holds. Then conditional on $\mathcal{E}_{t}$, we have the following:

\vspace{1em}

\begin{lemma}[Strong convexity of $\widehat{h}$]
\label{appendix:strong-convexity-of-empirical-h}
    Let $\ell_{\mathcal{D}_{t}}(\theta)$ be the empirical (sample) robust REBEL loss. Assume that Assumption \ref{assumption:data-coverage} holds. Then conditional on $\mathcal{E}_{t}$, $\widehat{h}_{t}(\theta;P) = \mathbb{E}_{z \sim P} \left[ \ell_{\mathcal{D}_{t}}(\theta) \right]$ is $2/\eta^{2}$-strongly convex with respect to norm $||\cdot||_{\hat{\Sigma}_{P, t}}$ where $$\widehat{\Sigma}_{P, t} = \frac{1}{n_{t}} \sum_{i=1}^{n_{t}} \left( \psi(x_{t ,i}, a_{t, i}^{1}) -  \psi(x_{t ,i}, a_{t, i}^{2})  \right) \left( \psi(x_{t ,i}, a_{t, i}^{1}) -  \psi(x_{t ,i}, a_{t, i}^{2}) \right)^{\top}.$$
\end{lemma}

\begin{proof}[Proof Sketch]
    This proof is exactly the same as Lemma \ref{appendix:strong-convexity-of-h} except instead of using linear approximation theorem (Lemma \ref{appendix:beck-linear-approx}) on $\ell_{t}$, we use it on $\ell_{\mathcal{D}_{t}}(\theta)$. For brevity, we omit the details.
\end{proof}

Likewise we can prove strong convexity of $\sup_{P \in \mathcal{B}_{\mathcal{W}_{p}} \left( P_{n, t}^{\circ} ; \varepsilon \right)} \mathbb{E}_{z \sim P}[\ell_{\mathcal{D}_{t}}(\theta)]$:

\vspace{1em}

\begin{lemma}[Strong convexity of empirical robust risk]
    \label{appendix:strong-convexity-of-empirical-risk}
    Let $\ell_{\mathcal{D}_{t}}(\theta)$ be the empirical (sample) robust REBEL loss. Then $$\widehat{\mathcal{L}}^{\mathcal{W}_{p}}\left(\theta; \varepsilon\right) = \sup_{P \in \mathcal{B}_{\mathcal{W}_{p}} \left( P_{n, t}^{\circ} ; \varepsilon \right)} \mathbb{E}_{z \sim P}[\ell_{\mathcal{D}_{t}}(\theta)],$$ is $\kappa / \eta^{2}$-strongly convex with respect to Euclidean norm $||\cdot||_{2}$ where $\kappa$ is the regularity 
\end{lemma}

\begin{proof}[Proof Sketch]
    This proof is exactly the same as Lemma \ref{appendix:strong-convexity-of-wasserstein-loss}. For brevity, we omit the details.
\end{proof}

We now prove the key inequality that will allow to us to recover $\widetilde{\mathcal{O}}_{P}(d/n)$ rates. 

\vspace{1em}

\begin{lemma}[Reduction to gradient concentration]
    \label{appendix:reduction-to-gradient-concentration-wasserstein}
    If $\hat{\mathcal{L}}_{t}$ is $\hat{\mu}_{t}$-strongly convex in $\norm{\cdot}_{2}$ and $\widehat{\theta}_{t} \in \argmin_{\theta \in \Theta} \hat{\mathcal{L}}_{t}(\theta)$ and $\theta^{*}_{t} \in \argmin_{\theta \in \Theta} \mathcal{L}_{t}(\theta)$, then for $\widehat{g} \in \nabla_{\theta}\widehat{\mathcal{L}}_{t}(\theta_{t}^{*})$
    \[
        \norm{\widehat{\theta}_{t} - \theta_{t}^{*}}_{2} \leq \frac{1}{\widehat{\mu}_{t}} \norm{\widehat{g}}_{2}.
    \]
\end{lemma}

\begin{proof} 
    Conditional on $\mathcal{E}_{t}$, we have by Lemma \ref{appendix:strong-convexity-of-empirical-risk} that for $\widehat{g} \in \nabla_{\theta}\widehat{\mathcal{L}}_{t}(\theta_{t}^{*})$
    \begin{align*}
        \widehat{\mathcal{L}}_{t}(\widehat{\theta}_{t}) &\geq \widehat{\mathcal{L}}_{t}(\theta_{t}^{*}) + \langle \widehat{g},  \widehat{\theta}_{t} - \theta_{t}^{*} \rangle + \frac{\widehat{\mu}_{t}}{2} \norm{\widehat{\theta}_{t} - \theta_{t}^{*}}_{2}^2 \\
        &= \widehat{\mathcal{L}}_{t}(\theta_{t}^{*}) - \widehat{\mathcal{L}}_{t}(\widehat{\theta}_{t}) + \widehat{\mathcal{L}}_{t}(\widehat{\theta}_{t}) + \langle \widehat{g},  \widehat{\theta}_{t} - \theta_{t}^{*} \rangle + \frac{\widehat{\mu}_{t}}{2} \norm{\widehat{\theta}_{t} - \theta_{t}^{*}}_{2}^2 \\
        &\geq \widehat{\mathcal{L}}_{t}(\widehat{\theta}_{t}) + \langle \widehat{g},  \widehat{\theta}_{t} - \theta_{t}^{*} \rangle + \frac{\widehat{\mu}_{t}}{2} \norm{\widehat{\theta}_{t} - \theta_{t}^{*}}_{2}^2,
    \end{align*}
    where the second inequality holds by definition that $\widehat{\mathcal{L}}_{t}(\widehat{\theta}_{t}) \leq \widehat{\mathcal{L}}_{t}(\theta_{t}^{*})$. Rearranging terms and using Cauchy-Schwarz on the inner-product, we conclude.
\end{proof}

We can finally establish the $\widetilde{\mathcal{O}}_{P}(d/n)$. Consider the dual of the empirical robust type-p Wasserstein risk
\[
    \widehat{\mathcal{L}}_{t}^{\mathcal{W}_{p}}(\theta ; \varepsilon) = \inf_{\Delta \in [0, \overline{
    \Delta}]} \left\{ \Delta \varepsilon^{p} + \frac{1}{n_{t}} \sum_{i=1}^{n_{t}} \sup_{z^{\prime} \in \mathcal{Z}} (\ell_{t}(z^{\prime} ; \theta) - \Delta d(z, z^{\prime})^{p})  \right\}.
\]
Applying Lemma \ref{appendix:measurable-argmax-selection} gives us a measurable minimizer $z^{\sharp}(\Delta, \theta, z) \in \mathcal{M}(\Delta, \theta, z)$. Likewise, applying Lemma \ref{appendix:measurable-selection-of-delta-wasserstein} gives us a measurable minimizer $\widehat{\Delta}(\theta) \in \argmin_{\Delta \in K}H(\Delta, \theta)$. We can now apply Lemma \ref{appendix:marginal-subgradient-rule} with $$H(\Delta, \theta) = \Delta \varepsilon^{p} + \frac{1}{n_{t}} \sum_{i=1}^{n_{t}} \sup_{z^{\prime} \in \mathcal{Z}} (\ell_{t}(z^{\prime} ; \theta) - \Delta d(z, z^{\prime})^{p}).$$ 
For fixed $z$, $(\Delta, \theta) \mapsto \ell_{t}(z^{\prime}; \theta) - \Delta d^{p}(z, z^{\prime})^{p}$ is jointly convex. Choosing the probability measure $P_{n, t}^{\circ}$ and using the fact that if $g_{\Delta}(\theta ; z) \in \partial_{\theta} \Phi(\Delta, \theta; z)$, then $\mathbb{E}_{Z \sim P_{n, t}^{\circ}}[g_{\Delta}(\theta ; Z)] \in \partial_{\theta}\Phi_{P}(\Delta, \theta)$, we have that 
\[
    \widehat{g}(\theta) = \frac{1}{n_{t}} \sum_{i=1}^{n_{t}} \nabla_{\theta}\ell_{t}(z^{\sharp}(\widehat{\Delta}(\theta), \theta; Z_{i}); \theta) \in \partial_{\theta} \widehat{\mathcal{L}}_{t}^{\mathcal{W}_{p}}(\theta ; \varepsilon).
\]
The reason we go through these arguments is to rigorously ensure that we can turn the dual DRO objective into an well-defined measurable object we can treat like an ordinary empirical risk. Thus it is suffices for us to bound $\frac{1}{n} \sum_{i=1}^{n} \nabla_{\theta}\ell_{t}(z^{\sharp}(\widehat{\Delta}(\theta), \theta; Z_{i}); \theta)$ at $\theta = \theta_{t}^{\mathcal{W}_{p}}$ and conclude with Lemma \ref{appendix:reduction-to-gradient-concentration-wasserstein}.

\subsection{Proving conditional uniform concentration over $[0, \overline{\Delta}]$ and the main claim}

One thing that we must note is the measurability of $\widehat{\Delta}(\theta_{t}^{\mathcal{W}_{p}})$. Recall our definitions as follows: 
Let $\theta_{t}^{\mathcal{W}_{p}}$ denote the true population minimizer $\mathrm{arg}\min_{\theta \in \Theta} \mathcal{L}_{t}^{\mathcal{W}_{p}}\left(\theta; \varepsilon\right)$ and $\hat{\theta}_{n, t}^{\mathcal{W}_{p}}$ denote the empirical minimizer $\mathrm{arg}\min_{\theta \in \Theta} \mathcal{L}_{n, t}^{\mathcal{W}_{p}}\left(\theta; \varepsilon\right)$. Let $t \in \left\{ 0, \dots, T-1 \right\}$ and for each $t$, we collect a dataset $\mathcal{D}_{t} = \left\{ (x_{t, i}, y_{t, i}, y_{t, i}^{\prime}) \right\}_{i=1}^{n_{t}}$ with $x_{t, i} \sim \rho$, $y_{t, i}, y_{t, i}^{\prime} \stackrel{\mathrm{i.i.d}}{\sim} \pi_{\theta_{t-1}}(\cdot \mid x_{t, i})$. Define $\mathcal{F}_{t} = \sigma(\theta_{0}, \mathcal{D}_{0}, \dots, \theta_{t-1}, \mathcal{D}_{t-1})$ be the sigma-field containing everything revealed up to the start of iteration $t$
and $\overline{\mathcal{F}}_{t} = \sigma(\mathcal{F}_{t}, \mathcal{D}_{t})$. In particular, $\theta_{t}$ is $\mathcal{F}_{t}$ measurable. However note that $\widehat{\Delta}(\theta_{t}^{\mathcal{W}_{p}})$ is not $\mathcal{F}_{t}$ measurable as it depends on the entire batch of data from $\mathcal{D}_{t}$. Rather it is $\overline{\mathcal{F}}_{t}$ measurable. Therefore even though we avoid needing to do a covering argument over $\Theta$, we must still use a covering argument over $[0, \overline{\Delta}]$. The final result we will establish is a conditional uniform concentration bound over $[0, \overline{\Delta}]$ for $\nabla_{\theta} \widehat{\mathcal{L}}_{t}^{\mathcal{W}_{p}}(\theta ; \varepsilon)$. 

Before we do this, we must make one more assumption that controls the complexity of the gradients we are dealing with in $\Delta$.

\vspace{1em}

\begin{assumption}[H\"older continuity in $\Delta$ and induced $L_2$ entropy for fixed $p$]
\label{appendix:assumption-holder-delta-entropy}
Fix $p\ge 1$, an iteration $t$, and a parameter $\theta\in\Theta$. For each coordinate $j\in[d]$, define the
$\Delta$-indexed scalar function class
\[
\mathcal{G}_{t,j}^{(p)}(\theta)
:=\Big\{\, z\mapsto (g_{\Delta,t}^{(p)}(\theta;z))_j \ :\ \Delta\in[0,\overline\Delta_p] \Big\}
\subset L_2(Q),
\]
where $Q$ is any probability measure on $(\mathcal Z,\mathcal B(\mathcal Z))$.
Assume there exist constants $\alpha_p\in(0,1]$ and $L_{\Delta,p}<\infty$ such that for every $Q$, every
$j\in[d]$, and all $\Delta,\Delta'\in[0,\overline\Delta_p]$,
\begin{equation}
\label{appendix:holder-modulus-delta}
\big\|(g_{\Delta,t}^{(p)}(\theta;\cdot))_j-(g_{\Delta',t}^{(p)}(\theta;\cdot))_j\big\|_{L_2(Q)}
\ \le\ L_{\Delta,p}\,|\Delta-\Delta'|^{\alpha_p}.
\end{equation}
Moreover, assume an envelope bound: for some $G_p>0$,
\[
\sup_{\Delta\in[0,\overline\Delta_p]}\big\|(g_{\Delta,t}^{(p)}(\theta;\cdot))_j\big\|_{L_2(Q)}\le G_p,
\qquad \forall Q,\ \forall j\in[d].
\]
Then for every $\epsilon\in(0,1]$, every probability measure $Q$, and every $j\in[d]$,
\begin{equation}
\label{appendix:entropy-from-holder}
\log \mathcal N_{\epsilon G_p}\!\big(\mathcal G_{t,j}^{(p)}(\theta);L_2(Q)\big)
\ \le\
\frac{1}{\alpha_p}\log \left( 1 + \frac{L_{\Delta,p}\,\overline{\Delta}_p^{\alpha_p}}{\epsilon G_p} \right).
\end{equation}
\end{assumption}

\begin{remark} Although the prompt and completion spaces are large, each round operates on a finite logged batch $\mathcal{D}_t$, so the empirical nominal measure $\mathbb{P}^{\circ}_{n,t}$ is supported on $n_t$ atoms. Standard robustifications in off-policy learning therefore reduce the inner adversary to a finite-dimensional problem over reweightings or couplings on the logged samples—the same discretization used throughout contextual-bandit DRO theory, whether for Wasserstein policy evaluation under discrete contexts and exploration conditions \cite{shen2024wasserstein_opl} or for duality-based offline formulations around a logged law \cite{si2020dro_offline_cb,kallus2022dr_dro_opl}. Under this finite support, the $\Delta$-indexed class ${\ell_{\Delta,t}(\cdot;\theta)}_{\Delta\in[0,\bar\Delta]}$ has metric entropy logarithmic in $1/\varepsilon$ once a modulus of continuity holds, so uniform concentration over $\Delta$ follows without requiring the measurable selector $z^\star(z,\Delta,\theta)$ to be single-valued or Lipschitz. Covering conditions of this kind are standard in contextual-bandit learning theory beyond finite policy classes \cite{bibaut2020nonparametric_cb}. \end{remark}

We now proceed with completing the argument.

\vspace{1em}

\begin{lemma}[Uniform conditional concentration over $ \lbrack 0, \overline{\Delta} \rbrack$]
\label{appendix:uniform-conditional-concentration}
    Fix $\theta \in \Theta$. Assume $\norm{g_{\Delta, t}(\theta)}_{2} \leq G_{p}$. Conditional on $\mathcal{F}_{t}$, for every $t \in [T]$, $\delta \in (0, 1)$, with probability atleast $1 - \delta$, we have
    \[
\sup_{\Delta \in [0, \overline{\Delta}_p]}
\norm{(P_{n, t}^{\circ}  - P_{t}^{\circ})g_{\Delta, t}^{(p)}(\theta)}_{2}
\le
C G_p
\sqrt{
\frac{
d \left(
\log \frac{2dT}{\delta}
+
\frac{1}{\alpha_p}\log\!\left(1 + \frac{L_{\Delta,p}\overline{\Delta}_p^{\alpha_p}}{G_p}\right)
\right)
}{
n_t
}
},
\]
    for some universal constant $C > 0$.
\end{lemma}

\begin{proof}
    First note that for any vector $u \in \mathbb{R}^{d}$, we have $\norm{u}_{2} \leq \sqrt{d} \max_{j \in [d]} \abs{u_{j}}$. Therefore we have
    \begin{align*}
        \sup_{\Delta \in [0, \overline{\Delta}]} \norm{(P_{n, t}^{\circ}  - P_{t}^{\circ})g_{\Delta, t}(\theta)}_{2} \leq \sup_{\Delta \in [0, \overline{\Delta}]} \max_{j \in [d]} \sqrt{d} \abs{(P_{n, t}^{\circ}  - P_{t}^{\circ})g_{\Delta, t}(\theta)}.
    \end{align*}
    Thus it suffices to bound the empirical process. Fix a coordinate $j$. Then by Theorem \ref{theorem:uniform-rademacher-deviation} (Uniform concentration via Rademacher complexity), Corollary \ref{cor:dudley} (Dudley's entropy integral), and Assumption \ref{appendix:assumption-holder-delta-entropy} ($\Delta$ entropy condition), with probability atleast $1 - \delta$ for $\delta > 0$
    \begin{align*}
        \sup_{f \in \mathcal{G}_{t, j}} \abs{(P_{n, t}^{\circ}  - P_{t}^{\circ})f} &\leq 2\mathfrak{R}_{n, t}(\mathcal{G}_{t, j}^{(p)}) + 4G_{p} \sqrt{\frac{2 \log(2/\delta)}{n_{t}}} \\
        &\leq \frac{24}{\sqrt{n_{t}}} \int_{0}^{2G} \sqrt{\log \mathcal{N}_{u}(\mathcal{G}_{t, j}^{(p)} ; L_{2}(P_{n, t}^{\circ}))}du  + 4G_{p} \sqrt{\frac{2 \log(2/\delta)}{n_{t}}} \\
        &\leq CG_{p} \sqrt{\frac{\frac{1}{\alpha_{p}} \log(1 + L_{\Delta, p}\overline{\Delta}_{p}^{\alpha_{p}} / G_{p} )}{n_{t}}} + 4G_{p} \sqrt{\frac{2 \log(2/\delta)}{n_{t}}}.
    \end{align*}
    Now we can simply union bound over $j \in [d]$ and $t \in [T]$ with $\delta = \delta / dT$ (abusing notation), we conclude by using the inequality $\sqrt{a} + \sqrt{b} \lesssim \sqrt{a + b}$
    \[
        \sup_{\Delta \in [0, \overline{\Delta}]} \norm{(P_{n, t}^{\circ}  - P_{t}^{\circ})g_{\Delta, t}(\theta)}_{2} \leq CG_{p} \sqrt{\frac{d \left(\log (2dT/\delta) + \frac{1}{\alpha_{p}} \log(1 + L_{\Delta}\overline{\Delta}_{p}^{\alpha_{p}} / G_{p} \right)}{n_{t}}}.
    \]
\end{proof}

We can now conclude the proof.

\vspace{1em}

\begin{theorem}["Fast-Rate" Wasserstein Parameter Convergence]
    Let $\delta \in (0,1)$. Then for all $t \in [T]$, provided that $n_{t} \gtrsim \mathcal{O}(\log(dT/\delta)/\kappa)$, with probability at least $1-\delta$, 
\[
    ||\theta_{t}^{\mathcal{W}_{p}} - \hat{\theta}_{n, t}^{\mathcal{W}_{p}}||_{2} \lesssim \frac{K_{g}\eta}{\kappa}\sqrt{\frac{d \left(\log (2dT/\delta) + \frac{1}{\alpha_{p}} \log(1 + L_{\Delta, p}\overline{\Delta}_{p}^{\alpha_{p}} / G_{p} \right)}{n_{t}}}.
\]
where $\kappa$ is from the regularity condition in Assumption \ref{assumption:data-coverage} and $K_{g} = 8B/\eta + 2F$ where $B$ is from the assumption that the policy parameter set is bounded in Assumption \ref{assumption:log-linear-policy}, $F$ is from Assumption \ref{assumption:linear-reward-class}, $T$ is the number of iterations we run REBEL, and $G, L_{\Delta}$ are constants from Assumption \ref{appendix:assumption-holder-delta-entropy}.
\end{theorem}

\begin{proof}
    Suppose $n_{t} \gtrsim \mathcal{O}(\log(dT/\delta)/\kappa)$. Then by Lemma \ref{appendix:empirical-covariance-matrix-data-coverage} and Lemma \ref{appendix:reduction-to-gradient-concentration-wasserstein}, we have for $\widehat{g} \in \partial_{\theta}\widehat{\mathcal{L}}_{t}(\theta_{t}^{\mathcal{W}_{p}})$, $g \in \partial_{\theta}\mathcal{L}_{t}(\theta_{t}^{\mathcal{W}_{p}})$
    \begin{align*}
        ||\theta_{t}^{\mathcal{W}_{p}} - \hat{\theta}_{n, t}^{\mathcal{W}_{p}}||_{2} &\leq \frac{\eta^2}{\kappa} \norm{\widehat{g}}_{2} \leq \frac{\eta^2}{\kappa} \sup_{\Delta \in [0, \overline{\Delta}]} \norm{\widehat{g} - g}_{2},
    \end{align*}
    where the last inequality holds since $\theta_{t}^{\mathcal{W}_{p}} \in \mathrm{arg}\min_{\theta \in \Theta} \mathcal{L}_{t}^{\mathcal{W}_{p}}\left(\theta; \varepsilon\right)$ implies $0 \in \partial_{\theta}\mathcal{L}_{t}(\theta_{t}^{\mathcal{W}_{p}})$. Applying Lemma \ref{appendix:uniform-conditional-concentration} and using the fact that $G$ is upper bounded by Lemma \ref{appendix:lipschitz-bound-l}, we can conclude.
\end{proof}

\section{Proof of "Slow Rate" KL-DRO-REBEL}
\label{appendix:Slow-KL-DRO-REBEL}
Before we prove the necessary results to obtain the "slow rate" for KL-DRO-REBEL, we need to make an assumption on the loss functions $\ell(\cdot; \theta)$, $\theta \in \Theta$. Note that this assumption is only used to prove the dual reformulation of the KL-DRO-REBEL objective.

\vspace{1em}

\begin{assumption}
\label{Assumption:KL-DRO-Assumption}
We assume that $\ell(z; \theta) \leq L$ for all $\theta \in \Theta$. That is, the loss function is upper bounded by $L$. In addition, we also assume that $\Theta$ permits a uniform upper bound on $\lambda_\theta$. That is, we assume that
\[
\sup_{\theta \in \Theta} \lambda_\theta < \bar{\lambda}.
\]
\end{assumption}

We state the following dual reformulation result. The proof of this reformulation can be found in \cite{xu2025distributionallyrobustdirectpreference}, Appendix C.

\vspace{1em}

\begin{lemma}[Dual reformulation of KL-DRO-REBEL]
\label{appendix:strong-duality-KL}
Let $\ell(z; \theta)$ be the REBEL loss. The KL-DRO-REBEL loss function admits the following dual reformulation:
\[
\mathcal{L}^{\mathrm{KL}}\left(\theta; \varepsilon\right) = \sup_{\mathbb{P} \in \mathcal{B}_{\varepsilon} \left( \mathbb{P}^{\circ} ; \mathrm{KL} \right)} \mathbb{E}_{z \sim \mathbb{P}} \left[ \ell(z;\theta) \right]
= \inf_{\lambda \in [\underline{\lambda}, \bar{\lambda}]} 
\left\{ 
\lambda \varepsilon + \lambda \log \left( \mathbb{E}_{z \sim \mathbb{P}^\circ} \left[ \exp\left( \frac{\ell(z; \theta)}{\lambda} \right) \right] \right)
\right\},
\]
where $0 < \underline{\lambda} < \bar{\lambda} < \infty$ are constants.
\end{lemma}

\subsection{Proof of Strong  Convexity of KL-DRO-REBEL}

We will now establish the strong convexity of $\mathcal{L}_{t}^{\mathrm{KL}}\left(\theta; \varepsilon\right) = \sup_{\mathbb{P} \in \mathcal{B}_{\varepsilon} \left( P_{t}^{\circ} ; \mathrm{KL} \right)} \mathbb{E}_{z \sim \mathbb{P}} \left[ \ell_{t}(z;\theta) \right]$. This proof is the same as Lemma \ref{appendix:strong-convexity-of-wasserstein-loss}.

\vspace{1em}

\begin{lemma}[Strong convexity of $\mathcal{L}_{t}^{\mathrm{KL}}$]
\label{appendix:strong-convexity-of-KL-loss}
Let $\ell_{t}(z;\theta)$ be the REBEL loss function. Then $\mathcal{L}_{t}^{\mathrm{KL}}\left(\theta; \varepsilon\right) = \sup_{\mathbb{P} \in \mathcal{B}_{\varepsilon} \left( P_{t}^{\circ} ; \mathrm{KL} \right)} \mathbb{E}_{z \sim \mathbb{P}} \left[ \ell_{t}(z;\theta) \right]$ is $2\kappa / \eta^{2}$-strongly convex with respect to Euclidean norm $||\cdot||_{2}$ where $\lambda$ is the regularity parameter from Assumption \ref{assumption:data-coverage}.
\end{lemma}
\begin{proof}[Proof of Lemma \ref{appendix:strong-convexity-of-KL-loss}]
In Lemma \ref{appendix:strong-convexity-of-h}, we proved strong convexity of $h_t$. By Lemma \ref{appendix:beck-strong-convexity}, for $\theta, \theta^{\prime} \in \Theta$ and $\alpha \in [0, 1]$, this is equivalent to
\[
    h_t(\alpha\theta + (1-\alpha)\theta^{\prime};\mathbb{P}) \leq \alpha h_t(\theta;\mathbb{P}) + (1-\alpha)h_t(\theta^{\prime};\mathbb{P}) - \frac{\mu}{2} \alpha(1-\alpha)\|\theta - \theta^{\prime}\|_{\Sigma_{\mathbb{P},t}}^{2}.
\]
Taking the supremum over $\mathbb{P}$ preserves the convex combination and the negative quadratic term so we get 
\begin{align*}
    \mathcal{L}_{t}^{\mathrm{KL}}\left(\alpha\theta + (1-\alpha)\theta^{\prime}; \varepsilon\right) 
    &= \sup_{\mathbb{P} \in \mathcal{B}_{\mathrm{KL}} \left( P_{t}^{\circ} ; \varepsilon \right)} h_t(\alpha\theta + (1-\alpha)\theta^{\prime};\mathbb{P}) \\
    &\leq \sup_{\mathbb{P} \in \mathcal{B}_{\mathrm{KL}} \left( P_{t}^{\circ} ; \varepsilon \right)} \left[  \alpha h_t(\theta;\mathbb{P}) + (1-\alpha)h_t(\theta^{\prime};\mathbb{P}) - \frac{\mu}{2} \alpha(1-\alpha)\|\theta - \theta^{\prime}\|_{\Sigma_{\mathbb{P},t}}^{2} \right] \\
    &\leq \alpha \mathcal{L}_{t}^{\mathrm{KL}} \left( \theta; \varepsilon \right) + (1-\alpha) \mathcal{L}_{t}^{\mathrm{KL}} \left( \theta^{\prime}; \varepsilon \right) - \frac{\mu}{2} \alpha(1-\alpha) \inf_{\mathbb{P} \in \mathcal{B}_{\mathrm{KL}} \left( P_{t}^{\circ} ; \varepsilon \right)}\|\theta - \theta^{\prime}\|_{\Sigma_{\mathbb{P},t}}^{2} \\
    &\leq \alpha \mathcal{L}_{t}^{\mathrm{KL}} \left( \theta; \varepsilon \right) + (1-\alpha) \mathcal{L}_{t}^{\mathrm{KL}} \left( \theta^{\prime}; \varepsilon \right) - \frac{\mu}{2} \alpha(1-\alpha) \inf_{\mathbb{P} \in \mathcal{B}_{\mathrm{KL}} \left( P_{t}^{\circ} ; \varepsilon \right)} \lambda_{\mathrm{min}} \left( \Sigma_{\mathbb{P},t} \right) \|\theta - \theta^{\prime}\|_{2}^{2} \\
    &\leq \alpha \mathcal{L}_{t}^{\mathrm{KL}} \left( \theta; \varepsilon \right) + (1-\alpha) \mathcal{L}_{t}^{\mathrm{KL}} \left( \theta^{\prime}; \varepsilon \right) - \frac{\mu \kappa}{2} \alpha(1-\alpha)  \|\theta - \theta^{\prime}\|_{2}^{2}.
\end{align*}
where the second inequality holds from $\sup_{x} \left( f(x) + g(x) \right) \leq \sup_{x} f(x) + \sup_{x}g(x)$, the third inequality holds by the fact that $\Sigma_{\mathbb{P},t} \succeq \lambda_{\mathrm{min}}\left( \Sigma_{\mathbb{P},t} \right)I$, and the last inequality holds from Assumption \ref{assumption:data-coverage}. Thus we conclude that $\mathcal{L}_{t}^{\mathrm{KL}}$ is $\mu\kappa$-strongly convex in the $\|\cdot\|_{2}$ norm.
\end{proof}

\subsection{Proof of Slow Parameter Estimation Rate of KL-DRO-REBEL}

We are now ready to prove the "slow rate" estimation error of KL-DRO-REBEL.

\begin{proof}[Proof of Theorem \ref{theorem:KL-DRO-REBEL}]
\label{appendix:proof-slow-rate-KL}
Let $t \in \left\{ 0, \dots, T-1 \right\}$ and for each $t$, we collect a dataset $\mathcal{D}_{t} = \left\{ (x_{t, i}, y_{t, i}, y_{t, i}^{\prime}) \right\}_{i=1}^{n_{t}}$ with $x_{t, i} \sim \rho$, $y_{t, i}, y_{t, i}^{\prime} \stackrel{\mathrm{i.i.d}}{\sim} \pi_{\theta_{t-1}}(\cdot \mid x_{t, i})$. Define $\mathcal{F}_{t} = \sigma(\theta_{0}, \mathcal{D}_{0}, \dots, \theta_{t-1}, \mathcal{D}_{t-1})$ be the sigma-field containing everything revealed up to the start of iteration $t$ and $\overline{\mathcal{F}}_{t} = \sigma(\mathcal{F}_{t}, \mathcal{D}_{t})$. In particular, $\theta_{t}$ is $\mathcal{F}_{t}$ measurable. By the strong duality result for KL-DRO [\ref{appendix:strong-duality-KL}], for fixed $\theta$ we have
\[
\mathcal{L}_{t}^{\mathrm{KL}}(\theta;\varepsilon)
=
\inf_{\lambda\in[\underline{\lambda},\bar{\lambda}]}
\Big\{
\lambda\varepsilon + \lambda\log\big(\mathbb{E}_{P_t^\circ}[\varphi_t(Z,\lambda;\theta)]\big)
\Big\},
\;\;
\varphi_t(z,\lambda;\theta):=\exp\Big(\frac{\ell_t(z;\theta)}{\lambda}\Big).
\]
Consider the difference between the population and empirical KL DRO losses: 
\begin{align*}
\big|\mathcal{L}_{t}^{\mathrm{KL}}(\theta;\varepsilon)-\mathcal{L}^{\mathrm{KL}}_{n,t}(\theta;\varepsilon)\big|
&=
\Big|
\inf_{\lambda} \{\lambda\varepsilon + \lambda\log(\mathbb{E}_{P_t^\circ}[\varphi_t])\}
-
\inf_{\lambda} \{\lambda\varepsilon + \lambda\log(\mathbb{E}_{\mathbb{P}_{n,t}^\circ}[\varphi_t])\}
\Big|\\
&\le
\sup_{\lambda\in[\underline{\lambda},\bar{\lambda}]}
\lambda\Big|
\log\big(\mathbb{E}_{\mathbb{P}_{n,t}^\circ}[\varphi_t(Z,\lambda;\theta)]\big)
-
\log\big(\mathbb{E}_{P_{t}^\circ}[\varphi_t(Z,\lambda;\theta)]\big)
\Big|.
\end{align*}
Next, since $\ell_t\ge 0$, we have $\varphi_t(z,\lambda;\theta)\ge 1$ for all $z,\lambda,\theta$ and hence
\[
\mathbb{E}_{P_{t}^\circ}[\varphi_t(Z,\lambda;\theta)]\ge 1,
\;\;
\mathbb{E}_{\mathbb{P}_{n,t}^\circ}[\varphi_t(Z,\lambda;\theta)]\ge 1.
\]
Because $|\frac{d}{dx}\log x| = 1/x \le 1$ on $[1,\infty)$, the map $\log(\cdot)$ is $1$-Lipschitz on $[1,\infty)$.
Therefore, for every $\lambda\in[\underline{\lambda},\bar{\lambda}]$,
\[
\Big|
\log(\mathbb{E}_{\mathbb{P}_{n,t}^\circ}[\varphi_t])-\log(\mathbb{E}_{P_{t}^\circ}[\varphi_t])
\Big|
\le
\Big|
\mathbb{E}_{\mathbb{P}_{n,t}^\circ}[\varphi_t]-\mathbb{E}_{P_{t}^\circ}[\varphi_t]
\Big|.
\]
Thus
\begin{equation}
\label{eq:KL-reduction}
\big|\mathcal{L}_{t}^{\mathrm{KL}}(\theta;\varepsilon)-\mathcal{L}^{\mathrm{KL}}_{n,t}(\theta;\varepsilon)\big|
\le
\bar{\lambda}\cdot
\sup_{\lambda\in[\underline{\lambda},\bar{\lambda}]}
\Big|
\mathbb{E}_{\mathbb{P}_{n,t}^\circ}[\varphi_t(Z,\lambda;\theta)]-\mathbb{E}_{P_{t}^\circ}[\varphi_t(Z,\lambda;\theta)]
\Big|.
\end{equation}
Define
\[
F(\lambda)
:=
\mathbb{E}_{\mathbb{P}_{n,t}^\circ}\!\big[\varphi_t(Z,\lambda;\theta)\big]
-
\mathbb{E}_{P_{t}^\circ}\!\big[\varphi_t(Z,\lambda;\theta)\big],
\;\;
\lambda\in[\underline{\lambda},\bar{\lambda}].
\]
Since $0\le \ell_t(z;\theta)\le K_\ell$ and $\lambda\ge \underline{\lambda}$,
\[
1 \le \varphi_t(z,\lambda;\theta) \le \exp\Big(\frac{K_\ell}{\underline{\lambda}}\Big)=:\mathcal{J}_{\max}.
\]
Hence $\varphi_t(\cdot,\lambda;\theta)\in[1,\mathcal{J}_{\max}]$ and the range satisfies
$\mathcal{J}_{\max}-1 \le \mathcal{J}_{\max}$. For fixed $(z,\theta)$, the map $\lambda\mapsto \varphi_t(z,\lambda;\theta)=\exp(\ell_t(z;\theta)/\lambda)$ is differentiable on
$[\underline{\lambda},\bar{\lambda}]$ with
\[
\Big|\nabla_{\lambda} \varphi_t(z,\lambda;\theta)\Big|
=
\exp\Big(\frac{\ell_t(z;\theta)}{\lambda}\Big)\cdot\frac{\ell_t(z;\theta)}{\lambda^2}
\le
\exp\Big(\frac{K_\ell}{\underline{\lambda}}\Big)\cdot\frac{K_\ell}{\underline{\lambda}^2}
=:L_{\lambda}.
\]
Thus, for all $\lambda,\lambda'\in[\underline{\lambda},\bar{\lambda}]$ and all $z$,
\[
|\varphi_t(z,\lambda;\theta)-\varphi_t(z,\lambda';\theta)|\le L_{\lambda}\,|\lambda-\lambda'|.
\]
Consequently,
\begin{align}
\label{eq:F-Lipschitz}
|F(\lambda)-F(\lambda')|
\le
\mathbb{E}_{\mathbb{P}_{n,t}^\circ}\!\big[|\varphi_t(Z,\lambda;\theta)-\varphi_t(Z,\lambda';\theta)|\big]
+
\mathbb{E}_{P_{t}^\circ}\!\big[|\varphi_t(Z,\lambda;\theta)-\varphi_t(Z,\lambda';\theta)|\big]\notag
\le 2L_{\lambda}\,|\lambda-\lambda'|.
\end{align}
Again note that in this case, $\lambda^{*}$, the optimal dual variable for the dual problem, is a function of the data drawn from $P_{t}^{\circ}$. Therefore it is not $\mathcal{F}_{t}$ measurable and we must use a covering argument. Let $\mathcal{N}^{\lambda}_{\alpha}$ be an $\alpha$-net of $[0,\bar{\lambda}]$ in absolute value.
Its cardinality satisfies
\[
|\mathcal{N}^{\lambda}_{\alpha}|
\le
1+\frac{\bar{\lambda}}{\alpha}.
\]
For any $\lambda\in[0,\bar{\lambda}]$ pick $\lambda^\sharp\in\mathcal{N}^{\lambda}_{\alpha}$ such that
$|\lambda-\lambda^\sharp|\le \alpha$. Then we have
\[
|F(\lambda)|\le |F(\lambda^\sharp)| + |F(\lambda)-F(\lambda^\sharp)|
\le
\max_{\lambda^\sharp\in\mathcal{N}^{\lambda}_{\alpha}}|F(\lambda^\sharp)|
+ 2L_{\lambda}\alpha,
\]
and therefore
\begin{equation}
\label{eq:supF-net}
\sup_{\lambda\in[0,\bar{\lambda}]}|F(\lambda)|
\le
\max_{\lambda^\sharp\in\mathcal{N}^{\lambda}_{\alpha}}|F(\lambda^\sharp)|
+ 2L_{\lambda}\alpha.
\end{equation}

Fix any $\lambda^\sharp\in\mathcal{N}^{\lambda}_{\alpha}$. Conditional on $\overline{\mathcal{F}}_t$, the variables $\varphi_t(Z_{t,i},\lambda^\sharp;\theta)$ are i.i.d. in $[1,\mathcal{J}_{\max}]$, hence conditional Hoeffding gives
for any $\epsilon>0$:
\[
\Pr\Big(|F(\lambda^\sharp)|\ge \epsilon \mid \overline{\mathcal{F}}_t \Big)
\le
2\exp\Big(-\frac{2n_t\epsilon^2}{(\mathcal{J}_{\max}-1)^2}\Big)
\le
2\exp\Big(-\frac{2n_t\epsilon^2}{\mathcal{J}_{\max}^2}\Big).
\]
Union bound over $\lambda^\sharp\in\mathcal{N}^{\lambda}_{\alpha}$ yields
\[
\mathbb{P}\Big(\max_{\lambda^\sharp\in\mathcal{N}^{\lambda}_{\alpha}}|F(\lambda^\sharp)|\ge \epsilon \mid \mathcal{F}_t\Big)
\le
2|\mathcal{N}^{\lambda}_{\alpha}|\exp\Big(-\frac{2n_t\epsilon^2}{\mathcal{J}_{\max}^2}\Big).
\]
Setting the right-hand side to $\delta_t$ and solving for $\epsilon$ gives that with conditional probability at least $1-\delta_t$,
\begin{equation}
\label{eq:maxF-eps}
\max_{\lambda^\sharp\in\mathcal{N}^{\lambda}_{\alpha}}|F(\lambda^\sharp)|
\le
\mathcal{J}_{\max}\sqrt{\frac{\log\big(2|\mathcal{N}^{\lambda}_{\alpha}|/\delta_t\big)}{2n_t}}.
\end{equation}
With conditional probability at least $1-\delta_t$,
\begin{equation}
\label{eq:supF-final}
\sup_{\lambda\in[\underline{\lambda},\bar{\lambda}]}|F(\lambda)|
\le
\mathcal{J}_{\max}\sqrt{\frac{\log\big(2|\mathcal{N}^{\lambda}_{\alpha}|/\delta_t\big)}{2n_t}}
+2L_{\lambda}\alpha.
\end{equation}
Choose $\alpha := \mathcal{J}_{\max} / (2L_{\lambda}\sqrt{n_t})$
so that $2L_{\lambda}\alpha = \mathcal{J}_{\max}/\sqrt{n_t}$ and
\[
\log|\mathcal{N}^{\lambda}_{\alpha}|
\le
\log\!\Big(1+(\bar{\lambda}-\underline{\lambda})\cdot\frac{2L_{\lambda}\sqrt{n_t}}{\mathcal{J}_{\max}}\Big)
\lesssim \log n_t + \log(\bar{\lambda}),
\]
where the implicit constant depends only on $K_\ell,\underline{\lambda}$ through $L_{\lambda}$ and $\mathcal{J}_{\max}$. Plugging \eqref{eq:supF-final} into \eqref{eq:KL-reduction}, we conclude that with conditional probability at least $1-\delta_t$,
for every fixed $\theta\in\Theta$,
\[
\big|\mathcal{L}_{t}^{\mathrm{KL}}(\theta;\varepsilon)-\mathcal{L}^{\mathrm{KL}}_{n,t}(\theta;\varepsilon)\big|
\le
\bar{\lambda}\Big[
\mathcal{J}_{\max}\sqrt{\frac{\log\big(2|\mathcal{N}^{\lambda}_{\alpha}|/\delta_t\big)}{2n_t}}
+2L_{\lambda}\alpha
\Big]
\;\lesssim\;
\bar{\lambda}\mathcal{J}_{\max}\sqrt{\frac{\log n_t + \log(\bar{\lambda})+\log(1/\delta_t)}{n_t}}.
\]

We now establish that both the population and empirical KL-DRO losses are Lipschitz in $\theta$. Fix any $\lambda\in[\underline{\lambda},\bar{\lambda}]$ and define
\[
\phi_t(\theta,\lambda)
:=
\lambda\varepsilon + \lambda\log\!\Big(\mathbb{E}_{Z\sim\mathbb{P}^\circ_t}\big[\varphi_t(Z,\lambda;\theta)\big]\Big),
\;\;
\phi_{n,t}(\theta,\lambda)
:=
\lambda\varepsilon + \lambda\log\!\Big(\mathbb{E}_{Z\sim\mathbb{P}^\circ_{n,t}}\big[\varphi_t(Z,\lambda;\theta)\big]\Big).
\]
Then
\[
\mathcal{L}^{\mathrm{KL}}_{t}(\theta;\varepsilon)=\inf_{\lambda\in[\underline{\lambda},\bar{\lambda}]}\phi_t(\theta,\lambda),
\;\;
\mathcal{L}^{\mathrm{KL}}_{n,t}(\theta;\varepsilon)=\inf_{\lambda\in[\underline{\lambda},\bar{\lambda}]}\phi_{n,t}(\theta,\lambda).
\]

Let $\theta,\theta'\in\Theta$. Using $|\inf_x f(x)-\inf_x g(x)|\le \sup_x|f(x)-g(x)|$, we have
\begin{align*}
\big|\mathcal{L}^{\mathrm{KL}}_{t}(\theta;\varepsilon)-\mathcal{L}^{\mathrm{KL}}_{t}(\theta';\varepsilon)\big|
&\le
\sup_{\lambda\in[\underline{\lambda},\bar{\lambda}]}\big|\phi_t(\theta,\lambda)-\phi_t(\theta',\lambda)\big|=
\sup_{\lambda\in[\underline{\lambda},\bar{\lambda}]}
\lambda\left|
\log\!\Big(\mathbb{E}_{\mathbb{P}^\circ_t}[\varphi_t(Z,\lambda;\theta)]\Big)
-
\log\!\Big(\mathbb{E}_{\mathbb{P}^\circ_t}[\varphi_t(Z,\lambda;\theta')]\Big)
\right|.
\end{align*}
Since $\ell_t\ge 0$, we have $\varphi_t(\cdot,\lambda;\cdot)\ge 1$ and hence both expectations are at least $1$.
Thus $\log(\cdot)$ is $1$-Lipschitz on $[1,\infty)$, and so
\[
\left|
\log\!\Big(\mathbb{E}_{\mathbb{P}^\circ_t}[\varphi_t(\theta)]\Big)
-
\log\!\Big(\mathbb{E}_{\mathbb{P}^\circ_t}[\varphi_t(\theta')]\Big)
\right|
\le
\left|
\mathbb{E}_{\mathbb{P}^\circ_t}[\varphi_t(Z,\lambda;\theta)]
-
\mathbb{E}_{\mathbb{P}^\circ_t}[\varphi_t(Z,\lambda;\theta')]
\right|.
\]
Moreover, for each $z$ and $\lambda$,
\[
\big|\varphi_t(z,\lambda;\theta)-\varphi_t(z,\lambda;\theta')\big|
=
\left|\exp\Big(\frac{\ell_t(z;\theta)}{\lambda}\Big)-\exp\Big(\frac{\ell_t(z;\theta')}{\lambda}\Big)\right|
\le
\exp\Big(\frac{K_\ell}{\underline{\lambda}}\Big)\cdot \frac{1}{\underline{\lambda}}\,
\big|\ell_t(z;\theta)-\ell_t(z;\theta')\big|,
\]
where we used the mean value theorem and $\lambda\ge\underline{\lambda}$.
By Lemma \ref{appendix:lipschitz-bound-l}, $\ell_t(z;\theta)$ is $L_{K_g,\eta}$-Lipschitz in $\theta$, hence
\[
\big|\varphi_t(z,\lambda;\theta)-\varphi_t(z,\lambda;\theta')\big|
\le
\frac{\mathcal{J}_{\max}}{\underline{\lambda}}\,L_{K_g,\eta}\,\|\theta-\theta'\|_2,
\;\;
\mathcal{J}_{\max}:=\exp\Big(\frac{K_\ell}{\underline{\lambda}}\Big).
\]
Taking expectations and combining yields
\[
\big|\mathcal{L}^{\mathrm{KL}}_{t}(\theta;\varepsilon)-\mathcal{L}^{\mathrm{KL}}_{t}(\theta';\varepsilon)\big|
\le
\sup_{\lambda\in[\underline{\lambda},\bar{\lambda}]}
\lambda\cdot \frac{\mathcal{J}_{\max}}{\underline{\lambda}}\,L_{K_g,\eta}\,\|\theta-\theta'\|_2
\le
L_{\mathrm{KL},\theta}\,\|\theta-\theta'\|_2,
\]
where we define the (population) Lipschitz constant
\[
L_{\mathrm{KL},\theta}
:=
\frac{\bar{\lambda}\mathcal{J}_{\max}}{\underline{\lambda}}\,L_{K_g,\eta}.
\]
The same argument (replacing $P_t^\circ$ by $\mathbb{P}_{n,t}^\circ$) shows that
$\mathcal{L}^{\mathrm{KL}}_{n,t}(\cdot;\varepsilon)$ is also $L_{\mathrm{KL},\theta}$-Lipschitz on $\Theta$. Let $\mathcal{N}^{\theta}_{\alpha}$ be an $\alpha$-net of $\Theta$ under $\|\cdot\|_2$. Then $|\mathcal{N}^{\theta}_{\alpha}|\le \left(3B / \alpha \right)^d$.  For each $\theta^\sharp\in\mathcal{N}^{\theta}_{\alpha}$, apply the fixed-$\theta$ deviation bound obtained in Step 3
with failure probability $\delta_t/|\mathcal{N}^{\theta}_{\alpha}|$, and union bound over $\theta^\sharp$.
Thus, with conditional probability at least $1-\delta_t$,
\begin{equation}
\label{eq:KL-dev-net-theta}
\sup_{\theta^\sharp\in\mathcal{N}^{\theta}_{\alpha}}
\big|\mathcal{L}^{\mathrm{KL}}_{t}(\theta^\sharp;\varepsilon)-\mathcal{L}^{\mathrm{KL}}_{n,t}(\theta^\sharp;\varepsilon)\big|
\;\lesssim\;
\bar{\lambda}\mathcal{J}_{\max}\sqrt{\frac{\log n_t + \log(\bar{\lambda})
+\log\!\Big(\frac{|\mathcal{N}^{\theta}_{\alpha}|}{\delta_t}\Big)}{n_t}}.
\end{equation}
Now fix any $\theta\in\Theta$ and choose $\theta^\sharp\in\mathcal{N}^{\theta}_{\alpha}$ with $\|\theta-\theta^\sharp\|_2\le \alpha$.
Then using the Lipschitzness of both $\mathcal{L}_t^{\mathrm{KL}}$ and $\mathcal{L}^{\mathrm{KL}}_{n,t}$,
\begin{align*}
\big|\mathcal{L}^{\mathrm{KL}}_{t}(\theta;\varepsilon)-\mathcal{L}^{\mathrm{KL}}_{n,t}(\theta;\varepsilon)\big|
&\le
\big|\mathcal{L}^{\mathrm{KL}}_{t}(\theta;\varepsilon)-\mathcal{L}^{\mathrm{KL}}_{t}(\theta^\sharp;\varepsilon)\big|
+
\big|\mathcal{L}^{\mathrm{KL}}_{t}(\theta^\sharp;\varepsilon)-\mathcal{L}^{\mathrm{KL}}_{n,t}(\theta^\sharp;\varepsilon)\big|
+
\big|\mathcal{L}^{\mathrm{KL}}_{n,t}(\theta^\sharp;\varepsilon)-\mathcal{L}^{\mathrm{KL}}_{n,t}(\theta;\varepsilon)\big|\\
&\le
\sup_{\theta^\sharp\in\mathcal{N}^{\theta}_{\alpha}}
\big|\mathcal{L}^{\mathrm{KL}}_{t}(\theta^\sharp;\varepsilon)-\mathcal{L}^{\mathrm{KL}}_{n,t}(\theta^\sharp;\varepsilon)\big|
+2L_{\mathrm{KL},\theta}\alpha.
\end{align*}
Taking $\sup_{\theta\in\Theta}$ and combining with \eqref{eq:KL-dev-net-theta}, we obtain: with conditional probability at least $1-\delta_t$,
\begin{equation}
\label{eq:KL-dev-all-theta}
\sup_{\theta\in\Theta}
\big|\mathcal{L}^{\mathrm{KL}}_{t}(\theta;\varepsilon)-\mathcal{L}^{\mathrm{KL}}_{n,t}(\theta;\varepsilon)\big|
\;\lesssim\;
\bar{\lambda}\mathcal{J}_{\max}\sqrt{\frac{\log n_t + \log(\bar{\lambda})
+\log\!\Big(\frac{|\mathcal{N}^{\theta}_{\alpha}|}{\delta_t}\Big)}{n_t}}
\;+\;2L_{\mathrm{KL},\theta}\alpha.
\end{equation}

Choose
\[
\alpha := \frac{1}{\sqrt{n_t}}
\;\;\text{so that}\;\;
\log|\mathcal{N}^{\theta}_{\alpha}|
\le d\log(3B\sqrt{n_t})
\lesssim d\log n_t,
\;\; 
2L_{\mathrm{KL},\theta}\alpha \lesssim \frac{L_{\mathrm{KL},\theta}}{\sqrt{n_t}}.
\]
Plugging this into \eqref{eq:KL-dev-all-theta} yields with conditional probability at least $1-\delta_t$,
\begin{equation}
\label{eq:KL-dev-summary}
\sup_{\theta\in\Theta}
\big|\mathcal{L}^{\mathrm{KL}}_{t}(\theta;\varepsilon)-\mathcal{L}^{\mathrm{KL}}_{n,t}(\theta;\varepsilon)\big|
\;\lesssim\;
\bar{\lambda}\mathcal{J}_{\max}\sqrt{\frac{d\log n_t + \log (\overline{\lambda}) + \log(1/\delta_t)}{n_t}},
\end{equation}
where we absorbed constants depending only on $(B,K_\ell,\underline{\lambda},L_{K_g,\eta})$. Setting $\delta_t := \delta/T$ and apply a union bound over $t\in\{0,\dots,T-1\}$ to conclude that, with probability at least $1-\delta$,
for all $t\in[T]$ simultaneously,
\begin{equation}
\label{eq:KL-dev-all-t}
\sup_{\theta\in\Theta}
\big|\mathcal{L}^{\mathrm{KL}}_{t}(\theta;\varepsilon)-\mathcal{L}^{\mathrm{KL}}_{n,t}(\theta;\varepsilon)\big|
\;\lesssim\;
\bar{\lambda}\mathcal{J}_{\max}\sqrt{\frac{d\log n_t + \log (\overline{\lambda}) + \log(T/\delta)}{n_t}}.
\end{equation}
On the event in \eqref{eq:KL-dev-all-t}, for each fixed $t$ we have the standard decomposition
\begin{align*}
\mathcal{L}^{\mathrm{KL}}_{t}(\hat{\theta}^{\mathrm{KL}}_{n,t};\varepsilon)-\mathcal{L}^{\mathrm{KL}}_{t}(\theta_t^{\mathrm{KL}};\varepsilon)
&=
\underbrace{\mathcal{L}^{\mathrm{KL}}_{t}(\hat{\theta}^{\mathrm{KL}}_{n,t};\varepsilon)-\mathcal{L}^{\mathrm{KL}}_{n,t}(\hat{\theta}^{\mathrm{KL}}_{n,t};\varepsilon)}_{\le \sup_{\theta}|\mathcal{L}_t-\mathcal{L}_{n,t}|}
+
\underbrace{\mathcal{L}^{\mathrm{KL}}_{n,t}(\hat{\theta}^{\mathrm{KL}}_{n,t};\varepsilon)-\mathcal{L}^{\mathrm{KL}}_{n,t}(\theta_t^{\mathrm{KL}};\varepsilon)}_{\le 0}
+
\underbrace{\mathcal{L}^{\mathrm{KL}}_{n,t}(\theta_t^{\mathrm{KL}};\varepsilon)-\mathcal{L}^{\mathrm{KL}}_{t}(\theta_t^{\mathrm{KL}};\varepsilon)}_{\le \sup_{\theta}|\mathcal{L}_t-\mathcal{L}_{n,t}|}\\
&\le
2\sup_{\theta\in\Theta}\big|\mathcal{L}^{\mathrm{KL}}_{t}(\theta;\varepsilon)-\mathcal{L}^{\mathrm{KL}}_{n,t}(\theta;\varepsilon)\big|.
\end{align*}

Invoking Lemma \ref{appendix:strong-convexity-of-KL-loss}, we have for all $\theta\in\Theta$,
\[
\mathcal{L}^{\mathrm{KL}}_{t}(\theta;\varepsilon)-\mathcal{L}^{\mathrm{KL}}_{t}(\theta_t^{\mathrm{KL}};\varepsilon)
\ge
\frac{\lambda}{\eta^2}\|\theta-\theta_t^{\mathrm{KL}}\|_2^2.
\]
Conditional on $\mathcal{F}_{t}$, applying this inequality at $\theta=\hat{\theta}^{\mathrm{KL}}_{n,t}$ yields, with probability at least $1-\delta$, for all $t\in[T]$,
\[
\|\hat{\theta}^{\mathrm{KL}}_{n,t}-\theta_t^{\mathrm{KL}}\|_2^2
\le
\frac{\eta^2}{\lambda}
\Big(\mathcal{L}^{\mathrm{KL}}_{t}(\hat{\theta}^{\mathrm{KL}}_{n,t};\varepsilon)-\mathcal{L}^{\mathrm{KL}}_{t}(\theta_t^{\mathrm{KL}};\varepsilon)\Big)
\;\lesssim\;
\frac{\eta^2\bar{\lambda}\mathcal{J}_{\max}}{\kappa}\sqrt{\frac{d\log n_t + \log (\overline{\lambda}) + \log(T/\delta)}{n_t}}.
\]
This completes the proof.
\end{proof}

\section{Proof of "Fast Rate" KL-DRO-REBEL}
\label{appendix:proofs-fast-kl}
For proving the fast rate for REBEL with KL ambiguity sets, we will take a different approach. Recognizing from the proof of the slow rate that we simply have a single dual variable to deal with, we will not take a uniform gradient concentration via measurable selection route as it is unnecessary. Rather we will prove that the excess robust risk rate scales as $\widetilde{\mathcal{O}}(d/n)$ which will allow us to conclude the parameter error scales as $\widetilde{\mathcal{O}}(\sqrt{d/n})$ in parameter estimation error via strong convexity. 

\subsection{Preliminary technical results for Local Rademacher Complexity analysis}

We will prove this with a more refined analysis typically referred to as Local Rademacher Complexity \cite{Bartlett_2005} which allows us to prove faster concentration rates via Talagrand-Bousquet's inequality.

\vspace{1em}

\begin{theorem}[\cite{BOUSQUET2002495}]\label{thm:bousquet-2002-633}
Consider i.i.d.\ random variables $(Z_1,\dots,Z_n)\sim D^n$. Let $\zeta$ be a real-valued
function of $(Z_1,\dots,Z_n)$. Moreover, for each $k\in[n]$, let $\zeta_k$ be a real-valued
function of $(Z_1,\dots,Z_{k-1},Z_{k+1},\dots,Z_n)$ such that
\[
\sum_{k=1}^n \bigl[\zeta-\zeta_k\bigr]\le \zeta.
\]
Assume that for each $k$, there exists a function $\zeta_k'$ of $(Z_1,\dots,Z_n)$ such that
\[
\zeta_k' \le \zeta-\zeta_k \le M, 
\;\; 
\mathbb{E}_{Z_k}\,\zeta_k' \ge 0,
\;\; 
\zeta_k' \le uM.
\]
Then for all $t\ge 0$,
\[
\Pr\!\left[
\zeta \ge \mathbb{E}\zeta
+\sqrt{\,2\Bigl((1+u)M\,\mathbb{E}\zeta + n\sigma^2\Bigr)t\,}
+\frac{tM}{3}
\right]\le e^{-t},
\]
where
\[
\sigma^2 \;\ge\; \frac{1}{n}\sum_{k=1}^n \mathbb{E}_{Z_k}\bigl(\zeta_k'\bigr)^2.
\]
\end{theorem}

A simple corollary of this yields a (two-sided) uniform convergence result:

\vspace{1em}

\begin{corollary}[\cite{Zhang_2023}, Corollary 6.34]\label{cor:bousquet-634}
Consider a real-valued function class $\mathcal F=\{f:\mathcal Z\to\mathbb R\}$ and let $D$ be a
distribution on $\mathcal Z$. Assume that there exist constants $M>0$ and $\sigma>0$ such that,
for all $f\in\mathcal F$,
\[
\sigma^2 \;\ge\; \mathrm{Var}_{Z\sim D}\!\bigl[f(Z)\bigr],
\;\;\text{and}\;\;
\sup_{z'\in\mathcal Z}\Bigl[\mathbb E_{Z\sim D} f(Z) - f(z')\Bigr] \;\le\; M.
\]
Let $S_n=\{Z_1,\dots,Z_n\}$ be $n$ independent random variables from $D$. Then with probability
at least $1-\delta$ over $S_n$, for all $f\in\mathcal F$,
\begin{align*}
\mathbb E_{Z\sim D} f(Z) - \frac{1}{n}\sum_{i=1}^n f(Z_i)
&\le \epsilon_n(\mathcal F,D)
+ \sqrt{\frac{\bigl(4M\,\epsilon_n(\mathcal F,D)+2\sigma^2\bigr)\ln(1/\delta)}{n}}
+ \frac{M\ln(1/\delta)}{3n} \\
&\le 2\,\epsilon_n(\mathcal F,D)
+ \sqrt{\frac{2\sigma^2\ln(1/\delta)}{n}}
+ \frac{4M\ln(1/\delta)}{3n},
\end{align*}
where $\epsilon_n(\mathcal F,D)$ is the expected supremum of the corresponding empirical process. More explicitly, given an empirical process $\{\phi(w,S_n): w\in\Omega\}$ with $S_n\sim D^n$,
we define $\epsilon_n(\mathcal G,D)$ as 
\[
\epsilon_n(\mathcal G,D)
\;:=\;
\mathbb E_{S_n}\,\sup_{w\in\Omega}\Bigl[\phi(w,D)-\phi(w,S_n)\Bigr].
\]
\end{corollary}
 We will also need to introduce the following definition of a rate function:

\vspace{1em}

\begin{definition}\label{def:rate-function-localization-635}
Given a distribution $D$ and a function class $\mathcal F$, consider a localization function
$h:\mathcal F\to\mathbb R$ such that
\[
b_0 := \inf_{f\in\mathcal F} h(f) \;>\; -\infty.
\]
Define the localized function class, for all $b>b_0$, by
\[
\mathcal F_h(b) := \{\, f\in\mathcal F : h(f)\le b \,\}.
\]
For any $\alpha>0$, the rate function with respect to the localization $h$ is defined as
\[
\overline r^{\,h}_n(\alpha,\mathcal F,D)
\;:=\;
\sup\left\{
r:\;
r \le \inf_{\,r'>\max(r,\alpha b_0)}
\epsilon_n\!\left(\mathcal F_h\!\left(r' / \alpha \right),\, D\right)
\right\}.
\]
We note that the requirement $r'>\alpha b_0$ is only to ensure that
$\mathcal F_h(r'/\alpha)$ is always non-empty, and thus
$\epsilon_n(\mathcal F_h(r'/\alpha),D)$ is well-defined.
\end{definition}

We should note that the rate function is always non-negative as $\epsilon_n(\mathcal{F}_{h}(r^{\prime} / \lambda),D)$ is always non-negative. The usefulness of the rate function is based on the following result, which shows that $\epsilon_{n}$ can be upper bounded by this quantity:

\vspace{1em}

\begin{proposition}[\cite{Zhang_2023}, Proposition 6.38]\label{prop:localized-eps-bound-638}
For all $\alpha>0$ and all $b>\inf_{f\in\mathcal F} h(f)$, we have
\[
\epsilon_n \bigl(\mathcal F_h(b), D\bigr)
\;\le\;
\max\!\left(\overline r^{\,h}_n(\alpha,\mathcal F,D),\; \alpha b\right).
\]
\end{proposition}

By using Corollary \ref{cor:bousquet-634} and Proposition \ref{prop:localized-eps-bound-638} with a localization/peeling argument, one can deduce the following result provided the localization function $h(f)$ satisfies a Bernstein variance condition:

\vspace{1em}

\begin{theorem}[\cite{Zhang_2023}, Theorem 6.41]\label{thm:localized-fast-rate-641}
Consider a function class $\mathcal F$, a distribution $D$, and i.i.d.\ samples
$S_n=\{Z_1,\dots,Z_n\}\sim D^n$. For $f\in\mathcal F$, define
\[
f(D) := \mathbb E_{Z\sim D}\bigl[f(Z)\bigr],
\;\;
f(S_n) := \frac{1}{n}\sum_{i=1}^n f(Z_i).
\]
Assume that for all $f\in\mathcal F$,
\[
\mathrm{Var}_{D}\!\bigl[f(Z)\bigr] \;\le\; c_0^2 + c_1 h(f)
\]
for some constants $c_0,c_1\ge 0$ and some function $h(\cdot)\ge 0$. Assume also that
$\mathcal F$ is bounded in the sense that for all $f\in\mathcal F$,
\[
\sup_{z'\in\mathcal Z}\bigl[f(D) - f(z')\bigr] \;\le\; M.
\]
Then with probability at least $1-\delta$ over $S_n$, for all $f\in\mathcal F$ and all $\alpha>0$,
\[
f(D)
\;\le\;
f(S_n) + 5\alpha\,h(f) + 5r_0
+ \sqrt{\frac{2c_0^2\ln(1/\delta)}{n}}
+ \frac{(3c_1+4\alpha M)\ln(1/\delta)}{3\alpha n},
\]
where
\[
r_0
\;=\;
\overline r^{\,h}_n(\alpha,\mathcal F,D)
+ \alpha \inf_{f\in\mathcal F} h(f)
+ \sqrt{\frac{2c_0^2}{n}}
+ \frac{3c_1+4\alpha M}{3\alpha n}.
\]
\end{theorem}

Lastly we will require a result that shows that the rate function can be obtained from a uniform upper bound of the Rademacher complexity:

\vspace{1em}

\begin{proposition}[\cite{Zhang_2023}, Proposition 6.46]\label{prop:local-rad-fixed-point-646}
Consider a function class $\mathcal G := \{\phi(w,z): w\in\Omega\}$, 
with localization $h(w):=\phi(w,D)$ and $\inf_{w\in\Omega} h(w)=0$. Assume that
$|\phi(\cdot)|\le M$ and that the variance condition~(3.14) holds with $c_0=0$.
Assume further that for any $b>0$,
\[
\sup_{S_n}\,
\mathfrak R\!\left(
\Bigl\{\phi(w,\cdot): \frac{1}{n}\sum_{i=1}^n \phi(w,Z_i)^2 \le b\Bigr\},
\, S_n
\right)
\;\le\; r_n(b),
\]
where $r_n(b)$ is a continuous concave function of $b$. Let $\alpha \le 0.5\,c_1/M$ and define
\[
b_0 := \sup\Bigl\{ b>0:\; b \le \frac{4c_1}{\alpha}\,r_n(b)\Bigr\}.
\]
Then $\overline r^{\,h}_n(\alpha,\mathcal G,D)
\;\le\;
0.5\,\alpha b_0 / c_1$.
\end{proposition}

\subsection{Proof of Fast Rate}
Throughout this proof, we will have Assumption \ref{Assumption:KL-DRO-Assumption} hold. We will use these arguments to obtain a fast excess robust risk rate of $\widetilde{\mathcal{O}}(d/n)$. We first will recall the setup again for convenience. Let $t \in \left\{ 0, \dots, T-1 \right\}$ and for each $t$, we collect a dataset $\mathcal{D}_{t} = \left\{ (x_{t, i}, y_{t, i}, y_{t, i}^{\prime}) \right\}_{i=1}^{n_{t}}$ with $x_{t, i} \sim \rho$, $y_{t, i}, y_{t, i}^{\prime} \stackrel{\mathrm{i.i.d}}{\sim} \pi_{\theta_{t-1}}(\cdot \mid x_{t, i})$. Define $\mathcal{F}_{t} = \sigma(\theta_{0}, \mathcal{D}_{0}, \dots, \theta_{t-1}, \mathcal{D}_{t-1})$ be the sigma-field containing everything revealed up to the start of iteration $t$ and $\overline{\mathcal{F}}_{t} = \sigma(\mathcal{F}_{t}, \mathcal{D}_{t})$. In particular, $\theta_{t}$ is $\mathcal{F}_{t}$ measurable. 

\subsubsection{Part I: Establishing a basic inequality.}
First we will start by defining some notation. Let $\Lambda := [\underline{\lambda}, \overline{\lambda}]$ and define, for each $(\theta, \lambda)$,
\[
    \Psi_{t}(\theta, \lambda) := \lambda \varepsilon + \lambda \log \left( \mathbb{E}_{P_{t}^\circ}[\varphi_t(Z,\lambda;\theta)] \right), \quad \Psi_{n, t}(\theta, \lambda) := \lambda \varepsilon + \lambda \log \left( \mathbb{E}_{P_{n, t}^\circ}[\varphi_t(Z,\lambda;\theta)] \right)
\]
where $\varphi_t(z,\lambda;\theta):=\exp\Big(\ell_t(z;\theta) / \lambda \Big)$. Thus we have 
\[
    \mathcal{L}_{t}^{\mathrm{KL}}(\theta; \varepsilon) = \inf_{\lambda \in \Lambda} \Psi_{t}(\theta, \lambda), \quad \mathcal{L}_{n, t}^{\mathrm{KL}}(\theta; \varepsilon) = \inf_{\lambda \in \Lambda} \Psi_{n, t}(\theta, \lambda)
\]
by the strong duality result for KL-DRO [\ref{appendix:strong-duality-KL}]. 

First let $(\theta_{t}^{\mathrm{KL}}, \lambda^{\star}) \in \argmin_{(\theta, \lambda) \in \Theta \times \Lambda} \Psi_{t}(\theta, \lambda)$ and $(\widehat{\theta}_{n, t}^{\mathrm{KL}}, \widehat{\lambda}) \in \argmin_{(\theta, \lambda) \in \Theta \times \Lambda} \Psi_{n, t}(\theta, \lambda)$. Also define for each $\lambda \in \Lambda$, $\theta_{\lambda}^{\star} \in \argmin_{\theta \in \Theta} \Psi_{t}(\theta, \lambda)$ or equivalently $\theta_{\lambda}^{\star} \in \argmin_{\theta \in \Theta} P_{t}^{\circ} \varphi_{t}(z, \lambda; \theta)$. Notice $\theta_{\lambda}^{\star}$ is well-defined since by compactness of $\Theta$, $\theta \mapsto \ell_{t}(z; \theta)$ is continuous, and $\ell_{t}$ is uniformly bounded (Lemma \ref{appendix:uniform-bound-l}), then by Dominated Convergence Theorem and Weierstrass's Extreme Value Theorem (Theorem \ref{appendix:extreme-value-theorem}), we conclude $\theta_{\lambda}^{\star}$ exists for every $\lambda$. Then we have the following decomposition:
\begin{align*}
    \mathcal{L}_{t}^{\mathrm{KL}}(\widehat{\theta}_{n, t}^{\mathrm{KL}} ; \varepsilon) - \mathcal{L}_{t} (\theta_{t}^{\mathrm{KL}} ; \varepsilon) = \underbrace{\left[ \Psi_{t}(\widehat{\theta}_{n, t}^{\mathrm{KL}}, \widehat{\lambda}) - \Psi_{t}(\theta_{\widehat{\lambda}}^{\star}, \widehat{\lambda}) \right]}_{\textbf{$\theta$-suboptimality at random $\widehat{\lambda}$}} + \underbrace{\left[ \Psi_{t}(\theta_{\widehat{\lambda}}^{\star}, \widehat{\lambda}) - \Psi_{t}(\theta_{t}^{\mathrm{KL}}, \lambda^{\star}) \right]}_{ \textbf{$\lambda$-mismatch}}
\end{align*}

\subsubsection{Part II: Bounding $\theta$-suboptimality.}
We will first bound $\Psi_{t}(\widehat{\theta}_{n, t}^{\mathrm{KL}}, \widehat{\lambda}) - \Psi_{t}(\theta_{\widehat{\lambda}}^{\star}, \widehat{\lambda})$. First notice that since $\ell_t\ge 0$, we have $\varphi_t(z,\lambda;\theta)\ge 1$ for all $z,\lambda,\theta$ and hence
\[
\mathbb{E}_{P_{t}^\circ}[\varphi_t(Z,\lambda;\theta)]\ge 1,
\;\;
\mathbb{E}_{\mathbb{P}_{n,t}^\circ}[\varphi_t(Z,\lambda;\theta)]\ge 1.
\]
Since $|\frac{d}{dx}\log x| = 1/x \le 1$ on $[1,\infty)$, the map $\log(\cdot)$ is $1$-Lipschitz on $[1,\infty)$.
Therefore, for every $\lambda\in[\underline{\lambda},\bar{\lambda}]$,
\[
\abs{\log(\mathbb{E}_{\mathbb{P}_{n,t}^\circ}[\varphi_t])-\log(\mathbb{E}_{P_{t}^\circ}[\varphi_t])}
\le
\abs{\mathbb{E}_{\mathbb{P}_{n,t}^\circ}[\varphi_t]-\mathbb{E}_{P_{t}^\circ}[\varphi_t]}
\]
Using this we find and the fact that $\lambda^{*}, \widehat{\lambda} < \overline{\lambda}$
\begin{align*}
    \Psi_{t}(\widehat{\theta}_{n, t}^{\mathrm{KL}}, \widehat{\lambda}) - \Psi_{t}(\theta_{\widehat{\lambda}}^{\star}, \widehat{\lambda}) &= \widehat{\lambda} \left[ \log \left( \mathbb{E}_{P_{t}^\circ}[\varphi_t(Z,\widehat{\lambda};\widehat{\theta}_{n, t}^{\mathrm{KL}})] \right) - \log \left( \mathbb{E}_{P_{t}^\circ}[\varphi_t(Z,\widehat{\lambda};\theta_{\widehat{\lambda}}^{\star})] \right) \right] \\
    &\leq \widehat{\lambda} \abs{\log \left( \mathbb{E}_{P_{t}^\circ}[\varphi_t(Z,\widehat{\lambda};\widehat{\theta}_{n, t}^{\mathrm{KL}})] \right) - \log \left( \mathbb{E}_{P_{t}^\circ}[\varphi_t(Z,\widehat{\lambda};\theta_{\widehat{\lambda}}^{\star})] \right)} \\
    &\leq \overline{\lambda} \abs{P_{t}^{\circ} (\varphi_t(Z,\widehat{\lambda};\widehat{\theta}_{n, t}^{\mathrm{KL}}) -\varphi_t(Z,\widehat{\lambda};\theta_{\widehat{\lambda}}^{\star}) )}
\end{align*}
Thus it suffices to bound the empirical process. We will use a local Rademacher complexity argument here. First we will note that since $\ell_{t}(z; \theta) \in [0, K_{\ell}]$ by Lemma \ref{appendix:uniform-bound-l}, we have $\varphi_t(z,\lambda;\theta) \in [1, \mathcal{J}_{\mathrm{max}}]$ where $\mathcal{J}_{\mathrm{max}} := \exp\Big( K_\ell / \underline{\lambda}\Big)$. Define the shifted-and-scaled version of $\varphi_{t}$ as follows:
\[
    \widetilde{\varphi}_{t}(z, \lambda; \theta) = \frac{\varphi_t(z,\lambda;\theta) - 1}{\mathcal{J}_{\mathrm{max}} - 1} \in [0, 1]
\]
Define the centered "shifted-and-scaled" function class $$\Phi_{t} := \left\{ \widetilde{\varphi}_{t}(\cdot, \lambda; \theta): (\theta, \lambda) \in \Theta \times \Lambda  \right\}.$$ Let $\gamma_{t}(\cdot, \lambda; \theta) = \widetilde{\varphi}_{t}(\cdot, \lambda; \theta) - \widetilde{\varphi}_{t}(\cdot, \lambda; \theta_{\lambda}^{*})$. Then we can define the centered class as: $$\Gamma_{t} = \left\{ \gamma_{t}(\cdot, \lambda; \theta) : (\theta, \lambda) \in \Theta \times \Lambda \right\}.$$
Note that $\gamma_{t} \in [-1, 1]$ and by definition we have $P_{t}^{\circ}\gamma_{t}(\cdot, \lambda; \theta) \geq 0$. To apply Corollary \ref{cor:bousquet-634}, we must choose a localization function. We will take $h(\theta, \lambda) = P_{t}^{\circ} \gamma(\cdot, \lambda; \theta)$. With this choice, we clearly have $h(\theta_{\lambda}^{\star}, \lambda) = 0$ and thus $\inf_{(\theta, \lambda) \in \Theta \times \Lambda} h(\theta, \lambda) = 0$. Additionally notice that since $\gamma_{t} \in [-1, 1]$, $\abs{\gamma_t} \leq 1$ and so we can take $M=1$. What remains is to show our choice of localization function satisfies a Bernstein variance condition. We will need the following:

\vspace{1em}

\begin{claim}
\label{claim:strong-convexity-of-tilted-loss}
 The population map $\theta \mapsto P_{t}^{\circ} \widetilde{\varphi}_{t}(\cdot, \cdot; \theta)$ is $\mu$-strongly convex w.r.t $\norm{\cdot}_{\Sigma_{P, t}}$ with
 \[
    \mu = \frac{2}{\eta^2 \overline{\lambda} (\mathcal{J}_{\max} - 1)}
 \]
\end{claim}

\begin{proof}[Proof of Claim \ref{claim:strong-convexity-of-tilted-loss}]
    For each $z$, notice that
    \begin{align*}
        \nabla_{\theta}^2 \varphi_{t}(z, \lambda; \theta) &= \varphi_{t}(z, \lambda; \theta) \left[ \frac{1}{\lambda} \nabla_{\theta}^2 \ell_{t}(z; \theta) + \frac{1}{\lambda^2} \nabla_{\theta} \ell_{t}(z; \theta) \nabla_{\theta} \ell_{t}(z; \theta)^{\top} \right] \\
        &\succeq \frac{\varphi_{t}(z, \lambda; \theta)}{\lambda} \nabla_{\theta}^2 \ell_{t}(z; \theta) \\
        &\succeq \frac{1}{\overline{\lambda}}\nabla_{\theta}^2 \ell_{t}(z; \theta),
    \end{align*}
    where the first inequality holds since $\nabla_{\theta} \ell_{t}(z; \theta) \nabla_{\theta} \ell_{t}(z; \theta)^{\top} \succeq 0$ and the second inequality holds from $\varphi_{t} \geq 1$ and Assumption \ref{Assumption:KL-DRO-Assumption}. From  Lemma \ref{appendix:strong-convexity-of-h}, we find that 
    \begin{align*}
        \nabla_{\theta}^{2} \ell(z;\theta) &= \frac{2}{\eta^{2}} \left( \psi(x, a) - \psi(x, a^{\prime}) \right) \left( \psi(x, a) - \psi(x, a^{\prime}) \right)^{\top} = \frac{2}{\eta^{2}}\Sigma_z.
    \end{align*}
    Thus taking expectations with respect to the nominal data-generating distribution, using Assumption \ref{assumption:data-coverage} and scaling by $(\mathcal{J}_{\max} - 1)^{-1}$, we can conclude. 
\end{proof}

Now we will show $\widetilde{\varphi}_{t}(z, \lambda; \theta)$ is Lipschitz in $\theta$:

\begin{align*}
\|\widetilde{\varphi}_{t}(\cdot,\lambda;\theta)-\widetilde{\varphi}_{t}(\cdot,\lambda;\theta')\|_{n,t}
&\le \|\widetilde{\varphi}_{t}(\cdot,\lambda;\theta)-\widetilde{\varphi}_{t}(\cdot,\lambda;\theta')\|_{\infty,t}\\
&= \max_{i\in[n_t]}\big|\widetilde{\varphi}_{t}(Z_i,\lambda;\theta)-\widetilde{\varphi}_{t}(Z_i,\lambda;\theta')\big|\\
&\le \frac{\mathcal{J}_{\max}}{\mathcal{J}_{\max}-1}\cdot\frac{1}{\underline{\lambda}}
\max_{i\in[n_t]}|\ell_t(Z_i;\theta)-\ell_t(Z_i;\theta')|\\
&\le \frac{\mathcal{J}_{\max}}{\mathcal{J}_{\max}-1}\cdot\frac{1}{\underline{\lambda}}\cdot\frac{4K_g}{\eta}\,\|\theta-\theta'\|_2,
\end{align*}

where the first inequality holds by Lemma \ref{appendix:uniform-bound-l}, mean value theorem applied to $f(u) = e^{u}$, and $\lambda\ge\underline{\lambda}$. The last inequality holds from Lemma \ref{appendix:lipschitz-bound-l}. Furthermore note the identity using the Rayleigh quotient bound:
\[
    \norm{\theta - \theta^{\prime}}_{2}^2 \leq \frac{1}{\lambda_{\min}(\Sigma)} \norm{\theta - \theta^{\prime}}_{\Sigma}^2
\]

Using Lipschitzness of $\widetilde{\varphi}_{t}$ in $\theta$ and this identity coupled with Claim \ref{claim:strong-convexity-of-tilted-loss}, we find
\begin{align*}
    \mathrm{Var}\!\left(\gamma_{t}(\cdot,\lambda;\theta)\right)
    &\leq P_{t}^{\circ}\gamma_{t}^{2}(\cdot,\lambda;\theta) \\
    &\leq \frac{16K_{g}^{2}\,\mathcal{J}_{\max}^{2}}%
               {\eta^{2}\,\underline{\lambda}^{2}\,(\mathcal{J}_{\max}-1)^{2}}
          \,\|\theta-\theta_{\lambda}^{*}\|_{2}^{2} \\
    &\leq \frac{16K_{g}^{2}\,\mathcal{J}_{\max}^{2}}%
               {\eta^{2}\,\underline{\lambda}^{2}\,(\mathcal{J}_{\max}-1)^{2}}
          \cdot\frac{2}{\mu\kappa}
          \Bigl(
            P_{t}^{\circ}\widetilde{\varphi}_{t}(\cdot,\lambda;\theta)
            - P_{t}^{\circ}\widetilde{\varphi}_{t}(\cdot,\lambda;\theta_{\lambda}^{*})
          \Bigr)\\
    &= \frac{16K_{g}^{2}\,\overline{\lambda}\,\mathcal{J}_{\max}^{2}}%
            {\underline{\lambda}^{2}\,\kappa^{2}\,(\mathcal{J}_{\max}-1)}
       \,P_{t}^{\circ}\gamma_{t}(\cdot,\lambda;\theta) \\
    &=: V_{\!\mathcal{J}_{\max},\overline{\lambda},\kappa,K_{g}}\,
        P_{t}^{\circ}\gamma_{t}(\cdot,\lambda;\theta).
\end{align*}
Thus $\gamma_t$ satisfies the Bernstein condition with $c_{0} = 0, c_{1} = V_{\mathcal{J}_{\max}, \underline{\lambda}, \kappa, K_{g}}$. By Corollary \ref{cor:bousquet-634}, taking $c_0 = 0, c_{1} = V_{\mathcal{J}_{\max}, \underline{\lambda}, \kappa, K_{g}}, M=2$, with probability atleast $1-\delta$, $\forall \gamma_{t} \in \Gamma_{t}$ and $\forall \alpha > 0$:
\begin{align*}
    (1-5\alpha)P_{t}^{\circ} \gamma_{t}(z, \lambda; \theta) \leq P_{n, t}^{\circ} \gamma_{t}(z, \lambda; \theta) + 5 \overline{r}_{n}^{h}(\alpha, \Gamma_{t}, \mathcal{D}_t) + \frac{5(3V_{\mathcal{J}_{\max}, \underline{\lambda}, \kappa, K_{g}} + 8\alpha) \log(1/\delta)}{3\alpha n}
\end{align*}
It suffices to compute the rate $\overline{r}_{n}^{h}(\alpha, \Gamma_{t}, \mathcal{D}_t)$. First, we will show $\widetilde{\varphi}_{t}(z, \lambda; \theta)$ is Lipschitz in $(\theta, \lambda)$. We already showed it is Lipschitz in $\theta$ in Claim \ref{claim:strong-convexity-of-tilted-loss}. We now show this in $\lambda$. For fixed $(z,\theta)$, the map $\lambda\mapsto \varphi_t(z,\lambda;\theta)=\exp(\ell_t(z;\theta)/\lambda)$ is differentiable on
$[\underline{\lambda},\bar{\lambda}]$ with
\begin{align*}
    \Big|\nabla_{\lambda} \varphi_t(z,\lambda;\theta)\Big|
    &=
\exp\Big(\frac{\ell_t(z;\theta)}{\lambda}\Big)\cdot\frac{\ell_t(z;\theta)}{\lambda^2} \le
\frac{K_\ell \mathcal{J}_{\max}}{\underline{\lambda}^2}.
\end{align*}
Thus, for all $\lambda,\lambda'\in[\underline{\lambda},\bar{\lambda}]$ and all $z$, we can conclude with
\begin{align*}
     \norm{\widetilde{\varphi}_{t}(z, \lambda; \theta) - \widetilde{\varphi}_{t}(z, \lambda^{\prime}; \theta)}_{n,t} &= \frac{1}{\mathcal{J}_{\max} - 1}\norm{\varphi_{t}(z, \lambda; \theta) - \varphi_{t}(z, \lambda^{\prime}; \theta)}_{n,t} \leq \frac{K_\ell \mathcal{J}_{\max}}{\mathcal{J}_{\max} - 1} \cdot \frac{1}{\underline{\lambda}^2} |\lambda-\lambda'|.
\end{align*}
Thus we have
\begin{align*}
    \norm{\widetilde{\varphi}_{t}(z, \lambda; \theta) - \widetilde{\varphi}_{t}(z, \lambda^{\prime}; \theta^{\prime})}_{n,t} &\leq \frac{\mathcal{J}_{\max}}{\mathcal{J}_{\max} - 1} \cdot \frac{1}{\underline{\lambda}} \cdot \frac{4K_{g}}{\eta} \norm{\theta - \theta^{\prime}}_{2} + \frac{K_\ell \mathcal{J}_{\max}}{\mathcal{J}_{\max} - 1} \cdot \frac{1}{\underline{\lambda}^2} |\lambda-\lambda'| \\
    &:= L_{\mathcal{J}_{\max}, \underline{\lambda}, K_{g}}\norm{\theta - \theta^{\prime}}_{2} + L_{\mathcal{J}_{\max}, \underline{\lambda}, K_{\ell}}|\lambda-\lambda'|.
\end{align*}

Fix a sample $\mathcal{S}_{n}$. Consider the localized function class $$\Gamma_{t}^{b}(\mathcal{S}_{n}) = \left\{ \gamma_t(\cdot, \lambda; \theta) : P_{n, t}^{\circ} \gamma_t^2(\cdot, \lambda; \theta) \leq b \right\} \subseteq \Phi_{t}^{b}(\mathcal{S}_{n}) - \Phi_{t}^{b}(\mathcal{S}_{n}).$$ Let $\{ \theta_{j} \}_{j=1}^{\mathcal{N}_{\theta}}$ be a $\rho/2L_{\mathcal{J}_{\max}, \underline{\lambda}, K_{g}}$-cover of $(\Theta, \norm{\cdot}_{2})$ and $\{ \lambda_{k} \}_{k=1}^{\mathcal{N}_{\lambda}}$ be a $\rho / 2L_{\mathcal{J}_{\max}, \underline{\lambda}, K_{\ell}}$-cover of $(\Lambda, \abs{\cdot})$. Consider the product set $\mathcal{N}_{\theta, \lambda} = \left\{ u_{\theta_{j}, \lambda_{k}} : 1 \leq j \leq \mathcal{N}_{\theta}, \; 1 \leq k \leq \mathcal{N}_{\lambda} \right\}$. Then for any $(\theta, \lambda)$, choose $j, k$ such that $\norm{\theta - \theta_j}_{2} \leq \rho/2L_{\mathcal{J}_{\max}, \underline{\lambda}, K_{g}}$ and $\abs{\lambda - \lambda_{k}} \leq \rho / 2L_{\mathcal{J}_{\max}, \underline{\lambda}, K_{\ell}}$. Then by showing that $\widetilde{\varphi}_{t}(z, \lambda; \theta)$ is Lipschitz in $(\theta, \lambda)$, we have
\[
    \norm{u_{\theta, \lambda} - u_{\theta_j, \lambda_k}}_{n,t} \leq \rho
\]

Thus we have that $\mathcal{N}_{\theta, \lambda}$ is a $\rho$-cover of $\Phi_{t}^{b}(\mathcal{S}_{n})$ in $L_{2}(P_{n, t}^{\circ})$. Hence
\[
    \mathcal{N}_{\rho}(\Phi_{t}^{b}(\mathcal{S}_{n}) ;L_{2}(P_{n, t}^{\circ})) \leq \mathcal{N}_{\rho/2L_{\mathcal{J}_{\max}, \underline{\lambda}, K_{g}}} \left( \Theta ; \norm{\cdot}_{2} \right) \cdot \mathcal{N}_{\rho / 2L_{\mathcal{J}_{\max}, \underline{\lambda}, K_{\ell}}} \left( \Lambda ; \abs{\cdot}\right)
\]
Since $\Gamma_{t}^{b} \subseteq \Phi_{t}^b - \Phi_{t}^b$, we have that $\log \mathcal{N}_{\rho}(\Gamma_{t}^{b}(\mathcal{S}_{n}); L_{2}(P_{n, t}^{\circ})) \leq 2 \log \mathcal{N}_{\rho/2} (\Phi_{t}^{b}(\mathcal{S}_{n}); L_{2}(P_{n, t}^{\circ}))$. Standard volumetric bounds give us 
\begin{align*}
    \log \mathcal{N}_{\rho}(\Gamma_{t}^{b}(\mathcal{S}_{n}) ;L_{2}(P_{n, t}^{\circ})) &\leq 2\log \mathcal{N}_{\rho/4L_{\mathcal{J}_{\max}, \underline{\lambda}, K_{g}}} \left( \Theta ; \norm{\cdot}_{2} \right) + 2\log \mathcal{N}_{\rho / 4L_{\mathcal{J}_{\max}, \underline{\lambda}, K_{\ell}}} \left( [\underline{\lambda}, \overline{\lambda}] ; \abs{\cdot}\right) \\
    &\le 2d \log \left( \frac{4 L_{\mathcal{J}_{\max}, \underline{\lambda}, K_{g}} B}{\rho}\right) + 2\log \left( \frac{4L_{\mathcal{J}_{\max}, \underline{\lambda}, K_{\ell}}\overline{\lambda}}{\rho}\right) \\
    &\leq 2d \log \left( \frac{4L_{\mathcal{J}_{\max}, \underline{\lambda}, K_{g}} B + 4L_{\mathcal{J}_{\max}, \underline{\lambda}, K_{\ell}}\overline{\lambda}}{\rho}  \right).
\end{align*}
Then by Corollary \ref{cor:dudley}, we can conclude that the Rademacher complexity of the localized function class is as follows:
\begin{align*}
    \mathfrak{R}_{\mathcal{S}_{n}} \left( \Phi_{t}^{b}(\mathcal{S}_{n}) \right) &\lesssim \frac{1}{\sqrt{n}}  \int_{0}^{\mathrm{diam}(\Phi_{t}^{b}(\mathcal{S}_{n}))} \sqrt{\log \mathcal{N}_{\rho}(\mathcal{F}_{b}(\mathcal{S}_{n}) ;L_{2}(P_{n, t}^{\circ}))} d\rho \\
    &\leq \frac{1}{\sqrt{n}}  \int_{0}^{2\sqrt{b}} \sqrt{\log \mathcal{N}_{\rho}(\mathcal{F}_{b}(\mathcal{S}_{n}) ;L_{2}(P_{n, t}^{\circ}))} d\rho \\
    &\lesssim \sqrt{\frac{bd \log((4L_{\mathcal{J}_{\max}, \underline{\lambda}, K_{g}} B + 4L_{\mathcal{J}_{\max}, \underline{\lambda}, K_{\ell}}\overline{\lambda})^2 /b)}{n}} \\
    &:= r_{n}(b),
\end{align*}
where the second inequality holds since by definition $\norm{\widetilde{\varphi}_{t}(\cdot, \lambda ; \theta)}_{L_{2}(P_{n, t}^{\circ})} \leq \sqrt{b}$ so $\mathrm{diam}( \Phi_{t}^{b}(\mathcal{S}_{n})) \leq 2\sqrt{b}$. 

Clearly $r_{n}(b)$ is a continuous concave function of $b$. Consider the function $g(b) := b \log (A^2 / b) = 2b \log A - b\log b$, Then $g^{\prime \prime}(b) = -1/b < 0$ for all $b > 0$. Thus we have that $g$ is concave on $(0, \infty)$ and $g(b) \geq 0$ on $(0, A^2]$. Now since $\sqrt{\cdot}$ is concave and increasing on $[0, \infty)$, the composition $b \mapsto \sqrt{g(b)}$ is concave on $(0, A^{2}]$ so $r_{n}(b)$. It is worth noting that in \cite{Bartlett_2005}, most of their localized fixed-point lemmas rely on sub-root functions $\psi$ which are non-decreasing and $\psi(b) / \sqrt{b}$ is non-increasing. 

Define the truncation level
$B_0:=A^2 / e, \quad
\psi_n(b):=r_n\!\bigl(b\wedge B_0\bigr), \quad \forall b\ge 0$.
We claim that $\psi_n$ is a sub-root function, i.e.,
(i) $\psi_n$ is nonnegative and nondecreasing on $[0,\infty)$, and
(ii) the map $b\mapsto \psi_n(b)/\sqrt b$ is nonincreasing on $(0,\infty)$. Consider
\[
f(b)\;:=\;b\log\!\Big(\frac{A^2}{b}\Big)\;=\;b(\log A^2-\log b),\;\; b\in(0,A^2].
\]
Then
\[
f'(b)\;=\;\log\!\Big(\frac{A^2}{b}\Big)-1.
\]
Hence $f'(b)\ge 0$ iff $\log(A^2/b)\ge 1$, equivalently $b\le A^2/e=B_0$.
Therefore $f$ is nondecreasing on $(0,B_0]$, and since $r_n(b)=C\sqrt{f(b)/n}$ with
$\sqrt{\cdot}$ increasing, it follows that $r_n$ is nondecreasing on $(0,B_0]$.
Because $b\mapsto b\wedge B_0$ is nondecreasing and $r_n$ is nondecreasing on $[0,B_0]$
(and $\psi_n$ is constant for $b\ge B_0$), we conclude that $\psi_n$ is nondecreasing on $[0,\infty)$. To show $b\mapsto \psi_n(b)/\sqrt b$ is nonincreasing, we split into two regimes.

\smallskip
\noindent\emph{(a) $0<b\le B_0$.}
Here $\psi_n(b)=r_n(b)$, so
\[
\frac{\psi_n(b)}{\sqrt b}
\;=\;
\frac{C}{\sqrt n}\,\sqrt{\log\!\Big(\frac{A^2}{b}\Big)}.
\]
Let $g(b):=\sqrt{\log(A^2/b)}$. For $b\in(0,B_0]$, we have $\log(A^2/b)\ge 1$ and
\[
g'(b)
=\frac{1}{2\sqrt{\log(A^2/b)}}\cdot\frac{d}{db}\log\!\Big(\frac{A^2}{b}\Big)
=\frac{1}{2\sqrt{\log(A^2/b)}}\cdot\Bigl(-\frac{1}{b}\Bigr)
<0.
\]
Thus $g$ is strictly decreasing on $(0,B_0]$, and therefore
$b\mapsto \psi_n(b)/\sqrt b$ is decreasing on $(0,B_0]$.

\smallskip
\noindent\emph{(b) $b\ge B_0$.}
Here $\psi_n(b)=r_n(B_0)$ is constant, so
\[
\frac{\psi_n(b)}{\sqrt b}
=\frac{r_n(B_0)}{\sqrt b},
\]
which is decreasing in $b$. Moreover, there is no upward jump at $B_0$ since for $b\ge B_0$,
\[
\frac{\psi_n(b)}{\sqrt b}
=\frac{r_n(B_0)}{\sqrt b}
\le
\frac{r_n(B_0)}{\sqrt{B_0}}
=\lim_{u\uparrow B_0}\frac{r_n(u)}{\sqrt u}
=\lim_{u\uparrow B_0}\frac{\psi_n(u)}{\sqrt u}.
\]
Combining (a)--(b), the ratio $b\mapsto \psi_n(b)/\sqrt b$ is nonincreasing on $(0,\infty)$.

Now it suffices to solve the fixed point equation $b \leq 4r_{n}(b) / \alpha$. Reorganizing terms and letting $\beta = c_{1}^2d / \alpha^2n$, we must solve $b \leq \beta \log(A^2 / b)$. This does not admit a closed form solution with elementary functions, so we appeal to the Lambert W function. This yields a closed form solution $b = \beta W(A^2 / \beta)$. Using the fact that $W(x) \asymp \log x$ (up to $\log \log$ terms), we find that 
\[
    b_{0} \asymp \frac{d}{\alpha^2 n} \log \left( \frac{(4L_{\mathcal{J}_{\max}, \underline{\lambda}, K_{g}} B + 4L_{\mathcal{J}_{\max}, \underline{\lambda}, K_{\ell}}\overline{\lambda})^2 \alpha^2 n}{d} \right) \implies \overline{r}_{n}^{h}(\alpha, \Phi_{t}, \mathcal{D}_{t}) \asymp \frac{d}{\alpha n} \log \left( \frac{(4L_{\mathcal{J}_{\max}, \underline{\lambda}, K_{g}} B + 4L_{\mathcal{J}_{\max}, \underline{\lambda}, K_{\ell}}\overline{\lambda})^2 \alpha^2 n}{d} \right).
\]

Thus we conclude that probability atleast $1-\delta$, $\forall \gamma_{t} \in \Phi_{t}$ and $\forall \alpha > 0$:
\begin{align*}
    (1-5\alpha)P_{t}^{\circ} \gamma_{t}(z, \lambda; \theta) \leq P_{n, t}^{\circ} \gamma_{t}(z, \lambda; \theta) + \frac{5d}{\alpha n} \log \left( \frac{(2L_{\mathcal{J}_{\max}, \underline{\lambda}, K_{g}} B + 2L_{\mathcal{J}_{\max}, \underline{\lambda}, K_{\ell}}\overline{\lambda})^2 \alpha^2 n}{d} \right) + \frac{5(3V_{\mathcal{J}_{\max}, \underline{\lambda}, \kappa, K_{g}} + 8\alpha) \log(1/\delta)}{3\alpha n}.
\end{align*}
Since this bound is uniform over $(\theta, \lambda) \in \Theta \times \Lambda$, we can note that $\widehat{\theta}_{n, t}^{\mathrm{KL}} \in \argmin_{\theta \in \Theta} \Psi_{n, t}(\theta, \widehat{\lambda}_{n, t})$. Since $\theta \mapsto \lambda \log(\cdot)$ is monotone in $P_{n, t}^{\circ}e^{\ell / \lambda}$, we have that $P_{n, t}^{\circ} \widetilde{\varphi}_{t}(\cdot, \hat{\lambda}; \widehat{\theta}_{n, t}^{\mathrm{KL}}) \leq P_{n, t}^{\circ} \widetilde{\varphi}_{t}(\cdot, \hat{\lambda}; \theta_{t}^{\star})$ so this immediately implies $P_{n, t}^{\circ} \gamma_{t}(;\cdot, \hat{\lambda}; \widehat{\theta}_{n, t}^{\mathrm{KL}}) \leq 0$. Applying this to the inequality above, we have with probability atleast $1- \delta$, 
\begin{align*}
    P_{t}^{\circ} \gamma_{t}(z, \hat{\lambda}; \widehat{\theta}_{n, t}^{\mathrm{KL}}) \leq \frac{1}{1-5\alpha} \left\{ \frac{5d}{\alpha n} \log \left( \frac{(4L_{\mathcal{J}_{\max}, \underline{\lambda}, K_{g}} B + 4L_{\mathcal{J}_{\max}, \underline{\lambda}, K_{\ell}}\overline{\lambda})^2 \alpha^2 n}{d} \right) + \frac{5(3V_{\mathcal{J}_{\max}, \underline{\lambda}, \kappa, K_{g}} + 4\alpha) \log(1/\delta)}{3\alpha n} \right\}.
\end{align*}

Now we can scale and shift back to our original empirical process since
\[
P_{n, t}^{\circ} \varphi_t(Z, \lambda; \theta) - P_{t}^{\circ} \varphi_t(Z, \lambda; \theta) \leq (\mathcal{J}_{\max} - 1) [P_{t}^{\circ} \widetilde{\varphi}_{t}(z, \lambda; \theta) - P_{n, t}^{\circ}\widetilde{\varphi}_{t}(z, \lambda; \theta)].
\]
Since this bound holds uniformly, we find that
\begin{align*}
    \Psi_{t}(\widehat{\theta}_{n, t}^{\mathrm{KL}}, \widehat{\lambda}) - \Psi_{t}(\theta_{\widehat{\lambda}}^{\star}, \widehat{\lambda})
&\leq
\overline{\lambda} \abs{P_{t}^{\circ} (\varphi_t(Z,\widehat{\lambda};\widehat{\theta}_{n, t}^{\mathrm{KL}}) -\varphi_t(Z,\widehat{\lambda};\theta_{\widehat{\lambda}}^{\star}) )} \\
&\lesssim \frac{\overline{\lambda} (\mathcal{J}_{\max} - 1)}{1-\alpha} \left\{ \frac{d}{\alpha n} \log \left( \frac{(L_{\mathcal{J}_{\max}, \underline{\lambda}, K_{g}} B + L_{\mathcal{J}_{\max}, \underline{\lambda}, K_{\ell}}\overline{\lambda})^2 \alpha^2 n}{d} \right) + \frac{(V_{\mathcal{J}_{\max}, \underline{\lambda}, \kappa, K_{g}} + \alpha) \log(1/\delta)}{\alpha n} \right\}.
\end{align*}

\subsubsection{Part III: Bounding $\lambda$-mismatch.} 
We must now bound $\Psi_{t}(\theta_{\widehat{\lambda}}^{\star}, \widehat{\lambda}) - \Psi_{t}(\theta_{t}^{\mathrm{KL}}, \lambda^{\star})$. For this, we will need to make one mild assumption.

\vspace{1em}

\begin{assumption}[Noisy reward function]
\label{assumption:noisy-reward-function}
    Let $U = (x, a^{1}, a^{2})$ from $P_{t}^{\circ}$ and assume we obtain noisy differences of rewards:
    \[
        \Delta r = \langle \omega, \Delta \phi \rangle + \xi, \quad \abs{\xi} \leq \overline{\xi}, \quad \mathbb{E}[\xi \mid U] = 0, \quad \mathbb{E}[\xi^{3} \mid U] = 0, \quad \mathrm{Var}_{P_{t}^{\circ}}(\xi^{2} \mid U) \geq \tau_{\xi}^2 > 0 \; \text{a.s.}
    \]
\end{assumption}

\begin{remark}[Generality of the noise model]
    We treat $\Delta r$ as the observable surrogate feedback (e.g., reward model scores) rather than an unobserved ground truth. Consequently, the decomposition $\Delta r = \langle \omega, \Delta \phi \rangle + \xi$ is general, as we may define the linear term to capture the conditional mean such that $\mathbb{E}[\xi \mid U]=0$ holds by construction.
\end{remark}

Using Assumption \ref{assumption:noisy-reward-function}, we can prove the following:

\vspace{1em}

\begin{lemma}[Uniform variance lower bound under exponential tilt]
\label{lemma:uniform-variance-lower-bound-exponential-tilt}
    For every $(\theta, \lambda) \in \Theta \times \Lambda$, we have 
    \[
        \mathrm{Var}_{Q_{\theta, \lambda}}(\ell_{t}(Z; \theta)) \geq \frac{\tau_{\xi}^2}{\mathcal{J}_{\max}},
    \]
    where $Q_{\theta, \lambda}$ has density
    \[
        q_{\theta, \lambda}(z) := \frac{\exp \left( \ell_{t}(z; \theta) / \lambda \right)}{\mathbb{E}_{P_{t}^{\circ}}[\exp \left(\ell_{t}(z; \theta) / \lambda \right)]} p_{t}^{\circ}(z) = w_{\theta, \lambda}(z)p_{t}^{\circ}(z).
    \]
\end{lemma}

\begin{proof}[Proof of Claim \ref{lemma:uniform-variance-lower-bound-exponential-tilt}]
    By the definition of the REBEL loss, we have that $\ell_{t} = (m_{\theta, \Delta r} - \xi)^2$ where $m_{\theta, \Delta \psi, \Delta r}$ takes the form
    \[
        m_{\theta, \Delta \psi, \Delta r} = \frac{1}{\eta} \langle \theta - \theta_t, \Delta \psi \rangle - \langle \omega, \Delta \phi \rangle.
    \]
    Note that we suppress dependence on $\eta$ and $\omega$ as these are known quantities and are not important in our proof of this lemma. Let $U = (x, a^{1}, a^{2})$. Conditional on $U$, we have that
    \begin{align*}
        \mathrm{Var}_{P_{t}^{\circ}}( \ell_{t}(Z; \theta) \mid U) &= \mathrm{Var}_{P_{t}^{\circ}} ( \xi^{2} \mid U) + 4m^2 \mathrm{Var}_{P_{t}^{\circ}} (\xi \mid U) - 4m \mathrm{Var}_{P_{t}^{\circ}}(\xi^2, \xi \mid U) \\
        &\geq \mathrm{Var}_{P_{t}^{\circ}} ( \xi^{2} \mid U)  - 4m \mathrm{Var}_{P_{t}^{\circ}}(\xi^2, \xi \mid U),
    \end{align*}
    where the inequality holds since $\mathrm{Var}_{P_{t}^{\circ}} (\xi \mid U) \geq 0$ and $m^2 \geq 0$. Now under Assumption \ref{assumption:noisy-reward-function}, we have 
    \begin{align*}
        \mathrm{Var}_{P_{t}^{\circ}}(\xi^2, \xi \mid U) &= \mathbb{E}_{P_{t}^{\circ}}[\xi^{3} \mid U] - \mathbb{E}_{P_{t}^{\circ}}[\xi^{2} \mid U]\mathbb{E}_{P_{t}^{\circ}}[\xi \mid U] = 0.
    \end{align*}
    Thus we conclude $\mathrm{Var}_{P_{t}^{\circ}}( \ell_{t}(Z; \theta) \mid U) \geq \mathrm{Var}_{P_{t}^{\circ}} ( \xi^{2} \mid U) \geq \tau_{\xi}^2$. Then by the total law of variance, we conclude
    \[
        \mathrm{Var}_{P_{t}^{\circ}}( \ell_{t}(Z; \theta)) \geq \mathbb{E}[\mathrm{Var}_{P_{t}^{\circ}} ( \xi^{2} \mid U)] \geq \tau_{\xi}^2.
    \]
    By a change of measure, we have
    \begin{align*}
        \mathrm{Var}_{Q_{\theta, \lambda}}(\ell_{t}(Z; \theta)) &= \mathbb{E}_{Q_{\theta, \lambda}} [ (\ell_{t}(Z;\theta) - Q_{\theta, \lambda} \ell_{t}(Z;\theta))^2] \\
        &= \mathbb{E}_{P_{t}^{\circ}} [ w_{\theta, \lambda}(\ell_{t}(Z;\theta) - Q_{\theta, \lambda} \ell_{t}(Z;\theta))^2 ] \\
        &\geq \frac{1}{\mathcal{J}_{\max}} \mathrm{Var}_{P_{t}^{\circ}}( \ell_{t}(Z; \theta)) \\
        &\geq \frac{\tau_{\xi}^2}{\mathcal{J}_{\max}},
    \end{align*}
    where the first inequality holds since $\exp \left( \ell_{t}(z; \theta) / \lambda \right) \geq 1$ and $\exp \left( \ell_{t}(z; \theta) / \lambda \right) \leq \mathcal{J}_{\max}$ and $$\mathbb{E}_{P_{t}^{\circ}} [ (\ell_{t}(Z;\theta) - Q_{\theta, \lambda} \ell_{t}(Z;\theta))^2 ] = \mathrm{Var}_{P_{t}^{\circ}}( \ell_{t}(Z; \theta)) + \left(\mathbb{E}_{P_{t}^{\circ}}( \ell_{t}(Z; \theta)) - Q_{\theta, \lambda} \ell_{t}(Z;\theta))^2 \right) \geq 0.$$
    This concludes the proof.
\end{proof}

Note that Lemma \ref{lemma:uniform-variance-lower-bound-exponential-tilt}, we get the map $\lambda \mapsto \Psi_{t}(\theta, \lambda)$ is $\tau_{\xi}^2 / \bar{\lambda}^3 \mathcal{J}_{\max}$ strongly convex since
\begin{align*}
    \nabla_{\lambda}^2 \Psi_{t}(\theta, \lambda) &= \frac{1}{\lambda^3} \mathrm{Var}_{Q_{\theta, \lambda}}(\ell_{t}(Z; \theta)) \geq \frac{\tau_{\xi}^2}{\bar{\lambda}^3 \mathcal{J}_{\max}}.
\end{align*}

We will need the following inequality to recover $\widetilde{\mathcal{O}}_{\mathbb{P}}(n^{-1/2})$ rates.

\vspace{1em}

\begin{lemma}[Reduction to gradient concentration]
    \label{appendix:reduction-to-gradient-concentration-KL}
    If $\hat{\mathcal{L}}_{t}$ is $\hat{\mu}_{t}$-strongly convex in $\norm{\cdot}_{2}$ and $\widehat{\theta}_{t} \in \argmin_{\theta \in \Theta} \hat{\mathcal{L}}_{t}(\theta)$ and $\theta^{*}_{t} \in \argmin_{\theta \in \Theta} \mathcal{L}_{t}(\theta)$, then
    \[
        \norm{\widehat{\theta}_{t} - \theta_{t}^{*}}_{2} \leq \frac{2}{\widehat{\mu}_{t}} \norm{\nabla_{\theta}\widehat{\mathcal{L}}_{t}(\theta_{t}^{*})}_{2}.
    \]
\end{lemma}

\begin{proof}[Proof of Lemma \ref{appendix:reduction-to-gradient-concentration-KL}]
    Conditional on $\mathcal{E}_{t}$, we have by Lemma \ref{appendix:strong-convexity-of-empirical-risk} that 
    \begin{align*}
        \widehat{\mathcal{L}}_{t}(\widehat{\theta}_{t}) &\geq \widehat{\mathcal{L}}_{t}(\theta_{t}^{*}) + \langle \nabla_{\theta}\widehat{\mathcal{L}}_{t}(\theta_{t}^{*}),  \widehat{\theta}_{t} - \theta_{t}^{*} \rangle + \frac{\widehat{\mu}_{t}}{2} \norm{\widehat{\theta}_{t} - \theta_{t}^{*}}_{2}^2 \\
        &= \widehat{\mathcal{L}}_{t}(\theta_{t}^{*}) - \widehat{\mathcal{L}}_{t}(\widehat{\theta}_{t}) + \widehat{\mathcal{L}}_{t}(\widehat{\theta}_{t}) + \langle \nabla_{\theta}\widehat{\mathcal{L}}_{t}(\theta_{t}^{*}),  \widehat{\theta}_{t} - \theta_{t}^{*} \rangle + \frac{\widehat{\mu}_{t}}{2} \norm{\widehat{\theta}_{t} - \theta_{t}^{*}}_{2}^2 \\
        &\geq \widehat{\mathcal{L}}_{t}(\widehat{\theta}_{t}) + \langle \nabla_{\theta}\widehat{\mathcal{L}}_{t}(\theta_{t}^{*}),  \widehat{\theta}_{t} - \theta_{t}^{*} \rangle + \frac{\widehat{\mu}_{t}}{2} \norm{\widehat{\theta}_{t} - \theta_{t}^{*}}_{2}^2,
    \end{align*}
    where the second inequality holds by definition that $\widehat{\mathcal{L}}_{t}(\widehat{\theta}_{t}) \leq \widehat{\mathcal{L}}_{t}(\theta_{t}^{*})$. Rearranging terms and using Cauchy-Schwarz on the inner-product, we conclude.
\end{proof}

By strong convexity of $\lambda \mapsto \Psi_{t}(\theta, \lambda)$ and Lemma \ref{appendix:reduction-to-gradient-concentration-KL}, we have 
\[
\Psi_{t}(\theta_{\widehat{\lambda}}^{\star}, \widehat{\lambda}) - \Psi_{t}(\theta_{t}^{\mathrm{KL}}, \lambda^{\star}) \leq \frac{\bar{\lambda}^3 \mathcal{J}_{\max}}{2\tau_{\xi}^2} \abs{\nabla_{\lambda} \Psi_{t}(\theta_{\widehat{\lambda}}^{\star}, \widehat{\lambda})}^2.
\]

It remains to bound $\abs{\nabla_{\lambda} \Psi_{t}(\theta_{\widehat{\lambda}}^{\star}, \widehat{\lambda})}$. We will first write $\nabla_{\lambda} \Psi_{t}(\theta_{\widehat{\lambda}}^{\star}, \widehat{\lambda})$ as a smooth function of two bounded means. By differentiation and interchange of gradients and expectation (by Dominated Convergence Theorem), we have 
\begin{align*}
    \nabla_{\lambda} \Psi_{t}(\theta, \lambda) = \varepsilon + \log (P_{t}^{\circ} \varphi_{t}(z, \lambda; \theta)) - \frac{1}{\lambda} \frac{P_{t}^{\circ}(\ell_{t}(z; \theta) \varphi_{t}(z, \lambda; \theta))}{P_{t}^{\circ} \varphi_{t}(z, \lambda; \theta)}.
\end{align*}
A similar quantity holds for the gradient of the empirical dual argument $\nabla_{\lambda} \Psi_{n, t}(\theta, \lambda)$. Since $(\widehat{\theta}_{n, t}^{\mathrm{KL}}, \widehat{\lambda}) \in \argmin_{(\theta, \lambda) \in \Theta \times \Lambda} \Psi_{n, t}(\theta, \lambda)$, we have $\nabla_{\lambda} \Psi_{n, t}(\widehat{\theta}_{n, t}^{\mathrm{KL}}, \widehat{\lambda}) = 0$. This gives us the following decomposition:
\begin{align*}
    \abs{\nabla_{\lambda} \Psi_{t}(\theta_{\widehat{\lambda}}^{\star}, \widehat{\lambda})} &= \abs{\nabla_{\lambda} \Psi_{t}(\theta_{\widehat{\lambda}}^{\star}, \widehat{\lambda}) - \nabla_{\lambda} \Psi_{n, t}(\theta_{\widehat{\lambda}}^{\star}, \widehat{\lambda}) + \nabla_{\lambda} \Psi_{n, t}(\theta_{\widehat{\lambda}}^{\star}, \widehat{\lambda})} \\
    &\leq \underbrace{\abs{\nabla_{\lambda} \Psi_{t}(\theta_{\widehat{\lambda}}^{\star}, \widehat{\lambda}) - \nabla_{\lambda} \Psi_{n, t}(\theta_{\widehat{\lambda}}^{\star}, \widehat{\lambda})}}_{(a) \; \textbf{Gradient empirical process}} + \underbrace{\abs{\nabla_{\lambda} \Psi_{n, t}(\theta_{\widehat{\lambda}}^{\star}, \widehat{\lambda}) - \nabla_{\lambda} \Psi_{n, t}(\widehat{\theta}_{n, t}^{\mathrm{KL}}, \widehat{\lambda})}}_{(b) \; \textbf{Gradient suboptimality in $\theta$ at random $\widehat{\lambda}$ }}.
\end{align*}

\subsubsection{Part III-(a): Bounding the gradient empirical process.}
We will first bound the gradient empirical process. Since we have a closed form formula for $\nabla_{\lambda} \Psi_{t}, \nabla_{\lambda} \Psi_{n, t}$, we have 

\begin{align*}
    \abs{\nabla_{\lambda} \Psi_{t}(\theta_{\widehat{\lambda}}^{\star}, \widehat{\lambda}) - \nabla_{\lambda} \Psi_{n, t}(\theta_{\widehat{\lambda}}^{\star}, \widehat{\lambda})} &\leq \abs{\log (P_{t}^{\circ} \varphi_{t}(z, \theta_{\widehat{\lambda}}^{\star}; \widehat{\lambda})) - \log (P_{n, t}^{\circ} \varphi_{t}(z, \theta_{\widehat{\lambda}}^{\star}; \widehat{\lambda}))} \\
    &+ \frac{1}{\lambda} \abs{\frac{P_{t}^{\circ}(\ell_{t}(z; \theta) \varphi_{t}(z, \theta_{\widehat{\lambda}}^{\star}; \widehat{\lambda}))}{P_{t}^{\circ} \varphi_{t}(z, \theta_{\widehat{\lambda}}^{\star}; \widehat{\lambda})} - \frac{P_{n, t}^{\circ}(\ell_{t}(z; \theta) \varphi_{t}(z, \theta_{\widehat{\lambda}}^{\star}; \widehat{\lambda}))}{P_{n, t}^{\circ} \varphi_{t}(z, \theta_{\widehat{\lambda}}^{\star}; \widehat{\lambda})}} \\
    &\leq \sup_{(\theta, \lambda)  \in \Theta \times \Lambda} \abs{(P_{t}^{\circ} - P_{n, t}^{\circ}) \varphi_{t}(z, \lambda; \theta)} + \frac{1}{\underline{\lambda}}\sup_{(\theta, \lambda)  \in \Theta \times \Lambda} \abs{(P_{t}^{\circ} - P_{n, t}^{\circ}) \ell_{t}(z; \theta) \varphi_{t}(z, \lambda; \theta)},
\end{align*}
where the second inequality holds from triangle inequality, Lipschitzness of $\log$ for $x \in [1, \infty)$, and boundedness of $\ell_t \in [0, K_{\ell}]$. We will bound the first and second quantity using standard symmetrization and Hoeffding. Notice that this is sufficient since we will square this quantity, maintaining fast rates. Let $\mathcal{F} = \left\{ \varphi_{t}(z, \lambda; \theta) : (\theta, \lambda) \in \Theta \times \Lambda \right\}$. We already know $\varphi_{t}(z, \lambda; \theta)$ is Lipschitz in $(\theta, \lambda)$. Let $$\widetilde{L}_{\mathcal{J}_{\max}, \underline{\lambda}, K_{g}} = \frac{\mathcal{J}_{\max}}{\underline{\lambda}} \cdot \frac{4K_{g}}{\eta} \norm{\theta - \theta^{\prime}}_{2},$$ and $$\widetilde{L}_{\mathcal{J}_{\max}, \underline{\lambda}, K_{\ell}} = \frac{K_\ell \mathcal{J}_{\max}}{\underline{\lambda}^2}.$$
Then we have that 
\begin{align*}
    \log \mathcal{N}_{\rho}(\mathcal{F} ;L_{2}(P_{n, t}^{\circ}))
    &\leq d \log \left( \frac{2\widetilde{L}_{\mathcal{J}_{\max}, \underline{\lambda}, K_{g}} B + 2\widetilde{L}_{\mathcal{J}_{\max}, \underline{\lambda}, K_{\ell}}\overline{\lambda}}{\rho}  \right).
\end{align*}

By Corollary \ref{cor:dudley}, we can conclude that the Rademacher complexity of the function class $\mathcal{F}$ is as follows:
\begin{align*}
    \mathfrak{R}_{\mathcal{S}_{n}} \left( \mathcal{F} \right) &\lesssim \frac{1}{\sqrt{n}}  \int_{0}^{\mathrm{diam}(\mathcal{F})} \sqrt{\log \mathcal{N}_{\rho}(\mathcal{F} ;L_{2}(P_{n, t}^{\circ}))} d\rho \\
    &\leq \frac{1}{\sqrt{n}}  \int_{0}^{\mathcal{J}_{\max}} \sqrt{\log \mathcal{N}_{\rho}(\mathcal{F} ;L_{2}(P_{n, t}^{\circ}))} d\rho \\
    &\lesssim \sqrt{\frac{d \log((\widetilde{L}_{\mathcal{J}_{\max}, \underline{\lambda}, K_{g}} B + \widetilde{L}_{\mathcal{J}_{\max}, \underline{\lambda}, K_{\ell}}\overline{\lambda})^2 /\mathcal{J}_{\max})}{n}} ,
\end{align*}
where the $\mathrm{diam}(\mathcal{F}) \leq \mathcal{J}_{\max}$ holds from $\varphi_{t} \leq \mathcal{J}_{\max}$. Now by Theorem \ref{theorem:uniform-rademacher-deviation}, we have with probability atleast $1-\delta$,
\begin{align*}
    \sup_{\varphi \in \mathcal{F}} \abs{(P_{t}^{\circ} - P_{n, t}^{\circ}) \varphi_{t}(z, \lambda; \theta)}\lesssim \mathcal{J}_{\max} \sqrt{\frac{d \log((\widetilde{L}_{\mathcal{J}_{\max}, \underline{\lambda}, K_{g}} B + \widetilde{L}_{\mathcal{J}_{\max}, \underline{\lambda}, K_{\ell}}\overline{\lambda})^2 /\mathcal{J}_{\max}) + \log(1/\delta)}{n}}.
\end{align*}

We now bound the second quantity. Let $\mathcal{G} = \left\{ \ell_{t}(z;\theta)\varphi_{t}(z, \lambda; \theta) : (\theta, \lambda) \in \Theta \times \Lambda \right\}$. Using boundedness of $\ell_{t}(z; \theta)$ and subsequently $\varphi_{t}(z, \lambda; \theta)$, one can show
\begin{align*}
    \norm{\ell_{t}(z; \theta)\varphi_{t}(z, \lambda; \theta) - \ell_{t}(z; \theta)\varphi_{t}(z, \lambda^{\prime}; \theta^{\prime})}_{n,t} &\leq \frac{4K_{g}\mathcal{J}_{\max}}{\eta} \left( 1 + \frac{K_{\ell}}{\underline{\lambda}} \right) \norm{\theta - \theta^{\prime}}_{2} + \frac{K_{\ell}^2 \mathcal{J}_{\max}}{\underline{\lambda}^2} \abs{\lambda - \lambda^{\prime}} \\
    &:=G_{K_{g}, K_{\ell}, \mathcal{J}_{\max}, \eta, \underline{\lambda}} \norm{\theta - \theta^{\prime}}_{2} + G_{K_{\ell}, \mathcal{J}_{\max}, \underline{\lambda}}\abs{\lambda - \lambda^{\prime}}.
\end{align*}
Using the fact that $\mathrm{diam}(\mathcal{G}) \leq K_{\ell}\mathcal{J}_{\max}$, we can also conclude with probability atleast $1-\delta$,
\begin{align*}
    \sup_{\phi \in \mathcal{G}} \abs{(P_{t}^{\circ} - P_{n, t}^{\circ}) \ell_{t}(z; \theta) \varphi_{t}(z, \lambda; \theta)} \lesssim K_{\ell}\mathcal{J}_{\max} \sqrt{\frac{d \log((G_{K_{g}, K_{\ell}, \mathcal{J}_{\max}, \eta, \underline{\lambda}} B + G_{K_{\ell}, \mathcal{J}_{\max}, \underline{\lambda}} \overline{\lambda})^2 /K_{\ell}\mathcal{J}_{\max}) + \log(1/\delta)}{n}}.
\end{align*}

\subsubsection{Part III-(b): Bounding the gradient suboptimality in $\theta$.}
We first bound the gradient suboptimality in $\theta$. We will use Lipschitzness of the map $\theta \mapsto \nabla_{\lambda} \Psi_{n, t}(\theta, \lambda)$ in $\theta$. We establish this result below:

\vspace{1em}

\begin{lemma}[Sharp $\theta$-Lipschitzness of $\nabla_\lambda \Psi_{n,t}$]
\label{lem:theta_lip_partial_lambda_Psi}
Fix $t$ and condition on $\mathcal F_t$ so that $\{Z_i\}_{i=1}^{n_t}$ are i.i.d.\ from the nominal
distribution $P_t^\circ$. Let $n:=n_t$ and define, for $(\theta,\lambda)\in\Theta\times[\underline\lambda,\overline\lambda]$,
\[
\Psi_{n,t}(\theta,\lambda)
:= \lambda\varepsilon
+\lambda \log\!\left(\frac1n\sum_{i=1}^n \exp\!\left(\frac{\ell_t(Z_i;\theta)}{\lambda}\right)\right).
\]
Then for every fixed $\lambda\in[\underline\lambda,\overline\lambda]$, the map
$\theta\mapsto \nabla_\lambda \Psi_{n,t}(\theta,\lambda)$ is Lipschitz, and for all $\theta,\theta'\in\Theta$,
\[
\big|\nabla_\lambda\Psi_{n,t}(\theta,\lambda)-\nabla_\lambda\Psi_{n,t}(\theta',\lambda)\big|
\le
\frac{4K_g}{\eta}\Big(\frac{2}{\lambda}+\frac{K_\ell}{2\lambda^2}\Big)\,\|\theta-\theta'\|_2.
\]
In particular, for any $\hat\lambda\in[\underline\lambda,\overline\lambda]$,
\begin{align*}
\big|\nabla_\lambda\Psi_{n,t}(\theta,\hat\lambda)-\nabla_\lambda\Psi_{n,t}(\theta',\hat\lambda)\big|
&\leq
\frac{4K_g}{\eta}\Big(\frac{2}{\underline\lambda}+\frac{K_\ell}{2\underline\lambda^2}\Big)\,\|\theta-\theta'\|_2 := L_{K_{g}, K_{\ell} \eta, \underline{\lambda}}^{\partial \lambda} \|\theta-\theta'\|_2.
\end{align*}
\end{lemma}

\begin{proof}[Proof of Lemma \ref{lem:theta_lip_partial_lambda_Psi}]
Fix $\lambda\in[\underline\lambda,\overline\lambda]$ and write $\ell_i(\theta):=\ell_t(Z_i;\theta)$.
Define
\[
A(\theta):=\frac1n\sum_{i=1}^n \exp\!\left(\frac{\ell_i(\theta)}{\lambda}\right),
\;\;
w_i(\theta):=\frac{\exp(\ell_i(\theta)/\lambda)}{\sum_{j=1}^n \exp(\ell_j(\theta)/\lambda)},
\;\;
\mu(\theta):=\sum_{i=1}^n w_i(\theta)\,\ell_i(\theta).
\]
A direct differentiation yields the closed form
\begin{equation}
\label{eq:closed_form_partial_lambda_Psi}
\nabla_\lambda \Psi_{n,t}(\theta,\lambda)
=
\varepsilon+\log A(\theta)-\frac{1}{\lambda}\mu(\theta).
\end{equation}
Therefore, for any $\theta,\theta'\in\Theta$,
\begin{equation}
\label{eq:split_partial_lambda_Psi}
\big|\nabla_\lambda\Psi_{n,t}(\theta,\lambda)-\nabla_\lambda\Psi_{n,t}(\theta',\lambda)\big|
\le
\big|\log A(\theta)-\log A(\theta')\big|
+\frac{1}{\lambda}\,|\mu(\theta)-\mu(\theta')|.
\end{equation}
We bound the two terms on the right-hand side separately. Define $u(\theta)\in\mathbb R^n$ by $u_i(\theta):=\ell_i(\theta)/\lambda$. Then
$\log A(\theta)=\log\left(\frac1n\sum_{i=1}^n e^{u_i(\theta)}\right)$.
We claim that for any $u,v\in\mathbb R^n$,
\begin{equation}
\label{eq:lse_lipschitz}
\left|
\log\left(\frac1n\sum_{i=1}^n e^{u_i}\right)
-
\log\left(\frac1n\sum_{i=1}^n e^{v_i}\right)
\right|
\le \|u-v\|_\infty.
\end{equation}
To see this, let $\Delta:=\|u-v\|_\infty$. Then $u_i\le v_i+\Delta$ for all $i$, hence
$\sum_i e^{u_i}\le e^\Delta \sum_i e^{v_i}$, which implies
$\log\left(\frac1n\sum_i e^{u_i}\right)\le \log\left(\frac1n\sum_i e^{v_i}\right)+\Delta$.
By symmetry (swap $u$ and $v$), \eqref{eq:lse_lipschitz} follows. Applying \eqref{eq:lse_lipschitz} with $u=u(\theta)$ and $v=u(\theta')$ gives
\[
|\log A(\theta)-\log A(\theta')|
\le
\|u(\theta)-u(\theta')\|_\infty
=
\frac{1}{\lambda}\max_{i\in[n]} |\ell_i(\theta)-\ell_i(\theta')|.
\]
Using (A2) pointwise, $\max_i |\ell_i(\theta)-\ell_i(\theta')|\le L_\ell \|\theta-\theta'\|_2$, hence
\begin{equation}
\label{eq:logA_bound}
|\log A(\theta)-\log A(\theta')|
\le \frac{L_\ell}{\lambda}\,\|\theta-\theta'\|_2.
\end{equation}

We decompose
\[
\mu(\theta)-\mu(\theta')
=
\sum_{i=1}^n w_i(\theta)\big(\ell_i(\theta)-\ell_i(\theta')\big)
+
\sum_{i=1}^n\big(w_i(\theta)-w_i(\theta')\big)\ell_i(\theta').
\]
For the first term, since $w(\theta)$ is a probability vector,
\[
\left|\sum_{i=1}^n w_i(\theta)\big(\ell_i(\theta)-\ell_i(\theta')\big)\right|
\le
\max_{i\in[n]}|\ell_i(\theta)-\ell_i(\theta')|
\le
L_\ell\|\theta-\theta'\|_2.
\]
For the second term, by (A1), $\ell_i(\theta')\in[0,K_\ell]$, so
\[
\left|\sum_{i=1}^n\big(w_i(\theta)-w_i(\theta')\big)\ell_i(\theta')\right|
\le
K_\ell \|w(\theta)-w(\theta')\|_1.
\]
It remains to bound $\|w(\theta)-w(\theta')\|_1$. Consider the softmax map
$\sigma:\mathbb R^n\to\Delta^{n-1}$ given by
\[
\sigma(v)_i=\frac{e^{v_i}}{\sum_{j=1}^n e^{v_j}}.
\]
A direct Jacobian calculation yields
\[
\frac{\partial \sigma_i}{\partial v_j}(v)=\sigma_i(v)\big(\mathbf 1\{i=j\}-\sigma_j(v)\big).
\]
Fix $j$. The $\ell_1$-column sum satisfies
\[
\sum_{i=1}^n\left|\frac{\partial \sigma_i}{\partial v_j}(v)\right|
=
\sigma_j(v)\big(1-\sigma_j(v)\big)+\sum_{i\neq j}\sigma_i(v)\sigma_j(v)
=
2\sigma_j(v)\big(1-\sigma_j(v)\big)
\le \frac12,
\]
since $x(1-x)\le 1/4$ for $x\in[0,1]$. Hence the operator norm obeys
$\|\nabla\sigma(v)\|_{\infty\to 1}\le 1/2$ for all $v$, and the mean value theorem implies
\begin{equation}
\label{eq:softmax_lip}
\|\sigma(u)-\sigma(v)\|_1\le \frac12\|u-v\|_\infty,
\;\; \forall u,v\in\mathbb R^n.
\end{equation}
Now observe $w(\theta)=\sigma(u(\theta))$ with $u_i(\theta)=\ell_i(\theta)/\lambda$. Therefore,
\[
\|w(\theta)-w(\theta')\|_1
\le
\frac12\|u(\theta)-u(\theta')\|_\infty
=
\frac{1}{2\lambda}\max_{i\in[n]}|\ell_i(\theta)-\ell_i(\theta')|
\le
\frac{2 K_{g}}{\lambda \eta}\|\theta-\theta'\|_2.
\]
Combining the pieces yields
\begin{equation}
\label{eq:mu_bound}
|\mu(\theta)-\mu(\theta')|
\le
\frac{4K_{g}}{\eta}\|\theta-\theta'\|_2
+
K_\ell\cdot \frac{2K_{g}}{\lambda \eta}\|\theta-\theta'\|_2
=
\frac{4K_{g}}{\eta}\Big(1+\frac{K_\ell}{2\lambda}\Big)\|\theta-\theta'\|_2.
\end{equation}

Substituting \eqref{eq:logA_bound} and \eqref{eq:mu_bound} into \eqref{eq:split_partial_lambda_Psi},
\[
\big|\nabla_\lambda\Psi_{n,t}(\theta,\lambda)-\nabla_\lambda\Psi_{n,t}(\theta',\lambda)\big|
\le
\frac{4K_{g}}{\lambda \eta}\|\theta-\theta'\|_2
+\frac{1}{\lambda}\cdot
\frac{4K_{g}}{\eta}\Big(1+\frac{K_\ell}{2\lambda}\Big)\|\theta-\theta'\|_2
=
\frac{4K_{g}}{\eta}\Big(\frac{2}{\lambda}+\frac{K_\ell}{2\lambda^2}\Big)\|\theta-\theta'\|_2,
\]
which proves the first inequality. The second follows by using $\lambda\ge\underline\lambda$.
\end{proof}

Using Lemma \ref{lem:theta_lip_partial_lambda_Psi}, we find
\begin{align*}
    \abs{\nabla_{\lambda} \Psi_{n, t}(\theta_{\widehat{\lambda}}^{\star}, \widehat{\lambda}) - \nabla_{\lambda} \Psi_{n, t}(\widehat{\theta}_{n, t}^{\mathrm{KL}}, \widehat{\lambda})}^2 &\leq L_{K_{g}, K_{\ell} \eta, \underline{\lambda}}^{2, \partial \lambda}\|\theta_{\widehat{\lambda}}^{\star}-\widehat{\theta}_{n, t}^{\mathrm{KL}}\|_2 \leq \frac{L_{K_{g}, K_{\ell} \eta, \underline{\lambda}}^{2, \partial \lambda}\eta^2 \underline{\lambda} (\mathcal{J}_{\max} - 1)}{\kappa} P_{t}^{\circ}\gamma_{t}(\cdot, \widehat{\lambda}; \widehat{\theta}_{n, t}^{\mathrm{KL}}).
\end{align*}

\subsubsection{Part IV: Concluding the fast-rate.}

Combining everything, we find
\begin{align*}
    \Psi_{t}(\theta_{\widehat{\lambda}}^{\star}, \widehat{\lambda}) - \Psi_{t}(\theta_{t}^{\mathrm{KL}}, \lambda^{\star}) &\lesssim  \frac{\mathcal{J}_{\max}^2 \; d \log((\widetilde{L}_{\mathcal{J}_{\max}, \underline{\lambda}, K_{g}} B + \widetilde{L}_{\mathcal{J}_{\max}, \underline{\lambda}, K_{\ell}}\overline{\lambda})^2 /\mathcal{J}_{\max}) + \log(1/\delta)}{n} \\
    &+  \frac{K_{\ell}^2\mathcal{J}_{\max}^2\;d \log((G_{K_{g}, K_{\ell}, \mathcal{J}_{\max}, \eta, \underline{\lambda}} B + G_{K_{\ell}, \mathcal{J}_{\max}, \underline{\lambda}} \overline{\lambda})^2 /K_{\ell}\mathcal{J}_{\max}) + \log(1/\delta)}{\underline{\lambda} n} \\
    &+ \frac{L_{K_{g}, K_{\ell} \eta, \underline{\lambda}}^{2, \partial \lambda}\eta^2 \underline{\lambda} (\mathcal{J}_{\max} - 1)}{\kappa} P_{t}^{\circ}\gamma_{t}(\cdot, \widehat{\lambda}; \widehat{\theta}_{n, t}^{\mathrm{KL}}).
\end{align*}
Thus we can conclude with probability atleast $1-\delta$,
\begin{align*}
    \mathcal{L}_{t}^{\mathrm{KL}}(\widehat{\theta}_{n, t}^{\mathrm{KL}} ; \varepsilon) - \mathcal{L}_{t} (\theta_{t}^{\mathrm{KL}} ; \varepsilon) &\lesssim \frac{L_{K_{g}, K_{\ell} \eta, \underline{\lambda}}^{\partial \lambda}\eta^2 \underline{\lambda} (\mathcal{J}_{\max} - 1)}{\kappa(1-\alpha)} \left\{ \frac{d}{\alpha n} \log \left( \frac{(L_{\mathcal{J}_{\max}, \underline{\lambda}, K_{g}} B + L_{\mathcal{J}_{\max}, \underline{\lambda}, K_{\ell}}\overline{\lambda})^2 \alpha^2 n}{d} \right) + \frac{(V_{\mathcal{J}_{\max}, \underline{\lambda}, \kappa, K_{g}} + \alpha) \log(1/\delta)}{\alpha n} \right\} \\
    &+ \frac{\mathcal{J}_{\max}^2 \; d \log((\widetilde{L}_{\mathcal{J}_{\max}, \underline{\lambda}, K_{g}} B + \widetilde{L}_{\mathcal{J}_{\max}, \underline{\lambda}, K_{\ell}}\overline{\lambda})^2 /\mathcal{J}_{\max}) + \log(1/\delta)}{n} \\
    &+  \frac{K_{\ell}^2\mathcal{J}_{\max}^2\;d \log((G_{K_{g}, K_{\ell}, \mathcal{J}_{\max}, \eta, \underline{\lambda}} B + G_{K_{\ell}, \mathcal{J}_{\max}, \underline{\lambda}} \overline{\lambda})^2 /K_{\ell}\mathcal{J}_{\max}) + \log(1/\delta)}{\underline{\lambda} n}. \\
\end{align*}
Invoking Lemma \ref{appendix:strong-convexity-of-KL-loss}, we have for all $\theta\in\Theta$,
\[
\mathcal{L}^{\mathrm{KL}}_{t}(\theta;\varepsilon)-\mathcal{L}^{\mathrm{KL}}_{t}(\theta_t^{\mathrm{KL}};\varepsilon)
\ge
\frac{\kappa}{\eta^2}\|\theta-\theta_t^{\mathrm{KL}}\|_2^2.
\]
Conditional on $\mathcal{F}_{t}$, applying this inequality at $\theta=\hat{\theta}^{\mathrm{KL}}_{n,t}$ yields, with probability at least $1-\delta$, for all $t\in[T]$,
\begin{align*}
    \|\hat{\theta}^{\mathrm{KL}}_{n,t}-\theta_t^{\mathrm{KL}}\|_2^2
&\lesssim \frac{L_{K_{g}, K_{\ell} \eta, \underline{\lambda}}^{\partial \lambda}\eta^{4} \underline{\lambda} (\mathcal{J}_{\max} - 1)}{\kappa^{2}(1-\alpha)} \left\{ \frac{d}{\alpha n_{t}} \log \left( \frac{(L_{\mathcal{J}_{\max}, \underline{\lambda}, K_{g}} B + L_{\mathcal{J}_{\max}, \underline{\lambda}, K_{\ell}}\overline{\lambda})^2 \alpha^2 n_{t}}{d} \right) + \frac{(V_{\mathcal{J}_{\max}, \underline{\lambda}, \kappa, K_{g}} + \alpha) \log(T/\delta)}{\alpha n_{t}} \right\} \\
    &+ \frac{\eta^{2}\mathcal{J}_{\max}^2 \; d \log((\widetilde{L}_{\mathcal{J}_{\max}, \underline{\lambda}, K_{g}} B + \widetilde{L}_{\mathcal{J}_{\max}, \underline{\lambda}, K_{\ell}}\overline{\lambda})^2 /\mathcal{J}_{\max}) + \log(T/\delta)}{\kappa n_{t}} \\
    &+  \frac{\eta^{2} K_{\ell}^2\mathcal{J}_{\max}^2\;d \log((G_{K_{g}, K_{\ell}, \mathcal{J}_{\max}, \eta, \underline{\lambda}} B + G_{K_{\ell}, \mathcal{J}_{\max}, \underline{\lambda}} \overline{\lambda})^2 /K_{\ell}\mathcal{J}_{\max}) + \log(T/\delta)}{\underline{\lambda} \kappa n_{t}}. \\
\end{align*}

\section{Proof of "Slow Rate" $\chi^{2}$-DRO-REBEL}
\label{appendix:Slow-chisquared-DRO-REBEL}
We first state the following dual reformulation for $\chi^{2}$-DRO.

\vspace{1em}

\begin{lemma}[Mean--variance form of $\chi^{2}$-DRO-REBEL]
\label{appendix:strong-duality-chi2}
Fix $\theta\in\Theta$ and $\varepsilon>0$. Define the Pearson $\chi^2$-divergence
\[
D_{\chi^2}(\mathbb{P}\,\|\,\mathbb{P}^{\circ})
:=\mathbb{E}_{z\sim\mathbb{P}^{\circ}}\!\left[\left(\frac{d\mathbb{P}}{d\mathbb{P}^{\circ}}(z)-1\right)^{2}\right],
\;\; 
\mathcal{B}_{\varepsilon}\bigl(\mathbb{P}^{\circ};\chi^2\bigr)
:=\Bigl\{\mathbb{P}\ll \mathbb{P}^{\circ}:\;
D_{\chi^2}(\mathbb{P}\,\|\,\mathbb{P}^{\circ})\le \varepsilon\Bigr\}.
\]
Let $\ell(\cdot;\theta)$ be measurable and satisfy $\ell(\cdot;\theta)\in L^2(\mathbb{P}^{\circ})$.
Define
\[
\mu_\theta := \mathbb{E}_{\mathbb{P}^{\circ}}\!\bigl[\ell(Z;\theta)\bigr],
\;\;
\sigma_\theta^2 := \mathrm{Var}_{\mathbb{P}^{\circ}}\!\bigl(\ell(Z;\theta)\bigr),
\;\;
\sigma_\theta := \sqrt{\sigma_\theta^2}.
\]
Consider the $\chi^2$-DRO objective
\[
\mathcal{L}^{\chi^2}\bigl(\theta;\varepsilon\bigr)
:=
\sup_{\mathbb{P}\in\mathcal{B}_{\varepsilon}(\mathbb{P}^{\circ};\chi^2)}
\mathbb{E}_{Z\sim\mathbb{P}}\bigl[\ell(Z;\theta)\bigr].
\]
Assume the following \emph{nonnegativity / inactive-truncation condition} holds:
\begin{equation}
\label{eq:chi2-inactive-truncation}
\ell(Z;\theta)\;\ge\;\mu_\theta-\frac{\sigma_\theta}{\sqrt{\varepsilon}}
\quad\text{for $\mathbb{P}^{\circ}$-a.e.\ }Z.
\end{equation}
Then the objective admits the mean--variance dual form
\begin{align*}
\mathcal{L}^{\chi^2}\bigl(\theta;\varepsilon\bigr)
&=
\mu_\theta \;+\;
\inf_{\lambda>0}\Bigl\{\varepsilon\,\lambda + \frac{\sigma_\theta^{2}}{4\lambda}\Bigr\}
\;=\;
\mu_\theta \;+\;\sqrt{\varepsilon}\,\sigma_\theta.
\end{align*}
Moreover, if $\sigma_\theta>0$ the infimum is uniquely attained at
\[
\lambda^\star=\frac{\sigma_\theta}{2\sqrt{\varepsilon}}.
\]
(If $\sigma_\theta=0$, then $\inf_{\lambda>0}\{\varepsilon\lambda+\sigma_\theta^2/(4\lambda)\}=0$ and
$\mathcal{L}^{\chi^2}(\theta;\varepsilon)=\mu_\theta$.)
In particular, when $\sigma_\theta>0$, one may equivalently restrict the infimum to any interval
$[\underline{\lambda},\bar{\lambda}]$ with $0<\underline{\lambda}\le \lambda^\star\le \bar{\lambda}<\infty$.
\end{lemma}

\subsection{Proof of Strong Convexity of $\chi^{2}$-DRO-REBEL}

We will again demonstrate strong convexity for $\chi^{2}$.

\vspace{1em}

\begin{lemma}[Strong convexity of $\mathcal{L}^{\chi^{2}}$]
\label{appendix:strong-convexity-of-chi-loss}
Let $l(z;\theta)$ be the REBEL loss function. Then $\mathcal{L}^{\chi^{2}}\left(\theta; \varepsilon\right) = \sup_{\mathbb{P} \in \mathcal{B}_{\varepsilon} \left( \mathbb{P}^{\circ} ; \chi^{2} \right)} \mathbb{E}_{z \sim \mathbb{P}} \left[ \ell(z;\theta) \right]$ is $2\kappa / \eta$-strongly convex with respect to Euclidean norm $||\cdot||_{2}$ where $\kappa$ is the regularity parameter from Assumption \ref{assumption:data-coverage}.
\end{lemma}

\begin{proof}[Proof of Lemma \ref{appendix:strong-convexity-of-chi-loss}]
In Lemma \ref{appendix:strong-convexity-of-h}, we proved the strong convexity of $h$. By Lemma \ref{appendix:beck-strong-convexity}, for $\theta, \theta^{\prime} \in \Theta$ and $\alpha \in [0, 1]$, this is equivalent to

\[
    h(\alpha\theta + (1-\alpha)\theta^{\prime};\mathbb{P}) \leq \alpha h(\theta;\mathbb{P}) + (1-\alpha)h(\theta^{\prime};\mathbb{P}) - \frac{\mu}{2} \alpha(1-\alpha)||\theta - \theta^{\prime}||_{\Sigma_{\mathbb{P}}}^{2}.
\]

Taking the supremum over $\mathbb{P}$ preserves the convex combination and the negative quadratic term so we get 

\begin{align*}
    \mathcal{L}^{\chi^{2}}\left(\alpha\theta + (1-\alpha)\theta^{\prime}; \varepsilon\right) &= \sup_{\mathbb{P} \in \mathcal{B}_{\varepsilon} \left( \mathbb{P}^{\circ} ; \chi^{2} \right)} h(\alpha\theta + (1-\alpha)\theta^{\prime};\mathbb{P}) \\
    &\leq \sup_{\mathbb{P} \in \mathcal{B}_{\varepsilon} \left( \mathbb{P}^{\circ} ; \chi^{2} \right)} \left[  \alpha h(\theta;\mathbb{P}) + (1-\alpha)h(\theta^{\prime};\mathbb{P}) - \frac{\mu}{2} \alpha(1-\alpha)||\theta - \theta^{\prime}||_{\Sigma_{\mathbb{P}}}^{2} \right] \\
    &\leq \alpha \mathcal{L}^{\chi^{2}} \left( \theta; \varepsilon \right) + (1-\alpha) \mathcal{L}^{\chi^{2}} \left( \theta^{\prime}; \varepsilon \right) - \frac{\mu}{2} \alpha(1-\alpha) \inf_{\mathbb{P} \in \mathcal{B}_{\varepsilon} \left( \mathbb{P}^{\circ} ; \chi^{2} \right)}||\theta - \theta^{\prime}||_{\Sigma_{\mathbb{P}}}^{2} \\
    &\leq \alpha \mathcal{L}^{\chi^{2}} \left( \theta; \varepsilon \right) + (1-\alpha) \mathcal{L}^{\chi^{2}} \left( \theta^{\prime}; \varepsilon \right) - \frac{\mu}{2} \alpha(1-\alpha) \inf_{\mathbb{P} \in \mathcal{B}_{\varepsilon} \left( \mathbb{P}^{\circ} ; \chi^{2} \right)} \lambda_{\mathrm{min}} \left( \Sigma_{\mathbb{P}} \right) ||\theta - \theta^{\prime}||_{2}^{2} \\
    &\leq \alpha \mathcal{L}^{\chi^{2}} \left( \theta; \varepsilon \right) + (1-\alpha) \mathcal{L}^{\chi^{2}} \left( \theta^{\prime}; \varepsilon \right) - \frac{\mu \kappa}{2} \alpha(1-\alpha)  ||\theta - \theta^{\prime}||_{2}^{2},
\end{align*}

where the second inequality holds from $\sup_{x} \left( f(x) + g(x) \right) \leq \sup_{x} f(x) + \sup_{x}g(x)$, the third inequality holds by the fact that $\Sigma_{\mathbb{P}} \succeq \lambda_{\mathrm{min}}\left( \Sigma_{\mathbb{P}} \right)I$, and the last inequality holds from Assumption \ref{assumption:data-coverage}. Thus we conclude that $\mathcal{L}^{\chi^{2}}$ is $\mu\kappa$-strongly convex in the $||\cdot||_{2}$ norm.
\end{proof}

\subsection{Proof of Slow Parameter Estimation Rate of $\chi^{2}$-DRO-REBEL}

We now prove the "slow rate" estimation error of $\chi^{2}$-DRO-REBEL.

\begin{proof}[Proof of Theorem \ref{theorem:chi-squared-DRO-REBEL}]
\label{appendix:proof-slow-rate-chi2}
Let $t \in \left\{ 0, \dots, T-1 \right\}$ and for each $t$, we collect a dataset $\mathcal{D}_{t} = \left\{ (x_{t, i}, y_{t, i}, y_{t, i}^{\prime}) \right\}_{i=1}^{n_{t}}$ with $x_{t, i} \sim \rho$, $y_{t, i}, y_{t, i}^{\prime} \stackrel{\mathrm{i.i.d}}{\sim} \pi_{\theta_{t-1}}(\cdot \mid x_{t, i})$. Define $\mathcal{F}_{t} = \sigma(\theta_{0}, \mathcal{D}_{0}, \dots, \theta_{t-1}, \mathcal{D}_{t-1})$ be the sigma-field containing everything revealed up to the start of iteration $t$ and $\overline{\mathcal{F}}_{t} = \sigma(\mathcal{F}_{t}, \mathcal{D}_{t})$. In particular, $\theta_{t}$ is $\mathcal{F}_{t}$ measurable. We assume:
(i) $\ell_t(\cdot;\theta)$ satisfies the nonnegativity / inactive-truncation condition
\eqref{eq:chi2-inactive-truncation} (so Lemma \ref{appendix:strong-duality-chi2} applies),
and (ii) there exist constants $0<\underline{\lambda}\le \bar{\lambda}<\infty$ such that for every
$\theta\in\Theta$ with $\sigma_t(\theta)>0$, the optimizer
$\lambda^\star_t(\theta)=\sigma_t(\theta)/(2\sqrt{\varepsilon})$ lies in $[\underline{\lambda},\bar{\lambda}]$.
Equivalently $\sigma_t(\theta)\in[2\underline{\lambda}\sqrt{\varepsilon},\,2\bar{\lambda}\sqrt{\varepsilon}]
\quad \text{uniformly over $\theta$ on the set where $\sigma_t(\theta)>0$}$. Define the population mean/variance under $P_t^\circ$:
\[
\mu_t(\theta):=\mathbb{E}_{P_t^\circ}[\ell_t(Z;\theta)],
\;\;
\sigma_t^2(\theta):=\mathrm{Var}_{P_t^\circ}(\ell_t(Z;\theta)),
\;\;
\sigma_t(\theta):=\sqrt{\sigma_t^2(\theta)}.
\]
Define their empirical counterparts under $\mathbb{P}^\circ_{n,t}$:
\[
\mu_{n,t}(\theta):=\mathbb{E}_{\mathbb{P}_{n,t}^\circ}[\ell_t(Z;\theta)],
\;\;
\sigma_{n,t}^2(\theta):=\mathrm{Var}_{\mathbb{P}_{n,t}^\circ}(\ell_t(Z;\theta)),
\;\;
\sigma_{n,t}(\theta):=\sqrt{\sigma_{n,t}^2(\theta)}.
\]

By Lemma \ref{appendix:strong-duality-chi2}, for fixed $\theta\in\Theta$,
\[
\mathcal{L}^{\chi^2}_{t}(\theta;\varepsilon)
=
\mu_t(\theta)+\sqrt{\varepsilon}\,\sigma_t(\theta),
\;\;
\mathcal{L}^{\chi^2}_{n,t}(\theta;\varepsilon)
=
\mu_{n,t}(\theta)+\sqrt{\varepsilon}\,\sigma_{n,t}(\theta).
\]
Hence
\begin{equation}
\label{eq:chi2-fixed-theta-reduction-new}
\big|\mathcal{L}^{\chi^2}_{t}(\theta;\varepsilon)-\mathcal{L}^{\chi^2}_{n,t}(\theta;\varepsilon)\big|
\le
|\mu_t(\theta)-\mu_{n,t}(\theta)|
+
\sqrt{\varepsilon}\,|\sigma_t(\theta)-\sigma_{n,t}(\theta)|.
\end{equation}
If $\sigma_t(\theta)=0$, then $\ell_t(Z;\theta)$ is $P_t^\circ$-a.s.\ constant, and thus
$\mu_{n,t}(\theta)=\mu_t(\theta)$ and $\sigma_{n,t}(\theta)=0$ $\mathbb{P}(\cdot\mid\mathcal F_t)$-a.s.;
therefore the left-hand side of \eqref{eq:chi2-fixed-theta-reduction-new} is $0$.
Henceforth assume $\sigma_t(\theta)>0$.

Using $|\sigma-\sigma'|=\frac{|\sigma^2-(\sigma')^2|}{\sigma+\sigma'}$ and $\sigma+\sigma'\ge \sigma$, we have
\[
\sqrt{\varepsilon}\,|\sigma_t(\theta)-\sigma_{n,t}(\theta)|
=
\sqrt{\varepsilon}\,\frac{|\sigma_t^2(\theta)-\sigma_{n,t}^2(\theta)|}{\sigma_t(\theta)+\sigma_{n,t}(\theta)}
\le
\sqrt{\varepsilon}\,\frac{|\sigma_t^2(\theta)-\sigma_{n,t}^2(\theta)|}{\sigma_t(\theta)}.
\]
By assumption (iii), $\sigma_t(\theta)\ge 2\underline{\lambda}\sqrt{\varepsilon}$ on the set where $\sigma_t(\theta)>0$.
Therefore,
\[
\sqrt{\varepsilon}\,|\sigma_t(\theta)-\sigma_{n,t}(\theta)|
\le
\frac{|\sigma_t^2(\theta)-\sigma_{n,t}^2(\theta)|}{2\underline{\lambda}}.
\]
Plugging into \eqref{eq:chi2-fixed-theta-reduction-new} yields the fixed-$\theta$ reduction
\begin{equation}
\label{eq:chi2-fixed-theta-reduction-variance}
\big|\mathcal{L}^{\chi^2}_{t}(\theta;\varepsilon)-\mathcal{L}^{\chi^2}_{n,t}(\theta;\varepsilon)\big|
\le
|\mu_t(\theta)-\mu_{n,t}(\theta)|
+\frac{|\sigma_t^2(\theta)-\sigma_{n,t}^2(\theta)|}{2\underline{\lambda}}.
\end{equation}

Fix $\theta\in\Theta$ and define the second moments
\[
m_{2,t}(\theta):=\mathbb{E}_{P_t^\circ}[\ell_t(Z;\theta)^2],
\;\;
m_{2,n,t}(\theta):=\mathbb{E}_{P_{n,t}^\circ}[\ell_t(Z;\theta)^2].
\]
Then $\sigma_t^2(\theta)=m_{2,t}(\theta)-\mu_t(\theta)^2$ and
$\sigma_{n,t}^2(\theta)=m_{2,n,t}(\theta)-\mu_{n,t}(\theta)^2$, so
\begin{align}
|\sigma_t^2(\theta)-\sigma_{n,t}^2(\theta)|
&\le
|m_{2,t}(\theta)-m_{2,n,t}(\theta)|
+
|\mu_t(\theta)^2-\mu_{n,t}(\theta)^2|\notag\\
&\le
|m_{2,t}(\theta)-m_{2,n,t}(\theta)|
+
(\,|\mu_t(\theta)|+|\mu_{n,t}(\theta)|\,)\,|\mu_t(\theta)-\mu_{n,t}(\theta)|\notag\\
&\le
|m_{2,t}(\theta)-m_{2,n,t}(\theta)|
+2K_\ell\,|\mu_t(\theta)-\mu_{n,t}(\theta)|,
\label{eq:var-decomp}
\end{align}
where we used $0\le \mu_t(\theta),\mu_{n,t}(\theta)\le K_\ell$. Now apply conditional Hoeffding (Lemma \ref{appendix:hoeffding}) given $\mathcal{F}_{t}$.
Conditioned on $\mathcal F_t$, the samples in $\mathcal D_t$ are i.i.d.\ from $\mathbb P_t^\circ$.
Since $\ell_t(\cdot;\theta)\in[0,K_\ell]$,
\[
\Pr\Big(|\mu_t(\theta)-\mu_{n,t}(\theta)|\ge \epsilon_1\mid \mathcal{F}_{t}\Big)
\le
2\exp\Big(-\frac{2n_t\epsilon_1^2}{K_\ell^2}\Big).
\]
Also $\ell_t(\cdot;\theta)^2\in[0,K_\ell^2]$, hence
\[
\Pr\Big(|m_{2,t}(\theta)-m_{2,n,t}(\theta)|\ge \epsilon_2\mid \mathcal{F}_{t}\Big)
\le
2\exp\Big(-\frac{2n_t\epsilon_2^2}{K_\ell^4}\Big).
\]
By a union bound, with conditional probability at least $1-\delta_t$ (given $\mathcal F_t$) we have simultaneously
\[
|\mu_t(\theta)-\mu_{n,t}(\theta)|
\le
K_\ell\sqrt{\frac{\log(4/\delta_t)}{2n_t}},
\;\;
|m_{2,t}(\theta)-m_{2,n,t}(\theta)|
\le
K_\ell^2\sqrt{\frac{\log(4/\delta_t)}{2n_t}}.
\]
Plugging these into \eqref{eq:var-decomp} yields that on the same event,
\[
|\sigma_t^2(\theta)-\sigma_{n,t}^2(\theta)|
\le
K_\ell^2\sqrt{\frac{\log(4/\delta_t)}{2n_t}}
+
2K_\ell\cdot K_\ell\sqrt{\frac{\log(4/\delta_t)}{2n_t}}
=
3K_\ell^2\sqrt{\frac{\log(4/\delta_t)}{2n_t}}.
\]
Finally, substituting into \eqref{eq:chi2-fixed-theta-reduction-variance}, we obtain: with conditional probability at least $1-\delta_t$ (given $\mathcal F_t$),
\begin{equation}
\label{eq:chi2-fixed-theta-bound}
\big|\mathcal{L}^{\chi^2}_{t}(\theta;\varepsilon)-\mathcal{L}^{\chi^2}_{n,t}(\theta;\varepsilon)\big|
\le
K_\ell\sqrt{\frac{\log(4/\delta_t)}{2n_t}}
+
\frac{3K_\ell^2}{2\underline{\lambda}}\sqrt{\frac{\log(4/\delta_t)}{2n_t}}.
\end{equation}

\emph{Note that we do not require any covering argument over the scalar dual variable $\lambda$: under the inactive-truncation condition, Lemma \ref{appendix:strong-duality-chi2} reduces the robust objective to a mean--standard-deviation functional, and the deviation analysis reduces to concentration of the first two moments.}

Let $\mathcal{N}^{\theta}_{\alpha}$ be an $\alpha$-net of $\Theta$ under $\|\cdot\|_2$ with
$|\mathcal{N}^{\theta}_{\alpha}|\le (3B/\alpha)^d$.
Apply \eqref{eq:chi2-fixed-theta-bound} to each $\theta^\sharp\in\mathcal{N}^{\theta}_{\alpha}$ with failure probability
$\delta_t/|\mathcal{N}^{\theta}_{\alpha}|$ and union bound over $\theta^\sharp$, yielding that with conditional probability at least $1-\delta_t$,
\begin{align}
\label{eq:chi2-net-bound}
\sup_{\theta^\sharp\in\mathcal{N}^{\theta}_{\alpha}}
\big|\mathcal{L}^{\chi^2}_{t}(\theta^\sharp;\varepsilon)-\mathcal{L}^{\chi^2}_{n,t}(\theta^\sharp;\varepsilon)\big|
\le
\Big(K_\ell+\frac{3K_\ell^2}{2\underline{\lambda}}\Big)
\sqrt{\frac{\log\!\big(4|\mathcal{N}^{\theta}_{\alpha}|/\delta_t\big)}{2n_t}}.
\end{align}

Let $\theta,\theta'\in\Theta$. Using Lemma \ref{appendix:strong-duality-chi2} in the infimum form and the inequality
$|\inf_x f(x)-\inf_x g(x)|\le \sup_x|f(x)-g(x)|$, we have for any fixed $\lambda_0>0$,
\begin{align*}
\big|\mathcal{L}^{\chi^2}_{t}(\theta;\varepsilon)-\mathcal{L}^{\chi^2}_{t}(\theta';\varepsilon)\big|
&\le
|\mu_t(\theta)-\mu_t(\theta')|
+
\frac{|\sigma_t^2(\theta)-\sigma_t^2(\theta')|}{4\lambda_0},\\
\big|\mathcal{L}^{\chi^2}_{n,t}(\theta;\varepsilon)-\mathcal{L}^{\chi^2}_{n,t}(\theta';\varepsilon)\big|
&\le
|\mu_{n,t}(\theta)-\mu_{n,t}(\theta')|
+
\frac{|\sigma_{n,t}^2(\theta)-\sigma_{n,t}^2(\theta')|}{4\lambda_0}.
\end{align*}
Taking $\lambda_0=\underline{\lambda}$ and using that $\ell_t(z;\theta)$ is $L_{K_g,\eta}$-Lipschitz in $\theta$ uniformly over $z$
(Lemma \ref{appendix:lipschitz-bound-l}), we obtain
\[
|\mu_t(\theta)-\mu_t(\theta')|
\le
L_{K_g,\eta}\|\theta-\theta'\|_2,
\;\;
|\mu_{n,t}(\theta)-\mu_{n,t}(\theta')|
\le
L_{K_g,\eta}\|\theta-\theta'\|_2.
\]
For the variance, write $\sigma^2=m_2-\mu^2$ and use
$|\ell^2-\ell'^2|\le |\ell-\ell'|\cdot(|\ell|+|\ell'|)\le 2K_\ell|\ell-\ell'|$ to get
\[
|m_{2,t}(\theta)-m_{2,t}(\theta')|
\le
2K_\ell L_{K_g,\eta}\|\theta-\theta'\|_2,
\;\;
|m_{2,n,t}(\theta)-m_{2,n,t}(\theta')|
\le
2K_\ell L_{K_g,\eta}\|\theta-\theta'\|_2.
\]
Also,
\[
|\mu_t(\theta)^2-\mu_t(\theta')^2|\le 2K_\ell|\mu_t(\theta)-\mu_t(\theta')|
\le 2K_\ell L_{K_g,\eta}\|\theta-\theta'\|_2,
\]
and similarly for $\mu_{n,t}$. Therefore,
\[
|\sigma_t^2(\theta)-\sigma_t^2(\theta')|
\le 4K_\ell L_{K_g,\eta}\|\theta-\theta'\|_2,
\;\;
|\sigma_{n,t}^2(\theta)-\sigma_{n,t}^2(\theta')|
\le 4K_\ell L_{K_g,\eta}\|\theta-\theta'\|_2.
\]
Hence both $\mathcal{L}^{\chi^2}_{t}(\cdot;\varepsilon)$ and $\mathcal{L}^{\chi^2}_{n,t}(\cdot;\varepsilon)$ are Lipschitz on $\Theta$ with constant
\[
L_{\chi^2,\theta}
:=
L_{K_g,\eta}+\frac{4K_\ell L_{K_g,\eta}}{4\underline{\lambda}}
=
L_{K_g,\eta}\Big(1+\frac{K_\ell}{\underline{\lambda}}\Big).
\]

Now fix any $\theta\in\Theta$ and choose $\theta^\sharp\in\mathcal{N}^{\theta}_{\alpha}$ such that $\|\theta-\theta^\sharp\|_2\le \alpha$.
Then
\begin{align*}
\big|\mathcal{L}^{\chi^2}_{t}(\theta;\varepsilon)-\mathcal{L}^{\chi^2}_{n,t}(\theta;\varepsilon)\big|
&\le
\big|\mathcal{L}^{\chi^2}_{t}(\theta;\varepsilon)-\mathcal{L}^{\chi^2}_{t}(\theta^\sharp;\varepsilon)\big|
+
\big|\mathcal{L}^{\chi^2}_{t}(\theta^\sharp;\varepsilon)-\mathcal{L}^{\chi^2}_{n,t}(\theta^\sharp;\varepsilon)\big|
+
\big|\mathcal{L}^{\chi^2}_{n,t}(\theta^\sharp;\varepsilon)-\mathcal{L}^{\chi^2}_{n,t}(\theta;\varepsilon)\big|\\
&\le
\sup_{\theta^\sharp\in\mathcal{N}^{\theta}_{\alpha}}
\big|\mathcal{L}^{\chi^2}_{t}(\theta^\sharp;\varepsilon)-\mathcal{L}^{\chi^2}_{n,t}(\theta^\sharp;\varepsilon)\big|
+2L_{\chi^2,\theta}\alpha.
\end{align*}
Taking the supremum over $\Theta$ and using \eqref{eq:chi2-net-bound} gives that with conditional probability at least $1-\delta_t$,
\begin{align}
\label{eq:chi2-uniform-theta}
\sup_{\theta\in\Theta}
\big|\mathcal{L}^{\chi^2}_{t}(\theta;\varepsilon)-\mathcal{L}^{\chi^2}_{n,t}(\theta;\varepsilon)\big|
\le
\Big(K_\ell+\frac{3K_\ell^2}{2\underline{\lambda}}\Big)
\sqrt{\frac{\log\!\big(4|\mathcal{N}^{\theta}_{\alpha}|/\delta_t\big)}{2n_t}}
+2L_{\chi^2,\theta}\alpha.
\end{align}
Choose $\alpha:=1/\sqrt{n_t}$ so that $\log|\mathcal{N}^{\theta}_{\alpha}|\le d\log(3B\sqrt{n_t})\lesssim d\log n_t$ and
$2L_{\chi^2,\theta}\alpha\lesssim L_{\chi^2,\theta}/\sqrt{n_t}$.
Absorbing constants, we obtain: with conditional probability at least $1-\delta_t$,
\begin{equation}
\label{eq:chi2-uniform-theta-summary}
\sup_{\theta\in\Theta}
\big|\mathcal{L}^{\chi^2}_{t}(\theta;\varepsilon)-\mathcal{L}^{\chi^2}_{n,t}(\theta;\varepsilon)\big|
\;\lesssim\;
\Big(K_\ell+\frac{K_\ell^2}{\underline{\lambda}}\Big)
\sqrt{\frac{d\log n_t+\log(1/\delta_t)}{n_t}}.
\end{equation}

Set $\delta_t:=\delta/T$ and apply a union bound over $t\in\{0,\dots,T-1\}$.
Then with probability at least $1-\delta$, for all $t\in[T]$ simultaneously,
\begin{equation}
\label{eq:chi2-uniform-all-t}
\sup_{\theta\in\Theta}
\big|\mathcal{L}^{\chi^2}_{t}(\theta;\varepsilon)-\mathcal{L}^{\chi^2}_{n,t}(\theta;\varepsilon)\big|
\;\lesssim\;
\Big(K_\ell+\frac{K_\ell^2}{\underline{\lambda}}\Big)
\sqrt{\frac{d\log n_t+\log(T/\delta)}{n_t}}.
\end{equation}
On the event \eqref{eq:chi2-uniform-all-t}, for each fixed $t$ we have the standard three-term decomposition:
\begin{align*}
\mathcal{L}^{\chi^2}_{t}(\hat{\theta}_{n,t}^{\chi^2};\varepsilon)-\mathcal{L}^{\chi^2}_{t}(\theta_t^{\chi^2};\varepsilon)
&\le
2\sup_{\theta\in\Theta}\big|\mathcal{L}^{\chi^2}_{t}(\theta;\varepsilon)-\mathcal{L}^{\chi^2}_{n,t}(\theta;\varepsilon)\big|.
\end{align*}
Combining with \eqref{eq:chi2-uniform-all-t}, we obtain with probability at least $1-\delta$ for all $t\in[T]$,
\begin{equation}
\label{eq:chi2-excess-risk}
\mathcal{L}^{\chi^2}_{t}(\hat{\theta}_{n,t}^{\chi^2};\varepsilon)-\mathcal{L}^{\chi^2}_{t}(\theta_t^{\chi^2};\varepsilon)
\;\lesssim\;
\Big(K_\ell+\frac{K_\ell^2}{\underline{\lambda}}\Big)
\sqrt{\frac{d\log n_t+\log(T/\delta)}{n_t}}.
\end{equation}

Finally, by the strong convexity of $\mathcal{L}^{\chi^2}_{t}(\cdot;\varepsilon)$ in $\|\cdot\|_2$
(Lemma \ref{appendix:strong-convexity-of-chi-loss}), we have
\[
\mathcal{L}^{\chi^2}_{t}(\theta;\varepsilon)-\mathcal{L}^{\chi^2}_{t}(\theta_t^{\chi^2};\varepsilon)
\ge
\frac{\kappa}{\eta^2}\|\theta-\theta_t^{\chi^2}\|_2^2
\;\; \forall\theta\in\Theta.
\]
Conditional on $\mathcal{F}_{t}$, applying this at $\theta=\hat{\theta}_{n,t}^{\chi^2}$ and combining with
\eqref{eq:chi2-excess-risk} yields, with probability at least $1-\delta$, for all $t\in[T]$,
\[
\|\hat{\theta}_{n,t}^{\chi^2}-\theta_t^{\chi^2}\|_2^2
\lesssim
\frac{\eta^2}{\kappa}
\Big(K_\ell+\frac{K_\ell^2}{\underline{\lambda}}\Big)
\sqrt{\frac{d\log n_t+\log(T/\delta)}{n_t}}.
\]
This completes the proof.
\end{proof}

\section{Proof of Tractable $\chi^2$-DRO-REBEL}
\label{appendix:tractable-chi-squared-algo}
\begin{proof}[Proof of \ref{prop:chi2-worst-case}]
The proof that follows is standard in the analysis of f-divergences and follows from \cite{namkoong2017variance}. We include it for completeness. Let $P_{n, t}$ be the empirical distribution. The robust optimization problem is given by
\[
\mathcal{L}^{\chi^2}_n(\theta;\rho) = \sup_{P} \mathbb{E}_{P}[\ell(z;\theta)] \quad \text{s.t.} \quad D_{\chi^2}(P\|P_{n, t}) \le \rho, \quad \mathbb{P} \ge 0, \quad \mathbb{E}_\mathbb{P}[1]=1.
\]
The $\chi^2$-divergence is defined by $f(t) = \frac{1}{2}(t-1)^2$. The Fenchel conjugate $f^*(s) = \sup_{t \ge 0} \{st - f(t)\}$.
For $s \in \mathbb{R}$, $f'(t) = t-1$. Setting $s = t-1$, we get $t=s+1$.
Substituting this into the definition of $f^*(s)$,
$f^*(s) = s(s+1) - \frac{1}{2}((s+1)-1)^2 = s^2+s - \frac{1}{2}s^2 = \frac{1}{2}s^2+s$.
This derivation holds for $t \ge 0$, which implies $s+1 \ge 0 \implies s \ge -1$. If $s < -1$, the optimal $t$ would be negative, violating $t \ge 0$. In this case, $f^*(s)$ becomes $\infty$ due to the constraint $t \ge 0$. According to Lemma \ref{appendix:duchi-strong-duality-KL} \cite{duchi2020learningmodelsuniformperformance}, the dual form of the $f$-divergence based DRO problem is:
\[
\sup_{\mathbb{P}\colon D_f(\mathbb{P}\Vert P_{n, t})\le\rho}
\mathbb E_\mathbb{P}[\ell(z;\theta)]
\;=\;
\inf_{\substack{\lambda\ge0\\\eta\in\mathbb R}}
\Bigl\{
\lambda\,\E_{P_{n}}\left[f^*\!\left(\frac{\ell(z;\theta)-\eta}{\lambda}\right)\right]
\;+\;\lambda\,\rho\;+\;\eta
\Bigr\}.
\]
Substituting $f^*(s) = \frac{1}{2}s^2+s$ into this dual formulation, with $s = \frac{\ell_i-\eta}{\lambda}$:
\[
\mathcal{L}^{\chi^2}_n(\theta;\rho) = \inf_{\substack{\lambda\ge0\\\eta\in\mathbb R}} \left\{ \lambda\,\rho\;+\;\eta\;+\;\E_{P_{n}}\left[\lambda\left(\frac{1}{2}\left(\frac{\ell_i-\eta}{\lambda}\right)^2 + \frac{\ell_i-\eta}{\lambda}\right)\right] \right\}.
\]
This simplifies to
\[
\mathcal{L}^{\chi^2}_n(\theta;\rho) = \inf_{\substack{\lambda\ge0\\\eta\in\mathbb R}} \left\{ \lambda\,\rho\;+\;\eta\;+\;\E_{P_{n}}\left[\frac{(\ell_i-\eta)^2}{2\lambda} + (\ell_i-\eta)\right] \right\}.
\]
Let $X_i = \ell_i - \eta$. The objective becomes
\[
\inf_{\substack{\lambda\ge0\\\eta\in\mathbb R}} \left\{ \lambda\,\rho\;+\;\eta\;+\;\frac{1}{n}\sum_{i=1}^n \left[\frac{X_i^2}{2\lambda} + X_i\right] \right\}.
\]
The non-negativity constraint $\mathbb{P}(z_i) \ge 0$ in the primal problem implies $1 + \frac{\ell_i-\eta}{\lambda} \ge 0$. This is equivalent to $\frac{\ell_i-\eta}{\lambda} \ge -1$. This constraint is handled by a special form of $f^*$ or by considering the dual's objective piecewise. When $1 + \frac{\ell_i-\eta}{\lambda} < 0$, this instance $z_i$ is excluded from the worst-case distribution. This leads to the presence of the positive part $(\cdot)_+$ in the objective. 

Specifically, for $\chi^2$-divergence, it is a known result in robust optimization that the problem is equivalent to:
\[
\mathcal{L}^{\chi^2}_n(\theta;\rho) = \inf_{\eta\in\mathbb{R}} \left\{ \eta + \inf_{\lambda>0} \left\{ \lambda\rho + \frac{1}{n}\sum_{i=1}^n \frac{(\ell_i-\eta)_+^2}{2\lambda} \right\} \right\}.
\]
Now, we solve the inner minimization with respect to $\lambda$ for a fixed $\eta$. Let $Y_i = (\ell_i-\eta)_+$. The inner objective is:
\[
G(\lambda) = \lambda\rho + \frac{1}{n}\sum_{i=1}^n \frac{Y_i^2}{2\lambda}.
\]
To find the optimal $\lambda^*$, we differentiate $G(\lambda)$ with respect to $\lambda$ and set it to zero:
\[
\frac{dG(\lambda)}{d\lambda} = \rho - \frac{1}{n}\sum_{i=1}^n \frac{Y_i^2}{2\lambda^2} = 0.
\]
Solving for $\lambda^2$:
\[
\lambda^2 = \frac{\sum_{i=1}^n Y_i^2}{2n\rho} = \frac{\E_{P_{n}}[(\ell_i-\eta)_+^2]}{2\rho}.
\]
Since $\lambda > 0$ and $\rho > 0$, we take the positive square root
\[
\lambda^* = \sqrt{\frac{\E_{P_{n}}[(\ell_i-\eta)_+^2]}{2\rho}}.
\]
Substitute $\lambda^*$ back into the inner objective $G(\lambda)$
\begin{align*}
G(\lambda^*) &= \sqrt{\frac{\E_{P_{n}}[(\ell_i-\eta)_+^2]}{2\rho}}\cdot\rho + \frac{1}{n}\sum_{i=1}^n \frac{(\ell_i-\eta)_+^2}{2\sqrt{\frac{\E_{P_{n}}[(\ell_i-\eta)_+^2]}{2\rho}}} \\
&= \rho\sqrt{\frac{\E_{P_{n}}[(\ell_i-\eta)_+^2]}{2\rho}} + \frac{1}{2}\E_{P_{n}}[(\ell_i-\eta)_+^2]\sqrt{\frac{2\rho}{\E_{P_{n}}[(\ell_i-\eta)_+^2]}} \\
&= \sqrt{\frac{\rho^2 \E_{P_{n}}[(\ell_i-\eta)_+^2]}{2\rho}} + \frac{1}{2}\sqrt{2\rho \E_{P_{n}}[(\ell_i-\eta)_+^2]} \\
&= \sqrt{\frac{\rho \E_{P_{n}}[(\ell_i-\eta)_+^2]}{2}} + \frac{1}{2}\sqrt{2\rho \E_{P_{n}}[(\ell_i-\eta)_+^2]} \\
&= \sqrt{\frac{2\rho \E_{P_{n}}[(\ell_i-\eta)_+^2]}{4}} + \sqrt{\frac{2\rho \E_{P_{n}}[(\ell_i-\eta)_+^2]}{4}} \\
&= 2 \sqrt{\frac{2\rho \E_{P_{n}}[(\ell_i-\eta)_+^2]}{4}} \\
&= \sqrt{2\rho \E_{P_{n}}[(\ell_i-\eta)_+^2]}.
\end{align*}
Therefore, the robust objective simplifies to:
\[
\mathcal{L}^{\chi^2}_n(\theta;\rho)
=\inf_{\eta\in\mathbb{R}}
\Bigl\{\,
\eta \;+\;\sqrt{\frac{2\rho}{n}\sum_{i=1}^n(\ell_i-\eta)_+^2}
\Bigr\},
\]
which matches the claim of the proposition. We now show that this problem can be solved efficiently by
establishing the convexity of the objective in $\eta$ and bounding the search space for its minimizer. Define
\[
f(\eta) \;:=\; \eta \;+\; \sqrt{\frac{2\rho}{n}\sum_{i=1}^n(\ell_i-\eta)_+^2}\,.
\]
Write $f(\eta)=g(\eta)+h(\eta)$ with $g(\eta)=\eta$ and
$$h(\eta)=\sqrt{(2\rho/n)\sum_{i=1}^n v_i(\eta)^2},$$ where $v_i(\eta)=(\ell_i-\eta)_+$. The term $g$ is
linear. Each $v_i$ is convex and non-negative, so the mapping
$\eta\mapsto v(\eta)=(v_1(\eta),\dots,v_n(\eta))^\top$ has convex, non-negative components; composing with the
scaled $\ell_2$-norm $\phi(v)=\sqrt{2\rho/n}\,\|v\|_2$, which is convex and non-decreasing on
$\mathbb{R}_{\ge 0}^n$, shows that $h$ is convex. Hence $f$ is convex as the sum of two convex functions.

At any point $\eta\notin\{\ell_1,\dots,\ell_n\}$ the function $f$ is differentiable with derivative
\[
f'(\eta) \;=\; 1 \;-\; \sqrt{\frac{2\rho}{n}} \cdot
\frac{\sum_{i:\,\ell_i>\eta}(\ell_i-\eta)}
     {\sqrt{\sum_{i:\,\ell_i>\eta}(\ell_i-\eta)^2}}\,.
\]
Because $f$ is convex, its subdifferential $\partial f$ is monotonically non-decreasing, and a point
$\eta^\star$ is a minimizer if and only if $0\in\partial f(\eta^\star)$. For any
$\eta>\max_{i}\ell_i$ we have $(\ell_i-\eta)_+=0$ for all $i$, so $f(\eta)=\eta$ and $f'(\eta)=1>0$.
Likewise, $f(\eta)\to\infty$ as $\eta\to-\infty$. The minimizer therefore satisfies
$\eta^\star\le\max_i\ell_i$, and the search can be restricted to a finite interval $[L,\,\max_i\ell_i]$ for
any sufficiently small $L$.

Monotonicity of the subdifferential permits a binary search for $\eta^\star$:
\begin{enumerate}
    \item Initialize the search interval $[L, U]$ with $U = \max_{i}\ell_i$ and $L$ a sufficiently small
          lower bound.
    \item Set the candidate $\eta_c = (L+U)/2$ and compute a subgradient $g_c \in \partial f(\eta_c)$.
    \item If $g_c > 0$, update $U \leftarrow \eta_c$; if $g_c < 0$, update $L \leftarrow \eta_c$.
    \item Repeat until $U - L < \epsilon$.
\end{enumerate}
Each subgradient evaluation requires a pass over the $n$ loss values, so the total cost to reach precision
$\epsilon$ is $\mathcal{O}(n\log((U-L)/\epsilon))$. In Algorithm~\ref{alg:chi2-rebel} the losses
$\ell_1,\dots,\ell_n$ are first sorted in $\mathcal{O}(n\log n)$ time; given the sorted order, each
subgradient can be evaluated in $\mathcal{O}(\log n)$ time via binary search and prefix sums, reducing the
overall runtime to $\mathcal{O}(n\log n)$.
\end{proof}

\section{Additional Experimental Results}
\label{appendix:additional-experimental-results}
Below you can find results for both convex and geometric reward mixtures on each REBEL variant discussed in Figure \ref{fig:W-REBEL-emotion-alignment}, \ref{fig:KL-REBEL-emotion-alignment}, and \ref{fig:chi-REBEL-emotion-alignment}. The takeaways are largely the same. Utilizing the DRO framework in a sample efficient algorithm like REBEL allows us to maintain generalization and prevent overoptimization by adapting to test-time distribution shifts.

\begin{figure}[htbp!]
    \centering
    \includegraphics[width=0.49\linewidth]{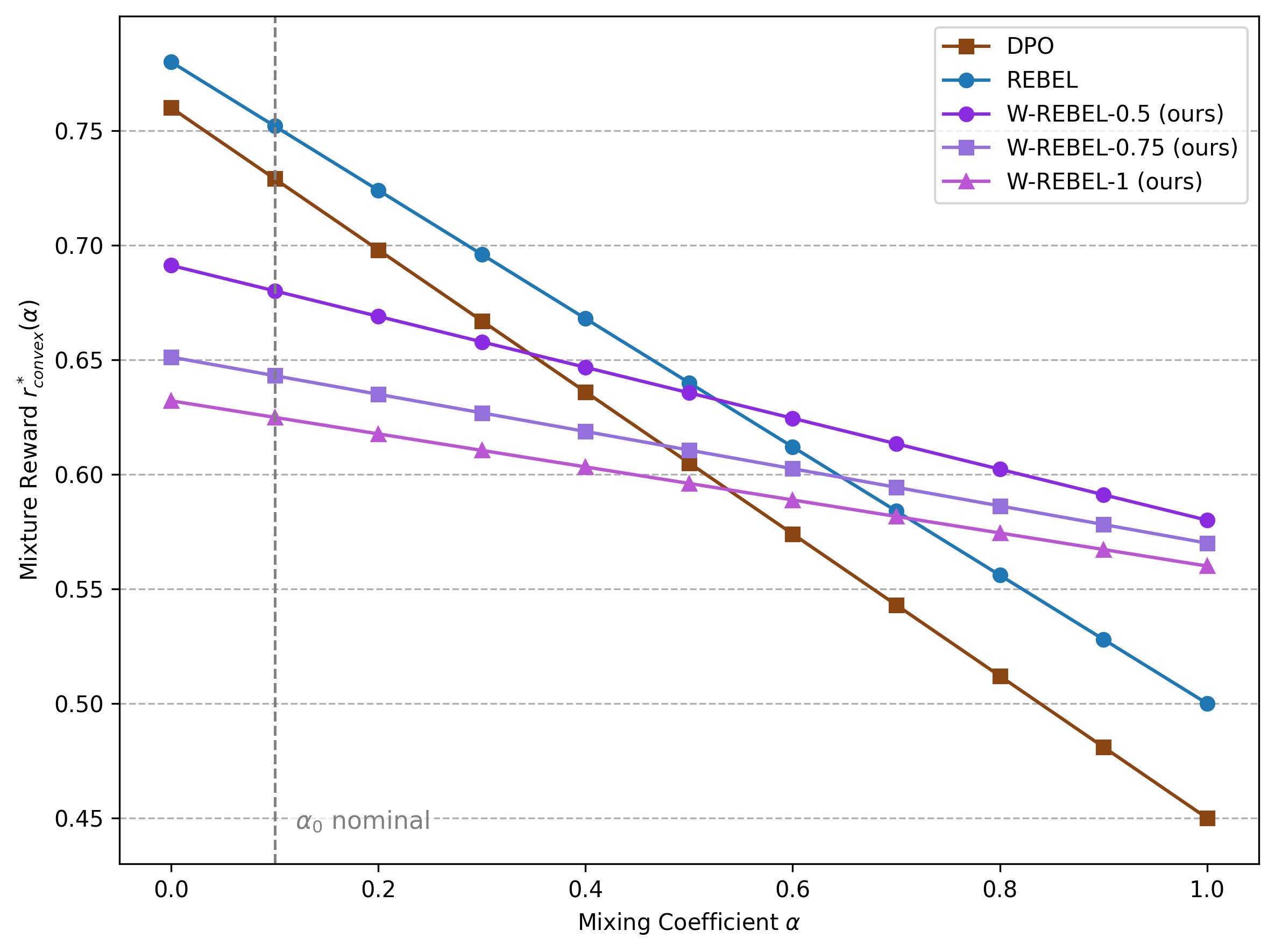}
    \includegraphics[width=0.49\linewidth]{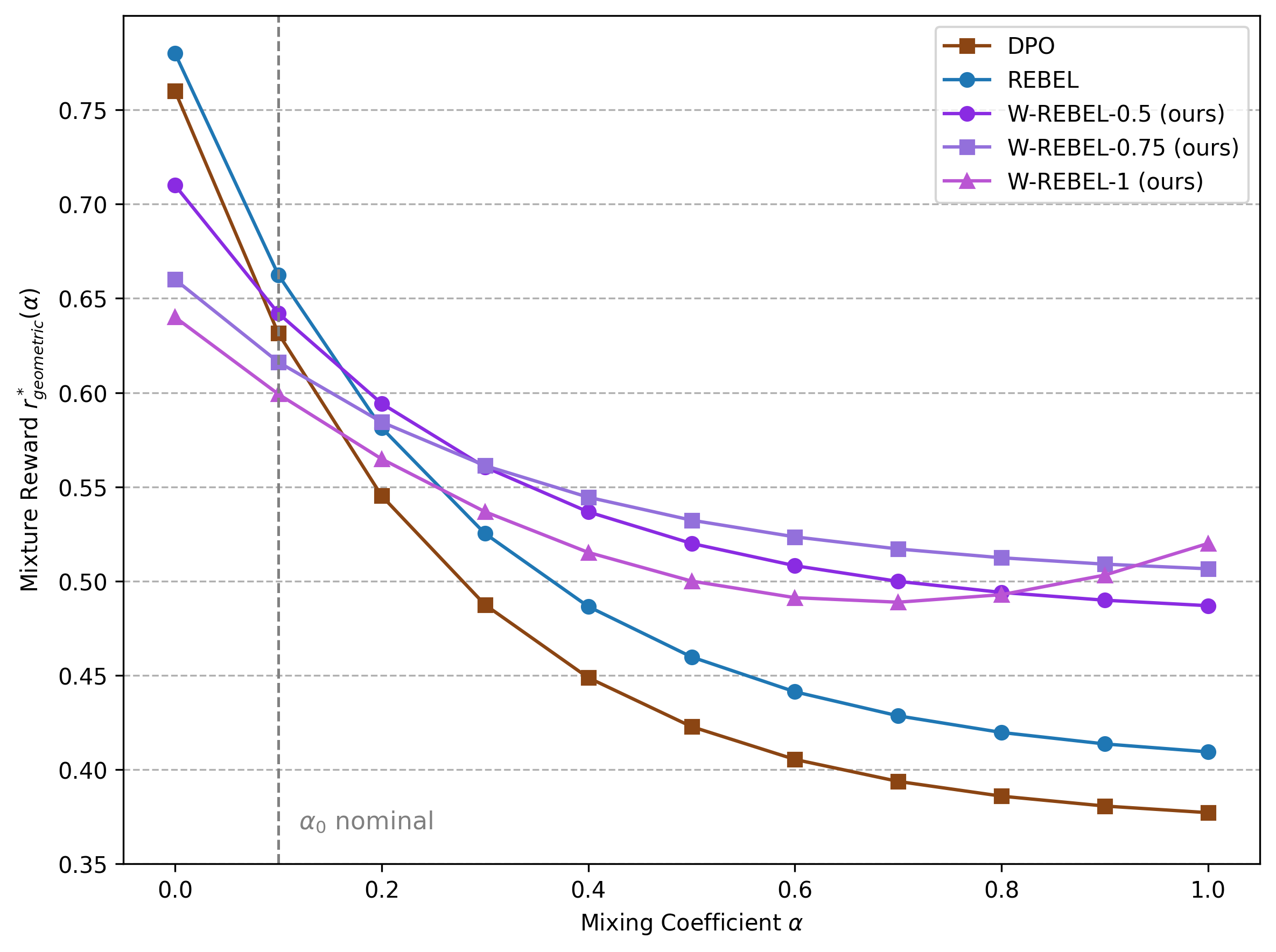}
    \caption{Emotion alignment performance for W-REBEL under convex (left) and geometric (right) reward mixing.}
    \label{fig:W-REBEL-emotion-alignment}
\end{figure}
\begin{figure}[htbp!]
    \centering
    \includegraphics[width=0.49\linewidth]{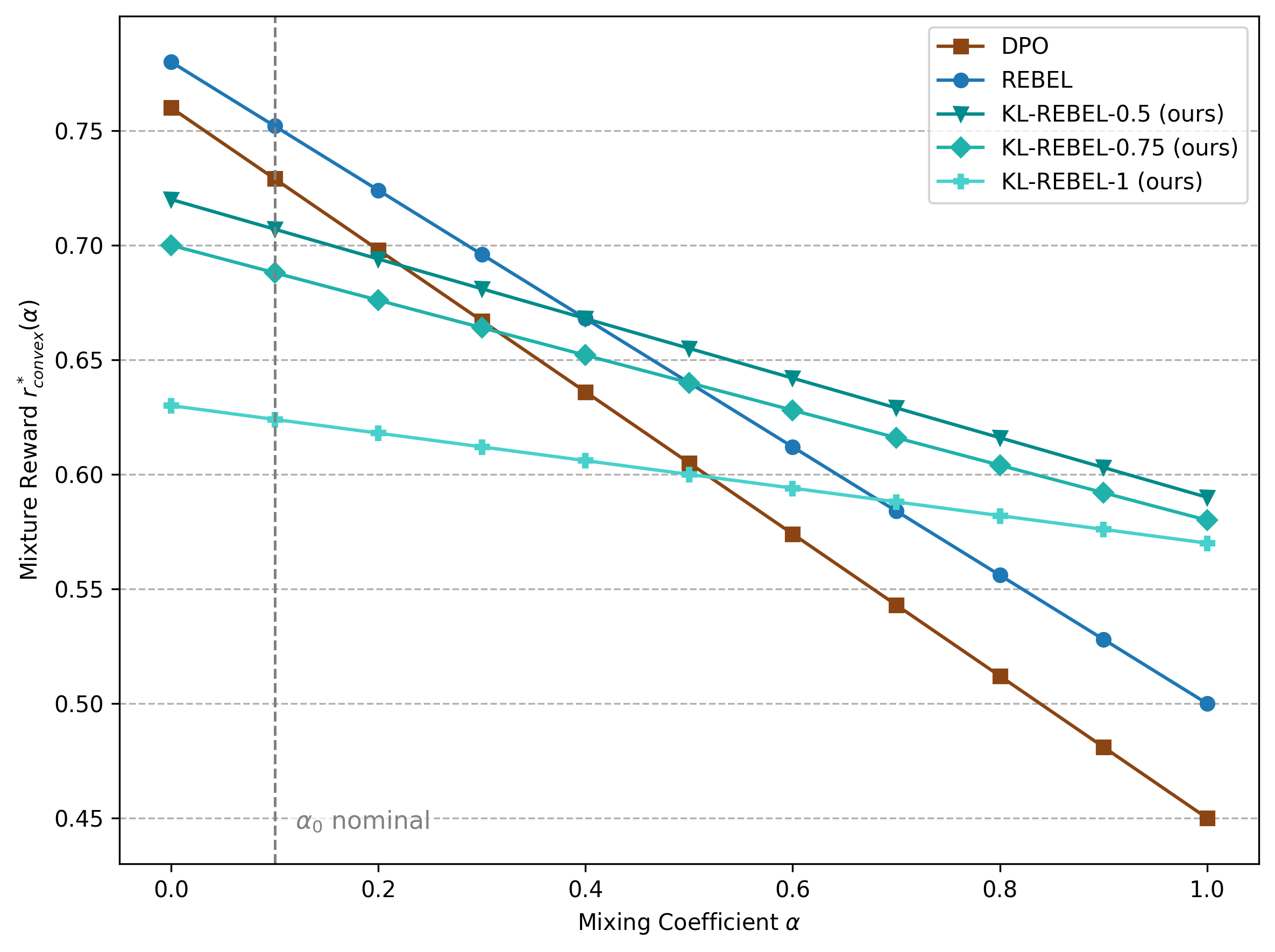}
    \includegraphics[width=0.49\linewidth]{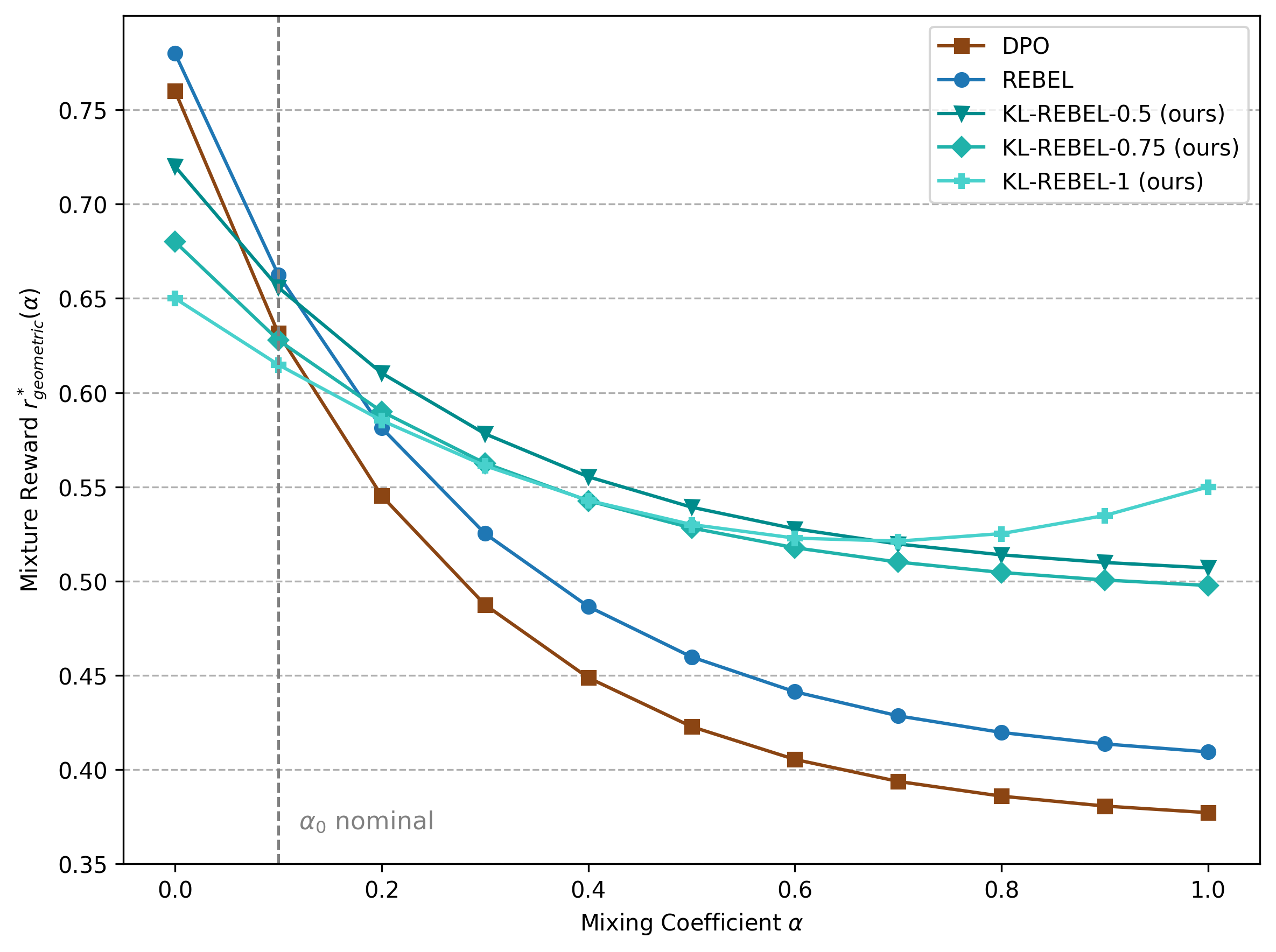}
    \caption{Emotion alignment performance for KL-REBEL under convex (left) and geometric (right) reward mixing.}
    \label{fig:KL-REBEL-emotion-alignment}
\end{figure}
\begin{figure}[htbp!]
    \centering
    \includegraphics[width=0.49\linewidth]{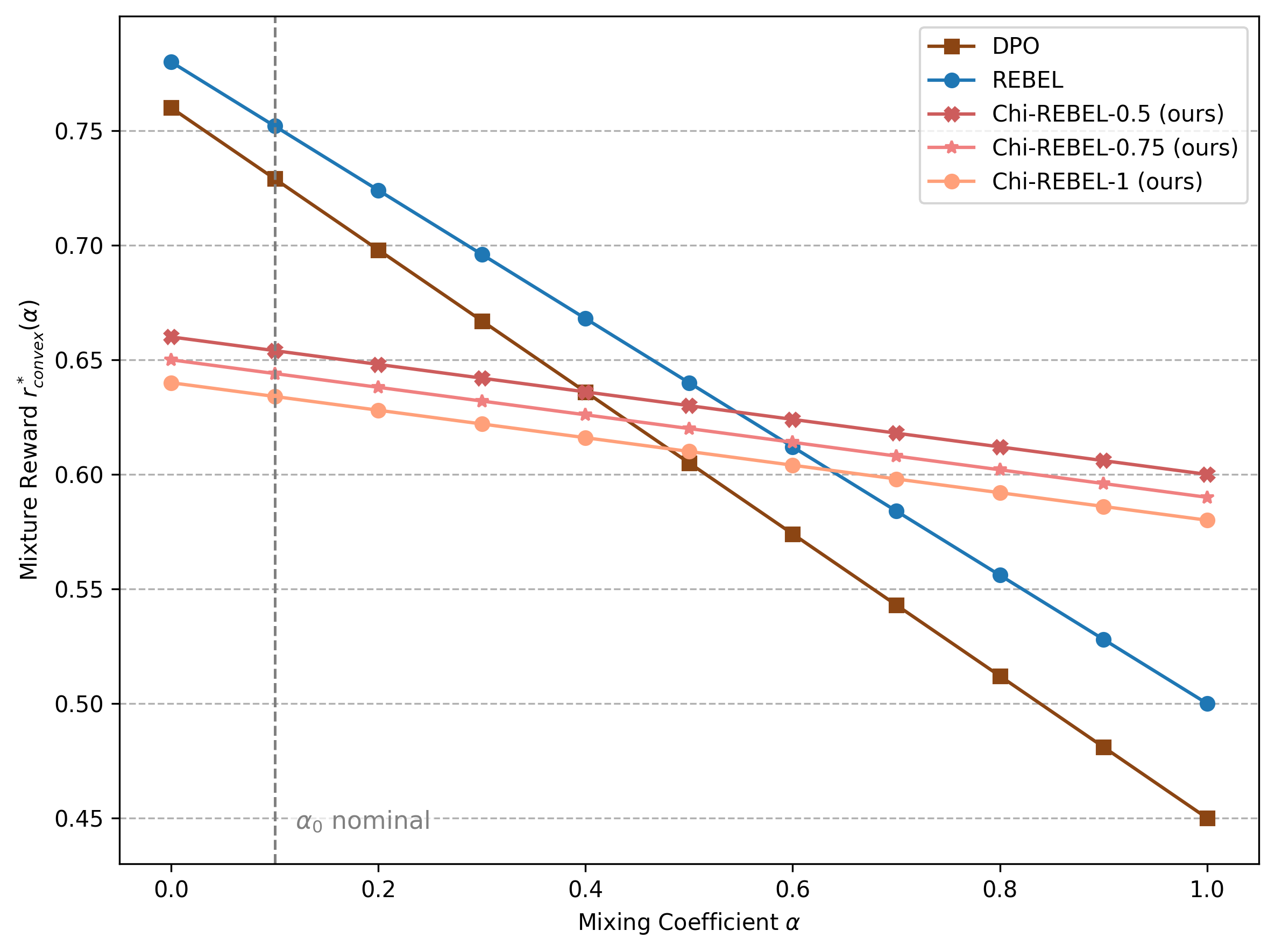}
    \includegraphics[width=0.49\linewidth]{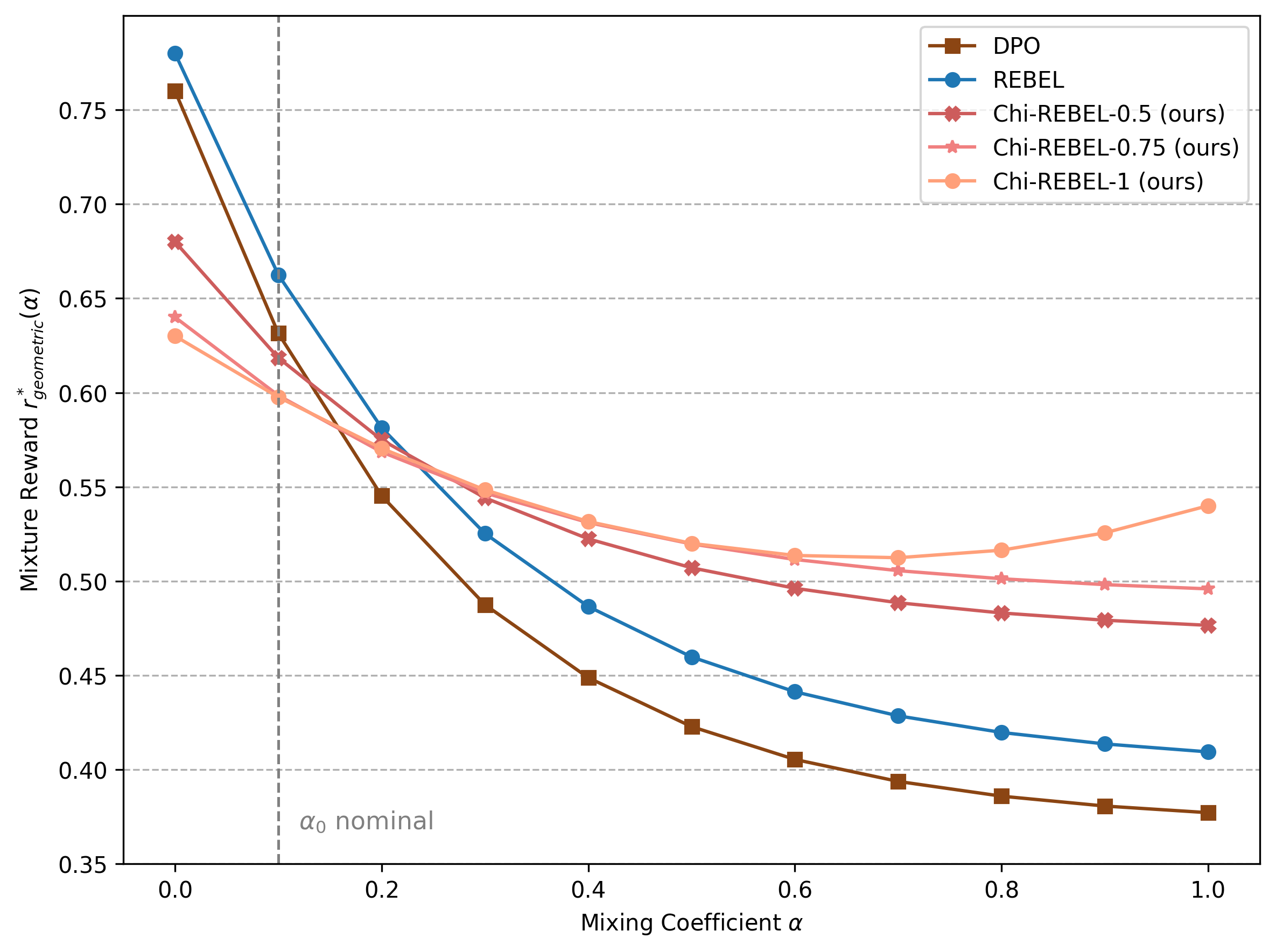}
    \caption{Emotion alignment performance for $\chi^2$-REBEL under convex (left) and geometric (right) reward mixing.}
    \label{fig:chi-REBEL-emotion-alignment}
\end{figure}

\section{Experiment Training Details}
\label{appendix:experiment-training-details}
Our empirical evaluation consists of four complementary components: (i) Emotion Alignment, which serves as a controlled preference-shift benchmark, (ii) a simulated ArmoRM multi-objective alignment setting, which captures heterogeneous and potentially competing objectives, and (iii) a robustness evaluation on HH-RLHF, which assesses generalization on a standard real-world alignment benchmark. This section details the methodology for each experiment, including the data construction, model architectures, training procedures, hyperparameter choices, evaluation metrics, and implementation details.

\subsection{Emotion Alignment Setup}

\textbf{Reward Model Training.}
The reward model, which serves to quantify emotion-specific preferences, was trained on the "emotion" dataset \cite{saravia-etal-2018-carer}. This dataset comprises text samples annotated with single-class labels across six distinct emotion categories: joy, sadness, love, anger, fear, and surprise. The raw text data was preprocessed by tokenization, with sequence lengths capped at a maximum as defined in shared training configurations (e.g., 68 tokens). The original single-label emotion classifications were used directly as targets for the reward model.

For the reward model architecture, we employed a standard GPT-2 model (\code{GPT2LMHeadModel}) fine-tuned for sequence classification by attaching an \code{AutoModelForSequenceClassification} head. This head processes the last token's representation to output logits corresponding to the emotion classes. The model was trained using a standard multi-class classification loss, which is implicitly Cross-Entropy Loss when using \code{AutoModelForSequenceClassification} with multiple labels. Training was conducted over 8 epochs. Optimization was performed with the AdamW optimizer using a learning rate of $5.0 \times 10^{-5}$. Common training arguments, including a \code{per\_device\_train\_batch\_size} (e.g., 64, with potential gradient accumulation to reach an effective batch size), \code{gradient\_accumulation\_steps}, \code{fp16} precision, \code{warmup\_steps}, \code{logging\_steps}, and \code{evaluation\_strategy}, were configured via a shared dictionary (\code{TRAINER\_ARGS\_COMMON}). 

The performance of the model was monitored by \code{eval\_f1\_score} (weighted average), which was set as the metric to select the best model. The trained reward model achieved a test accuracy of 87\% and a test ROC-AUC score of 0.99. The class-wise probability scores predicted by this model were subsequently utilized as our scalar rewards ($r_{\mathrm{emotion}}$) for the preference alignment process.

\textbf{Supervised Fine-Tuning (SFT).}
    We selected a GPT-2 model (\code{GPT2LMHeadModel}) as our base language model. It was trained to predict the next token given preceding context from the emotion dataset. Text samples were tokenized and truncated to a maximum sequence length, typically 68 tokens, as defined by \code{MAX\_SEQ\_LENGTH} in our training configuration. The SFT model was trained for 10 epochs using the AdamW optimizer with a learning rate of $5.0 \times 10^{-7}$. The training schedule included 12 warmup steps, where the learning rate gradually increased to its peak. To ensure training stability and prevent exploding gradients, a maximum gradient norm of 10 was applied during optimization. This SFT-trained model served as both the initial policy ($\pi_0$) and the fixed reference policy ($\pi_{\mathrm{ref}}$) for all subsequent training runs of the DPO and REBEL variants.

\textbf{Data Generation for Alignment.}
A preference dataset for Emotion Alignment was dynamically constructed during the training iterations of each alignment algorithm. Each data point consisted of a prompt and two generated completions, paired with a preference label. The detailed data generation process was as follows:
\begin{itemize}
    \item \textbf{Prompts:} Prompts were directly sampled from the `text` field of the emotion dataset's training split, ensuring they were drawn from the same domain as the SFT model's training data.
    \item \textbf{Completion Generation:} For each prompt, two distinct completions ($a_1$ and $a_2$) were generated by the current policy model ($\pi_{\theta}$). Text generation employed sampling-based decoding with specific parameters: \code{do\_sample=True}, \code{top\_k=50}, \code{top\_p=0.95}, and a \code{temperature=0.7}. Each completion was constrained to a maximum length corresponding to the \code{max\_seq\_length} used during SFT (e.g., 68 tokens), ensuring consistency.
    \item \textbf{Reward Calculation:} The generated completions ($a_1$ and $a_2$) were then evaluated by the pre-trained emotion reward model. This yielded emotion-specific scores for each completion. These scores were combined into a single scalar reward ($r_{a_1}$, $r_{a_2}$) using a configurable mixing function, either "convex" or "geometric", parameterized by a specific $\alpha_{0}$ value. This allowed for emphasis on particular emotions or combinations thereof.
    \item \textbf{Preference Labeling:} Instead of deterministic selection, a binary preference label (`preference`, typically 0 or 1) was assigned to the pair ($a_1$, $a_2$). This was not a deterministic selection based on the mixed reward, but rather a stochastic process following a Bradley-Terry model. Specifically, a random number was drawn, and if it was less than $p = \frac{\exp(r_{a_1})}{\exp(r_{a_1}) + \exp(r_{a_2})}$, then $a_1$ was marked as preferred (preference = 1); otherwise, $a_2$ was preferred (preference = 0).
\end{itemize}

\textbf{REBEL and DPO Variant Training.}
We conducted comprehensive experiments comparing seven distinct preference alignment algorithms: Direct Preference Optimization (DPO), Wasserstein Distributionally Robust DPO (WDPO), KL Distributionally Robust DPO (KL-DPO), Reinforcement Learning via Regressing Relative Rewards (REBEL), Wasserstein Distributionally Robust REBEL (W-REBEL), KL Distributionally Robust REBEL (KL-REBEL), and Chi-squared Distributionally Robust REBEL ($\chi^2$-REBEL).

Each variant was trained for 40 iterations (epochs). In each iteration, a fresh batch of 64 new data points (prompt-completion pairs with preferences/rewards) was collected using the dynamic data generation process described above. The policy model parameters were optimized using the AdamW optimizer with a fixed learning rate of $5.0 \times 10^{-7}$. A DPO $\beta$ parameter of 0.1 was consistently applied across all DPO and its DRO variants. Algorithm-specific robustness hyperparameters, including REBEL's $\eta$ (set to 0.01), WDPO/W-REBEL's $\rho_{0}$, KL-DPO/KL-REBEL's $\tau$, and $\chi^2$-REBEL's $\rho$, were configured through shared experiment settings. All Emotion Alignment experiments were executed on a single NVIDIA A100 GPU with 40 GB VRAM. To accommodate the chosen batch size and model requirements, gradient accumulation was performed over two steps per optimization update.

\subsection{ArmoRM Multi-objective Alignment Setup}
For ArmoRM Multi-objective Alignment, our experimental design focused on scenarios where pre-trained models are aligned to multiple, potentially conflicting, objectives, leveraging reward signals derived from a specialized ArmoRM reward model.

\textbf{Reward Model and SFT.}
Distinct from the Emotion Alignment setup, the ArmoRM configurations did not involve separate training of a reward model or explicit supervised fine-tuning (SFT) of a base language model for the specific alignment task since the use of a foundational model like Meta LLaMA-3.2-1B-Instruct has already undergone extensive pre-training on vast text corpora, followed by multiple rounds of SFT and preliminary alignment on broad human preference datasets. These pre-aligned models are intrinsically capable of generating responses reflecting general human preferences and providing granular, multiobjective reward scores across various axes such as helpfulness, harmlessness, truthfulness, and conciseness. 

\textbf{Data Generation for Alignment.}
The preference dataset for ArmoRM alignment was constructed by sampling prompt-completion pairs from large, diverse datasets designed for evaluating instruction-following and safety, specifically a subset of the publicly available HelpSteer2 dataset \cite{wang2024helpsteer2opensourcedatasettraining}. These prompts typically consisted of user queries, instructions, and open-ended questions designed to elicit varied and complex responses. The generation process for candidate completions for alignment was configured as follows:
\begin{itemize}
    \item \textbf{Completions:} For each sampled prompt, two distinct candidate completions were generated by the current policy model. Text generation employed sampling-based decoding to encourage diversity and creativity, utilizing specific parameters: a \code{temperature} of 0.7 (to balance creativity with coherence), a \code{top\_p} of 1.0 (to allow for maximal diversity in token sampling), and a maximum generation length of up to 1024 new tokens, enabling the generation of comprehensive and elaborate responses. Prompts themselves were also truncated to a maximum of 1024 tokens before being fed to the model.
    \item \textbf{Multi-objective Rewards:} These generated prompt-completion pairs were then input into the \emph{first stage} of a pre-existing ArmoRM model \cite{wang2024helpsteer2opensourcedatasettraining}. This "first stage" is a multi-headed reward architecture designed to output a comprehensive vector of scores, quantifying a completion's performance across several predefined objectives (e.g., helpfulness, harmlessness, creativity, factuality). The chosen and rejected completions within each pair were determined based on a composite mixed metric derived from these multi-objective reward vectors.
\end{itemize}

\textbf{REBEL and DPO Variant Training.}
Given the substantial size of the models and datasets involved, these experiments were performed on a high-performance distributed computing setup comprising 8xH100 GPUs. Training leveraged the DeepSpeed framework for efficient memory management and optimized distributed training. Specifically, we primarily utilized DeepSpeed ZeRO-2 (Zero Redundancy Optimizer Stage 2) \cite{10.5555/3433701.3433727} for parameter, gradient, and optimizer state partitioning across GPUs. This included features like \code{overlap\_comm} for overlapping computation and communication, and \code{contiguous\_gradients} for memory efficiency. The training process was configured for automatic mixed precision, with \code{bf16} enabled for bfloat16 training and \code{fp16} also available (with a \code{loss\_scale} of 512). DeepSpeed automatically managed the optimizer parameters (learning rate, betas, epsilon, weight decay) and the \code{WarmupDecayLR} scheduler, as well as \code{gradient\_accumulation\_steps} and \code{gradient\_clipping}.

\subsection{HH\textendash RLHF Pairwise Alignment Setup}
We evaluate robustness on the HH\textendash RLHF preference dataset \cite{bai2022traininghelpfulharmlessassistant} with Llama-1B and Llama-8B policies, comparing non\mbox{-}robust baselines (DPO, REBEL) to robust DPO variants (WDPO, KL\mbox{-}DPO) and our robust REBEL variants (W\mbox{-}REBEL, KL\mbox{-}REBEL, $\chi^2$-REBEL) \cite{rafailov2023direct,gao2024rebel,xu2025distributionallyrobustdirectpreference}. The runner is model\mbox{-}agnostic.

\textbf{Tokenizer \& models.}  We load the policy and create a frozen reference copy (placed in \code{eval()} and \code{requires\_grad=False} for all params). \code{bf16/fp16} is enabled via \code{torch.autocast}. 

\textbf{Data pairing.} We load data \(\{x, y^{\text{chosen}}, y^{\text{rejected}}\}\). A small held\mbox{-}out test set (by default \(2\%\), min 256 and capped at 1{,}024 pairs) is reserved for evaluation. The remainder forms the training set.  At each step we uniformly sample indices and map each example to the optimizer’s pairwise format:
\[
\{\text{prompt}=x,\ \text{response\_a1}=y^{\text{chosen}},\ \text{response\_a2}=y^{\text{rejected}},\ \text{preference}=1\}.
\]
Inputs are truncated to \texttt{max\_seq\_length}=1024 tokens. Teacher\mbox{-}forced log\mbox{-}probabilities are always computed on \emph{response tokens only}, conditioned on the prompt.  

\textbf{Evaluation.} Every 200 steps of upstream training (on non\mbox{-}HH sources), we evaluate on the fixed HH\textendash RLHF set in chunks of 64 examples. We report \emph{Win} rate, the fraction of pairs where the policy increases the chosen–rejected margin relative to the reference

$$
\Big[\log \mathbb{P}_{\theta} \big(y^{\text{chosen}}\!\mid x\big)-\log \mathbb{P}_{\theta}\big(y^{\text{rejected}}\!\mid x\big)\Big]
\;-\;
\Big[\log \mathbb{P}_{\text{ref}}\big(y^{\text{chosen}}\mid x\big)-\log \mathbb{P}_{\text{ref}}\big(y^{\text{rejected}}\!\mid x\big)\Big]
> 0,
$$

and \emph{Lose} rate, the fraction with a negative margin (ties excluded).

\subsection{Wasserstein Variants (WDPO, W-REBEL) Implementation Details}
Recall that regularization term in Algorithm \ref{alg:wd-rebel} is defined as $R(\pi_{\theta}; D) = \rho_0 (\mathbb{E}_{z \sim D} \|\nabla_z \ell(z; \theta)\|_2^2)^{1/2}$, where $\ell(z; \theta)$ is the pointwise loss. In a distributed LLM training setting, computing the exact expectation over the entire data distribution $D$ for this regularizer, or even accurately averaging gradient norms over small, local micro-batches, presents a key implementation challenge. A naive approach of averaging gradient norms over local micro-batches can lead to a highly noisy and unstable gradient penalty due to the typically small number of samples per GPU. To mitigate this instability and ensure tractability, we used the trick utilized by \cite{xu2025distributionallyrobustdirectpreference} which exploits the inequality $\sqrt{x} \leq x$ for $x \geq 1$. This allows us to upper bound the regularizer. This leads to a tractable approximation of the pointwise WDPO loss:
$$\ell_W (z_i, \rho_0) = \ell(z_i; \theta) + \rho_0 \|\nabla_z \ell(z_i; \theta)\|_2^2,$$
where $\ell(z_i; \theta)$ denotes the standard DPO or REBEL loss for sample $z_i$.

For computing $\|\nabla_z \ell(z_i; \theta)\|_2^2$, gradient tracking was enabled on the input embeddings of the policy model using \code{requires\_grad=True} in \code{get\_log\_probs\_and\_input\_embeddings}. This is because since we cannot directly compute $\nabla_{z} \ell(z;\theta)$ since our input is tokenized as integers. The \code{torch.autograd.grad} function was then used to calculate the gradient of the pointwise loss $\ell(z_i; \theta)$ with respect to these differentiable input representations. The sum of squared norms of these gradients was calculated for each sample. This term, scaled by $\rho_0$, was directly incorporated as a penalty into the total loss for each sample, effectively regularizing the policy towards smoother loss landscapes.

\subsection{KL Variants (KL-DPO, KL-REBEL) Implementation Details}
Recall that the re-weighting factor for each sample $i$, $P(i)$, was calculated proportional to $$\exp\left(\frac{1}{\tau_{\mathrm{eff}}} (\ell(z_i; \theta) - \text{mean}(\ell(z_j; \theta)))\right),$$ where $\ell(z_i; \theta)$ is the pointwise loss for sample $i$ in Algorithm \ref{alg:kl-rebel}. Critically, $\text{mean}(\ell(z_j; \theta))$ represents the average pointwise loss computed over the global batch across all participating GPUs. To achieve this global consistency, each GPU first computes the pointwise losses for its local mini-batch. Then, a synchronization step involving a \code{torch.distributed.all\_gather} operation is performed. This operation collects all individual losses from all workers onto every GPU, allowing each GPU to compute the exact global mean of $\ell(z_j; \theta)$ across the entire distributed batch. This ensures that the re-weighting factors $P(i)$ are consistent and correctly reflect the global worst-case distribution. $\tau_{\mathrm{eff}} = \max(\tau, 1e-6)$ ensures numerical stability. The total loss for these methods was then computed as a weighted sum of the individual losses $\sum P(i) \cdot \ell(z_i; \theta)$.

\subsection{$\chi^2$ Variant ($\chi^2$-REBEL) Implementation Details}
This method seeks to find a robust policy by optimizing against an ambiguity set defined by $\chi^2$-divergence. The optimization involves determining an optimal dual variable, $\eta^*$, at each training step. This is achieved by searching over a set of candidate $\eta$ values, including unique individual loss values and boundary points, to identify the one that minimizes the expression $$\eta + \sqrt{\frac{2\rho}{n} \sum (\text{loss}_i - \eta)^2_+},$$ where $(\cdot)_+ = \max(\cdot, 0)$. Implementing this in a distributed setting presented a significant engineering challenge, particularly in ensuring global consistency for the $\eta^*$ search. At each step, each GPU first computes its \code{individual\_ell\_losses} for its local mini-batch. To enable the global search for $\eta^*$, these local loss tensors are then gathered from all GPUs onto every GPU using a \code{torch.distributed.all\_gather} operation, creating a \code{global\_ell\_losses} tensor on each rank. The \code{\_find\_eta\_star} method then executes on this \code{global\_ell\_losses} tensor independently on each GPU; since all GPUs possess identical global loss data, they deterministically identify the same $\eta^*$ value. This process of efficiently finding the optimal dual variable across distributed data, without excessive communication, was one of the most difficult parts of the implementation. The term $$\sum (\text{loss}_i - \eta^*)^2_+,$$ in the expression for $\eta^*$ and for deriving $\lambda^*$ requires a sum over all samples across all GPUs, which is achieved using a \code{torch.distributed.all\_reduce} operation on the locally computed sums. Once $\eta^*$ is found, a corresponding $\lambda^*$ is derived. Finally, the gradients for the policy model parameters are computed as a weighted sum of the gradients of individual losses, where the weights $w_i = (\text{loss}_i - \eta^*)_+ / (n \cdot \lambda^*)$ emphasize samples contributing most to the robust objective. These weighted gradients are directly applied to the model's parameters, with DeepSpeed's ZeRO-2 optimizer handling the implicit aggregation across all GPUs during its optimization step.

\end{document}